\documentclass[11pt]{article}
\usepackage[T1]{fontenc}
\usepackage[letterpaper]{geometry}
\usepackage{color}
\usepackage{array}
\usepackage{longtable}
\usepackage{calc}
\usepackage{textcomp}
\usepackage{multirow}
\usepackage{amsmath}
\usepackage{amsthm}
\usepackage{amssymb}
\usepackage{graphicx}
\usepackage[authoryear,round]{natbib}
\usepackage[unicode=true,
 bookmarks=true,bookmarksnumbered=true,bookmarksopen=true,bookmarksopenlevel=2,
 breaklinks=false,pdfborder={0 0 1},backref=false,colorlinks=true]
 {hyperref}
\hypersetup{
 linkcolor=blue,citecolor=blue}

\makeatletter

\providecommand{\tabularnewline}{\\}

\theoremstyle{plain}
\newtheorem{thm}{\protect\theoremname}[section]
\theoremstyle{plain}
\newtheorem{lem}[thm]{\protect\lemmaname}
\theoremstyle{remark}
\newtheorem{rem}[thm]{\protect\remarkname}

\makeatother

\providecommand{\lemmaname}{Lemma}
\providecommand{\remarkname}{Remark}
\providecommand{\theoremname}{Theorem}

\begin{document}
\clubpenalty=10000 \widowpenalty=10000 
\allowdisplaybreaks

\global\long\def\tr{\text{Tr}}%
\global\long\def\R{\mathbb{R}}%
\global\long\def\E{\mathbb{E}}%
\global\long\def\V{\mathrm{Var}}%
\global\long\def\dom{\mathcal{K}}%
\global\long\def\sm{\mathcal{B}}%
\global\long\def\diag{\textnormal{diag}}%
\global\long\def\regret{\textnormal{Regret}}%

\global\long\def\dualitygap{\textnormal{DualityGap}}%
\global\long\def\dualgap{\textnormal{DualGap}}%
\global\long\def\adagradplus{\textsc{AdaGrad+}}%
\global\long\def\adaagdplus{\textsc{AdaAGD+}}%
\global\long\def\adamp{\textsc{Adaptive Mirror Prox}}%

\global\long\def\adaacsa{\textsc{AdaACSA}}%
\global\long\def\adagrad{\textsc{AdaGrad}}%
\global\long\def\adam{\textsc{Adam}}%
\global\long\def\jrgs{\textsc{JRGS}}%
\global\long\def\sgd{\textsc{SGD}}%
\global\long\def\log{\ln}%

\title{Adaptive Gradient Methods for Constrained Convex Optimization and
Variational Inequalities}
\author{Alina Ene\thanks{Department of Computer Science, Boston University, \texttt{{aene@bu.edu}}}
\and Huy L. Nguy\~{ê}n\thanks{Khoury College of Computer and Information Science, Northeastern University,
\texttt{{hu.nguyen@northeastern.edu}}} \and Adrian Vladu\thanks{CNRS \& IRIF, Université de Paris, \texttt{{vladu@irif.fr}}}}
\date{(version 3)\thanks{The first version of this paper appeared on Arxiv on July 17, 2020.}}

\maketitle
 
\begin{abstract}
We provide new adaptive first-order methods for constrained convex
optimization. Our main algorithms $\adaacsa$ and $\adaagdplus$ are
accelerated methods, which are universal in the sense that they achieve
nearly-optimal convergence rates for both smooth and non-smooth functions,
even when they only have access to stochastic gradients. In addition,
they do not require any prior knowledge on how the objective function
is parametrized, since they automatically adjust their per-coordinate
learning rate. These can be seen as truly accelerated $\adagrad$
methods for constrained optimization.

We complement them with a simpler algorithm $\adagradplus$ which
enjoys the same features, and achieves the standard non-accelerated
convergence rate. We also present a set of new results involving adaptive
methods for unconstrained optimization and monotone operators. 
\end{abstract}

\section{Introduction}

Gradient methods are a fundamental building block of modern machine
learning. Their scalability and small memory footprint makes them
exceptionally well suited to the massive volumes of data used for
present-day learning tasks.

While such optimization methods perform very well in practice, one
of their major limitations consists of their inability to converge
faster by taking advantage of specific features of the input data.
For example, the training data used for classification tasks may exhibit
a few very informative features, while all the others have only marginal
relevance. Having access to this information a priori would enable
practitioners to appropriately tune first-order optimization methods,
thus allowing them to train much faster. Lacking this knowledge, one
may attempt to reach a similar performance by very carefully tuning
hyper-parameters, which are all specific to the learning model and
input data.

This limitation has motivated the development of adaptive methods,
which in absence of prior knowledge concerning the importance of various
features in the data, adapt their learning rates based on the information
they acquired in previous iterations. The most notable example is
$\adagrad$ \citep{duchi2011adaptive}, which adaptively modifies
the learning rate corresponding to each coordinate in the vector of
weights. Following its success, a host of new adaptive methods appeared,
including $\textsc{Adam}$ \citep{kingma2014adam}, $\textsc{AmsGrad}$
\citep{reddi2018convergence}, and $\textsc{Shampoo}$ \citep{gupta2018shampoo},
which attained optimal rates for generic online learning tasks.

A significant series of recent works on adaptive methods addresses
the regime of smooth convex functions. Notably, Levy \citeyearpar{levy2017online},
Cutkosky \citeyearpar{Cutkosky19}, Kavis et al. \citeyearpar{KavisLBC19},
and Bach and Levy \citeyearpar{BachL19} consider the case of minimizing
smooth convex functions without having prior knowledge of the smoothness
parameter. While a standard convergence rate of $1/T$ is fairly easily
attainable in the case of unconstrained optimization, achieving the
optimal $1/T^{2}$ rate becomes significantly more challenging. Even
worse, for constrained minimization objectives, where the gradient
is nonzero at the optimum, it is generally unclear how an adaptive
method can pick the correct step sizes even when aiming for the weaker
non-accelerated rate of $1/T$. These difficulties occur when one
merely attempts to find the correct learning rate; taking advantage
of non-uniform per-coordinate learning rates, as in the case of the
original $\adagrad$ method has remained largely open. In \citep{KavisLBC19},
finding such a method with an accelerated $1/T^{2}$ convergence is
posed as an open problem, since it would allow the development of
robust algorithms that are applicable to non-convex problems such
as training deep neural networks.

In this paper, we address this problem and present adaptive algorithms
which achieve nearly-optimal convergence with per-coordinate learning
rates, even in constrained domains. Our algorithms are \textit{universal}
in the sense that they achieve nearly-optimal convergence rate even
when the objective function is non-smooth \citep{nesterov2015universal}.
Furthermore, they automatically extend to the case of stochastic optimization,
achieving up to logarithmic factors optimal dependence in the standard
deviation of the stochastic gradient norm. We complement them with
a simpler non-accelerated algorithm which enjoys the same features:
it achieves the standard convergence rate on both smooth and non-smooth
functions, and does not require prior knowledge of the smoothness
parameters, or the variance of the stochastic gradients.

\begin{table}
\begin{centering}
\begin{tabular}{|r|l|l|}
\hline 
{\small{}method} & {\small{}non-smooth convergence} & {\small{}smooth convergence}\tabularnewline
\hline 
\multirow{2}{*}{{\small{}$\adagrad$}} & {\small{}$O\left(\frac{R_{\infty}\sqrt{d}G}{\sqrt{T}}+\frac{R_{\infty}\sqrt{d}\sigma}{\sqrt{T}}\right)$} & {\small{}$O\left(\frac{R_{\infty}^{2}\sum_{i=1}^{d}\beta_{i}}{T}+\frac{R_{\infty}\sqrt{d}\sigma}{\sqrt{T}}\right)$}\tabularnewline
 & {\small{}Follows from \citep{duchi2011adaptive}} & {\small{}Theorem \ref{thm:stoch-unconstr}}\tabularnewline
\hline 
\multirow{2}{*}{{\small{}$\adagradplus$}} & {\small{}$O\left(\frac{R_{\infty}\sqrt{d}G\sqrt{\ln\left(\frac{GT}{R_{\infty}}\right)}}{\sqrt{T}}+\frac{R_{\infty}\sqrt{d}\sigma\sqrt{\ln\left(\frac{T\sigma}{R_{\infty}}\right)}}{\sqrt{T}}+\frac{R_{\infty}^{2}d}{T}\right)$} & {\small{}$O\left(\frac{R_{\infty}^{2}\sum_{i=1}^{d}\beta_{i}\ln\beta_{i}}{T}+\frac{R_{\infty}\sqrt{d}\sigma\sqrt{\ln\left(\frac{T\sigma}{R_{\infty}}\right)}}{\sqrt{T}}\right)$}\tabularnewline
 & {\small{}Theorems \ref{thm:adagrad+-deterministic}, \ref{thm:adagrad+}} & {\small{}Theorems \ref{thm:adagrad+-deterministic}, \ref{thm:adagrad+}}\tabularnewline
\hline 
\multirow{2}{*}{{\small{}$\adaacsa$}} & {\small{}$O\left(\frac{R_{\infty}\sqrt{d}G\sqrt{\ln\left(\frac{GT}{R_{\infty}}\right)}+R_{\infty}\sqrt{d}\sigma\sqrt{\ln\left(\frac{T\sigma}{R_{\infty}}\right)}}{\sqrt{T}}+\frac{R_{\infty}^{2}d}{T^{2}}\right)$} & {\small{}$O\left(\frac{R_{\infty}^{2}\sum_{i=1}^{d}\beta_{i}\ln\beta_{i}}{T^{2}}+\frac{R_{\infty}\sqrt{d}\sigma\sqrt{\ln\left(\frac{T\sigma}{R_{\infty}}\right)}}{\sqrt{T}}\right)$}\tabularnewline
 & {\small{}Theorems \ref{thm:acsa-deterministic}, \ref{thm:acsa}} & {\small{}Theorems \ref{thm:acsa-deterministic}, \ref{thm:acsa}}\tabularnewline
\hline 
\multirow{2}{*}{{\small{}$\adaagdplus$}} & {\small{}$O\left(\frac{R_{\infty}\sqrt{d}G\sqrt{\ln\left(\frac{GT}{R_{\infty}}\right)}+R_{\infty}\sqrt{d}\sigma\sqrt{\ln\left(\frac{T\sigma}{R_{\infty}}\right)}}{\sqrt{T}}+\frac{R_{\infty}^{2}d}{T^{2}}\right)$} & {\small{}$O\left(\frac{R_{\infty}^{2}\sum_{i=1}^{d}\beta_{i}\ln\beta_{i}}{T^{2}}+\frac{R_{\infty}\sqrt{d}\sigma\sqrt{\ln\left(\frac{T\sigma}{R_{\infty}}\right)}}{\sqrt{T}}\right)$}\tabularnewline
 & {\small{}Theorems \ref{thm:agd+-deterministic}, \ref{thm:agd+}} & {\small{}Theorems \ref{thm:agd+-deterministic}, \ref{thm:agd+}}\tabularnewline
\hline 
{\small{}$\textsc{Adaptive}$} & {\small{}$O\left(\frac{R_{\infty}\sqrt{d}G\sqrt{\ln\left(\frac{GT}{R_{\infty}}\right)}+R_{\infty}\sqrt{d}\sigma\sqrt{\ln\left(\frac{T\sigma}{R_{\infty}}\right)}}{\sqrt{T}}+\frac{R_{\infty}^{2}d}{T}\right)$} & {\small{}$O\left(\frac{R_{\infty}^{2}\sum_{i=1}^{d}\beta_{i}\ln\beta_{i}}{T}+\frac{R_{\infty}\sqrt{d}\sigma\sqrt{\ln\left(\frac{T\sigma}{R_{\infty}}\right)}}{\sqrt{T}}\right)$}\tabularnewline
{\small{}$\textsc{Mirror Prox}$} & {\small{}Theorem \ref{thm:variational}} & {\small{}Theorem \ref{thm:variational}}\tabularnewline
\hline 
\end{tabular}
\par\end{centering}
\caption{Convergence rates of adaptive methods in the vector setting. We assume
that $f:\protect\dom\rightarrow\mathbb{R}$, with $\protect\dom\subseteq\mathbb{R}^{d}$,
is either smooth with respect to an unknown norm $\left\Vert \cdot\right\Vert _{\protect\sm}$,
where $\protect\sm=\protect\diag\left(\beta_{1},\dots,\beta_{d}\right)$,
or non-smooth and $G$-Lipschitz. We assume access to stochastic gradients
$\widetilde{\nabla}f(x)$ which are unbiased estimators for the true
gradient and have bounded variance $\mathbb{E}\left[\left\Vert \widetilde{\nabla}f(x)-\nabla f(x)\right\Vert ^{2}\right]\protect\leq\sigma^{2}$.
The $\protect\adamp$ algorithm is for the more general setting of
variational inequalities.}

\label{table:results}
\end{table}

\paragraph*{Previous Work. }

Work on adaptive methods has been extensive, and resulted in a broad
range of algorithms \citep{duchi2011adaptive,kingma2014adam,reddi2018convergence,tieleman2012lecture,dozat2016incorporating,chen2018convergence}.
A significant body of work is dedicated to non-convex optimization
\citep{zou2018weighted,ward2019adagrad,zou2019sufficient,li2019convergence,defossez2020convergence}.
In a slightly different line of research, there has been recent progress
on obtaining improved convergence bounds in the online non-smooth
setting; these methods appear in the context of parameter-free optimization,
whose main feature is that they adapt to the radius of the domain
\citep{cutkosky2019matrix,cutkosky2020parameter}.

Here we discuss, for comparison, relevant previous results on adaptive
first order methods for smooth convex optimization where the function
$f:\mathbb{R}^{d}\rightarrow\mathbb{R}$ to be minimized is smooth
with respect to some unknown norm $\left\Vert \cdot\right\Vert _{\sm}$,
where $\sm$ is a non-negative diagonal matrix. The case where $\sm=\beta I$
is a multiple of the identity corresponds to the standard assumption
that $f$ is $\beta$-smooth, and we refer to this as the \textit{scalar}
version of the problem. In the case when $\sm$ is a non-negative
diagonal matrix, we optimize using the \textit{vector} version of
the problem.

Notably, Levy \citeyearpar{levy2017online} presents an adaptive first
order method, achieving an optimal convergence rate of $O\left(\beta R^{2}/T\right)$,
without requiring prior knowledge of the smoothness $\beta$. While
this method also applies to the case where the domain is constrained,
it requires the strong condition that the global optimum lies within
the domain. In \citep{levy2018online}, this issue is discussed explicitly,
and the line of work is pushed further in the unconstrained case to
obtain an accelerated rate of $O\left(\beta R^{2}\log\left(\beta R/\left\Vert g_{0}\right\Vert \right)/T^{2}\right)$,
where $g_{0}$ is the gradient evaluated at the initial point. In
\citep{BachL19}, the authors consider constrained variational inequalities,
which are more general, as they include both convex optimization and
convex-concave saddle point problems. The rate they achieve is $O\left(\beta R^{2}/T\right)$,
where $\beta$ is the an upper bound on the unknown Lipschitz parameter
of the monotone operator, generalizing the case of $\beta$-smooth
convex functions. Based on this scheme, in \citep{KavisLBC19} the
authors deliver an accelerated adaptive method with nearly optimal
rate for the scalar version of the problem. There, they pose as an
open problem the question of delivering an accelerated adaptive method
for the vector case. We give a more in-depth comparison to previous
work in Section \ref{sec:scalar-schemes}.

\paragraph*{Our Contributions. }

We give the first adaptive algorithms with per-coordinate step sizes
achieving nearly-optimal rates for both constrained convex minimization
and variational inequalities arising from monotone operators. Variational
inequalities are a very general framework that captures convex minimization,
convex-concave saddle point problems, and many other problems of interest
\citep{BachL19,nemirovski2004prox}. Our algorithms are universal,
in the sense defined by Nesterov \citeyearpar{nesterov2015universal}.
They automatically achieve optimal convergence rates (up to a $\sqrt{\ln T}$
factor) in the smooth and non-smooth setting, both in the deterministic
setting as well as the stochastic setting where we have access to
noisy gradient or operator evaluations. Our algorithms automatically
adapt to problem parameters such as smoothness, gradient or operator
norms, and the variance of the stochastic gradient or operator norms.
Our results answer several open questions raised in previous work
\citep{KavisLBC19,BachL19}.

For constrained convex minimization, we present three algorithms:
$\adagradplus$, $\adaacsa$, and $\adaagdplus$. For $\beta$-smooth
functions, $\adagradplus$ convergences at the rate $O\left(R_{\infty}^{2}d\cdot\beta\ln\beta/T\right)$,
and $\adaacsa$ and $\adaagdplus$ converge at the rate $O\left(R_{\infty}^{2}d\cdot\beta\ln\beta/T^{2}\right)$.
Since $R_{\infty}d^{1/2}$ is the $\ell_{2}$ diameter of the region
containing the $\ell_{\infty}$ ball of radius $R_{\infty}$, these
exactly match the rates of standard non-accelerated and accelerated
gradient decent, when the domain is an $\ell_{\infty}$ ball \citep{nesterov2013introductory}.
Therefore these schemes can be interpreted as learning the optimal
\textit{diagonal preconditioner} for a smooth function $f$.

For variational inequalities, we present the $\adamp$ algorithm that
couples the Universal Mirror-Prox scheme \citep{BachL19,nemirovski2004prox}
with novel per-coordinate step sizes. The Universal Mirror-Prox algorithm
of \citet{BachL19} sets a single step size for all coordinates that
is initialized using an estimate for the gradient norms. In contrast,
our algorithm uses per-coordinate step sizes that are initialized
to an absolute constant. In addition to eliminating a hyperparameter
that we would need to tune, this approach leads to larger stepsizes.
Adaptive methods such as $\adagrad$ are also implemented and used
in practice using step sizes initialized to a small constant, such
as $\epsilon=10^{-10}$. We show that the algorithm simultaneously
achieves convergence guarantees that are optimal (up to a $\sqrt{\ln T}$
factor) for both smooth and non-smooth operators, as well as in the
deterministic and stochastic settings.

Algorithmically, we provide a new rule for updating the diagonal preconditioner,
which is better suited to constrained optimization. While the unconstrained
$\adagrad$ algorithm updates the preconditioner based on the previously
seen gradients, here we update based on the movement performed by
the iterate (see Figure \ref{alg:adagrad+}). In the unconstrained
setting, our update rule matches the standard $\adagrad$ update.
The works \citep{KavisLBC19,joulanisimpler} tackled the difficulties
introduced by constraining the domain by using a different update
rule based on the change in gradients.

\paragraph*{Contemporaneous Work. }

Joulani et al. \citeyearpar{joulanisimpler} also obtain an accelerated
algorithm with coordinate-wise adaptive rates, in constrained domains.
The convergence guarantee is stronger than ours by a $O(\ln\beta)$
factor in the smooth setting, where $\beta$ is the smoothness constant,
and by a $O(\sqrt{\ln T})$ factor in the non-smooth and stochastic
settings. On the other hand, we obtain adaptive schemes for a wide-range
of settings, including a non-dual-averaging scheme ($\text{\ensuremath{\adaacsa}},$
based on the AC-SA algorithm \citep{acsa2012}), a dual-averaging
scheme ($\adaagdplus$, based on the AGD+ algorithm \citep{CohenDO18}),
and an adaptive mirror-prox scheme \citep{BachL19,nemirovski2004prox}
for solving variational inequalities which generalizes both convex
minimization and convex-concave zero-sum games. The latter answers
an open question~\citep{BachL19}. Joulani et al. \citeyearpar{joulanisimpler}
propose a very different dual-averaging scheme for convex minimization
based on the online-to-batch conversion \citep{Cutkosky19,KavisLBC19}
and the online learning with optimism framework \citep{MohriYang16}.
Our algorithms use the iterate movement to set the per-coordinate
step sizes, whereas the algorithm presented in \citep{joulanisimpler}
uses the change in gradients.

\section*{Roadmap}

The rest of the paper is organized as follows.

\begin{longtable}[l]{>{\raggedright}p{2cm}>{\raggedright}p{14cm}}
\textbf{Section \ref{sec:prelim}} & We introduce relevant notation and concepts.\tabularnewline
\textbf{Section \ref{sec:adaptive-schemes}} & We present our adaptive schemes for constrained convex minimization
$(\adagradplus$, $\adaacsa$, $\adaagdplus$) and variational inequalities
($\adamp$), and state their convergence guarantees.\tabularnewline
\textbf{Section \ref{sec:analysis-adagrad+-smooth}} & We analyze the convergence of $\adagradplus$ for smooth functions
in the deterministic setting.\tabularnewline
\textbf{Section \ref{sec:analysis-acc-smooth-1}} & We analyze the convergence of $\adaacsa$ for smooth functions in
the deterministic setting.\tabularnewline
\textbf{Section \ref{sec:scalar-schemes}} & We present the scalar versions of our schemes, provide their convergence
guarantees, and discuss their relation to previous work.\tabularnewline
\textbf{Section \ref{sec:analysis-adagrad+-nonsmooth}} & We analyze the convergence of $\adagradplus$ for non-smooth functions
in the deterministic setting.\tabularnewline
\textbf{Section \ref{sec:analysis-adagrad+-stoch}} & We analyze the convergence of $\adagradplus$ for both smooth and
non-smooth functions in the stochastic setting.\tabularnewline
\textbf{Section \ref{sec:analysis-acc-nonsmooth-1}} & We analyze the convergence of $\adaacsa$ for non-smooth functions
in the deterministic setting.\tabularnewline
\textbf{Section \ref{sec:analysis-acsa-stoch}} & We analyze the convergence of $\adaacsa$ for both smooth and non-smooth
functions in the stochastic setting.\tabularnewline
\textbf{Section \ref{sec:analysis-acc-smooth}} & We analyze the convergence of $\adaagdplus$ for smooth functions
in the deterministic setting.\tabularnewline
\textbf{Section \ref{sec:analysis-acc-nonsmooth}} & We analyze the convergence of $\adaagdplus$ for non-smooth functions
in the deterministic setting.\tabularnewline
\textbf{Section \ref{sec:analysis-agd+-stoch}} & We analyze the convergence of $\adaagdplus$ for both smooth and non-smooth
functions in the stochastic setting.\tabularnewline
\textbf{Section \ref{sec:adagrad-unconstrained}} & We extend the analysis of \citet{levy2018online} to the vector setting,
and obtain a sharp analysis for the standard $\adagrad$ algorithm
for smooth functions in the unconstrained setting (Theorem \ref{thm:adagrad-smooth-unconstrained}),
which saves the extra logarithmic factors that $\adagradplus$ pays
in constrained domains. We also provide its guarantees in the stochastic
setup (Theorem \ref{thm:stoch-unconstr}).\tabularnewline
\textbf{Section \ref{sec:mirror-prox}} & We extend the universal mirror prox method of \citet{BachL19} to
the vector setting, and resolve the open question asked by them.\tabularnewline
\textbf{Section \ref{sec:Experiments}} & We provide experimental results.\tabularnewline
\end{longtable}

\section{Preliminaries}

\label{sec:prelim}

\paragraph*{Constrained Convex Optimization. }

We consider the problem $\min_{x\in\dom}f(x)$, where $f\colon\R^{d}\to\R$
is a convex function and $\dom\subseteq\R^{d}$ is an arbitrary convex
set. For simplicity, we assume that $f$ is continuously differentiable
and we let $\nabla f(x)$ denote the gradient of $f$ at $x$. We
assume access to projections over $\dom$ in the sense that we can
efficiently solve problems of the form $\arg\min_{x\in\dom}\left\langle g,x\right\rangle +\frac{1}{2}\left\Vert x\right\Vert _{D}^{2}$,
where $D$ is an arbitrary non-negative diagonal matrix and $\left\Vert x\right\Vert _{D}=\sqrt{x^{\top}Dx}$. 

We say that $f$ is smooth with respect to the norm $\left\Vert \cdot\right\Vert _{\sm}$
if $\nabla^{2}f(x)\preceq\sm$, for all $x\in\dom$. Equivalently,
we have $f(y)\leq f(x)+\left\langle \nabla f(x),y-x\right\rangle +\frac{1}{2}\left\Vert x-y\right\Vert _{\sm}^{2}$
, for all $x,y\in\dom$. We say that $f$ is strongly convex with
respect to the norm $\left\Vert \cdot\right\Vert _{\sm}$ if $\nabla^{2}f(x)\succeq\sm$,
for all $x\in\dom$.

\paragraph*{Variational Inequalities. }

We also consider the more general problem setting of variational inequalities
arising from monotone operators. Let $\dom\subseteq\R^{d}$ be a convex
set and let $F\colon\dom\to\mathbb{R}^{d}$ be an operator. The operator
$F$ is monotone if it satisfies $\left\langle F(x)-F(y),x-y\right\rangle \geq0$
for all $x,y\in\dom$ and it is smooth with respect to the norm $\left\Vert \cdot\right\Vert _{\sm}$
if $\left\Vert F(x)-F(y)\right\Vert _{\sm^{-1}}\leq\left\Vert x-y\right\Vert _{\sm}$
for all $x,y\in\dom$. The goal is to find a strong solution $x^{*}$
for the variational inequality arising from $F$, i.e., a solution
$x^{*}\in\dom$ satisfying $\left\langle F(x^{*}),x^{*}-x\right\rangle \leq0$
for all $x\in\dom$. Variational inequalities are a very general framework
that captures convex minimization, convex-concave saddle point problems,
and many other problems of interest \citep{BachL19,nemirovski2004prox}.
For convex minimization, the operator $F(x)$ is simply the gradient
$\nabla f(x)$.

\paragraph*{Notation. }

For diagonal matrices $D$, we use $D_{i}$ to refer to the $i^{th}$
diagonal entry. We use $R$ to denote the $\ell_{2}$ diameter of
the domain $\dom$, $R=\max_{x,y\in\dom}\left\Vert x-y\right\Vert _{2}$,
and similarly $R_{\infty}$ to denote the $\ell_{\infty}$ diameter
of $\dom$. When the function is not continuously differentiable,
we abuse notation and use $\nabla f(x)$ to denote a subgradient of
$f$ at $x$. We use $G$ to denote the Lipschitz constant of $f$
i.e. $G=\max_{x\in\dom}\left\Vert \nabla f(x)\right\Vert _{2}$. In
the stochastic setting, our algorithms assume access to gradient estimators
$\widetilde{\nabla}f(x)$ satisfying the following standard assumptions
for a fixed (but unknown) scalar $\sigma$:
\begin{align}
\mathbb{E}\left[\widetilde{\nabla}f(x)\vert x\right] & =\nabla f(x)\ ,\label{eq:stoch-assumption-unbiased}\\
\mathbb{E}\left[\left\Vert \widetilde{\nabla}f(x)-\nabla f(x)\right\Vert ^{2}\right] & \le\sigma^{2}\ .\label{eq:stoch-assumption-variance}
\end{align}

\section{Adaptive Schemes for Constrained Convex Optimization and Variational
Inequalities}

\label{sec:adaptive-schemes}

In this section, we present our algorithms for constrained convex
minimization and variational inequalities. 

\subsection{Constrained $\protect\adagrad$ Scheme}

\begin{figure}[t]
\noindent\fbox{\begin{minipage}[t]{1\columnwidth - 2\fboxsep - 2\fboxrule}%
Let $x_{0}\in\dom$, $D_{0}=I$, $R_{\infty}\geq\max_{x,y\in\dom}\left\Vert x-y\right\Vert _{\infty}$.

For $t=0,\dots,T-1$, update: 
\begin{align*}
x_{t+1}= & \arg\min_{x\in\dom}\left\{ \left\langle \nabla f(x_{t}),x\right\rangle +\frac{1}{2}\left\Vert x-x_{t}\right\Vert _{D_{t}}^{2}\right\} \ ,\\
D_{t+1,i}^{2} & =D_{t,i}^{2}\left(1+\frac{\left(x_{t+1,i}-x_{t,i}\right)^{2}}{R_{\infty}^{2}}\right)\ ,\text{for all }i\in[d].
\end{align*}

Return $\overline{x}_{T}=\frac{1}{T}\sum_{t=1}^{T}x_{t}$.%
\end{minipage}}

\noindent \caption{$\protect\adagradplus$ algorithm.}
\label{alg:adagrad+} 
\end{figure}

Figure \ref{alg:adagrad+} shows our $\adagradplus$ algorithm for
constrained convex optimization. The algorithm can be viewed as a
generalization of the celebrated $\adagrad$ algorithm of \citet{duchi2011adaptive}
to the constrained setting where the feasible set $\dom$ is an arbitrary
convex set. To see the parallel with $\adagrad$, consider the gradient
mapping:
\[
g_{t}=-D_{t}\left(x_{t+1}-x_{t}\right)\Leftrightarrow x_{t+1}=x_{t}-D_{t}^{-1}g_{t}\ .
\]
Letting $\eta=R_{\infty}$, the update is
\begin{align*}
x_{t+1,i} & =x_{t,i}-\frac{\eta}{\sqrt{\eta^{2}+\sum_{s=1}^{t-1}g_{s,i}^{2}}}g_{t,i}\ , & \forall i\in[d]\ .
\end{align*}
In the unconstrained setting, we have $g_{t}=\nabla f(x_{t})$ and
our scheme almost coincides with $\adagrad$. We have chosen the initial
scaling to be the identity, whereas the original $\adagrad$ scheme
uses $D_{0}=\epsilon I$. Our analysis extends to this choice and
we incur an additional $O\left(\log(1/\epsilon)\right)$ factor in
the convergence guarantee. In addition, the diagonal matrix $D_{t}$
we use is off by one iterate, in the sense that it does not contain
information about $g_{t}$. This is an essential feature of our method,
since in the constrained setting computing the gradient mapping requires
access to $D_{t}$.

Similarly to \citep{BachL19}, we can motivate the choice of updating
$D$ by the iterate movement as follows. The algorithm simultaneously
addresses the unconstrained setting and the more challenging constrained
setting. Since our goal is to design a universal method, intuitively
we would like the step size to decay in the non-smooth setting and
to remain constant in the smooth setting, similarly to the standard
(non-adaptive) gradient descent schemes. In the unconstrained setting,
the iterate movement coincides with the gradient. In the constrained
setting, the gradient is non-zero at the optimum and we cannot hope
that the gradient norm decreases as we approach the optimum. Instead,
as the iterate converges to the optimum, the movement also goes to
zero and thus our adaptive step size remains around the optimal value.

\paragraph*{Covergence Guarantees for $\protect\adagradplus$. }

We show that the algorithm is universal and it obtains the smooth
rate of $\frac{1}{T}$ if the function is smooth while retaining the
optimal $\frac{1}{\sqrt{T}}$ rate if the function is non-smooth.
The algorithm automatically adapts to the smoothness parameters, the
gradient norm, and the variance parameter. The following theorem states
the precise convergence guarantees in the deterministic setting.
\begin{thm}
\label{thm:adagrad+-deterministic}Let $x^{*}\in\arg\min_{x\in\dom}f(x)$,
$R_{\infty}\geq\max_{x,y\in\dom}\left\Vert x-y\right\Vert _{\infty}$,
$G\geq\max_{x\in\dom}\left\Vert \nabla f(x)\right\Vert _{2}$, Let
$x_{t}$ be the iterates constructed by the algorithm in Figure \ref{alg:adagrad+}
and let $\overline{x}_{T}=\frac{1}{T}\sum_{t=0}^{T-1}x_{t}$. If $f$
is a convex function, we have
\[
f(\overline{x}_{T})-f(x^{*})\leq O\left(\frac{R_{\infty}\sqrt{d}G\sqrt{\ln\left(\frac{GT}{R_{\infty}}\right)}}{\sqrt{T}}+\frac{R_{\infty}^{2}d}{T}\right)\ .
\]

If $f$ is additionally $1$-smooth with respect to the norm $\left\Vert \cdot\right\Vert _{\sm}$,
where $\sm=\diag(\beta_{1},\dots,\beta_{d})$ is a diagonal matrix
with $\beta_{1},\dots,\beta_{d}\geq1$, we have
\[
f(\overline{x}_{T})-f(x^{*})\leq O\left(\frac{R_{\infty}^{2}\sum_{i=1}^{d}\beta_{i}\ln\left(2\beta_{i}\right)}{T}\right)\ .
\]
\end{thm}

In Section \ref{sec:analysis-adagrad+-stoch}, we extend the algorithm
and its analysis to the stochastic setting where we are given stochastic
gradients $\widetilde{\nabla}f(x)$ satisfying the assumptions \eqref{eq:stoch-assumption-unbiased}
and \eqref{eq:stoch-assumption-variance}. The following theorem state
the precise convergence guarantees in the stochastic setting.
\begin{thm}
\label{thm:adagrad+}Let $x^{*}\in\arg\min_{x\in\dom}f(x)$, $R_{\infty}\geq\max_{x,y\in\dom}\left\Vert x-y\right\Vert _{\infty}$,
$G\geq\max_{x\in\dom}\left\Vert \nabla f(x)\right\Vert _{2}$, and
$\sigma^{2}$ be the variance of the stochastic gradients (\eqref{eq:stoch-assumption-unbiased}
and \eqref{eq:stoch-assumption-variance}). Let $x_{t}$ be the iterates
constructed by the algorithm in Figure \ref{alg:adagrad-stoch} and
let $\overline{x}_{T}=\frac{1}{T}\sum_{t=0}^{T-1}x_{t}$. If $f$
is a convex function, we have
\[
\mathbb{E}\left[f(\overline{x}_{T})-f(x^{*})\right]\leq O\left(\frac{R_{\infty}\sqrt{d}G\sqrt{\ln\left(\frac{GT}{R_{\infty}}\right)}}{\sqrt{T}}+\frac{R_{\infty}\sqrt{d}\sigma\sqrt{\ln\left(\frac{T\sigma}{R_{\infty}}\right)}}{\sqrt{T}}+\frac{R_{\infty}^{2}d}{T}\right)\ .
\]

If $f$ is additionally $1$-smooth with respect to the norm $\left\Vert \cdot\right\Vert _{\sm}$,
where $\sm=\diag(\beta_{1},\dots,\beta_{d})$ is a diagonal matrix
with $\beta_{1},\dots,\beta_{d}\geq1$, we have
\[
\mathbb{E}\left[f(\overline{x}_{T})-f(x^{*})\right]\leq O\left(\frac{R_{\infty}^{2}\sum_{i=1}^{d}\beta_{i}\ln\left(2\beta_{i}\right)}{T}+\frac{R_{\infty}\sigma\sqrt{d\ln\left(\frac{T\sigma}{R_{\infty}}\right)}}{\sqrt{T}}\right)\ .
\]
\end{thm}

In Sections \ref{sec:analysis-adagrad+-smooth} and \ref{sec:analysis-adagrad+-nonsmooth},
we analyze the algorithm in the deterministic setting where we have
access to the actual gradients ($\sigma=0$). We extend the analysis
to the stochastic setting in Section \ref{sec:analysis-adagrad+-stoch}.
We note that we have $D_{t+1,i}^{2}\leq2D_{t,i}^{2}$. This property
will play an important role in our analysis.

\subsection{Accelerated Schemes}

We give two adaptive schemes for constrained convex optimization that
achieve the optimal rate of $\frac{1}{T^{2}}$ for smooth functions
without knowing the smoothness parameters. Our algorithms are adaptive
versions of the AC-SA algorithm \citep{acsa2012}, and the AGD+ algorithm
\citep{CohenDO18}. For this reason, we coin the names $\adaacsa$
(Figure \ref{alg:acsa}) and $\adaagdplus$ (Figure \ref{alg:agd+}).
The AGD+ algorithm is a dual-averaging version of AC-SA. The algorithms
and their adaptive versions have different iterates and they may be
useful in different contexts.

We show that our algorithms simultaneously achieve convergence rates
that are optimal (up to a $\sqrt{\ln T}$ factor) for both smooth
and non-smooth functions, both in the deterministic and stochastic
setting. The algorithms automatically adapt to the smoothness parameters,
the gradient norm, and the variance parameter.

\begin{figure}[t]
\noindent %
\noindent\fbox{\begin{minipage}[t]{1\columnwidth - 2\fboxsep - 2\fboxrule}%
Let $D_{0}=I$, $z_{0}\in\dom$, $\alpha_{t}=\gamma_{t}=1+\frac{t}{3}$,
$R_{\infty}^{2}\geq\max_{x,y\in\dom}\left\Vert x-y\right\Vert _{\infty}^{2}$.

For $t=0,\dots,T-1$, update: 
\begin{align*}
x_{t} & =\left(1-\alpha_{t}^{-1}\right)y_{t}+\alpha_{t}^{-1}z_{t}\ ,\\
z_{t+1} & =\arg\min_{u\in\dom}\left\{ \gamma_{t}\left\langle \nabla f(x_{t}),u\right\rangle +\frac{1}{2}\left\Vert u-z_{t}\right\Vert _{D_{t}}^{2}\right\} \ ,\\
y_{t+1} & =\left(1-\alpha_{t}^{-1}\right)y_{t}+\alpha_{t}^{-1}z_{t+1}\ ,\\
D_{t+1,i}^{2} & =D_{t,i}^{2}\left(1+\frac{\left(z_{t+1,i}-z_{t,i}\right)^{2}}{R_{\infty}^{2}}\right)\ ,\text{for all }i\in[d].
\end{align*}

Return $y_{T}$.%
\end{minipage}}

\caption{$\protect\adaacsa$ algorithm.}
\label{alg:acsa} 
\end{figure}

\paragraph*{Convergence Guarantees for $\protect\adaacsa$. }

We show that $\adaacsa$ is universal and it simultaneously achieves
the optimal convergence rate for both smooth and non-smooth optimization.
The following theorem states the precise convergence guarantees in
the deterministic setting. 
\begin{thm}
\label{thm:acsa-deterministic}Let $x^{*}\in\arg\min_{x\in\dom}f(x)$,
$R_{\infty}\geq\max_{x,y\in\dom}\left\Vert x-y\right\Vert _{\infty}$,
$G\geq\max_{x\in\dom}\left\Vert \nabla f(x)\right\Vert _{2}$. Let
$y_{t}$ be the iterates constructed by the algorithm in Figure \ref{alg:acsa}.
If $f$ is a convex function, we have
\begin{align*}
f(y_{T})-f(x^{*}) & \leq O\left(\frac{R_{\infty}\sqrt{d}G\sqrt{\ln\left(\frac{GT}{R_{\infty}}\right)}}{\sqrt{T}}+\frac{R_{\infty}^{2}d}{T^{2}}\right)\ .
\end{align*}

If $f$ is additionally $1$-smooth with respect to the norm $\left\Vert \cdot\right\Vert _{\sm}$,
where $\sm=\diag(\beta_{1},\dots,\beta_{d})$ is a diagonal matrix
with $\beta_{1},\dots,\beta_{d}\geq1$, we have
\[
f(y_{T})-f(x^{*})\leq O\left(\frac{R_{\infty}^{2}\sum_{i=1}^{d}\beta_{i}\ln\left(2\beta_{i}\right)}{T^{2}}\right)\ .
\]
\end{thm}

In Section \ref{sec:analysis-acsa-stoch}, we extend the \textbf{$\adaacsa$}
algorithm and its analysis to the stochastic setting where we are
given stochastic gradients $\widetilde{\nabla}f(x)$ satisfying the
assumptions \eqref{eq:stoch-assumption-unbiased} and \eqref{eq:stoch-assumption-variance}.
The following theorem state the precise convergence guarantees in
the stochastic setting.
\begin{thm}
\label{thm:acsa}Let $x^{*}\in\arg\min_{x\in\dom}f(x)$, $R_{\infty}\geq\max_{x,y\in\dom}\left\Vert x-y\right\Vert _{\infty}$,
$G\geq\max_{x\in\dom}\left\Vert \nabla f(x)\right\Vert _{2}$, and
$\sigma^{2}$ be the variance of the stochastic gradients (\eqref{eq:stoch-assumption-unbiased}
and \eqref{eq:stoch-assumption-variance}). Let $y_{t}$ be the iterates
constructed by the algorithm in Figure \ref{alg:acc-stoch-1}. If
$f$ is a convex function, we have
\begin{align*}
\mathbb{E}\left[f(y_{T})-f(x^{*})\right] & \leq O\left(\frac{R_{\infty}\sqrt{d}G\sqrt{\ln\left(\frac{GT}{R_{\infty}}\right)}+R_{\infty}\sqrt{d}\sigma\sqrt{\ln\left(\frac{T\sigma}{R_{\infty}}\right)}}{\sqrt{T}}+\frac{R_{\infty}^{2}d}{T^{2}}\right)\ .
\end{align*}

If $f$ is additionally $1$-smooth with respect to the norm $\left\Vert \cdot\right\Vert _{\sm}$,
where $\sm=\diag(\beta_{1},\dots,\beta_{d})$ is a diagonal matrix
with $\beta_{1},\dots,\beta_{d}\geq1$, we have
\[
\mathbb{E}\left[f(y_{T})-f(x^{*})\right]\leq O\left(\frac{R_{\infty}^{2}\sum_{i=1}^{d}\beta_{i}\ln\left(2\beta_{i}\right)}{T^{2}}+\frac{R_{\infty}\sqrt{d}\sigma\sqrt{\ln\left(\frac{T\sigma}{R_{\infty}}\right)}}{\sqrt{T}}\right)\ .
\]
\end{thm}

In Sections \ref{sec:analysis-acc-smooth-1} and \ref{sec:analysis-acc-nonsmooth-1},
we analyze the \textbf{$\adaacsa$} algorithm in the deterministic
setting where we have access to the actual gradients ($\sigma=0$).
We extend the analysis to the stochastic setting in Section \ref{sec:analysis-acsa-stoch}.

\begin{figure}[t]
\noindent %
\noindent\fbox{\begin{minipage}[t]{1\columnwidth - 2\fboxsep - 2\fboxrule}%
Let $D_{1}=I$, $z_{0}\in\dom$, $a_{t}=t$, $A_{t}=\sum_{i=1}^{t}a_{i}=\frac{t(t+1)}{2}$,
$R_{\infty}^{2}\geq\max_{x,y\in\dom}\left\Vert x-y\right\Vert _{\infty}^{2}$.

For $t=1,\dots,T$, update: 
\begin{align*}
x_{t} & =\frac{A_{t-1}}{A_{t}}y_{t-1}+\frac{a_{t}}{A_{t}}z_{t-1}\ ,\\
z_{t} & =\arg\min_{u\in\dom}\left(\sum_{i=1}^{t}\left\langle a_{i}\nabla f(x_{i}),u\right\rangle +\frac{1}{2}\left\Vert u-z_{0}\right\Vert _{D_{t}}^{2}\right)\\
y_{t} & =\frac{A_{t-1}}{A_{t}}y_{t-1}+\frac{a_{t}}{A_{t}}z_{t}\ ,\\
D_{t+1,i}^{2} & =D_{t,i}^{2}\left(1+\frac{\left(z_{t,i}-z_{t-1,i}\right)^{2}}{R_{\infty}^{2}}\right)\ ,\text{for all }i\in[d].
\end{align*}

Return $y_{T}$.%
\end{minipage}}

\caption{$\protect\adaagdplus$ algorithm.}
\label{alg:agd+} 
\end{figure}

\paragraph*{Convergence Guarantees for $\protect\adaagdplus$. }

We show that $\adaagdplus$ is universal and it simultaneously achieves
the optimal convergence rate for both smooth and non-smooth optimization.
The following theorem states the precise convergence guarantees in
the deterministic setting. 
\begin{thm}
\label{thm:agd+-deterministic}Let $x^{*}\in\arg\min_{x\in\dom}f(x)$,
$R_{\infty}\geq\max_{x,y\in\dom}\left\Vert x-y\right\Vert _{\infty}$,
$G\geq\max_{x\in\dom}\left\Vert \nabla f(x)\right\Vert _{2}$. Let
$y_{t}$ be the iterates constructed by the algorithm in Figure \ref{alg:agd+}.
If $f$ is a convex function, we have
\begin{align*}
f(y_{T})-f(x^{*}) & \leq O\left(\frac{R_{\infty}\sqrt{d}G\sqrt{\ln\left(\frac{GT}{R_{\infty}}\right)}}{\sqrt{T}}+\frac{R_{\infty}^{2}d}{T^{2}}\right)\ .
\end{align*}

If $f$ is additionally $1$-smooth with respect to the norm $\left\Vert \cdot\right\Vert _{\sm}$,
where $\sm=\diag(\beta_{1},\dots,\beta_{d})$ is a diagonal matrix
with $\beta_{1},\dots,\beta_{d}\geq1$, we have
\[
f(y_{T})-f(x^{*})\leq O\left(\frac{R_{\infty}^{2}\sum_{i=1}^{d}\beta_{i}\ln\left(2\beta_{i}\right)}{T^{2}}\right)\ .
\]
\end{thm}

In Section \ref{sec:analysis-agd+-stoch}, we extend the \textbf{$\adaagdplus$}
algorithm and its analysis to the stochastic setting where we are
given stochastic gradients $\widetilde{\nabla}f(x)$ satisfying the
assumptions \eqref{eq:stoch-assumption-unbiased} and \eqref{eq:stoch-assumption-variance}.
The following theorem state the precise convergence guarantees in
the stochastic setting.
\begin{thm}
\label{thm:agd+}Let $x^{*}\in\arg\min_{x\in\dom}f(x)$, $R_{\infty}\geq\max_{x,y\in\dom}\left\Vert x-y\right\Vert _{\infty}$,
$G\geq\max_{x\in\dom}\left\Vert \nabla f(x)\right\Vert _{2}$, and
$\sigma^{2}$ be the variance of the stochastic gradients (\eqref{eq:stoch-assumption-unbiased}
and \eqref{eq:stoch-assumption-variance}). Let $y_{t}$ be the iterates
constructed by the algorithm in Figure \ref{alg:acc-stoch}. If $f$
is a convex function, we have
\begin{align*}
\mathbb{E}\left[f(y_{T})-f(x^{*})\right] & \leq O\left(\frac{R_{\infty}\sqrt{d}G\sqrt{\ln\left(\frac{GT}{R_{\infty}}\right)}+R_{\infty}\sqrt{d}\sigma\sqrt{\ln\left(\frac{T\sigma}{R_{\infty}}\right)}}{\sqrt{T}}+\frac{R_{\infty}^{2}d}{T^{2}}\right)\ .
\end{align*}

If $f$ is additionally $1$-smooth with respect to the norm $\left\Vert \cdot\right\Vert _{\sm}$,
where $\sm=\diag(\beta_{1},\dots,\beta_{d})$ is a diagonal matrix
with $\beta_{1},\dots,\beta_{d}\geq1$, we have
\[
\mathbb{E}\left[f(y_{T})-f(x^{*})\right]\leq O\left(\frac{R_{\infty}^{2}\sum_{i=1}^{d}\beta_{i}\ln\left(2\beta_{i}\right)}{T^{2}}+\frac{R_{\infty}\sqrt{d}\sigma\sqrt{\ln\left(\frac{T\sigma}{R_{\infty}}\right)}}{\sqrt{T}}\right)\ .
\]
\end{thm}

In Sections \ref{sec:analysis-acc-smooth} and \ref{sec:analysis-acc-nonsmooth},
we analyze the algorithm in the deterministic setting where we have
access to the actual gradients ($\sigma=0$). We extend the analysis
to the stochastic setting in Section \ref{sec:analysis-agd+-stoch}.

\subsection{Variational Inequalities}

Building on the work of Bach and Levy \citeyearpar{BachL19}, we give
the first universal method with per-coordinate adaptive step sizes
for variational inequalities arising from monotone operators, and
answer the open question asked by them. The algorithm, shown in Figure
\ref{alg:mirror-prox-main-body}, is the natural extension to the
vector setting of the scheme of \citet{BachL19}. A notable difference
is that the algorithm provided in \citep{BachL19} uses an estimate
for $G\geq\max_{x\in\dom}\left\Vert F(x)\right\Vert $ as part of
the step size. Our algorithm does not use the $G$ parameter and it
automatically adapts to it, as well as the smoothness and variance
parameters.

\begin{figure}[t]
\noindent %
\noindent\fbox{\begin{minipage}[t]{1\columnwidth - 2\fboxsep - 2\fboxrule}%
Let $y_{0}\in\dom$, $D_{1}=I$, $R_{\infty}\geq\max_{x,y\in K}\left\Vert x-y\right\Vert _{\infty}$.

For $t=1,\dots,T$, update: 
\begin{align*}
x_{t} & =\arg\min_{x\in\dom}\left\{ \left\langle F(y_{t-1}),x\right\rangle +\frac{1}{2}\left\Vert x-y_{t-1}\right\Vert _{D_{t}}^{2}\right\} \ ,\\
y_{t} & =\arg\min_{x\in\dom}\left\{ \left\langle F(x_{t}),x\right\rangle +\frac{1}{2}\left\Vert x-y_{t-1}\right\Vert _{D_{t}}^{2}\right\} \ ,\\
D_{t+1,i}^{2} & =D_{t,i}^{2}\left(1+\frac{\left(x_{t,i}-y_{t-1,i}\right)^{2}+\left(x_{t,i}-y_{t,i}\right)^{2}}{2R_{\infty}^{2}}\right)\ .
\end{align*}
Return $\overline{x}_{T}=\frac{1}{T}\sum_{t=1}^{T}x_{t}$.%
\end{minipage}}

\caption{$\protect\adamp$ algorithm, extending \citet{BachL19} to the vector
setting.}

\label{alg:mirror-prox-main-body} 
\end{figure}

\paragraph*{Convergence Guarantees for $\protect\adamp$. }

By combining the analysis of \citet{BachL19} with our techniques
from the other sections, we show that the algorithm is universal and
it simultaneously achieves the nearly-optimal rates (up to a $\sqrt{\ln T}$
factor) for both smooth and non-smooth operators. The following theorem
shows the precise convergence guarantees in the deterministic setting
($\sigma=0$), and we give the proof in Section \ref{sec:mirror-prox}.
We refer the reader to Section \ref{sec:mirror-prox} for the precise
definitions which concern this setting.
\begin{thm}
\label{thm:variational} Consider the problem $\min_{x\in\dom}F(x)$,
where $F$ is a monotone operator and $\dom$ is a convex set. Let
$R_{\infty}\geq\max_{x,y\in K}\left\Vert x-y\right\Vert _{\infty}$
and $G\geq\max_{x\in\dom}\left\Vert F(x)\right\Vert _{2}$ . Let $x_{t}$
be the iterates constructed by the algorithm in Figure \ref{alg:mirror-prox-main-body}
and let $\overline{x}_{T}=\frac{1}{T}\sum_{t=1}^{T}x_{t}$. We have
\[
\dualitygap(\overline{x}_{T})\leq O\left(\frac{\sqrt{d}R_{\infty}G\sqrt{\ln\left(\frac{GT}{R_{\infty}}\right)}}{\sqrt{T}}+\frac{R_{\infty}^{2}d}{T} \right) \ .
\]
If $F$ is additionally $1$-smooth with respect to the norm $\left\Vert \cdot\right\Vert _{\sm}$,
where $\sm=\diag(\beta_{1},\dots,\beta_{d})$ is a diagonal matrix
with $\beta_{1},\dots,\beta_{d}\geq1$, we have 
\[
\dualitygap(\overline{x}_{T})\leq O\left( \frac{R_{\infty}^{2}\sum_{i=1}^{d}\beta_{i}\ln\left(2\beta_{i}\right)}{T} \right)\ .
\]
\end{thm}

\paragraph*{Extension to the stochastic setting. }

Analogously to our other results, in the stochastic setting, we assume
that we are given noisy evaluations $\widetilde{F}(x_{t})$ satisfying
the expectation and variance assumptions $\E\left[\widetilde{F}(x_{t})\vert x_{t}\right]=F(x_{t})$
and $\E\left[\left\Vert \widetilde{F}(x_{t})-F(x_{t})\right\Vert ^{2}\right]\leq\sigma^{2}$.
The algorithm and its analysis can be easily extended to the stochastic
setting using the techniques developed in Sections \ref{sec:analysis-adagrad+-stoch},
\ref{sec:analysis-acsa-stoch} and \ref{sec:analysis-agd+-stoch},
and we omit this straightforward extension. We refer the reader to
\ref{table:results} for the convergence guarantee in the stochastic
setting.

We note that, in the sthocastic setting, the analysis of \citet{BachL19}
makes the additional assumption that the stochastic values have bounded
norms almost surely, i.e., $\left\Vert \widetilde{F}(x_{t})\right\Vert \leq G$
with probability one, which is stronger than our assumption of bounded
variance. This assumption simplifies the analysis, as it allows one
to directly upper bound $D_{T}$ (equivalently, lower bound the step
size $\eta_{T}=1/D_{T}$), which is the key loss term in the convergence
analysis. Our analysis removes this assumption by employing a more
involved argument that does not upper bound $\tr(D_{T})$ directly.

\section{Analysis of $\protect\adagradplus$ for Smooth Functions}

\label{sec:analysis-adagrad+-smooth}

We make the following observation that will be used in the analysis:
we have $D_{t+1,i}^{2}\leq2D_{t,i}^{2}$ for all iterations $t\in[T]$
and coordinates $i\in[d]$. We start with the following lemma, which
follows from the standard analysis of gradient descent. 
\begin{lem}
\label{lem:main} For any $y\in\dom$, we have 
\begin{align*}
f(x_{t+1})-f(y) & \leq\left\langle D_{t}\left(x_{t}-x_{t+1}\right),x_{t}-y\right\rangle -\left\Vert x_{t+1}-x_{t}\right\Vert _{D_{t}}^{2}+\frac{1}{2}\left\Vert x_{t+1}-x_{t}\right\Vert _{\sm}^{2}\\
 & =\frac{1}{2}\left(\left\Vert x_{t}-y\right\Vert _{D_{t}}^{2}-\left\Vert x_{t+1}-y\right\Vert _{D_{t}}^{2}-\left\Vert x_{t+1}-x_{t}\right\Vert _{D_{t}}^{2}\right)+\frac{1}{2}\left\Vert x_{t+1}-x_{t}\right\Vert _{\sm}^{2}\ .
\end{align*}
\end{lem}

\begin{proof}
We write $f(x_{t+1})-f(y)=f(x_{t+1})-f(x_{t})+f(x_{t})-f(y)$, and
we use smoothness to bound the first term and convexity to bound the
second term. 
\begin{align*}
f(x_{t+1})-f(y) & =f(x_{t+1})-f(x_{t})+f(x_{t})-f(y)\\
 & \leq\left\langle \nabla f(x_{t}),x_{t+1}-x_{t}\right\rangle +\frac{1}{2}\left\Vert x_{t+1}-x_{t}\right\Vert _{\sm}^{2}+\left\langle \nabla f(x_{t}),x_{t}-y\right\rangle \\
 & =\left\langle \nabla f(x_{t}),x_{t+1}-y\right\rangle +\frac{1}{2}\left\Vert x_{t+1}-x_{t}\right\Vert _{\sm}^{2}\ .
\end{align*}
Next, we use the first-order optimality condition for $x_{t+1}$ to
obtain 
\[
\left\langle \nabla f(x_{t})+D_{t}\left(x_{t+1}-x_{t}\right),x_{t+1}-y\right\rangle \leq0\ .
\]
By rearranging, we obtain 
\begin{align*}
\left\langle \nabla f(x_{t}),x_{t+1}-y\right\rangle  & \leq\left\langle D_{t}\left(x_{t}-x_{t+1}\right),x_{t+1}-y\right\rangle \\
 & =\left\langle D_{t}\left(x_{t}-x_{t+1}\right),x_{t}-y+x_{t+1}-x_{t}\right\rangle \\
 & =\left\langle D_{t}\left(x_{t}-x_{t+1}\right),x_{t}-y\right\rangle -\left\Vert x_{t+1}-x_{t}\right\Vert _{D_{t}}^{2}\ .
\end{align*}
Plugging into the previous inequality gives 
\begin{align*}
f(x_{t+1})-f(y) & \leq\left\langle D_{t}\left(x_{t}-x_{t+1}\right),x_{t}-y\right\rangle -\left\Vert x_{t+1}-x_{t}\right\Vert _{D_{t}}^{2}+\frac{1}{2}\left\Vert x_{t+1}-x_{t}\right\Vert _{\sm}^{2}\ .
\end{align*}
Finally, we note that 
\[
\left\langle D_{t}\left(x_{t}-x_{t+1}\right),x_{t}-y\right\rangle =\frac{1}{2}\left(\left\Vert x_{t}-y\right\Vert _{D_{t}}^{2}-\left\Vert x_{t+1}-y\right\Vert _{D_{t}}^{2}+\left\Vert x_{t+1}-x_{t}\right\Vert _{D_{t}}^{2}\right)\ .
\]
Indeed, we have 
\begin{align*}
 & \left\Vert x_{t}-y\right\Vert _{D_{t}}^{2}-\left\Vert x_{t+1}-y\right\Vert _{D_{t}}^{2}+\left\Vert x_{t+1}-x_{t}\right\Vert _{D_{t}}^{2}\\
 & =\left\Vert x_{t}-x_{t+1}+x_{t+1}-y\right\Vert _{D_{t}}^{2}-\left\Vert x_{t+1}-y\right\Vert _{D_{t}}^{2}+\left\Vert x_{t+1}-x_{t}\right\Vert _{D_{t}}^{2}\\
 & =2\left\Vert x_{t}-x_{t+1}\right\Vert _{D_{t}}^{2}+2\left\langle D_{t}\left(x_{t}-x_{t+1}\right),x_{t+1}-y\right\rangle \\
 & =2\left\Vert x_{t}-x_{t+1}\right\Vert _{D_{t}}^{2}+2\left\langle D_{t}\left(x_{t}-x_{t+1}\right),x_{t+1}-x_{t}+x_{t}-y\right\rangle \\
 & =2\left\Vert x_{t}-x_{t+1}\right\Vert _{D_{t}}^{2}-2\left\Vert x_{t+1}-x_{t}\right\Vert _{D_{t}}^{2}+2\left\langle D_{t}\left(x_{t}-x_{t+1}\right),x_{t}-y\right\rangle \\
 & =2\left\langle D_{t}\left(x_{t}-x_{t+1}\right),x_{t}-y\right\rangle \ .
\end{align*}
\end{proof}
Using the standard $\adagrad$ analysis, we obtain the following lemma. 
\begin{lem}
\label{lem:standard-adagrad} We have 
\[
\sum_{t=0}^{T-1}\left(f(x_{t+1})-f(x^{*})\right)\leq\frac{1}{2}R_{\infty}^{2}\tr(D_{T})-\frac{1}{2}\sum_{t=0}^{T-1}\left\Vert x_{t+1}-x_{t}\right\Vert _{D_{t}}^{2}+\frac{1}{2}\sum_{t=0}^{T-1}\left\Vert x_{t+1}-x_{t}\right\Vert _{\sm}^{2}\ .
\]
\end{lem}

\begin{proof}
Setting $y=x^{*}$ in Lemma \ref{lem:main} and summing up over all
iterations, 
\begin{align*}
 & 2\sum_{t=0}^{T-1}\left(f(x_{t+1})-f(x^{*})\right)\\
 & =\sum_{t=0}^{T-1}\left(\left\Vert x_{t}-x^{*}\right\Vert _{D_{t}}^{2}-\left\Vert x_{t+1}-x^{*}\right\Vert _{D_{t}}^{2}-\left\Vert x_{t+1}-x_{t}\right\Vert _{D_{t}}^{2}+\left\Vert x_{t+1}-x_{t}\right\Vert _{\sm}^{2}\right)\\
 & =\sum_{t=0}^{T-1}\left(\left\Vert x_{t}-x^{*}\right\Vert _{D_{t}}^{2}-\left\Vert x_{t+1}-x^{*}\right\Vert _{D_{t}}^{2}\right)-\sum_{t=0}^{T-1}\left\Vert x_{t+1}-x_{t}\right\Vert _{D_{t}}^{2}+\sum_{t=0}^{T-1}\left\Vert x_{t+1}-x_{t}\right\Vert _{\sm}^{2}\\
 & =\left\Vert x_{0}-x^{*}\right\Vert _{D_{0}}^{2}-\left\Vert x_{T}-x^{*}\right\Vert _{D_{T}}^{2}+\sum_{t=0}^{T-1}\left\Vert x_{t+1}-x^{*}\right\Vert _{D_{t+1}-D_{t}}^{2}-\sum_{t=0}^{T-1}\left\Vert x_{t+1}-x_{t}\right\Vert _{D_{t}}^{2}+\sum_{t=0}^{T-1}\left\Vert x_{t+1}-x_{t}\right\Vert _{\sm}^{2}\ .
\end{align*}
We use the following upper bound ($\tr$ denotes the trace of the
matrix): 
\[
\left\Vert x-y\right\Vert _{D}^{2}\leq\tr(D)\left\Vert x-y\right\Vert _{\infty}^{2}\leq\tr(D)R_{\infty}^{2}\ .
\]
We obtain 
\begin{align*}
 & 2\sum_{t=0}^{T-1}\left(f(x_{t+1})-f(x^{*})\right)\\
 & \leq\tr(D_{0})R_{\infty}^{2}+R_{\infty}^{2}\sum_{t=0}^{T-1}\left(\tr(D_{t+1})-\tr(D_{t})\right)-\sum_{t=0}^{T-1}\left\Vert x_{t+1}-x_{t}\right\Vert _{D_{t}}^{2}+\sum_{t=0}^{T-1}\left\Vert x_{t+1}-x_{t}\right\Vert _{\sm}^{2}\\
 & =R_{\infty}^{2}\tr(D_{T})-\sum_{t=0}^{T-1}\left\Vert x_{t+1}-x_{t}\right\Vert _{D_{t}}^{2}+\sum_{t=0}^{T-1}\left\Vert x_{t+1}-x_{t}\right\Vert _{\sm}^{2}\ .
\end{align*}
\end{proof}
We now proceed with the main part of the analysis. We will show that
the right-hand side in the above lemma is bounded by a constant (independent
of $T$). Our analysis can be viewed as a vector generalization of
the scalar analyses presented in previous work \citep{LevyYC18,BachL19,KavisLBC19}.
In our setting, we have a per-coordinate scaling $D_{t,i}$ whereas
previous work used the same scaling $D_{t}$ for each coordinate.

Note that the guarantee provided by the above lemma has two loss terms,
$R_{\infty}^{2}\tr(D_{T})$ and $\sum_{t=0}^{T-1}\left\Vert x_{t+1}-x_{t}\right\Vert _{\sm}^{2}$,
and the gain term $\sum_{t=0}^{T-1}\left\Vert x_{t+1}-x_{t}\right\Vert _{D_{t}}^{2}$.
We will use the gain to absorb most of the loss as follows. We write
\begin{align*}
 & R_{\infty}^{2}\tr(D_{T})-\sum_{t=0}^{T-1}\left\Vert x_{t+1}-x_{t}\right\Vert _{D_{t}}^{2}+\sum_{t=0}^{T-1}\left\Vert x_{t+1}-x_{t}\right\Vert _{\sm}^{2}\\
 & =\underbrace{R_{\infty}^{2}\tr(D_{T})-\frac{1}{2}\sum_{t=0}^{T-1}\left\Vert x_{t+1}-x_{t}\right\Vert _{D_{t}}^{2}}_{(\star)}+\underbrace{\sum_{t=0}^{T-1}\left\Vert x_{t+1}-x_{t}\right\Vert _{\sm}^{2}-\frac{1}{2}\sum_{t=0}^{T-1}\left\Vert x_{t+1}-x_{t}\right\Vert _{D_{t}}^{2}}_{(\star\star)}\ .
\end{align*}
We upper bound each of the terms $(\star)$ and $(\star\star)$ in
turn.

Before proceeding, we prove the following generic lemma that we will
use throughout the paper. The inequalities are standard and are equivalent
to the inequalities used in previous work. We give the proof in Section
\ref{sec:lemma-inequalities-proof}. 
\begin{lem}
\label{lem:inequalities}Let $d_{0}^{2},d_{1}^{2},d_{2}^{2},\dots,d_{T}^{2}$
and $R^{2}$ be scalars. Let $D_{0}>0$ and let $D_{1},\dots,D_{T}$
be defined according to the following recurrence 
\[
D_{t+1}^{2}=D_{t}^{2}\left(1+\frac{d_{t}^{2}}{R^{2}}\right)\ .
\]
We have 
\[
\sum_{t=a}^{b-1}D_{t}\cdot d_{t}^{2}\geq2R^{2}\left(D_{b}-D_{a}\right)\ .
\]
If $d_{t}^{2}\leq R^{2}$ for all $t$, we have: 
\begin{align*}
\sum_{t=a}^{b-1}D_{t}\cdot d_{t}^{2} & \leq\left(\sqrt{2}+1\right)R^{2}\left(D_{b}-D_{a}\right)\\
\sum_{t=a}^{b-1}d_{t}^{2} & \leq4R^{2}\ln\left(\frac{D_{b}}{D_{a}}\right)\ .
\end{align*}
\end{lem}

\begin{lem}
\label{lem:error1}We have 
\[
(\star)=R_{\infty}^{2}\tr(D_{T})-\frac{1}{2}\sum_{t=0}^{T-1}\left\Vert x_{t+1}-x_{t}\right\Vert _{D_{t}}^{2}\leq R_{\infty}^{2}\tr(D_{0})=R_{\infty}^{2}d\ .
\]
\end{lem}

\begin{proof}
For each coordinate $i$ separately, we apply Lemma \ref{lem:inequalities}
with $d_{t}^{2}=\left(x_{t+1,i}-x_{t,i}\right)^{2}$and $R^{2}=R_{\infty}^{2}$.
By the first inequality in the lemma, 
\[
\sum_{t=0}^{T-1}D_{t,i}\left(x_{t+1,i}-x_{t,i}\right)^{2}\geq2R_{\infty}^{2}\left(D_{T}-D_{0}\right)\ .
\]
Therefore 
\[
\frac{1}{2}\sum_{t=0}^{T-1}\left\Vert x_{t+1}-x_{t}\right\Vert _{D_{t}}^{2}=\frac{1}{2}\sum_{i=1}^{d}\sum_{t=0}^{T-1}D_{t,i}\left(x_{t+1,i}-x_{t,i}\right)^{2}\geq R_{\infty}^{2}\left(\tr(D_{T})-\tr(D_{0})\right)\ .
\]
Thus 
\[
(\star)=R_{\infty}^{2}\tr(D_{T})-\frac{1}{2}\sum_{t=0}^{T-1}\left\Vert x_{t+1}-x_{t}\right\Vert _{D_{t}}^{2}\leq R_{\infty}^{2}\tr(D_{0})=R_{\infty}^{2}d\ .
\]
\end{proof}
\begin{lem}
\label{lem:error2}We have 
\[
(\star\star)=\sum_{t=0}^{T-1}\left\Vert x_{t+1}-x_{t}\right\Vert _{\sm}^{2}-\frac{1}{2}\sum_{t=0}^{T-1}\left\Vert x_{t+1}-x_{t}\right\Vert _{D_{t}}^{2}\leq O\left(R_{\infty}^{2}\sum_{i=1}^{d}\beta_{i}\ln\left(2\beta_{i}\right)\right)\ .
\]
\end{lem}

\begin{proof}
Note that, for each coordinate $i$, $D_{t,i}$ is increasing with
$t$. For each coordinate $i\in[d]$ , we let $\tilde{T}_{i}$ be
the last iteration $t$ for which $D_{t,i}\leq2\beta_{i}$; if there
is no such iteration, we let $\tilde{T}_{i}=-1$. We have 
\begin{align*}
 & \sum_{t=0}^{T-1}\left\Vert x_{t+1}-x_{t}\right\Vert _{\sm}^{2}-\frac{1}{2}\sum_{t=0}^{T-1}\left\Vert x_{t+1}-x_{t}\right\Vert _{D_{t}}^{2}\\
 & =\sum_{i=1}^{d}\sum_{t=0}^{T-1}\left(\beta_{i}\left(x_{t+1,i}-x_{t,i}\right)^{2}-\frac{1}{2}D_{t,i}\left(x_{t+1,i}-x_{t,i}\right)^{2}\right)\\
 & \leq\sum_{i=1}^{d}\sum_{t=0}^{\tilde{T}_{i}}\beta_{i}\left(x_{t+1,i}-x_{t,i}\right)^{2}\ .
\end{align*}
Next, we upper bound the above sum for each coordinate separately.
We apply Lemma \ref{lem:inequalities} with $d_{t}^{2}=\left(x_{t+1,i}-x_{t,i}\right)^{2}$
and $R^{2}=R_{\infty}^{2}\geq d_{t}^{2}$. Using the third inequality
in the lemma, we obtain 
\begin{align*}
\sum_{t=0}^{\tilde{T}_{i}}\left(x_{t+1,i}-x_{t,i}\right)^{2} & \leq R_{\infty}^{2}+\sum_{t=0}^{\tilde{T}_{i}-1}\left(x_{t+1,i}-x_{t,i}\right)^{2}\le R_{\infty}^{2}+4R_{\infty}^{2}\ln\left(\frac{D_{\tilde{T}_{i},i}}{D_{0,i}}\right)=R_{\infty}^{2}+4R_{\infty}^{2}\ln\left(2\beta_{i}\right)\ .
\end{align*}
Therefore 
\[
(\star\star)\leq R_{\infty}^{2}\sum_{i=1}^{d}\beta_{i}\left(1+2\ln\left(2\beta_{i}\right)\right)=O\left(R_{\infty}^{2}\sum_{i=1}^{d}\beta_{i}\ln\left(2\beta_{i}\right)\right)\ .
\]
The convergence guarantee now follows from combining Lemmas \ref{lem:standard-adagrad},
\ref{lem:error1}, \ref{lem:error2}. 
\end{proof}
Our analysis above did not directly upper bound $\tr(D_{T})$. In
the remainder of this section, we show that it is indeed possible
to directly bound $\tr(D_{T})$ as well and show that it is a constant
independent of $T$. This can be viewed as providing a theoretical
justification for the intuition that, for smooth functions, the $\adagrad$
step size remains constant. Note that, in the unconstrained setting,
$(f(x_{0})-f(x^{*}))/R_{\infty}^{2}$ is the lower-order term, and
the following lemma implies that the step sizes are very close to
the ideal step sizes given by the smoothness parameters. 
\begin{lem}
\label{lem:final-trace} We have 
\[
\tr(D_{T})\leq\tr(D_{0})+O\left(\sum_{i=1}^{d}\beta_{i}\ln\left(2\beta_{i}\right)\right)+\frac{f(x_{0})-f(x^{*})}{2R_{\infty}^{2}}\ .
\]
\end{lem}

\begin{proof}
Setting $y=x_{t}$ in Lemma \ref{lem:main}, 
\[
f(x_{t+1})-f(x_{t})\leq-\left\Vert x_{t+1}-x_{t}\right\Vert _{D_{t}}^{2}+\frac{1}{2}\left\Vert x_{t+1}-x_{t}\right\Vert _{\sm}^{2}\ .
\]
Summing up over all iterations and using Lemma \ref{lem:error2},
\begin{align*}
f(x_{T})-f(x_{0}) & \leq\sum_{t=0}^{T-1}\left(-\left\Vert x_{t+1}-x_{t}\right\Vert _{D_{t}}^{2}+\frac{1}{2}\left\Vert x_{t+1}-x_{t}\right\Vert _{\sm}^{2}\right)\leq O\left(R_{\infty}^{2}\sum_{i=1}^{d}\beta_{i}\ln\left(2\beta_{i}\right)\right)\ .
\end{align*}
Rearranging and using that $f(x_{T})\leq f(x^{*})$, 
\[
\sum_{t=0}^{T-1}\left\Vert x_{t+1}-x_{t}\right\Vert _{D_{t}}^{2}\leq O\left(R_{\infty}^{2}\sum_{i=1}^{d}\beta_{i}\ln\beta_{i}\right)+\left(f(x_{0})-f(x^{*})\right)\ .
\]
By Lemma \ref{lem:error1}, 
\[
R_{\infty}^{2}\left(\tr(D_{T})-\tr(D_{0})\right)\leq\frac{1}{2}\sum_{t=0}^{T-1}\left\Vert x_{t+1}-x_{t}\right\Vert _{D_{t}}^{2}\ .
\]
Thus 
\[
\tr(D_{T})\leq\tr(D_{0})+O\left(\sum_{i=1}^{d}\beta_{i}\ln\left(2\beta_{i}\right)\right)+\frac{f(x_{0})-f(x^{*})}{2R_{\infty}^{2}}\ .
\]
\end{proof}

\section{Analysis of $\protect\adaacsa$ for Smooth Functions}

\label{sec:analysis-acc-smooth-1}

We note that we have $D_{t+1,i}^{2}\leq2D_{t,i}^{2}$, which will
play an important role in our analysis. The solution $y_{t}$ is the
primal solution, $z_{t}$ is the dual solution, and $x_{t}$ is the
solution at which we compute the gradient.

In the initial part of the analysis, we follow closely the converge
analysis given in \citep{acsa2012}. 
\begin{lem}
\label{lem:acsa-basic-bound}We have that 
\end{lem}

\begin{align*}
\alpha_{t}\gamma_{t}\left(f(y_{t+1})-f(x^{*})\right) & \leq\left(\alpha_{t}-1\right)\gamma_{t}\left(f(y_{t})-f(x^{*})\right)\\
 & +\frac{1}{2}\left\Vert z_{t}-x^{*}\right\Vert _{D_{t}}^{2}-\frac{1}{2}\left\Vert z_{t+1}-x^{*}\right\Vert _{D_{t}}^{2}-\frac{1}{2}\left\Vert z_{t}-z_{t+1}\right\Vert _{D_{t}}^{2}+\frac{1}{2}\frac{\gamma_{t}}{\alpha_{t}}\left\Vert z_{t+1}-z_{t}\right\Vert _{\sm}^{2}\ .
\end{align*}

\begin{proof}
Note that by definition we have 
\begin{align}
y_{t+1}-x_{t} & =\left(\alpha_{t}^{-1}z_{t+1}+\left(1-\alpha_{t}^{-1}\right)y_{t}\right)-\left(\alpha_{t}^{-1}z_{t}+\left(1-\alpha_{t}^{-1}\right)y_{t}\right)\nonumber \\
 & =\alpha_{t}^{-1}\left(z_{t+1}-z_{t}\right)\ .\label{eq:acsa-yx-zz}
\end{align}
Using smoothness, we upper bound $\alpha_{t}\gamma_{t}f(y_{t+1})$
as follows: 
\begin{align}
\alpha_{t}\gamma_{t}f(y_{t+1}) & \leq\alpha_{t}\gamma_{t}\left(f(x_{t})+\left\langle \nabla f(x_{t}),y_{t+1}-x_{t}\right\rangle +\frac{1}{2}\left\Vert y_{t+1}-x_{t}\right\Vert _{\sm}^{2}\right)\nonumber \\
 & =\alpha_{t}\gamma_{t}\left(f(x_{t})+\left\langle \nabla f(x_{t}),y_{t+1}-x_{t}\right\rangle +\frac{1}{2}\left\Vert \alpha_{t}^{-1}\left(z_{t+1}-z_{t}\right)\right\Vert _{\sm}^{2}\right)\nonumber \\
 & =\alpha_{t}\gamma_{t}\left(f(x_{t})+\left\langle \nabla f(x_{t}),y_{t+1}-x_{t}\right\rangle \right)+\frac{1}{2}\frac{\gamma_{t}}{\alpha_{t}}\left\Vert z_{t+1}-z_{t}\right\Vert _{\sm}^{2}\,,\label{eq:acsa-fy-ineq}
\end{align}
where we obtained the first identity by plugging \eqref{eq:acsa-yx-zz}.
Next, we further write, using \eqref{eq:acsa-yx-zz} and the definition
of $x_{t}$: 
\begin{align*}
y_{t+1}-x_{t} & =\alpha_{t}^{-1}\left(z_{t+1}-z_{t}\right)=\alpha_{t}^{-1}z_{t+1}-\alpha_{t}^{-1}z_{t}\\
 & =\alpha_{t}^{-1}z_{t+1}+\left(1-\alpha_{t}^{-1}\right)y_{t}-x_{t}\\
 & =\alpha_{t}^{-1}\left(z_{t+1}-x_{t}\right)+\left(1-\alpha_{t}^{-1}\right)\left(y_{t}-x_{t}\right)\,,
\end{align*}
which enables us to expand: 
\begin{align*}
 & \alpha_{t}\gamma_{t}\left(f(x_{t})+\left\langle \nabla f(x_{t}),y_{t+1}-x_{t}\right\rangle \right)\\
 & =\alpha_{t}\gamma_{t}\left(f(x_{t})+\left\langle \nabla f(x_{t}),\alpha_{t}^{-1}\left(z_{t+1}-x_{t}\right)+\left(1-\alpha_{t}^{-1}\right)\left(y_{t}-x_{t}\right)\right\rangle \right)\\
 & =\alpha_{t}\gamma_{t}\left(\left(1-\alpha_{t}^{-1}\right)\left(f(x_{t})+\left\langle \nabla f(x_{t}),y_{t}-x_{t}\right\rangle \right)+\alpha_{t}^{-1}\left(f(x_{t})+\left\langle \nabla f(x_{t}),z_{t+1}-x_{t}\right\rangle \right)\right)\\
 & =\left(\alpha_{t}-1\right)\gamma_{t}\left(f(x_{t})+\left\langle \nabla f(x_{t}),y_{t}-x_{t}\right\rangle \right)+\gamma_{t}\left(f(x_{t})+\left\langle \nabla f(x_{t}),z_{t+1}-x_{t}\right\rangle \right)\\
 & \leq\left(\alpha_{t}-1\right)\gamma_{t}\cdot f(y_{t})+\gamma_{t}\left(f(x_{t})+\left\langle \nabla f(x_{t}),z_{t+1}-x_{t}\right\rangle \right)\,,
\end{align*}
where we used convexity for the last inequality. Plugging this into
\eqref{eq:acsa-fy-ineq} we obtain: 
\begin{equation}
\alpha_{t}\gamma_{t}f(y_{t+1})\leq\left(\alpha_{t}-1\right)\gamma_{t}\cdot f(y_{t})+\underbrace{\gamma_{t}\left(f(x_{t})+\left\langle \nabla f(x_{t}),z_{t+1}-x_{t}\right\rangle \right)}_{(\diamond)}+\frac{1}{2}\frac{\gamma_{t}}{\alpha_{t}}\left\Vert z_{t+1}-z_{t}\right\Vert _{\sm}^{2}\ .\label{eq:acsa-fy-ineq2}
\end{equation}
We now upper bound $(\diamond)$. Let 
\[
\phi_{t}(u)=\gamma_{t}\left(f(x_{t})+\left\langle \nabla f(x_{t}),u-x_{t}\right\rangle \right)+\frac{1}{2}\left\Vert u-z_{t}\right\Vert _{D_{t}}^{2}\ .
\]
Since $\phi_{t}$ is $1$-strongly convex with respect to $\left\Vert \cdot\right\Vert _{D_{t}}$
and $z_{t+1}=\arg\min_{u\in\dom}\phi_{t}(u)$ by definition, we have
that for all $u\in\dom$: 
\begin{align*}
\phi_{t}(u) & \geq\phi_{t}(z_{t+1})+\underbrace{\left\langle \nabla\phi_{t}(z_{t+1}),u-z_{t+1}\right\rangle }_{\geq0}+\frac{1}{2}\left\Vert u-z_{t+1}\right\Vert _{D_{t}}^{2}\\
 & \geq\phi_{t}(z_{t+1})+\frac{1}{2}\left\Vert u-z_{t+1}\right\Vert _{D_{t}}^{2}\ .
\end{align*}
The non-negativity of the inner product term follows from first order
optimality: locally any move away from $z_{t+1}$ can not possibly
decrease the value of $\phi_{t}$. Thus 
\[
\phi_{t}(z_{t+1})\leq\phi_{t}(u)-\frac{1}{2}\left\Vert u-z_{t+1}\right\Vert _{D_{t}}^{2}\ .
\]
This allows us to bound:

\begin{align*}
(\diamond)= & \gamma_{t}\left(f(x_{t})+\left\langle \nabla f(x_{t}),z_{t+1}-x_{t}\right\rangle \right)\\
= & \phi_{t}(z_{t+1})-\frac{1}{2}\left\Vert z_{t+1}-z_{t}\right\Vert _{D_{t}}^{2}\\
\leq & \phi_{t}(x^{*})-\left\Vert x^{*}-z_{t+1}\right\Vert _{D_{t}}^{2}-\frac{1}{2}\left\Vert z_{t+1}-z_{t}\right\Vert _{D_{t}}^{2}\\
= & \gamma_{t}\left(f(x_{t})+\left\langle \nabla f(x_{t}),x^{*}-x_{t}\right\rangle \right)+\frac{1}{2}\left\Vert x^{*}-z_{t}\right\Vert _{D_{t}}^{2}-\frac{1}{2}\left\Vert x^{*}-z_{t+1}\right\Vert _{D_{t}}^{2}-\frac{1}{2}\left\Vert z_{t+1}-z_{t}\right\Vert _{D_{t}}^{2}\\
\leq & \gamma_{t}f(x^{*})+\frac{1}{2}\left\Vert z_{t}-x^{*}\right\Vert _{D_{t}}^{2}-\frac{1}{2}\left\Vert z_{t+1}-x^{*}\right\Vert _{D_{t}}^{2}-\frac{1}{2}\left\Vert z_{t}-z_{t+1}\right\Vert _{D_{t}}^{2}\ .
\end{align*}
Plugging back into \eqref{eq:acsa-fy-ineq2}, we obtain: 
\begin{align*}
\alpha_{t}\gamma_{t}f(y_{t+1}) & \leq\left(\alpha_{t}-1\right)\gamma_{t}\cdot f(y_{t})+\gamma_{t}f(x^{*})\\
 & +\frac{1}{2}\left\Vert z_{t}-x^{*}\right\Vert _{D_{t}}^{2}-\frac{1}{2}\left\Vert z_{t+1}-x^{*}\right\Vert _{D_{t}}^{2}-\frac{1}{2}\left\Vert z_{t}-z_{t+1}\right\Vert _{D_{t}}^{2}\ .+\frac{1}{2}\frac{\gamma_{t}}{\alpha_{t}}\left\Vert z_{t+1}-z_{t}\right\Vert _{\sm}^{2}\ .
\end{align*}
Subtracting $\alpha_{t}\gamma_{t}f(x^{*})$ from both sides and obtain

\begin{align*}
\alpha_{t}\gamma_{t}\left(f(y_{t+1})-f(x^{*})\right) & \leq\left(\alpha_{t}-1\right)\gamma_{t}\left(f(y_{t})-f(x^{*})\right)\\
 & +\frac{1}{2}\left\Vert z_{t}-x^{*}\right\Vert _{D_{t}}^{2}-\frac{1}{2}\left\Vert z_{t+1}-x^{*}\right\Vert _{D_{t}}^{2}-\frac{1}{2}\left\Vert z_{t}-z_{t+1}\right\Vert _{D_{t}}^{2}+\frac{1}{2}\frac{\gamma_{t}}{\beta_{t}}\left\Vert z_{t+1}-z_{t}\right\Vert _{\sm}^{2}\ ,
\end{align*}
which concludes the proof. 
\end{proof}
\begin{lem}
\label{lem:acsa-conditional-telescope}Suppose that the parameters
$\left\{ \alpha_{t}\right\} _{t}$, $\left\{ \gamma_{t}\right\} _{t}$
satisfy 
\[
0<\left(\alpha_{t+1}-1\right)\gamma_{t+1}\leq\alpha_{t}\gamma_{t}\ ,
\]
for all $t\geq0$. Then 
\begin{align*}
 & \left(\alpha_{T}-1\right)\gamma_{T}\left(f(y_{T})-f(x^{*})\right)-\left(\alpha_{0}-1\right)\gamma_{0}\left(f(y_{0})-f(x^{*})\right)\\
 & \leq\frac{1}{2}R_{\infty}^{2}\tr\left(D_{T-1}\right)+\sum_{t=0}^{T-1}\left(\frac{1}{2}\frac{\gamma_{t}}{\alpha_{t}}\left\Vert z_{t}-z_{t+1}\right\Vert _{\sm}^{2}-\frac{1}{2}\left\Vert z_{t}-z_{t+1}\right\Vert _{D_{t}}^{2}\right)\ .
\end{align*}
\end{lem}

\begin{proof}
Since $f(y_{t+1})-f(x^{*})\geq0$, we apply the hypothesis to Lemma
\ref{lem:acsa-basic-bound} and obtain:

\begin{align*}
\left(\alpha_{t+1}-1\right)\gamma_{t+1}\left(f(y_{t+1})-f(x^{*})\right) & \leq\left(\alpha_{t}-1\right)\gamma_{t}\left(f(y_{t})-f(x^{*})\right)\\
 & +\frac{1}{2}\left\Vert z_{t}-x^{*}\right\Vert _{D_{t}}^{2}-\frac{1}{2}\left\Vert z_{t+1}-x^{*}\right\Vert _{D_{t}}^{2}-\frac{1}{2}\left\Vert z_{t}-z_{t+1}\right\Vert _{D_{t}}^{2}\\
 & +\frac{1}{2}\frac{\gamma_{t}}{\beta_{t}}\left\Vert z_{t+1}-z_{t}\right\Vert _{\sm}^{2}\ .
\end{align*}
Summing up and telescoping we obtain that 
\begin{align*}
 & \left(\alpha_{T}-1\right)\gamma_{T}\left(f(y_{T})-f(x^{*})\right)-\left(\alpha_{0}-1\right)\gamma_{0}\left(f(y_{0})-f(x^{*})\right)\\
 & \leq\sum_{t=0}^{T-1}\left(\frac{1}{2}\left\Vert z_{t}-x^{*}\right\Vert _{D_{t}}^{2}-\frac{1}{2}\left\Vert z_{t+1}-x^{*}\right\Vert _{D_{t}}^{2}-\frac{1}{2}\left\Vert z_{t}-z_{t+1}\right\Vert _{D_{t}}^{2}+\frac{1}{2}\frac{\gamma_{t}}{\alpha_{t}}\left\Vert z_{t}-z_{t+1}\right\Vert _{\sm}^{2}\right)\\
 & =\left(\frac{1}{2}\left\Vert z_{0}-x^{*}\right\Vert _{D_{0}}^{2}-\frac{1}{2}\left\Vert z_{T}-x^{*}\right\Vert _{D_{T-1}}^{2}+\sum_{t=1}^{T-1}\frac{1}{2}\left\Vert z_{t}-x^{*}\right\Vert _{D_{t}-D_{t-1}}^{2}\right)\\
 & +\sum_{t=0}^{T-1}\left(\frac{1}{2}\frac{\gamma_{t}}{\alpha_{t}}\left\Vert z_{t}-z_{t+1}\right\Vert _{\sm}^{2}-\frac{1}{2}\left\Vert z_{t}-z_{t+1}\right\Vert _{D_{t}}^{2}\right)\\
 & \leq\frac{1}{2}\left(R_{\infty}^{2}\tr(D_{0})+\sum_{t=1}^{T-1}R_{\infty}^{2}\left(\tr(D_{t})-\tr(D_{t-1})\right)\right)+\sum_{t=0}^{T-1}\left(\frac{1}{2}\frac{\gamma_{t}}{\alpha_{t}}\left\Vert z_{t}-z_{t+1}\right\Vert _{\sm}^{2}-\frac{1}{2}\left\Vert z_{t}-z_{t+1}\right\Vert _{D_{t}}^{2}\right)\\
 & \leq\frac{1}{2}R_{\infty}^{2}\tr(D_{T-1})+\sum_{t=0}^{T-1}\left(\frac{1}{2}\frac{\gamma_{t}}{\alpha_{t}}\left\Vert z_{t}-z_{t+1}\right\Vert _{\sm}^{2}-\frac{1}{2}\left\Vert z_{t}-z_{t+1}\right\Vert _{D_{t}}^{2}\right)\ ,
\end{align*}
which is what we needed. 
\end{proof}
We analyze the upper bound provided by Lemma \ref{lem:acsa-conditional-telescope}
using an analogous argument to that we used in Section \ref{sec:analysis-adagrad+-smooth}.
As before, we split the upper bound into two terms and analyze each
of the terms analogously to Lemmas \ref{lem:error1} and \ref{lem:error2}.
To this extent we write 
\begin{align*}
 & \frac{1}{2}R_{\infty}^{2}\tr(D_{T-1})+\sum_{t=0}^{T-1}\left(\frac{1}{2}\frac{\gamma_{t}}{\alpha_{t}}\left\Vert z_{t}-z_{t+1}\right\Vert _{\sm}^{2}-\frac{1}{2}\left\Vert z_{t}-z_{t+1}\right\Vert _{D_{t}}^{2}\right)\\
 & =\underbrace{\sum_{t=0}^{T-1}\left(\frac{1}{2}\frac{\gamma_{t}}{\alpha_{t}}\left\Vert z_{t}-z_{t+1}\right\Vert _{\sm}^{2}-\left(\frac{1}{2}-\frac{1}{2\sqrt{2}}\right)\left\Vert z_{t}-z_{t+1}\right\Vert _{D_{t}}^{2}\right)}_{(\star)}+\underbrace{\left(\frac{1}{2}R_{\infty}^{2}\tr(D_{T-1})-\frac{1}{2\sqrt{2}}\sum_{t=0}^{T-1}\left\Vert z_{t}-z_{t+1}\right\Vert _{D_{t}}^{2}\right)}_{(\star\star)}\ .
\end{align*}
We now proceed to bound each of these terms. 
\begin{lem}
\label{lem:error1-acc-1}If $\gamma_{t}\leq\alpha_{t}$ for all $t$,
then we have 
\[
(\star)\leq O\left(R_{\infty}^{2}\sum_{i=1}^{d}\beta_{i}\ln\left(2\beta_{i}\right)\right)\ .
\]
\end{lem}

\begin{proof}
Let $c=\frac{1}{2}-\frac{1}{2\sqrt{2}}$. Note that, for each coordinate
$i$, $D_{t,i}$ is increasing with $t$. For each coordinate $i\in[d]$
, we let $\tilde{T}_{i}$ be the last iteration $t$ for which $D_{t,i}\leq\frac{1}{c}\beta_{i}$;
if there is no such iteration, we let $\tilde{T}_{i}=-1$. We have
\begin{align*}
(\star) & =\sum_{t=0}^{T-1}\left\Vert z_{t}-z_{t+1}\right\Vert _{\sm}^{2}-c\sum_{t=0}^{T-1}\left\Vert z_{t}-z_{t+1}\right\Vert _{D_{t}}^{2}\\
 & =\sum_{i=1}^{d}\sum_{t=0}^{T-1}\left(\beta_{i}\left(z_{t,i}-z_{t+1,i}\right)^{2}-cD_{t,i}\left(z_{t,i}-z_{t+1,i}\right)^{2}\right)\\
 & \leq\sum_{i=1}^{d}\sum_{t=0}^{\tilde{T}_{i}}\beta_{i}\left(z_{t,i}-z_{t+1,i}\right)^{2}\ .
\end{align*}
We bound the above sum by considering each coordinate separately.
We apply Lemma \ref{lem:inequalities} with $d_{t}^{2}=\left(z_{t,i}-z_{t+1,i}\right)^{2}$
and $R^{2}=R_{\infty}^{2}\geq d_{t}^{2}$. Using the third inequality
in the lemma, we obtain 
\begin{align*}
\sum_{t=0}^{\tilde{T}_{i}}\left(z_{t,i}-z_{t+1,i}\right)^{2} & \leq R_{\infty}^{2}+\sum_{t=0}^{\tilde{T}_{i}-1}\left(z_{t,i}-z_{t+1,i}\right)^{2}\\
 & \leq R_{\infty}^{2}+4R_{\infty}^{2}\ln\left(\frac{D_{\widetilde{T}_{i},i}}{D_{0,i}}\right)\\
 & \leq R_{\infty}^{2}+4R_{\infty}^{2}\ln\left(\frac{\beta_{i}}{c}\right)\ .
\end{align*}
Therefore 
\[
(\star)\leq O\left(R_{\infty}^{2}\sum_{i=1}^{d}\beta_{i}\ln\left(2\beta_{i}\right)\right)\ .
\]
\end{proof}
\begin{lem}
\label{lem:error2-acc-1}We have 
\[
(\star\star)\leq O(R_{\infty}^{2}d)\ .
\]
\end{lem}

\begin{proof}
We have 
\[
\sum_{t=0}^{T-1}\left\Vert z_{t}-z_{t+1}\right\Vert _{D_{t}}^{2}=\sum_{i=1}^{d}\sum_{t=0}^{T-1}D_{t,i}\left(z_{t,i}-z_{t+1,i}\right)^{2}\ .
\]
We apply Lemma \ref{lem:inequalities} with $d_{t}^{2}=\left(z_{t,i}-z_{t+1,i}\right)^{2}$
and $R^{2}=R_{\infty}^{2}$ and obtain 
\[
\sum_{t=0}^{T-1}D_{t,i}\left(z_{t,i}-z_{t+1,i}\right)^{2}\geq2R_{\infty}^{2}\left(D_{T,i}-D_{0,i}\right)\ .
\]
Therefore 
\[
\sum_{t=0}^{T-1}\left\Vert z_{t}-z_{t+1}\right\Vert _{D_{t}}^{2}\geq2R_{\infty}^{2}\left(\tr(D_{T})-\tr(D_{0})\right)
\]
and 
\begin{align*}
(\star\star) & =\frac{1}{2}R_{\infty}^{2}\tr(D_{T-1})-\frac{1}{2\sqrt{2}}\sum_{t=0}^{T-1}\left\Vert z_{t}-z_{t+1}\right\Vert _{D_{t}}^{2}\\
 & \leq\frac{1}{\sqrt{2}}R_{\infty}^{2}\tr(D_{0})=\frac{1}{\sqrt{2}}R_{\infty}^{2}d\ .
\end{align*}
\end{proof}
Combining Lemmas \ref{lem:acsa-conditional-telescope}, \ref{lem:error1-acc-1}
and \ref{lem:error2-acc-1} we see that as long as for all $t\geq0$,
$0<\left(\alpha_{t+1}-1\right)\gamma_{t+1}\leq\alpha_{t}\gamma_{t}$
and $\gamma_{t}\leq\alpha_{t}$, we have that 
\begin{align*}
 & \left(\alpha_{T}-1\right)\gamma_{T}\left(f(y_{T})-f(x^{*})\right)-\left(\alpha_{0}-1\right)\gamma_{0}\left(f(y_{0})-f(x^{*})\right)\\
 & \leq O\left(R_{\infty}^{2}\sum_{i=1}^{d}\beta_{i}\ln\left(2\beta_{i}\right)\right)+O(R_{\infty}^{2}d)=O\left(R_{\infty}^{2}\sum_{i=1}^{d}\beta_{i}\ln\left(2\beta_{i}\right)\right)\ .
\end{align*}
Picking $\gamma_{t}=\alpha_{t}=\frac{t}{3}+1$ we easily verify that
the the required conditions hold, and thus 
\[
f(y_{T})-f(x^{*})=O\left(\frac{R_{\infty}^{2}\sum_{i=1}^{d}\beta_{i}\ln\left(2\beta_{i}\right)}{T^{2}}\right)\:,
\]
which completes our convergence analysis. 
\begin{rem}
\label{rem:step-sizes}We can improve the convergence rate by a constant
factor by optimizing the choice of $\alpha_{t}$. More specifically,
we can force the inequality $\left(\alpha_{t+1}-1\right)\gamma_{t+1}\leq\alpha_{t}\gamma_{t}$
to be tight by setting $\gamma_{t}=\alpha_{t}$ for all $t$ and setting
$\alpha_{0}=1$, $\left(\alpha_{t+1}-1\right)\alpha_{t+1}=\alpha_{t}^{2}$
for $t\geq0$. Equivalently $\alpha_{t+1}=\frac{1+\sqrt{1+4\alpha_{t}^{2}}}{2}$,
which recovers the choice of step sizes from previous works on accelerated
methods (see \citep{BansalGupta17}, Section 5.2.1). 
\end{rem}

\section*{Acknowledgments}

AE was supported in part by NSF CAREER grant CCF-1750333, NSF grant
CCF-1718342, and NSF grant III-1908510. HN was supported in part by
NSF CAREER grant CCF-1750716 and NSF grant CCF-1909314. AV was supported
in part by NSF grant CCF-1718342. We thank Aleksander M\c{a}dry for
kindly providing us with computing resources to perform the experimental
component of this work.

\bibliographystyle{plainnat}
\bibliography{adagrad}

\begin{thebibliography}{34}
\providecommand{\natexlab}[1]{#1}
\providecommand{\url}[1]{\texttt{#1}}
\expandafter\ifx\csname urlstyle\endcsname\relax
  \providecommand{\doi}[1]{doi: #1}\else
  \providecommand{\doi}{doi: \begingroup \urlstyle{rm}\Url}\fi

\bibitem[Allen-Zhu and Orecchia(2017)]{allen2017linear}
Zeyuan Allen-Zhu and Lorenzo Orecchia.
\newblock Linear coupling: An ultimate unification of gradient and mirror
  descent.
\newblock In \emph{Innovations in Theoretical Computer Science Conference
  {(ITCS)}}, 2017.

\bibitem[Bach and Levy(2019)]{BachL19}
Francis Bach and Kfir~Y. Levy.
\newblock A universal algorithm for variational inequalities adaptive to
  smoothness and noise.
\newblock In \emph{Conference on Learning Theory {(COLT)}}, pages 164--194,
  2019.

\bibitem[Bansal and Gupta(2019)]{BansalGupta17}
Nikhil Bansal and Anupam Gupta.
\newblock Potential-function proofs for gradient methods.
\newblock \emph{Theory of Computing}, 15\penalty0 (4):\penalty0 1--32, 2019.

\bibitem[Chen et~al.(2018)Chen, Liu, Sun, and Hong]{chen2018convergence}
Xiangyi Chen, Sijia Liu, Ruoyu Sun, and Mingyi Hong.
\newblock On the convergence of a class of adam-type algorithms for non-convex
  optimization.
\newblock In \emph{International Conference on Learning Representations
  {(ICLR)}}, 2018.

\bibitem[Cohen et~al.(2018)Cohen, Diakonikolas, and Orecchia]{CohenDO18}
Michael Cohen, Jelena Diakonikolas, and Lorenzo Orecchia.
\newblock On acceleration with noise-corrupted gradients.
\newblock In \emph{International Conference on Machine Learning {(ICML)}},
  pages 1018--1027, 2018.

\bibitem[Cutkosky(2019)]{Cutkosky19}
Ashok Cutkosky.
\newblock Anytime online-to-batch, optimism and acceleration.
\newblock In \emph{International Conference on Machine Learning {(ICML)}},
  pages 1446--1454, 2019.

\bibitem[Cutkosky(2020)]{cutkosky2020parameter}
Ashok Cutkosky.
\newblock Parameter-free, dynamic, and strongly-adaptive online learning.
\newblock In \emph{International Conference on Machine Learning {(ICML)}},
  2020.

\bibitem[Cutkosky and Sarlos(2019)]{cutkosky2019matrix}
Ashok Cutkosky and Tamas Sarlos.
\newblock Matrix-free preconditioning in online learning.
\newblock In \emph{International Conference on Machine Learning (ICML)}, pages
  1455--1464, 2019.

\bibitem[D{\'e}fossez et~al.(2020)D{\'e}fossez, Bottou, Bach, and
  Usunier]{defossez2020convergence}
Alexandre D{\'e}fossez, L{\'e}on Bottou, Francis Bach, and Nicolas Usunier.
\newblock On the convergence of adam and adagrad.
\newblock \emph{arXiv preprint arXiv:2003.02395}, 2020.

\bibitem[Diakonikolas and Orecchia(2018)]{DiakonikolasO18}
Jelena Diakonikolas and Lorenzo Orecchia.
\newblock Accelerated extra-gradient descent: {A} novel accelerated first-order
  method.
\newblock In \emph{Innovations in Theoretical Computer Science Conference
  {(ITCS)}}, volume~94, pages 23:1--23:19, 2018.

\bibitem[Dozat(2016)]{dozat2016incorporating}
Timothy Dozat.
\newblock Incorporating nesterov momentum into adam.
\newblock In \emph{International Conference on Learning Representations
  {(ICLR)} Workshop}, 2016.

\bibitem[Duchi et~al.(2011)Duchi, Hazan, and Singer]{duchi2011adaptive}
John Duchi, Elad Hazan, and Yoram Singer.
\newblock Adaptive subgradient methods for online learning and stochastic
  optimization.
\newblock \emph{Journal of machine learning research}, 12\penalty0 (7), 2011.

\bibitem[Gupta et~al.(2018)Gupta, Koren, and Singer]{gupta2018shampoo}
Vineet Gupta, Tomer Koren, and Yoram Singer.
\newblock Shampoo: Preconditioned stochastic tensor optimization.
\newblock In \emph{International Conference on Machine Learning {(ICML)}},
  pages 1842--1850, 2018.

\bibitem[He et~al.(2016)He, Zhang, Ren, and Sun]{he2016deep}
Kaiming He, Xiangyu Zhang, Shaoqing Ren, and Jian Sun.
\newblock Deep residual learning for image recognition.
\newblock In \emph{Computer Vision and Pattern Recognition {(CVPR)}}, pages
  770--778, 2016.

\bibitem[Joulani et~al.(2020)Joulani, Raj, Gy{\"o}rgy, and
  Szepesv{\'a}ri]{joulanisimpler}
Pooria Joulani, Anant Raj, Andr{\'a}s Gy{\"o}rgy, and Csaba Szepesv{\'a}ri.
\newblock A simpler approach to accelerated stochastic optimization: Iterative
  averaging meets optimism.
\newblock In \emph{International Conference on Machine Learning {(ICML)}},
  2020.

\bibitem[Kavis et~al.(2019)Kavis, Levy, Bach, and Cevher]{KavisLBC19}
Ali Kavis, Kfir~Y. Levy, Francis Bach, and Volkan Cevher.
\newblock Unixgrad: {A} universal, adaptive algorithm with optimal guarantees
  for constrained optimization.
\newblock In \emph{Neural Information Processing Systems {(NeurIPS)}}, pages
  6257--6266, 2019.

\bibitem[Kingma and Ba(2014)]{kingma2014adam}
Diederik~P Kingma and Jimmy Ba.
\newblock Adam: A method for stochastic optimization.
\newblock \emph{arXiv preprint arXiv:1412.6980}, 2014.

\bibitem[Lan(2012)]{acsa2012}
Guanghui Lan.
\newblock An optimal method for stochastic composite optimization.
\newblock \emph{Mathematical Programming}, 133\penalty0 (1-2):\penalty0
  365--397, 2012.

\bibitem[Leclerc and Madry(2020)]{leclerc2020two}
Guillaume Leclerc and Aleksander Madry.
\newblock The two regimes of deep network training.
\newblock \emph{arXiv preprint arXiv:2002.10376}, 2020.

\bibitem[Levy(2017)]{levy2017online}
Kfir Levy.
\newblock Online to offline conversions, universality and adaptive minibatch
  sizes.
\newblock In \emph{Neural Information Processing Systems {(NeurIPS)}}, pages
  1613--1622, 2017.

\bibitem[Levy et~al.(2018{\natexlab{a}})Levy, Yurtsever, and
  Cevher]{levy2018online}
Kfir~Y Levy, Alp Yurtsever, and Volkan Cevher.
\newblock Online adaptive methods, universality and acceleration.
\newblock In \emph{Neural Information Processing Systems {(NeurIPS)}}, pages
  6500--6509, 2018{\natexlab{a}}.

\bibitem[Levy et~al.(2018{\natexlab{b}})Levy, Yurtsever, and Cevher]{LevyYC18}
Kfir~Yehuda Levy, Alp Yurtsever, and Volkan Cevher.
\newblock Online adaptive methods, universality and acceleration.
\newblock In \emph{Neural Information Processing Systems {(NeurIPS)}}, pages
  6501--6510, 2018{\natexlab{b}}.

\bibitem[Li and Orabona(2019)]{li2019convergence}
Xiaoyu Li and Francesco Orabona.
\newblock On the convergence of stochastic gradient descent with adaptive
  stepsizes.
\newblock In \emph{Artificial Intelligence and Statistics {(AISTATS)}}, pages
  983--992, 2019.

\bibitem[McMahan and Streeter(2010)]{McMahanS10}
H.~Brendan McMahan and Matthew~J. Streeter.
\newblock Adaptive bound optimization for online convex optimization.
\newblock In \emph{Conference on Learning Theory {(COLT)}}, pages 244--256,
  2010.

\bibitem[Mohri and Yang(2016)]{MohriYang16}
Mehryar Mohri and Scott Yang.
\newblock Accelerating online convex optimization via adaptive prediction.
\newblock In \emph{Artificial Intelligence and Statistics {(AISTATS)}}, pages
  848--856, 2016.

\bibitem[Nemirovski(2004)]{nemirovski2004prox}
Arkadi Nemirovski.
\newblock Prox-method with rate of convergence o (1/t) for variational
  inequalities with lipschitz continuous monotone operators and smooth
  convex-concave saddle point problems.
\newblock \emph{SIAM Journal on Optimization}, 15\penalty0 (1):\penalty0
  229--251, 2004.

\bibitem[Nesterov(2015)]{nesterov2015universal}
Yu~Nesterov.
\newblock Universal gradient methods for convex optimization problems.
\newblock \emph{Mathematical Programming}, 152\penalty0 (1-2):\penalty0
  381--404, 2015.

\bibitem[Nesterov(1998)]{nesterov1998introductory}
Yurii Nesterov.
\newblock Introductory lectures on convex programming volume i: Basic course.
\newblock \emph{Lecture notes}, 3\penalty0 (4):\penalty0 5, 1998.

\bibitem[Nesterov(2013)]{nesterov2013introductory}
Yurii Nesterov.
\newblock \emph{Introductory lectures on convex optimization: A basic course},
  volume~87.
\newblock Springer Science \& Business Media, 2013.

\bibitem[Reddi et~al.(2018)Reddi, Kale, and Kumar]{reddi2018convergence}
Sashank~J Reddi, Satyen Kale, and Sanjiv Kumar.
\newblock On the convergence of adam and beyond.
\newblock In \emph{International Conference on Learning Representations
  {(ICLR)}}, 2018.

\bibitem[Tieleman and Hinton(2012)]{tieleman2012lecture}
Tijmen Tieleman and Geoffrey Hinton.
\newblock Lecture 6.5-rmsprop: Divide the gradient by a running average of its
  recent magnitude.
\newblock \emph{COURSERA: Neural networks for machine learning}, 4\penalty0
  (2):\penalty0 26--31, 2012.

\bibitem[Ward et~al.(2019)Ward, Wu, and Bottou]{ward2019adagrad}
Rachel Ward, Xiaoxia Wu, and Leon Bottou.
\newblock Adagrad stepsizes: sharp convergence over nonconvex landscapes.
\newblock In \emph{International Conference on Machine Learning {(ICML)}},
  pages 6677--6686, 2019.

\bibitem[Zou et~al.(2018)Zou, Shen, Jie, Sun, and Liu]{zou2018weighted}
Fangyu Zou, Li~Shen, Zequn Jie, Ju~Sun, and Wei Liu.
\newblock Weighted adagrad with unified momentum.
\newblock \emph{arXiv preprint arXiv:1808.03408}, 2018.

\bibitem[Zou et~al.(2019)Zou, Shen, Jie, Zhang, and Liu]{zou2019sufficient}
Fangyu Zou, Li~Shen, Zequn Jie, Weizhong Zhang, and Wei Liu.
\newblock A sufficient condition for convergences of adam and rmsprop.
\newblock In \emph{Computer Vision and Pattern Recognition {(CVPR)}}, pages
  11127--11135, 2019.

\end{thebibliography}

\newpage{}

\appendix

\section{Scalar Schemes and Comparison to Previous Work}

\label{sec:scalar-schemes}

\begin{figure}
\noindent %
\noindent\fbox{\begin{minipage}[t]{1\columnwidth - 2\fboxsep - 2\fboxrule}%
Let $x_{0}\in\dom$, $D_{0}=1$, $R_{2}\geq\max_{x,y\in\dom}\left\Vert x-y\right\Vert _{2}$.

For $t=0,\dots,T-1$, update:

\begin{align*}
x_{t+1}= & \arg\min_{x\in\dom}\left\{ \left\langle \nabla f(x_{t}),x\right\rangle +\frac{1}{2}D_{t}\left\Vert x-x_{t}\right\Vert _{2}^{2}\right\} \ ,\\
D_{t+1}^{2} & =D_{t}^{2}\left(1+\frac{\left\Vert x_{t+1}-x_{t}\right\Vert _{2}^{2}}{R_{2}^{2}}\right)\ .
\end{align*}

Return $\overline{x}_{T}=\frac{1}{T}\sum_{t=1}^{T}x_{t}$.%
\end{minipage}}

\caption{Scalar version of the $\protect\adagradplus$ algorithm.}
\label{alg:adagrad-scalar}
\end{figure}
\begin{figure}
\noindent %
\noindent\fbox{\begin{minipage}[t]{1\columnwidth - 2\fboxsep - 2\fboxrule}%
Let $D_{0}=I$, $z_{0}\in\dom$, $\alpha_{t}=\gamma_{t}=1+\frac{t}{3}$,
$R_{\infty}^{2}\geq\max_{x,y\in\dom}\left\Vert x-y\right\Vert _{\infty}^{2}$.

For $t=0,\dots,T-1$, update:

\begin{align*}
x_{t} & =\left(1-\alpha_{t}^{-1}\right)y_{t}+\alpha_{t}^{-1}z_{t}\ ,\\
z_{t+1} & =\arg\min_{u\in\dom}\left\{ \gamma_{t}\left\langle \nabla f(x_{t}),u\right\rangle +\frac{1}{2}\left\Vert u-z_{t}\right\Vert _{D_{t}}^{2}\right\} \ ,\\
y_{t+1} & =\left(1-\alpha_{t}^{-1}\right)y_{t}+\alpha_{t}^{-1}z_{t+1}\ ,\\
D_{t+1}^{2} & =D_{t}^{2}\left(1+\frac{\left\Vert z_{t+1}-z_{t}\right\Vert ^{2}}{R_{\infty}^{2}}\right)\ .
\end{align*}

Return $y_{T}$.%
\end{minipage}}

\caption{Scalar version of the $\protect\adaacsa$ algorithm.}
\label{alg:acc-scalar-1}
\end{figure}
\begin{figure}
\noindent %
\noindent\fbox{\begin{minipage}[t]{1\columnwidth - 2\fboxsep - 2\fboxrule}%
Let $D_{1}=1$, $z_{0}\in\dom$, $a_{t}=t$, $A_{t}=\sum_{i=1}^{t}a_{i}=\frac{t(t+1)}{2}$,
$R_{2}^{2}\geq\max_{x,y\in\dom}\left\Vert x-y\right\Vert _{2}^{2}$.

For $t=1,\dots,T$, update:

\begin{align*}
x_{t} & =\frac{A_{t-1}}{A_{t}}y_{t-1}+\frac{a_{t}}{A_{t}}z_{t-1}\ ,\\
z_{t} & =\arg\min_{u\in\dom}\left(\sum_{i=1}^{t}\left\langle a_{i}\nabla f(x_{i}),u\right\rangle +\frac{1}{2}D_{t}\left\Vert u-z_{0}\right\Vert _{2}^{2}\right)\ ,\\
y_{t} & =\frac{A_{t-1}}{A_{t}}y_{t-1}+\frac{a_{t}}{A_{t}}z_{t}\ ,\\
D_{t+1}^{2} & =D_{t}^{2}\left(1+\frac{\left\Vert z_{t}-z_{t-1}\right\Vert ^{2}}{R_{2}^{2}}\right)\ .
\end{align*}

Return $y_{T}$.%
\end{minipage}}

\caption{Scalar version of the $\protect\adaagdplus$ algorithm.}
\label{alg:acc-scalar}
\end{figure}

In this section, we include for completeness the scalar version of
our algorithms that use a scalar step size $D_{t}$. We compare these
algorithms with previous works, which are all scalar schemes.

\paragraph{Scalar $\protect\adagradplus$ Algorithm.}

We describe the scalar version of $\adagradplus$ in Figure \ref{alg:adagrad-scalar}.

\begin{sloppypar}If $f$ is $\beta$-smooth with respect to the $\ell_{2}$-norm,
the scalar algorithm converges at the rate $O\left(R_{2}^{2}\beta\ln\left(2\beta\right)/T\right)$.
If $f$ is non-smooth, the scalar algorithm converges at the rate
$O\left(R_{2}G\sqrt{\ln\left(TG/R_{2}\right)}/\sqrt{T}+R_{2}^{2}/T\right)$,
where $G\geq\max_{x\in\dom}\left\Vert \nabla f(x)\right\Vert _{2}$.
In the stochastic setting, we obtain a rate of $O\left(R_{2}^{2}\beta\ln\left(2\beta\right)/T+R_{2}\sigma\sqrt{\ln\left(T\sigma/R_{2}\right)}/\sqrt{T}\right)$
for smooth functions and $O\left(R_{2}G\sqrt{\ln\left(TG/R_{2}\right)}/\sqrt{T}+R_{2}\sigma\sqrt{\ln\left(T\sigma/R_{2}\right)}/\sqrt{T}+R_{2}^{2}/T\right)$
for non-smooth functions.\end{sloppypar}

The convergence analysis for the scalar case follows readily from
our vector analysis, and we omit it. The scalar update and the analysis
readily extend to general Bregman distances, similarly to previous
work \citep{BachL19}.

\paragraph{Comparison With Previous Work.}

Bach and Levy \citeyearpar{BachL19} propose an adaptive scalar scheme
that is based on the mirror prox algorithm. In contrast, our algorithm
is an adaptive version of gradient descent. The algorithm of \citet{BachL19}
is universal and converges at the same rate (up to constants) as our
scalar algorithm in both the smooth and non-smooth setting. The main
differences between the two algorithms are the following. The algorithm
of \citet{BachL19} uses two gradient computations and two projections
per iteration, whereas our algorithm uses only one gradient computation
and one projection per iteration. Both schemes use iterate movement
to update the step size, but the scheme of \citet{BachL19} relies
on having an estimate for $G$ (in addition to $R_{2}$) in order
to set the step size. In the stochastic setting, \citet{BachL19}
also assumes that the $\ell_{2}$-norm of the stochastic gradients
is bounded with probability one and the step size relies on having
an estimate on this bound in order to set the step size.

Bach and Levy \citeyearpar{BachL19} leave as an open question to
generalize their algorithm to the vector setting. By building on their
work and the techniques introduced in this paper, we resolve this
open question in Section \ref{sec:mirror-prox}.

\paragraph{Scalar $\protect\adaacsa$ and $\protect\adaagdplus$ Algorithms.}

We describe the scalar versions of $\adaacsa$ and $\adaagdplus$
in Figures \ref{alg:acc-scalar-1} and \ref{alg:acc-scalar}.

\begin{sloppypar}If $f$ is $\beta$-smooth with respect to the $\ell_{2}$-norm,
both scalar algorithms converge at the rate $O\left(R_{2}^{2}\beta\ln\left(2\beta\right)/T^{2}\right)$.
If $f$ is non-smooth, the scalar algorithms converge at the rate
$O\left(R_{2}G\sqrt{\ln\left(TG/R_{2}\right)}/\sqrt{T}+R_{2}^{2}/T^{2}\right)$,
where $G\geq\max_{x\in\dom}\left\Vert \nabla f(x)\right\Vert _{2}$.
In the stochastic setting, we obtain a rate of $O\left(R_{2}^{2}\beta\ln\left(2\beta\right)/T^{2}+R_{2}\sigma\sqrt{\ln\left(T\sigma/R_{2}\right)}/\sqrt{T}\right)$
for smooth functions and $O\left(R_{2}G\sqrt{\ln\left(TG/R_{2}\right)}/\sqrt{T}+R_{2}\sigma\sqrt{\ln\left(T\sigma/R_{2}\right)}/\sqrt{T}+R_{2}^{2}/T^{2}\right)$
for non-smooth functions.\end{sloppypar}

The convergence analysis for the scalar cases follow readily from
our vector analyses, and we omit them. The scalar updates and the
analyses readily extend to general Bregman distances, similarly to
previous work \citep{KavisLBC19}.

\paragraph{Comparison With Previous Work.}

Kavis \emph{et al. }\citeyearpar{KavisLBC19} propose an accelerated
scalar scheme that builds on the accelerated mirror prox algorithm
of \citet{DiakonikolasO18}. The step sizes employed by their scheme
is very different from ours: whereas we use the iterate movement,
their scheme uses the norm of gradient differences. In the smooth
setting, the convergence guarantee of the algorithm of \citet{KavisLBC19}
is better than our scalar convergence by a $\log\beta$ factor. Their
algorithm uses two gradient computations and two projections per iteration,
whereas our algorithm uses only one gradient computation and one projection
per iteration. \citet{KavisLBC19}\emph{ }leave as an open question
to obtain an accelerated vector scheme, and we resolve this open question
in this paper.

The works \citep{Cutkosky19,LevyYC18} give accelerated scalar schemes
for unconstrained ($\dom=\R^{d}$) smooth optimization that build
on the linear coupling interpretation \citep{allen2017linear} of
Nesterov's accelerated gradient descent algorithm \citep{nesterov2013introductory}.
The convergence guarantees for smooth functions provided in these
works is the same as the convergence of our scalar algorithm. These
works leave as an open question to obtain accelerated schemes for
the constrained setting.

\section{Analysis of $\protect\adagradplus$ for Non-Smooth Functions}

\label{sec:analysis-adagrad+-nonsmooth}

Our analysis builds on the standard analysis of gradient descent (in
particular, the elegant potential-function proof of \citet{BansalGupta17})
and $\adagrad$ \citep{duchi2011adaptive,McMahanS10}, as well as
ideas from \citep{BachL19}.

Throughout this section, the norm $\left\Vert \cdot\right\Vert $
without a subscript denotes the $\ell_{2}$-norm. We analyze the potential
\[
\Phi_{t}:=\frac{1}{2}\left\Vert x_{t}-x^{*}\right\Vert _{D_{t}}^{2}\ .
\]
We analyze the difference in potential: 
\begin{align}
\Phi_{t+1}-\Phi_{t} & =\frac{1}{2}\left\Vert x_{t+1}-x^{*}\right\Vert _{D_{t+1}}^{2}-\frac{1}{2}\left\Vert x_{t}-x^{*}\right\Vert _{D_{t}}^{2}\nonumber \\
 & =\frac{1}{2}\left\Vert x_{t+1}-x^{*}\right\Vert _{D_{t+1}-D_{t}}^{2}+\frac{1}{2}\left\Vert x_{t+1}-x^{*}\right\Vert _{D_{t}}^{2}-\frac{1}{2}\left\Vert x_{t}-x^{*}\right\Vert _{D_{t}}^{2}\ .\label{eq:potential}
\end{align}
Using the first-order optimality condition for $x_{t+1}$ and straightforward
algebraic manipulations, we next show the following inequality: 
\begin{equation}
\frac{1}{2}\left\Vert x_{t+1}-x^{*}\right\Vert _{D_{t}}^{2}-\frac{1}{2}\left\Vert x_{t}-x^{*}\right\Vert _{D_{t}}^{2}+\left\langle \nabla f(x_{t}),x_{t}-x^{*}\right\rangle \leq\left\langle \nabla f(x_{t}),x_{t}-x_{t+1}\right\rangle -\frac{1}{2}\left\Vert x_{t+1}-x_{t}\right\Vert _{D_{t}}^{2}\ .\label{eq:p1}
\end{equation}
We recall the definition of $x_{t+1}$: 
\[
x_{t+1}=\arg\min_{x\in\dom}\left\{ \left\langle \nabla f(x_{t}),x\right\rangle +\frac{1}{2}\left\Vert x-x_{t}\right\Vert _{D_{t}}^{2}\right\} \ .
\]
By the first-order optimality condition for $x_{t+1}$, we have 
\[
\left\langle \nabla f(x_{t})+D_{t}\left(x_{t+1}-x_{t}\right),x^{*}-x_{t+1}\right\rangle \geq0\ .
\]
Rearranging, 
\[
\left\langle D_{t}\left(x_{t+1}-x_{t}\right),x_{t+1}-x^{*}\right\rangle \leq\left\langle \nabla f(x_{t}),x^{*}-x_{t+1}\right\rangle \ .
\]
Using the above inequality, we obtain

\begin{align*}
\frac{1}{2}\left\Vert x_{t+1}-x^{*}\right\Vert _{D_{t}}^{2}-\frac{1}{2}\left\Vert x_{t}-x^{*}\right\Vert _{D_{t}}^{2} & =\frac{1}{2}\left\Vert x_{t+1}-x_{t}+x_{t}-x^{*}\right\Vert _{D_{t}}^{2}-\frac{1}{2}\left\Vert x_{t}-x^{*}\right\Vert _{D_{t}}^{2}\\
 & =\frac{1}{2}\left\Vert x_{t+1}-x_{t}\right\Vert _{D_{t}}^{2}+\left\langle D_{t}\left(x_{t+1}-x_{t}\right),x_{t}-x^{*}\right\rangle \\
 & =\frac{1}{2}\left\Vert x_{t+1}-x_{t}\right\Vert _{D_{t}}^{2}+\left\langle D_{t}\left(x_{t+1}-x_{t}\right),x_{t}-x_{t+1}+x_{t+1}-x^{*}\right\rangle \\
 & =-\frac{1}{2}\left\Vert x_{t+1}-x_{t}\right\Vert _{D_{t}}^{2}+\left\langle D_{t}\left(x_{t+1}-x_{t}\right),x_{t+1}-x^{*}\right\rangle \\
 & \leq-\frac{1}{2}\left\Vert x_{t+1}-x_{t}\right\Vert _{D_{t}}^{2}+\left\langle \nabla f(x_{t}),x^{*}-x_{t+1}\right\rangle \\
 & =-\frac{1}{2}\left\Vert x_{t+1}-x_{t}\right\Vert _{D_{t}}^{2}+\left\langle \nabla f(x_{t}),x^{*}-x_{t}+x_{t}-x_{t+1}\right\rangle \\
 & =-\frac{1}{2}\left\Vert x_{t+1}-x_{t}\right\Vert _{D_{t}}^{2}+\left\langle \nabla f(x_{t}),x^{*}-x_{t}\right\rangle +\left\langle \nabla f(x_{t}),x_{t}-x_{t+1}\right\rangle \ .
\end{align*}
By rearranging the above inequality, we obtain \eqref{eq:p1}.

Next, we use Cauchy-Schwarz to bound 
\[
\left\langle \nabla f(x_{t}),x_{t}-x_{t+1}\right\rangle \leq\left\Vert \nabla f(x_{t})\right\Vert \left\Vert x_{t}-x_{t+1}\right\Vert \leq G\left\Vert x_{t}-x_{t+1}\right\Vert \ .
\]
We use convexity to bound 
\[
f(x_{t})-f(x^{*})\leq\left\langle \nabla f(x_{t}),x_{t}-x^{*}\right\rangle \ .
\]
Plugging in the two inequalities into \eqref{eq:p1}, we obtain

\[
\frac{1}{2}\left\Vert x_{t+1}-x^{*}\right\Vert _{D_{t}}^{2}-\frac{1}{2}\left\Vert x_{t}-x^{*}\right\Vert _{D_{t}}^{2}+f(x_{t})-f(x^{*})\leq G\left\Vert x_{t}-x_{t+1}\right\Vert -\frac{1}{2}\left\Vert x_{t+1}-x_{t}\right\Vert _{D_{t}}^{2}\ .
\]
Plugging in the above inequality into \eqref{eq:potential}, we obtain
\[
\Phi_{t+1}-\Phi_{t}+f(x_{t})-f(x^{*})\leq\frac{1}{2}\left\Vert x_{t+1}-x^{*}\right\Vert _{D_{t+1}-D_{t}}^{2}+G\left\Vert x_{t}-x_{t+1}\right\Vert -\frac{1}{2}\left\Vert x_{t+1}-x_{t}\right\Vert _{D_{t}}^{2}\ .
\]
Summing up over all iterations, 
\begin{align}
 & \Phi_{T}-\Phi_{0}+\sum_{t=0}^{T-1}\left(f(x_{t})-f(x^{*})\right)\nonumber \\
 & \leq\underbrace{\sum_{t=0}^{T-1}\frac{1}{2}\left\Vert x_{t+1}-x^{*}\right\Vert _{D_{t+1}-D_{t}}^{2}}_{(\star)}+\underbrace{\sum_{t=0}^{T-1}G\left\Vert x_{t}-x_{t+1}\right\Vert }_{(\star\star)}-\underbrace{\sum_{t=0}^{T-1}\frac{1}{2}\left\Vert x_{t+1}-x_{t}\right\Vert _{D_{t}}^{2}}_{(\star\star\star)}\ .\label{eq:p2}
\end{align}
We bound $(\star)$ as before:

\begin{align*}
(\star) & =\sum_{t=0}^{T-1}\frac{1}{2}\left\Vert x_{t+1}-x^{*}\right\Vert _{D_{t+1}-D_{t}}^{2}\leq\sum_{t=0}^{T-1}\frac{1}{2}R_{\infty}^{2}\left(\tr(D_{t+1})-\tr(D_{t})\right)=\frac{1}{2}R_{\infty}^{2}\left(\tr(D_{T})-\tr(D_{0})\right)\ .
\end{align*}
To bound $(\star\star)$, similarly to \citep{BachL19}, we use concavity
of $\sqrt{z}$ to push the sum under the square root. We then bound
the total movement as in \ref{sec:analysis-adagrad+-smooth}. Since
$\sqrt{z}$ is concave, we have 
\begin{align*}
(\star\star) & =G\sum_{t=0}^{T-1}\sqrt{\left\Vert x_{t}-x_{t+1}\right\Vert ^{2}}\leq G\sqrt{T}\cdot\sqrt{\sum_{t=0}^{T-1}\left\Vert x_{t}-x_{t+1}\right\Vert ^{2}}\ .
\end{align*}
We now apply Lemma \ref{lem:inequalities} with $d_{t}^{2}=\left(x_{t+1,i}-x_{t,i}\right)^{2}$
and $R^{2}=R_{\infty}^{2}\geq d_{t}^{2}$, and obtain 
\[
\sum_{t=0}^{T-1}\left(x_{t+1,i}-x_{t,i}\right)^{2}\leq4R_{\infty}^{2}\ln\left(\frac{D_{T,i}}{D_{0,i}}\right)=4R_{\infty}^{2}\ln\left(D_{T,i}\right)\ .
\]
Thus 
\begin{align*}
(\star\star) & \leq G\sqrt{T}\sqrt{\sum_{i=1}^{d}4R_{\infty}^{2}\ln\left(D_{T,i}\right)}=2GR_{\infty}\sqrt{T}\sqrt{\sum_{i=1}^{d}\ln\left(D_{T,i}\right)}\ .
\end{align*}
Finally, we bound $(\star\star\star)$. For each coordinate separately,
we apply Lemma \ref{lem:inequalities} with $d_{t}^{2}=\left(x_{t+1,i}-x_{t,i}\right)^{2}$
and $R^{2}=R_{\infty}^{2}\geq d_{t}^{2}$, and obtain 
\begin{align*}
(\star\star\star) & =\frac{1}{2}\sum_{t=0}^{T-1}\left\Vert x_{t+1}-x_{t}\right\Vert _{D_{t}}^{2}=\frac{1}{2}\sum_{i=1}^{d}\sum_{t=0}^{T-1}D_{t,i}\left(x_{t+1,i}-x_{t,i}\right)^{2}\\
 & \geq R_{\infty}^{2}\sum_{i=1}^{d}\left(D_{T,i}-D_{0,i}\right)=R_{\infty}^{2}\left(\tr(D_{T})-\tr(D_{0})\right)\ .
\end{align*}
Plugging in the bounds on $(\star)$, $(\star\star)$, and $(\star\star\star)$
into (\ref{eq:p2}), we obtain 
\begin{align*}
\Phi_{T}-\Phi_{0}+\sum_{t=0}^{T-1}\left(f(x_{t})-f(x^{*})\right) & \leq2GR_{\infty}\sqrt{T}\sqrt{\sum_{i=1}^{d}\ln\left(D_{T,i}\right)}-\frac{1}{2}R_{\infty}^{2}\left(\tr(D_{T})-\tr(D_{0})\right)\\
 & =2GR_{\infty}\sqrt{T}\sqrt{\sum_{i=1}^{d}\ln\left(D_{T,i}\right)}-\frac{1}{2}R_{\infty}^{2}\left(\sum_{i=1}^{d}D_{T,i}\right)+\frac{1}{2}R_{\infty}^{2}d\\
 & =\frac{1}{2}R_{\infty}^{2}\left(\underbrace{4\frac{G\sqrt{T}}{R_{\infty}}\sqrt{\sum_{i=1}^{d}\ln\left(D_{T,i}\right)}-\sum_{i=1}^{d}D_{T,i}}_{(\diamond)}\right)+\frac{3}{2}R_{\infty}^{2}d\ .
\end{align*}
Let $a=4\frac{G\sqrt{T}}{R_{\infty}}$ and $z_{i}=D_{T,i}\geq1$ and
$z=(z_{1},\dots,z_{d})$. With this notation, we have $(\diamond)$$=a\sqrt{\sum_{i=1}^{d}\ln\left(z_{i}\right)}-\sum_{i=1}^{d}z_{i}=:\phi(z)$.
Note that $\phi(z)$ is concave over $z\geq1$. We can upper bound
$\max_{z\geq1}\phi(z)$ using a straightforward calculation, which
we encapsulate in the following lemma for future use. We defer the
proof to Section \ref{sec:lem-phiz}. 
\begin{lem}
\label{lem:phiz}Let $\phi\colon\R^{d}\to\R$, $\phi(z)=a\sqrt{\sum_{i=1}^{d}\ln\left(z_{i}\right)}-\sum_{i=1}^{d}z_{i}$,
where $a$ is a non-negative scalar. Let $z^{*}\in\arg\max_{z\geq1}\phi(z)$.
We have 
\[
\phi(z^{*})\leq\sqrt{d}a\sqrt{\ln a}\ .
\]
Thus we obtain 
\[
(\diamond)\leq\phi(z^{*})\leq\sqrt{d}a\sqrt{\ln a}=O\left(\sqrt{d}\frac{G\sqrt{T}}{R_{\infty}}\sqrt{\ln\left(\frac{GT}{R_{\infty}}\right)}\right)\ .
\]
\end{lem}

Plugging into the previous inequality, we obtain 
\begin{align*}
\sum_{t=0}^{T-1}\left(f(x_{t})-f(x^{*})\right) & \leq O\left(\sqrt{d}R_{\infty}G\sqrt{\ln\left(\frac{GT}{R_{\infty}}\right)}\right)\sqrt{T}+O\left(R_{\infty}^{2}d\right)+\Phi_{0}-\Phi_{T}\\
 & =O\left(\sqrt{d}R_{\infty}G\sqrt{\ln\left(\frac{GT}{R_{\infty}}\right)}\right)\sqrt{T}+O\left(R_{\infty}^{2}d\right)+\frac{1}{2}\underbrace{\left\Vert x_{0}-x^{*}\right\Vert _{D_{0}}^{2}}_{\leq R_{\infty}^{2}\tr(D_{0})}-\frac{1}{2}\underbrace{\left\Vert x_{T}-x^{*}\right\Vert _{D_{T}}^{2}}_{\geq0}\\
 & \leq O\left(\sqrt{d}R_{\infty}G\sqrt{\ln\left(\frac{GT}{R_{\infty}}\right)}\right)\sqrt{T}+O\left(R_{\infty}^{2}d\right)\ .
\end{align*}
Therefore 
\begin{align*}
f(\overline{x}_{T})-f(x^{*}) & \leq\frac{1}{T}\sum_{t=0}^{T-1}\left(f(x_{t})-f(x^{*})\right)=O\left(\frac{\sqrt{d}R_{\infty}G\sqrt{\ln\left(\frac{GT}{R_{\infty}}\right)}}{\sqrt{T}}+\frac{R_{\infty}^{2}d}{T}\right)\ .
\end{align*}

\subsection{Proof of Lemma \ref{lem:inequalities}}

\label{sec:lemma-inequalities-proof}

Here we prove Lemma \ref{lem:inequalities}. As noted earlier, the
inequalities are standard and are equivalent to the inequalities used
in previous work. For convenience, we restate the lemma statement. 
\begin{lem}
Let $d_{0}^{2},d_{1}^{2},d_{2}^{2},\dots,d_{T}^{2}$ and $R^{2}$
be scalars. Let $D_{0}>0$ and let $D_{1},\dots,D_{T}$ be defined
according to the following recurrence: 
\[
D_{t+1}^{2}=D_{t}^{2}\left(1+\frac{d_{t}^{2}}{R^{2}}\right)\ .
\]
We have 
\[
\sum_{t=a}^{b-1}D_{t}\cdot d_{t}^{2}\geq2R^{2}\left(D_{b}-D_{a}\right)\ .
\]
If $d_{t}^{2}\leq R^{2}$ for all $t$, then: 
\begin{align*}
\sum_{t=a}^{b-1}D_{t}\cdot d_{t}^{2} & \leq\left(\sqrt{2}+1\right)R^{2}\left(D_{b}-D_{a}\right)\\
\sum_{t=a}^{b-1}d_{t}^{2} & \leq4R^{2}\ln\left(\frac{D_{b}}{D_{a}}\right)\ .
\end{align*}
\end{lem}

\begin{proof}
We have 
\[
d_{t}^{2}=R^{2}\frac{D_{t+1}^{2}-D_{t}^{2}}{D_{t}^{2}}\ .
\]
Therefore 
\begin{align*}
\sum_{t=a}^{b-1}D_{t}\cdot d_{t}^{2} & =R^{2}\sum_{t=a}^{b-1}\frac{D_{t+1}^{2}-D_{t}^{2}}{D_{t}}=R^{2}\sum_{t=a}^{b-1}\frac{\left(D_{t+1}-D_{t}\right)\left(D_{t+1}+D_{t}\right)}{D_{t}}\\
 & \geq R^{2}\sum_{t=a}^{b-1}\frac{\left(D_{t+1}-D_{t}\right)2D_{t}}{D_{t}}=2R^{2}\sum_{t=a}^{b-1}\left(D_{t+1}-D_{t}\right)=2R^{2}\left(D_{b}-D_{a}\right)\ .
\end{align*}
For the next two inequalities, we assume that $d_{t}^{2}\leq R^{2}$
for all $t$. It follows that $D_{t+1}^{2}\leq2D_{t}^{2}$. We have
\begin{align*}
\sum_{t=a}^{b-1}D_{t}\cdot d_{t}^{2} & =R^{2}\sum_{t=a}^{b-1}\frac{\left(D_{t+1}-D_{t}\right)\left(D_{t+1}+D_{t}\right)}{D_{t}}\leq R^{2}\sum_{t=a}^{b-1}\frac{\left(D_{t+1}-D_{t}\right)\left(\sqrt{2}+1\right)D_{t}}{D_{t}}\\
 & =\left(\sqrt{2}+1\right)R^{2}\sum_{t=a}^{b-1}\left(D_{t+1}-D_{t}\right)=\left(\sqrt{2}+1\right)R^{2}\left(D_{b}-D_{a}\right)\ .
\end{align*}
Since $D_{t+1}^{2}\leq2D_{t}^{2}$, we have 
\begin{align*}
\sum_{t=a}^{b-1}d_{t}^{2} & =R^{2}\sum_{t=a}^{b-1}\frac{D_{t+1}^{2}-D_{t}^{2}}{D_{t}^{2}}\leq2R^{2}\sum_{t=a}^{b-1}\frac{D_{t+1}^{2}-D_{t}^{2}}{D_{t+1}^{2}}\ .
\end{align*}
To upper bound the last sum, let $\phi(x)=D_{\lfloor x\rfloor}^{2}+\left(x-\lfloor x\rfloor\right)\left(D_{\lfloor x\rfloor+1}^{2}-D_{\lfloor x\rfloor}^{2}\right)$.
For integer $t$, we have $\phi'(t)=D_{t+1}^{2}-D_{t}^{2}$ and $\phi(t+1)=D_{t+1}^{2}$.
Thus 
\[
\sum_{t=a}^{b-1}\frac{D_{t+1}^{2}-D_{t}^{2}}{D_{t+1}^{2}}=\sum_{t=a}^{b}\frac{\phi'(t)}{\phi(t+1)}\ .
\]
Since $\phi'$ and $\phi$ are increasing, for all $x\in[t,t+1]$,
we have $\phi'(t)\leq\phi'(x)$ and $\phi(x)\leq\phi(t+1)$. Thus
we can upper bound 
\[
\frac{\phi'(t)}{\phi(t+1)}\leq\int_{t}^{t+1}\frac{\phi'(x)}{\phi(x)}dx\ ,
\]
and thus 
\[
\sum_{t=a}^{b-1}\frac{D_{t+1}^{2}-D_{t}^{2}}{D_{t+1}^{2}}=\sum_{t=a}^{b-1}\frac{\phi'(t)}{\phi(t+1)}\leq\int_{a}^{b}\frac{\phi'(x)}{\phi(x)}dx=\ln\left(\frac{\phi(b)}{\phi(a)}\right)=\ln\left(\frac{D_{b}^{2}}{D_{a}^{2}}\right)=2\ln\left(\frac{D_{b}}{D_{a}}\right)\ .
\]
Therefore 
\[
\sum_{t=a}^{b-1}d_{t}^{2}\leq4R^{2}\ln\left(\frac{D_{b}}{D_{a}}\right)\ .
\]
\end{proof}

\subsection{Proof of Lemma \ref{lem:phiz}}

\label{sec:lem-phiz}

For convenience, we restate the lemma here. 
\begin{lem}
Let $\phi\colon\R^{d}\to\R$, $\phi(z)=a\sqrt{\sum_{i=1}^{d}\ln\left(z_{i}\right)}-\sum_{i=1}^{d}z_{i}$,
where $a\geq0$ is a scalar. Let $z^{*}\in\arg\max_{z\geq1}\phi(z)$.
We have 
\[
\phi(z^{*})\leq\sqrt{d}a\sqrt{\ln a}\ .
\]
\end{lem}

\begin{proof}
By taking the gradient of $\phi(z)$ and setting it to $0$, we see
that $\phi(z)$ is maximized over $z\geq1$ at the point $z^{*}$
that satisfies the following. For every $i\in[d]$, either $z_{i}^{*}=1$
or $z_{i}^{*}=\frac{a}{Z}$, where we defined $Z:=2\sqrt{\sum_{i=1}^{d}\ln\left(z_{i}^{*}\right)}$.

If $Z\leq1$, we have 
\[
\phi(z^{*})=a\underbrace{\sqrt{\sum_{i=1}^{d}\ln\left(z_{i}^{*}\right)}}_{\leq1}-\underbrace{\sum_{i=1}^{d}z_{i}^{*}}_{\geq0}\leq a\ .
\]

If $Z\geq1$, we have 
\begin{align*}
\phi(z^{*}) & =a\sqrt{\sum_{i=1}^{d}\ln\left(z_{i}^{*}\right)}-\sum_{i=1}^{d}z_{i}^{*}\leq a\sqrt{\sum_{i=1}^{d}\ln\left(z_{i}^{*}\right)}=a\sqrt{\sum_{i\colon z_{i}^{*}\neq1}\ln\left(z_{i}^{*}\right)}\\
 & =a\sqrt{\sum_{i\colon z_{i}^{*}\neq1}\ln\left(\frac{a}{Z}\right)}=a\sqrt{\sum_{i\colon z_{i}^{*}\neq1}\left(\ln a-\underbrace{\ln Z}_{\geq0}\right)}\leq a\sqrt{d\ln a}\ .
\end{align*}
\end{proof}

\section{Analysis of $\protect\adagradplus$ in the Stochastic Setting}

\label{sec:analysis-adagrad+-stoch}

In this section, we extend the $\adagradplus$ algorithm and its analysis
to the setting where, in each iteration, the algorithm receives a
stochastic gradient $\widetilde{\nabla}f(x_{t})$ that satisfies the
assumptions \eqref{eq:stoch-assumption-unbiased} and \eqref{eq:stoch-assumption-variance}:
$\E\left[\widetilde{\nabla}f(x)\vert x\right]=\nabla f(x)$ and $\E\left[\left\Vert \widetilde{\nabla}f(x)-\nabla f(x)\right\Vert ^{2}\right]\leq\sigma^{2}$.
The algorithm is shown in Figure \ref{alg:adagrad-stoch}. Note that
we made a minor adjustment to the constant in the update in $D_{t}$.

\begin{figure}
\noindent %
\noindent\fbox{\begin{minipage}[t]{1\columnwidth - 2\fboxsep - 2\fboxrule}%
Let $x_{0}\in\dom$, $D_{0}=I$, $R_{\infty}\geq\max_{x,y\in\dom}\left\Vert x-y\right\Vert _{\infty}$.

For $t=0,\dots,T-1$, update:

\begin{align*}
x_{t+1}= & \arg\min_{x\in\dom}\left\{ \left\langle \widetilde{\nabla}f(x_{t}),x\right\rangle +\frac{1}{2}\left\Vert x-x_{t}\right\Vert _{D_{t}}^{2}\right\} \ ,\\
D_{t+1,i}^{2} & =D_{t,i}^{2}\left(1+\frac{\left(x_{t+1,i}-x_{t,i}\right)^{2}}{2R_{\infty}^{2}}\right)\ , & \forall i\in[d]
\end{align*}
Return $\overline{x}_{T}=\frac{1}{T}\sum_{t=1}^{T}x_{t}$.%
\end{minipage}}

\caption{$\protect\adagradplus$ algorithm with stochastic gradients $\widetilde{\nabla}f(x_{t})$.}
\label{alg:adagrad-stoch} 
\end{figure}

\subsection{Analysis for Smooth Functions}

\label{subsec:stoch-adagrad-smooth}

The analysis is an adaptation of the analysis from Section \ref{sec:analysis-adagrad+-smooth}.

Following the analysis from Lemma \ref{lem:main} we obtain a more
refined version of Lemma \ref{lem:standard-adagrad}. Specifically,
we can prove that the iterates produced by $\adagradplus$ satisfy:
\begin{align}
\sum_{t=0}^{T-1}\left(f(x_{t+1})-f(x^{*})\right) & \leq\frac{1}{2}R_{\infty}^{2}\tr\left(D_{T}\right)-\frac{1}{2}\sum_{t=0}^{T-1}\left\Vert x_{t+1}-x_{t}\right\Vert _{D_{t}}^{2}+\frac{1}{2}\sum_{t=0}^{T-1}\left\Vert x_{t+1}-x_{t}\right\Vert _{\sm}^{2}\label{eq:st-scal-1}\\
 & +\sum_{t=0}^{T-1}\left\langle \nabla f(x_{t})-\widetilde{\nabla}f(x_{t}),x_{t+1}-x^{*}\right\rangle \ .\nonumber 
\end{align}
To prove \eqref{eq:st-scal-1} we write $f(x_{t+1})-f(x^{*})=f(x_{t+1})-f(x_{t})+f(x_{t})-f(x^{*})$,
and use smoothness to bound the first term and convexity to bound
the second term.

\begin{align*}
f(x_{t+1})-f(x^{*}) & =f(x_{t+1})-f(x_{t})+f(x_{t})-f(x^{*})\\
 & \leq\left\langle \nabla f(x_{t}),x_{t+1}-x_{t}\right\rangle +\frac{1}{2}\left\Vert x_{t+1}-x_{t}\right\Vert _{\sm}^{2}+\left\langle \nabla f(x_{t}),x_{t}-x^{*}\right\rangle \\
 & =\left\langle \nabla f(x_{t}),x_{t+1}-x^{*}\right\rangle +\frac{1}{2}\left\Vert x_{t+1}-x_{t}\right\Vert _{\sm}^{2}\\
 & =\left\langle \widetilde{\nabla}f(x_{t}),x_{t+1}-x^{*}\right\rangle +\frac{1}{2}\left\Vert x_{t+1}-x_{t}\right\Vert _{\sm}^{2}+\left\langle \nabla f(x_{t})-\widetilde{\nabla}f(x_{t}),x_{t+1}-x^{*}\right\rangle \ .
\end{align*}
Then, we use the first-order optimality condition for $x_{t+1}$ to
obtain

\[
\left\langle \widetilde{\nabla}f(x_{t})+D_{t}\left(x_{t+1}-x_{t}\right),x_{t+1}-x^{*}\right\rangle \leq0\ ,
\]
which after rearranging gives 
\begin{align*}
\left\langle \widetilde{\nabla}f(x_{t}),x_{t+1}-x^{*}\right\rangle  & \leq\left\langle D_{t}\left(x_{t}-x_{t+1}\right),x_{t+1}-x^{*}\right\rangle \\
 & =\left\langle D_{t}\left(x_{t}-x_{t+1}\right),x_{t}-x^{*}+x_{t+1}-x_{t}\right\rangle \\
 & =\left\langle D_{t}\left(x_{t}-x_{t+1}\right),x_{t}-x^{*}\right\rangle -\left\Vert x_{t+1}-x_{t}\right\Vert _{D_{t}}^{2}\ .
\end{align*}
Plugging into the previous inequality gives 
\begin{align*}
f(x_{t+1})-f(x^{*}) & \leq\left\langle D_{t}\left(x_{t}-x_{t+1}\right),x_{t}-x^{*}\right\rangle -\left\Vert x_{t+1}-x_{t}\right\Vert _{D_{t}}^{2}+\frac{1}{2}\left\Vert x_{t+1}-x_{t}\right\Vert _{\sm}^{2}\\
 & +\left\langle \nabla f(x_{t})-\widetilde{\nabla}f(x_{t}),x_{t+1}-x^{*}\right\rangle \ .
\end{align*}
From here on, we can use the same analysis from Lemmas \ref{lem:main}
and \ref{lem:standard-adagrad} to obtain the inequality from \eqref{eq:st-scal-1}.

To shorten notation, let $\xi_{t}=\nabla f(x_{t})-\widetilde{\nabla}f(x_{t})$.
Compared to Lemma \ref{lem:standard-adagrad} we carry the additional
error term $\sum_{t=0}^{T-1}\left\langle \xi_{t},x_{t+1}-x^{*}\right\rangle $.

We write the guarantee provided by \eqref{eq:st-scal-1} as follows:

\begin{align*}
2\sum_{t=0}^{T-1}\left(f(x_{t+1})-f(x^{*})\right) & \leq\underbrace{R_{\infty}^{2}\tr\left(D_{T}\right)-\frac{1}{4}\sum_{t=0}^{T-1}\left\Vert x_{t+1}-x_{t}\right\Vert _{D_{t}}^{2}}_{(\star)}\\
 & +\underbrace{\sum_{t=0}^{T-1}\left\Vert x_{t+1}-x_{t}\right\Vert _{\sm}^{2}-\frac{1}{4}\sum_{t=0}^{T-1}\left\Vert x_{t+1}-x_{t}\right\Vert _{D_{t}}^{2}}_{(\star\star)}\\
 & +\underbrace{2\left(\sum_{t=0}^{T-1}\left\langle \xi_{t},x_{t+1}-x_{t}\right\rangle -\frac{1}{4}\sum_{t=0}^{T-1}\left\Vert x_{t+1}-x_{t}\right\Vert _{D_{t}}^{2}\right)}_{(\diamond)}\\
 & +\underbrace{2\sum_{t=0}^{T-1}\left\langle \xi_{t},x_{t}-x^{*}\right\rangle }_{(\diamond\diamond)}\ .
\end{align*}
By applying Lemma \ref{lem:inequalities} separately for each coordinate,
with $d_{t}^{2}=\left(x_{t+1,i}-x_{t,i}\right)^{2}\leq R_{\infty}^{2}$
and $R^{2}=2R_{\infty}^{2}$, we obtain: 
\begin{align*}
\sum_{t=0}^{T-1}\left\Vert x_{t+1}-x_{t}\right\Vert _{D_{t}}^{2} & \geq4R_{\infty}^{2}\left(\tr\left(D_{T}\right)-\tr\left(D_{0}\right)\right)\\
\sum_{t=0}^{T-1}\left\Vert x_{t+1}-x_{t}\right\Vert ^{2} & \leq8R_{\infty}^{2}\sum_{i=1}^{d}\ln\left(D_{T,i}\right)
\end{align*}
We proceed analogously to Lemmas \eqref{lem:error1} and \eqref{lem:error2},
and obtain

\begin{align*}
(\star) & \leq O\left(R_{\infty}^{2}d\right)\\
(\star\star) & \leq O\left(R_{\infty}^{2}\sum_{i=1}^{d}\beta_{i}\ln\left(2\beta_{i}\right)\right)
\end{align*}
To bound $(\diamond)$, similarly to \citet{BachL19}, we apply Cauchy-Schwarz
twice and obtain

\[
\sum_{t=0}^{T-1}\left\langle \xi_{t},x_{t+1}-x_{t}\right\rangle \leq\sum_{t=0}^{T-1}\left\Vert \xi_{t}\right\Vert \left\Vert x_{t+1}-x_{t}\right\Vert \leq\sqrt{\sum_{t=0}^{T-1}\left\Vert \xi_{t}\right\Vert ^{2}}\sqrt{\sum_{t=0}^{T-1}\left\Vert x_{t+1}-x_{t}\right\Vert ^{2}}\ .
\]
Therefore

\begin{align*}
(\diamond) & =2\left(\sum_{t=0}^{T-1}\left\langle \xi_{t},x_{t+1}-x_{t}\right\rangle -\frac{1}{4}\sum_{t=0}^{T-1}\left\Vert x_{t+1}-x_{t}\right\Vert _{D_{t}}^{2}\right)\\
 & \le2\left(\sqrt{\sum_{t=0}^{T-1}\left\Vert \xi_{t}\right\Vert ^{2}}\sqrt{8R_{\infty}^{2}\sum_{i=1}^{d}\ln\left(D_{T,i}\right)}-R_{\infty}^{2}\sum_{i=1}^{T}D_{T,i}+R_{\infty}^{2}d\right)\\
 & =2R_{\infty}^{2}\left(\sqrt{\frac{8}{R_{\infty}^{2}}\sum_{t=0}^{T-1}\left\Vert \xi_{t}\right\Vert ^{2}}\sqrt{\sum_{i=1}^{d}\ln\left(D_{T,i}\right)}-\sum_{i=1}^{T}D_{T,i}\right)+2R_{\infty}^{2}d\\
 & \leq O\left(R_{\infty}^{2}\right)\sqrt{d}\sqrt{\frac{\sum_{t=0}^{T-1}\left\Vert \xi_{t}\right\Vert ^{2}}{R_{\infty}^{2}}\ln\left(\frac{\sum_{t=0}^{T-1}\left\Vert \xi_{t}\right\Vert ^{2}}{R_{\infty}^{2}}\right)}+2R_{\infty}^{2}d
\end{align*}
In the last inequality, we have used Lemma \ref{lem:phiz}.

Taking expectation and using that $\sqrt{x\ln x}$ is concave and
the assumption $\E\left[\left\Vert \xi_{t}\right\Vert ^{2}\right]\leq\sigma^{2}$,
we obtain 
\begin{align*}
\E\left[(\diamond)\right] & \leq O\left(R_{\infty}^{2}\right)\sqrt{d}\cdot\E\left[\sqrt{\frac{\sum_{t=0}^{T-1}\left\Vert \xi_{t}\right\Vert ^{2}}{R_{\infty}^{2}}\ln\left(\frac{\sum_{t=0}^{T-1}\left\Vert \xi_{t}\right\Vert ^{2}}{R_{\infty}^{2}}\right)}\right]+2R_{\infty}^{2}d\\
 & \leq O\left(R_{\infty}^{2}\right)\sqrt{d}\cdot\sqrt{\E\left[\frac{\sum_{t=0}^{T-1}\left\Vert \xi_{t}\right\Vert ^{2}}{R_{\infty}^{2}}\right]\ln\left(\E\left[\frac{\sum_{t=0}^{T-1}\left\Vert \xi_{t}\right\Vert ^{2}}{R_{\infty}^{2}}\right]\right)}+2R_{\infty}^{2}d\\
 & \leq O\left(R_{\infty}^{2}\right)\sqrt{d}\cdot\sqrt{\frac{T\sigma^{2}}{R_{\infty}^{2}}\ln\left(\frac{T\sigma^{2}}{R_{\infty}^{2}}\right)}+2R_{\infty}^{2}d\\
 & =O\left(R_{\infty}\sqrt{d}\sigma\sqrt{T\ln\left(\frac{T\sigma}{R_{\infty}}\right)}\right)+2R_{\infty}^{2}d
\end{align*}
By assumption \eqref{eq:stoch-assumption-unbiased}, we have 
\[
\mathbb{E}\left[\left\langle \xi_{t},x_{t}-x^{*}\right\rangle \vert x_{t}\right]=0\ ,
\]
Taking expectation over the entire history we obtain that 
\[
\E\left[(\diamond\diamond)\right]=\mathbb{E}\left[2\sum_{t=0}^{T-1}\left\langle \xi_{t},x_{t}-x^{*}\right\rangle \right]=0\ .
\]
Putting everything together, we obtain 
\begin{align*}
\E\left[f(\overline{x}_{T})-f(x^{*})\right] & \leq\E\left[\frac{1}{T}\sum_{t=0}^{T-1}\left(f(x_{t+1})-f(x^{*})\right)\right]\leq O\left(\frac{R_{\infty}^{2}\sum_{i=1}^{d}\beta_{i}\ln\left(2\beta_{i}\right)}{T}+\frac{R_{\infty}\sqrt{d}\sigma\sqrt{\ln\left(\frac{T\sigma}{R_{\infty}}\right)}}{\sqrt{T}}\right)\ .
\end{align*}

\subsection{Analysis for Non-Smooth Functions}

The analysis is an extension of the analysis in Section \ref{sec:analysis-adagrad+-nonsmooth},
and it bounding the additional error term arising from stochasticity
as in the above section.

As in Section \ref{sec:analysis-adagrad+-nonsmooth}, we analyze the
potential

\[
\Phi_{t}:=\frac{1}{2}\left\Vert x_{t}-x^{*}\right\Vert _{D_{t}}^{2}\ .
\]
To analyze the difference in potential, we proceed similarly to Section
\ref{sec:analysis-adagrad+-nonsmooth}: 
\begin{align*}
\Phi_{t+1}-\Phi_{t} & =\frac{1}{2}\left\Vert x_{t+1}-x^{*}\right\Vert _{D_{t+1}}^{2}-\frac{1}{2}\left\Vert x_{t}-x^{*}\right\Vert _{D_{t}}^{2}\\
 & =\frac{1}{2}\left\Vert x_{t+1}-x^{*}\right\Vert _{D_{t+1}-D_{t}}^{2}+\frac{1}{2}\left\Vert x_{t+1}-x^{*}\right\Vert _{D_{t}}^{2}-\frac{1}{2}\left\Vert x_{t}-x^{*}\right\Vert _{D_{t}}^{2}\\
 & \leq\frac{1}{2}\left\Vert x_{t+1}-x^{*}\right\Vert _{D_{t+1}-D_{t}}^{2}-\frac{1}{2}\left\Vert x_{t+1}-x_{t}\right\Vert _{D_{t}}^{2}-\left\langle \widetilde{\nabla}f(x_{t}),x_{t+1}-x^{*}\right\rangle \ .
\end{align*}
In the last inequality, we used the optimality condition for $x_{t+1}$
and algebraic manipulations.

By convexity, we have 
\[
f(x_{t})-f(x^{*})\leq\left\langle \nabla f(x_{t}),x_{t}-x^{*}\right\rangle \ .
\]
Therefore 
\begin{align*}
 & \Phi_{t+1}-\Phi_{t}+f(x_{t})-f(x^{*})\\
 & \leq\frac{1}{2}\left\Vert x_{t+1}-x^{*}\right\Vert _{D_{t+1}-D_{t}}^{2}-\frac{1}{2}\left\Vert x_{t+1}-x_{t}\right\Vert _{D_{t}}^{2}-\left\langle \widetilde{\nabla}f(x_{t}),x_{t+1}-x^{*}\right\rangle +\left\langle \nabla f(x_{t}),x_{t}-x^{*}\right\rangle \\
 & =\frac{1}{2}\left\Vert x_{t+1}-x^{*}\right\Vert _{D_{t+1}-D_{t}}^{2}-\frac{1}{2}\left\Vert x_{t+1}-x_{t}\right\Vert _{D_{t}}^{2}+\left\langle \nabla f(x_{t}),x_{t}-x_{t+1}\right\rangle +\left\langle \nabla f(x_{t})-\widetilde{\nabla}f(x_{t}),x_{t+1}-x^{*}\right\rangle \\
 & \leq\frac{1}{2}\left\Vert x_{t+1}-x^{*}\right\Vert _{D_{t+1}-D_{t}}^{2}-\frac{1}{2}\left\Vert x_{t+1}-x_{t}\right\Vert _{D_{t}}^{2}+\left\Vert \nabla f(x_{t})\right\Vert \left\Vert x_{t}-x_{t+1}\right\Vert +\left\langle \nabla f(x_{t})-\widetilde{\nabla}f(x_{t}),x_{t+1}-x^{*}\right\rangle \\
 & \leq\frac{1}{2}\left\Vert x_{t+1}-x^{*}\right\Vert _{D_{t+1}-D_{t}}^{2}-\frac{1}{2}\left\Vert x_{t+1}-x_{t}\right\Vert _{D_{t}}^{2}+G\left\Vert x_{t}-x_{t+1}\right\Vert +\left\langle \nabla f(x_{t})-\widetilde{\nabla}f(x_{t}),x_{t+1}-x^{*}\right\rangle 
\end{align*}
We sum up over all iterations and obtain 
\begin{align*}
 & \Phi_{T}-\Phi_{0}+\sum_{t=0}^{T-1}\left(f(x_{t})-f(x^{*})\right)\\
 & \leq\underbrace{\sum_{t=0}^{T-1}\frac{1}{2}\left\Vert x_{t+1}-x^{*}\right\Vert _{D_{t+1}-D_{t}}^{2}}_{(\star)}-\underbrace{\sum_{t=0}^{T-1}\frac{1}{2}\left\Vert x_{t+1}-x_{t}\right\Vert _{D_{t}}^{2}}_{(\star\star)}+\underbrace{\sum_{t=0}^{T-1}G\left\Vert x_{t}-x_{t+1}\right\Vert }_{(\star\star\star)}\\
 & +\underbrace{\sum_{t=0}^{T-1}\left\langle \nabla f(x_{t})-\widetilde{\nabla}f(x_{t}),x_{t+1}-x_{t}\right\rangle }_{(\diamond)}\\
 & +\underbrace{\sum_{t=0}^{T-1}\left\langle \nabla f(x_{t})-\widetilde{\nabla}f(x_{t}),x_{t}-x^{*}\right\rangle }_{(\diamond\diamond)}\ .
\end{align*}
As before, we have 
\[
(\star)=\sum_{t=0}^{T-1}\frac{1}{2}\left\Vert x_{t+1}-x^{*}\right\Vert _{D_{t+1}-D_{t}}^{2}\leq\sum_{t=0}^{T-1}\frac{1}{2}R_{\infty}^{2}\left(\tr(D_{t+1})-\tr(D_{t})\right)=\frac{1}{2}R_{\infty}^{2}\left(\tr(D_{T})-\tr(D_{0})\right)\ .
\]
By applying Lemma \ref{lem:inequalities} separately for each coordinate,
with $d_{t}^{2}=\left(x_{t+1,i}-x_{t,i}\right)^{2}\leq R_{\infty}^{2}$
and $R^{2}=2R_{\infty}^{2}$, we obtain 
\begin{align*}
\sum_{t=0}^{T-1}\left\Vert x_{t+1}-x_{t}\right\Vert _{D_{t}}^{2} & \geq4R_{\infty}^{2}\left(\tr\left(D_{T}\right)-\tr\left(D_{0}\right)\right)\\
\sum_{t=0}^{T-1}\left\Vert x_{t+1}-x_{t}\right\Vert ^{2} & \leq8R_{\infty}^{2}\sum_{i=1}^{d}\ln\left(D_{T,i}\right)\ .
\end{align*}
Therefore 
\[
(\star\star)=\sum_{t=0}^{T-1}\frac{1}{2}\left\Vert x_{t+1}-x_{t}\right\Vert _{D_{t}}^{2}\geq2R_{\infty}^{2}\left(\tr\left(D_{T}\right)-\tr\left(D_{0}\right)\right)\ .
\]
\[
(\star\star\star)=G\sum_{t=0}^{T-1}\sqrt{\left\Vert x_{t}-x_{t+1}\right\Vert ^{2}}\leq G\sqrt{T}\cdot\sqrt{\sum_{t=0}^{T-1}\left\Vert x_{t}-x_{t+1}\right\Vert ^{2}}\leq G\sqrt{T}\cdot\sqrt{8R_{\infty}^{2}\sum_{i=1}^{d}\ln\left(D_{T,i}\right)}\ .
\]
Letting $\xi_{t}=\nabla f(x_{t})-\widetilde{\nabla}f(x_{t})$, we
apply Cauchy-Schwarz twice and obtain 
\begin{align*}
(\diamond) & =\sum_{t=0}^{T-1}\left\langle \xi_{t},x_{t+1}-x_{t}\right\rangle \leq\sum_{t=0}^{T-1}\left\Vert \xi_{t}\right\Vert \left\Vert x_{t+1}-x_{t}\right\Vert \\
 & \leq\sqrt{\sum_{t=0}^{T-1}\left\Vert \xi_{t}\right\Vert ^{2}}\sqrt{\sum_{t=0}^{T-1}\left\Vert x_{t+1}-x_{t}\right\Vert ^{2}}\le\sqrt{\sum_{t=0}^{T-1}\left\Vert \xi_{t}\right\Vert ^{2}}\sqrt{8R_{\infty}^{2}\sum_{i=1}^{d}\ln\left(D_{T,i}\right)}\ .
\end{align*}
Plugging in, 
\begin{align*}
 & \Phi_{T}-\Phi_{0}+\sum_{t=0}^{T-1}\left(f(x_{t})-f(x^{*})\right)\\
 & \leq G\sqrt{T}\cdot\sqrt{8R_{\infty}^{2}\sum_{i=1}^{d}\ln\left(D_{T,i}\right)}+\sqrt{\sum_{t=0}^{T-1}\left\Vert \xi_{t}\right\Vert ^{2}}\sqrt{8R_{\infty}^{2}\sum_{i=1}^{d}\ln\left(D_{T,i}\right)}-\frac{3}{2}R_{\infty}^{2}\sum_{i=1}^{d}D_{T,i}+\frac{3}{2}R_{\infty}^{2}d+(\diamond\diamond)\\
 & =\left(G\sqrt{T}\cdot\sqrt{8R_{\infty}^{2}\sum_{i=1}^{d}\ln\left(D_{T,i}\right)}-\frac{3}{4}R_{\infty}^{2}\sum_{i=1}^{d}D_{T,i}\right)\\
 & +\left(\sqrt{\sum_{t=0}^{T-1}\left\Vert \xi_{t}\right\Vert ^{2}}\sqrt{8R_{\infty}^{2}\sum_{i=1}^{d}\ln\left(D_{T,i}\right)}-\frac{3}{4}R_{\infty}^{2}\sum_{i=1}^{d}D_{T,i}\right)\\
 & +\frac{3}{2}R_{\infty}^{2}d+(\diamond\diamond)\\
 & \leq O\left(R_{\infty}^{2}\right)\sqrt{d}\sqrt{\frac{G^{2}T}{R_{\infty}^{2}}\ln\left(\frac{G^{2}T}{R_{\infty}^{2}}\right)}+\underbrace{O\left(R_{\infty}^{2}\right)\sqrt{d}\sqrt{\frac{\sum_{t=0}^{T-1}\left\Vert \xi_{t}\right\Vert ^{2}}{R_{\infty}^{2}}\ln\left(\frac{\sum_{t=0}^{T-1}\left\Vert \xi_{t}\right\Vert ^{2}}{R_{\infty}^{2}}\right)}}_{(\diamond\diamond\diamond)}+\frac{3}{2}R_{\infty}^{2}d+(\diamond\diamond)\ .
\end{align*}
In the last inequality, we applied Lemma \ref{lem:phiz} twice to
bound each of the first two terms.

Taking expectation and using that $\sqrt{x\ln x}$ is concave and
the assumption $\E\left[\left\Vert \xi_{t}\right\Vert ^{2}\right]\leq\sigma^{2}$,
we obtain 
\begin{align*}
\E\left[(\diamond\diamond\diamond)\right] & \leq O\left(R_{\infty}^{2}\right)\sqrt{d}\cdot\E\left[\sqrt{\frac{\sum_{t=0}^{T-1}\left\Vert \xi_{t}\right\Vert ^{2}}{R_{\infty}^{2}}\ln\left(\frac{\sum_{t=0}^{T-1}\left\Vert \xi_{t}\right\Vert ^{2}}{R_{\infty}^{2}}\right)}\right]\\
 & \leq O\left(R_{\infty}^{2}\right)\sqrt{d}\cdot\sqrt{\E\left[\frac{\sum_{t=0}^{T-1}\left\Vert \xi_{t}\right\Vert ^{2}}{R_{\infty}^{2}}\right]\ln\left(\E\left[\frac{\sum_{t=0}^{T-1}\left\Vert \xi_{t}\right\Vert ^{2}}{R_{\infty}^{2}}\right]\right)}\\
 & \leq O\left(R_{\infty}^{2}\right)\sqrt{d}\cdot\sqrt{\frac{T\sigma^{2}}{R_{\infty}^{2}}\ln\left(\frac{T\sigma^{2}}{R_{\infty}^{2}}\right)}\\
 & =O\left(R_{\infty}\sqrt{d}\sigma\sqrt{T\ln\left(\frac{T\sigma}{R_{\infty}}\right)}\right)
\end{align*}
By assumption \eqref{eq:stoch-assumption-unbiased}, we have 
\[
\mathbb{E}\left[\left\langle \xi_{t},x_{t}-x^{*}\right\rangle \vert x_{t}\right]=0\ ,
\]
Taking expectation over the entire history we obtain that 
\[
\E\left[(\diamond\diamond)\right]=\mathbb{E}\left[2\sum_{t=0}^{T-1}\left\langle \xi_{t},x_{t}-x^{*}\right\rangle \right]=0\ .
\]
Putting everything together and using that $\Phi_{0}=\left\Vert x_{0}-x^{*}\right\Vert _{D_{0}}^{2}\leq R_{\infty}^{2}d$
and $\Phi_{T}\geq0$, we obtain 
\begin{align*}
\E\left[f(\overline{x}_{T})-f(x^{*})\right] & \leq\E\left[\frac{1}{T}\sum_{t=0}^{T-1}\left(f(x_{t+1})-f(x^{*})\right)\right]\\
 & \leq O\left(\frac{R_{\infty}\sqrt{d}G\sqrt{\ln\left(\frac{GT}{R_{\infty}}\right)}}{\sqrt{T}}+\frac{R_{\infty}\sqrt{d}\sigma\sqrt{\ln\left(\frac{T\sigma}{R_{\infty}}\right)}}{\sqrt{T}}+\frac{R_{\infty}^{2}d}{T}\right)\ .
\end{align*}

\section{Analysis of $\protect\adaacsa$ for Non-Smooth Functions}

\label{sec:analysis-acc-nonsmooth-1}

Throughout this section, the norm $\left\Vert \cdot\right\Vert $
without a subscript denotes the $\ell_{2}$ norm. We follow the initial
part of the analysis from Section \ref{sec:analysis-acc-smooth-1}
by using only convexity rather than smoothness. 
\begin{lem}
\label{lem:acsa-basic-bound-nonsmooth}We have that 
\end{lem}

\begin{align*}
\alpha_{t}\gamma_{t}\left(f(y_{t+1})-f(x^{*})\right) & \leq\left(\alpha_{t}-1\right)\gamma_{t}\left(f(y_{t})-f(x^{*})\right)\\
 & +\frac{1}{2}\left\Vert z_{t}-x^{*}\right\Vert _{D_{t}}^{2}-\frac{1}{2}\left\Vert z_{t+1}-x^{*}\right\Vert _{D_{t}}^{2}-\frac{1}{2}\left\Vert z_{t}-z_{t+1}\right\Vert _{D_{t}}^{2}+2G\gamma_{t}\left\Vert z_{t+1}-z_{t}\right\Vert \ .
\end{align*}

\begin{proof}
We follow the proof of Lemma \ref{lem:acsa-basic-bound}, except that
instead of smoothness we use convexity and Cauchy-Schwarz. Specifically,
instead of \eqref{eq:acsa-fy-ineq} we plug in 
\begin{align*}
\alpha_{t}\gamma_{t}f(y_{t+1}) & \leq\alpha_{t}\gamma_{t}\left(f(x_{t})+\left\langle \nabla f(y_{t+1}),y_{t+1}-x_{t}\right\rangle \right)\\
 & =\alpha_{t}\gamma_{t}\left(f(x_{t})+\left\langle \nabla f(x_{t}),y_{t+1}-x_{t}\right\rangle +\left\langle \nabla f(y_{t+1})-\nabla f(x_{t}),y_{t+1}-x_{t}\right\rangle \right)\\
 & \leq\alpha_{t}\gamma_{t}\left(f(x_{t})+\left\langle \nabla f(x_{t}),y_{t+1}-x_{t}\right\rangle +\left\Vert \nabla f(y_{t+1})-\nabla f(x_{t})\right\Vert \left\Vert y_{t+1}-x_{t}\right\Vert \right)\\
 & \leq\alpha_{t}\gamma_{t}\left(f(x_{t})+\left\langle \nabla f(x_{t}),y_{t+1}-x_{t}\right\rangle +2G\left\Vert y_{t+1}-x_{t}\right\Vert \right)\ .
\end{align*}
Using \eqref{eq:acsa-yx-zz} we obtain 
\begin{align*}
\alpha_{t}\gamma_{t}f(y_{t+1}) & \leq\alpha_{t}\gamma_{t}\left(f(x_{t})+\left\langle \nabla f(x_{t}),y_{t+1}-x_{t}\right\rangle +2G\left\Vert \alpha_{t}^{-1}\left(z_{t+1}-z_{t}\right)\right\Vert \right)\\
 & =\alpha_{t}\gamma_{t}\left(f(x_{t})+\left\langle \nabla f(x_{t}),y_{t+1}-x_{t}\right\rangle \right)+2G\gamma_{t}\left\Vert z_{t+1}-z_{t}\right\Vert \ .
\end{align*}
Repeating the argument from before we obtain the inequality from Lemma
\ref{lem:acsa-basic-bound} with $2G\gamma_{t}\left\Vert z_{t+1}-z_{t}\right\Vert $
substituted instead of $\frac{1}{2}\frac{\gamma_{t}}{\alpha_{t}}\left\Vert z_{t+1}-z_{t}\right\Vert _{\sm}^{2}$. 
\end{proof}
Next we telescope the terms from Lemma \ref{lem:acsa-basic-bound-nonsmooth}
using the analogue of Lemma \ref{lem:acsa-conditional-telescope}. 
\begin{lem}
\label{lem:acsa-conditional-telescope-1}Suppose that the parameters
$\left\{ \alpha_{t}\right\} _{t}$, $\left\{ \gamma_{t}\right\} _{t}$
satisfy 
\[
0<\left(\alpha_{t+1}-1\right)\gamma_{t+1}\leq\alpha_{t}\gamma_{t}\ ,
\]
for all $t\geq0$. Then 
\begin{align*}
 & \left(\alpha_{T}-1\right)\gamma_{T}\left(f(y_{T})-f(x^{*})\right)-\left(\alpha_{0}-1\right)\gamma_{0}\left(f(y_{0})-f(x^{*})\right)\\
 & \leq\frac{1}{2}R_{\infty}^{2}\tr\left(D_{T-1}\right)+\sum_{t=0}^{T-1}\left(2G\gamma_{t}\left\Vert z_{t}-z_{t+1}\right\Vert -\frac{1}{2}\left\Vert z_{t}-z_{t+1}\right\Vert _{D_{t}}^{2}\right)\ .
\end{align*}
\end{lem}

From here on we are concerned with upper bounding 
\[
\frac{1}{2}R_{\infty}^{2}\tr\left(D_{T-1}\right)+\underbrace{2G\sum_{t=0}^{T-1}\gamma_{t}\left\Vert z_{t+1}-z_{t}\right\Vert }_{(\star)}-\underbrace{\frac{1}{2}\sum_{t=0}^{T-1}\left\Vert z_{t}-z_{t+1}\right\Vert _{D_{t}}^{2}}_{(\star\star)}\ .
\]
For $(\star)$, we use the concavity of the square root function to
write 
\[
(\star)=2G\sum_{t=0}^{T-1}\gamma_{t}\left\Vert z_{t+1}-z_{t}\right\Vert \leq2G\gamma_{T}\sum_{t=0}^{T-1}\sqrt{\left\Vert z_{t+1}-z_{t}\right\Vert ^{2}}\leq2G\gamma_{T}T^{1/2}\sqrt{\sum_{t=0}^{T-1}\left\Vert z_{t+1}-z_{t}\right\Vert ^{2}}\ .
\]
We now apply Lemma \ref{lem:inequalities} with $d_{t}^{2}=\left(z_{t+1,i}-z_{t,i}\right)^{2}$
and $R^{2}=R_{\infty}^{2}\geq d_{t}^{2}$, and obtain 
\[
\sum_{t=0}^{T-1}\left(z_{t+1,i}-z_{t,i}\right)\leq4R_{\infty}^{2}\ln\left(\frac{D_{T,i}}{D_{0,i}}\right)=4R_{\infty}^{2}\ln\left(D_{T,i}\right)\ ,
\]
which gives us that 
\[
(\star)\leq4G\gamma_{T}T^{1/2}R_{\infty}\sqrt{\sum_{i=1}^{d}\ln\left(D_{T,i}\right)}\ .
\]
In the proof of Lemma \ref{lem:error2-acc-1}, we have shown that

\begin{align*}
(\star\star) & =\frac{1}{2}\sum_{t=0}^{T-1}\left\Vert z_{t}-z_{t+1}\right\Vert _{D_{t}}^{2}\geq R_{\infty}^{2}\left(\tr(D_{T})-\tr(D_{0})\right)\ .
\end{align*}
Putting everything together, we obtain 
\begin{align*}
 & \left(\alpha_{T}-1\right)\gamma_{T}\left(f(y_{T})-f(x^{*})\right)-\left(\alpha_{0}-1\right)\gamma_{0}\left(f(y_{0})-f(x^{*})\right)\\
 & \leq\frac{1}{2}R_{\infty}^{2}\tr\left(D_{T-1}\right)+4G\gamma_{T}T^{1/2}R_{\infty}\sqrt{\sum_{i=1}^{d}\ln\left(D_{T,i}\right)}-R_{\infty}^{2}\left(\tr(D_{T})-\tr(D_{0})\right)\\
 & \leq4G\gamma_{T}T^{1/2}R_{\infty}\sqrt{\sum_{i=1}^{d}\ln\left(D_{T,i}\right)}-\frac{1}{2}R_{\infty}^{2}\tr\left(D_{T}\right)+R_{\infty}^{2}d\\
 & \leq O\left(G\gamma_{T}T^{1/2}R_{\infty}\sqrt{d}\sqrt{\ln\left(\frac{G\gamma_{T}T}{R_{\infty}}\right)}\right)+R_{\infty}^{2}d\ .
\end{align*}
where the last inequality follows from Lemma \ref{lem:phiz}.

Once again, picking $\gamma_{t}=\alpha_{t}=\frac{t}{3}+1$ we easily
verify that the the required conditions hold, and thus 
\[
f(y_{T})-f(x^{*})=O\left(\frac{\sqrt{d}R_{\infty}G\sqrt{\ln\left(\frac{GT}{R_{\infty}}\right)}}{\sqrt{T}}+\frac{R_{\infty}^{2}d}{T^{2}}\right)\:,
\]
which completes our convergence analysis.

\section{Analysis of $\protect\adaacsa$ in the Stochastic Setting}

\label{sec:analysis-acsa-stoch}

In this section, we extend the $\adaacsa$ algorithm and its analysis
to the setting where, in each iteration, the algorithm receives a
stochastic gradient $\widetilde{\nabla}f(x_{t})$ that satisfies the
assumptions \eqref{eq:stoch-assumption-unbiased} and \eqref{eq:stoch-assumption-variance}:
$\E\left[\widetilde{\nabla}f(x)\vert x\right]=\nabla f(x)$ and $\E\left[\left\Vert \widetilde{\nabla}f(x)-\nabla f(x)\right\Vert ^{2}\right]\leq\sigma^{2}$.
The algorithm is shown in Figure \ref{alg:acc-stoch-1}. Note that
we made a minor adjustment to the constant in the update in $D_{t}$.

\begin{figure}
\noindent %
\noindent\fbox{\begin{minipage}[t]{1\columnwidth - 2\fboxsep - 2\fboxrule}%
Let $D_{0}=I$, $z_{0}\in\dom$, $\alpha_{t}=\gamma_{t}=1+\frac{t}{3}$,
$R_{\infty}^{2}\geq\max_{x,y\in\dom}\left\Vert x-y\right\Vert _{\infty}^{2}$.

For $t=0,\dots,T-1$, update:

\begin{align*}
x_{t} & =\left(1-\alpha_{t}^{-1}\right)y_{t}+\alpha_{t}^{-1}z_{t}\ ,\\
z_{t+1} & =\arg\min_{u\in\dom}\left\{ \gamma_{t}\left\langle \widetilde{\nabla}f(x_{t}),u\right\rangle +\frac{1}{2}\left\Vert u-z_{t}\right\Vert _{D_{t}}^{2}\right\} \ ,\\
y_{t+1} & =\left(1-\alpha_{t}^{-1}\right)y_{t}+\alpha_{t}^{-1}z_{t+1}\ ,\\
D_{t+1,i}^{2} & =D_{t,i}^{2}\left(1+\frac{\left(z_{t+1,i}-z_{t,i}\right)^{2}}{2R_{\infty}^{2}}\right)\ , & \text{for all }i\in[d].
\end{align*}

Return $y_{T}$.%
\end{minipage}}

\caption{$\protect\adaacsa$ algorithm with stochastic gradients $\widetilde{\nabla}f(x_{t})$.}
\label{alg:acc-stoch-1} 
\end{figure}

\subsection{Analysis for Smooth Functions}

The analysis is very similar to the one from Section \ref{sec:analysis-acc-smooth-1}.
The main difference consists in properly tracking the error introduced
by stochasticity, and bounding it in expectation. We provide a version
of Lemma \ref{lem:acsa-basic-bound}. 
\begin{lem}
\label{lem:acsa-basic-bound-stoch}We have that 
\end{lem}

\begin{align*}
\alpha_{t}\gamma_{t}\left(f(y_{t+1})-f(x^{*})\right) & \leq\left(\alpha_{t}-1\right)\gamma_{t}\left(f(y_{t})-f(x^{*})\right)\\
 & +\frac{1}{2}\left\Vert z_{t}-x^{*}\right\Vert _{D_{t}}^{2}-\frac{1}{2}\left\Vert z_{t+1}-x^{*}\right\Vert _{D_{t}}^{2}-\frac{1}{2}\left\Vert z_{t}-z_{t+1}\right\Vert _{D_{t}}^{2}+\frac{1}{2}\frac{\gamma_{t}}{\alpha_{t}}\left\Vert z_{t+1}-z_{t}\right\Vert _{\sm}^{2}\\
 & +\gamma_{t}\left\langle \nabla f(x_{t})-\widetilde{\nabla}f(x_{t}),z_{t+1}-x^{*}\right\rangle \ .
\end{align*}

\begin{proof}
Following the same idea, we use bound using smoothness, and using
the definition of $x_{t}$: 
\begin{align}
\alpha_{t}\gamma_{t}f(y_{t+1}) & \leq\alpha_{t}\gamma_{t}\left(f(x_{t})+\left\langle \nabla f(x_{t}),y_{t+1}-x_{t}\right\rangle \right)+\frac{1}{2}\frac{\gamma_{t}}{\alpha_{t}}\left\Vert z_{t+1}-z_{t}\right\Vert _{\sm}^{2}\nonumber \\
 & \leq\left(\alpha_{t}-1\right)\gamma_{t}\cdot f(y_{t})+\gamma_{t}\left(f(x_{t})+\left\langle \nabla f(x_{t}),z_{t+1}-x_{t}\right\rangle \right)+\frac{1}{2}\frac{\gamma_{t}}{\alpha_{t}}\left\Vert z_{t+1}-z_{t}\right\Vert _{\sm}^{2}\nonumber \\
 & =\left(\alpha_{t}-1\right)\gamma_{t}\cdot f(y_{t})+\underbrace{\gamma_{t}\left(f(x_{t})+\left\langle \widetilde{\nabla}f(x_{t}),z_{t+1}-x_{t}\right\rangle \right)}_{(\diamond)}+\frac{1}{2}\frac{\gamma_{t}}{\alpha_{t}}\left\Vert z_{t+1}-z_{t}\right\Vert _{\sm}^{2}\nonumber \\
 & +\gamma_{t}\left\langle \nabla f(x_{t})-\widetilde{\nabla}f(x_{t}),z_{t+1}-x_{t}\right\rangle \ .\label{eq:acsa-fy-ineq2-stoch}
\end{align}
Next we bound $(\diamond)$. Let 
\[
\phi_{t}(u)=\gamma_{t}\left(f(x_{t})+\left\langle \widetilde{\nabla}f(x_{t}),u-x_{t}\right\rangle \right)+\frac{1}{2}\left\Vert u-z_{t}\right\Vert _{D_{t}}^{2}\ .
\]
Since $\phi_{t}$ is $1$-strongly convex with respect to $\left\Vert \cdot\right\Vert _{D_{t}}$
and $z_{t+1}=\arg\min_{u\in\dom}\phi_{t}(u)$ by definition, we have
that for all $u\in\dom$: 
\begin{align*}
\phi_{t}(u) & \geq\phi_{t}(z_{t+1})+\underbrace{\left\langle \nabla\phi_{t}(z_{t+1}),u-z_{t+1}\right\rangle }_{\geq0}+\frac{1}{2}\left\Vert u-z_{t+1}\right\Vert _{D_{t}}^{2}\\
 & \geq\phi_{t}(z_{t+1})+\frac{1}{2}\left\Vert u-z_{t+1}\right\Vert _{D_{t}}^{2}\ .
\end{align*}
The non-negativity of the inner product term follows from first order
optimality: locally any move away from $z_{t+1}$ can not possibly
decrease the value of $\phi_{t}$. Thus 
\[
\phi_{t}(z_{t+1})\leq\phi_{t}(u)-\frac{1}{2}\left\Vert u-z_{t+1}\right\Vert _{D_{t}}^{2}\ .
\]
This allows us to bound:

\begin{align*}
(\diamond)= & \gamma_{t}\left(f(x_{t})+\left\langle \widetilde{\nabla}f(x_{t}),z_{t+1}-x_{t}\right\rangle \right)\\
= & \phi_{t}(z_{t+1})-\frac{1}{2}\left\Vert z_{t+1}-z_{t}\right\Vert _{D_{t}}^{2}\\
\leq & \phi_{t}(x^{*})-\left\Vert x^{*}-z_{t+1}\right\Vert _{D_{t}}^{2}-\frac{1}{2}\left\Vert z_{t+1}-z_{t}\right\Vert _{D_{t}}^{2}\\
= & \gamma_{t}\left(f(x_{t})+\left\langle \widetilde{\nabla}f(x_{t}),x^{*}-x_{t}\right\rangle \right)+\frac{1}{2}\left\Vert x^{*}-z_{t}\right\Vert _{D_{t}}^{2}-\frac{1}{2}\left\Vert x^{*}-z_{t+1}\right\Vert _{D_{t}}^{2}-\frac{1}{2}\left\Vert z_{t+1}-z_{t}\right\Vert _{D_{t}}^{2}\\
\leq & \gamma_{t}f(x^{*})+\frac{1}{2}\left\Vert z_{t}-x^{*}\right\Vert _{D_{t}}^{2}-\frac{1}{2}\left\Vert z_{t+1}-x^{*}\right\Vert _{D_{t}}^{2}-\frac{1}{2}\left\Vert z_{t}-z_{t+1}\right\Vert _{D_{t}}^{2}+\gamma_{t}\left\langle \widetilde{\nabla}f(x_{t})-\nabla f(x_{t}),x^{*}-x_{t}\right\rangle \ .
\end{align*}
Plugging back into \eqref{eq:acsa-fy-ineq2-stoch} we obtain: 
\begin{align*}
\alpha_{t}\gamma_{t}f(y_{t+1}) & \leq\left(\alpha_{t}-1\right)\gamma_{t}\cdot f(y_{t})+\gamma_{t}f(x^{*})\\
 & +\frac{1}{2}\left\Vert z_{t}-x^{*}\right\Vert _{D_{t}}^{2}-\frac{1}{2}\left\Vert z_{t+1}-x^{*}\right\Vert _{D_{t}}^{2}-\frac{1}{2}\left\Vert z_{t}-z_{t+1}\right\Vert _{D_{t}}^{2}+\frac{1}{2}\frac{\gamma_{t}}{\alpha_{t}}\left\Vert z_{t+1}-z_{t}\right\Vert _{\sm}^{2}\\
 & +\gamma_{t}\left\langle \nabla f(x_{t})-\widetilde{\nabla}f(x_{t}),z_{t+1}-x_{t}\right\rangle +\gamma_{t}\left\langle \widetilde{\nabla}f(x_{t})-\nabla f(x_{t}),x^{*}-x_{t}\right\rangle \ ,
\end{align*}
which yields the claim. 
\end{proof}
We can telescope these terms exactly like in Lemma \ref{lem:acsa-conditional-telescope}: 
\begin{lem}
\label{lem:acsa-conditional-telescope-stochastic}Suppose that the
parameters $\left\{ \alpha_{t}\right\} _{t}$, $\left\{ \gamma_{t}\right\} _{t}$
satisfy $\alpha_{t}\geq\gamma_{t}$ and 
\[
0<\left(\alpha_{t+1}-1\right)\gamma_{t+1}\leq\alpha_{t}\gamma_{t}\ ,
\]
for all $t\geq0$. Then 
\begin{align*}
 & \left(\alpha_{T}-1\right)\gamma_{T}\left(f(y_{T})-f(x^{*})\right)-\left(\alpha_{0}-1\right)\gamma_{0}\left(f(y_{0})-f(x^{*})\right)\\
 & \leq\frac{1}{2}R_{\infty}^{2}\tr\left(D_{T-1}\right)+\sum_{t=0}^{T-1}\left(\frac{1}{2}\frac{\gamma_{t}}{\alpha_{t}}\left\Vert z_{t}-z_{t+1}\right\Vert _{\sm}^{2}-\frac{1}{2}\left\Vert z_{t}-z_{t+1}\right\Vert _{D_{t}}^{2}\right)\\
 & +\sum_{t=0}^{T-1}\gamma_{t}\left\langle \nabla f(x_{t})-\widetilde{\nabla}f(x_{t}),z_{t+1}-x^{*}\right\rangle \ .
\end{align*}
\end{lem}

Finally we can upper bound this term by writing it as 
\begin{align*}
 & \frac{1}{2}R_{\infty}^{2}\tr\left(D_{T-1}\right)+\sum_{t=0}^{T-1}\left(\frac{1}{2}\frac{\gamma_{t}}{\alpha_{t}}\left\Vert z_{t}-z_{t+1}\right\Vert _{\sm}^{2}-\frac{1}{2}\left\Vert z_{t}-z_{t+1}\right\Vert _{D_{t}}^{2}\right)\\
 & +\sum_{t=0}^{T-1}\left\langle \nabla f(x_{t})-\widetilde{\nabla}f(x_{t}),z_{t+1}-x^{*}\right\rangle \\
 & =\underbrace{\sum_{t=0}^{T-1}\left(\frac{1}{2}\frac{\gamma_{t}}{\alpha_{t}}\left\Vert z_{t}-z_{t+1}\right\Vert _{\sm}^{2}-\left(\frac{1}{2}-\frac{1}{2\sqrt{2}}\right)\left\Vert z_{t}-z_{t+1}\right\Vert _{D_{t}}^{2}\right)}_{(\star)}\\
 & +\underbrace{\left(\frac{1}{2}R_{\infty}^{2}\tr(D_{T-1})+\frac{\sqrt{2}-1}{2}R_{\infty}^{2}\tr(D_{T})-\frac{1}{2\sqrt{2}}\sum_{t=0}^{T-1}\left\Vert z_{t}-z_{t+1}\right\Vert _{D_{t}}^{2}\right)}_{(\star\star)}\\
 & +\underbrace{\left(\sum_{t=0}^{T-1}\gamma_{t}\left\langle \nabla f(x_{t})-\widetilde{\nabla}f(x_{t}),z_{t+1}-z_{t}\right\rangle \right)-\frac{\sqrt{2}-1}{2}R_{\infty}^{2}\tr(D_{T})}_{(\diamond)}+\underbrace{\sum_{t=0}^{T-1}\gamma_{t}\left\langle \nabla f(x_{t})-\widetilde{\nabla}f(x_{t}),z_{t}-x^{*}\right\rangle }_{(\diamond\diamond)}\ .
\end{align*}
Following the proofs from Lemmas \ref{lem:error1-acc-1} and \ref{lem:error2-acc-1}
we bound $(\star)\leq O\left(R_{\infty}^{2}\sum_{i=1}^{d}\beta_{i}\ln\left(2\beta_{i}\right)\right)$
and $(\star\star)\leq O\left(R_{\infty}^{2}d\right)$, assuming that
the condition stated in Lemma \ref{lem:acsa-conditional-telescope-stochastic}
holds and $\alpha_{t}\geq\gamma_{t}$.

Now we can bound $(\diamond)$. Letting $\xi_{t}=\nabla f(x_{t})-\widetilde{\nabla}f(x_{t})$,
we apply Cauchy-Schwarz twice and obtain 
\begin{align*}
(\diamond) & =\left(\sum_{t=0}^{T-1}\gamma_{t}\left\langle \xi_{t},z_{t+1}-z_{t}\right\rangle \right)-\frac{\sqrt{2}-1}{2}R_{\infty}^{2}\tr(D_{T})\\
 & \leq\sum_{t=0}^{T-1}\gamma_{t}\left\Vert \xi_{t}\right\Vert \left\Vert z_{t+1}-z_{t}\right\Vert -\frac{\sqrt{2}-1}{2}R_{\infty}^{2}\tr(D_{T})\\
 & \leq\sqrt{\left(\sum_{t=0}^{T-1}\gamma_{t}^{2}\left\Vert \xi_{t}\right\Vert ^{2}\right)\left(\sum_{t=0}^{T-1}\left\Vert z_{t+1}-z_{t}\right\Vert ^{2}\right)}-\frac{\sqrt{2}-1}{2}R_{\infty}^{2}\tr(D_{T})\ .
\end{align*}

By applying Lemma \ref{lem:inequalities} separately for each coordinate,
with $d_{t}^{2}=\left(z_{t+1,i}-z_{t,i}\right)^{2}\leq R_{\infty}^{2}$
and $R^{2}=2R_{\infty}^{2}$, we obtain 
\begin{align}
\sum_{t=0}^{T-1}\left\Vert z_{t}-z_{t+1}\right\Vert ^{2} & \leq8R_{\infty}^{2}\sum_{i=1}^{d}\ln\left(D_{T,i}\right)\,.\label{eq:sumznorms-1}
\end{align}

Plugging in and using Lemma \ref{lem:phiz} for $z_{i}=D_{T,i}$,
we obtain: 
\begin{align*}
(\diamond) & \leq R_{\infty}^{2}\cdot\frac{\sqrt{2}-1}{2}\left(\frac{2}{\sqrt{2}-1}\sqrt{\frac{\left(\sum_{t=0}^{T-1}\gamma_{t}^{2}\left\Vert \xi_{t}\right\Vert ^{2}\right)}{R_{\infty}^{2}}\cdot8\sum_{i=1}^{d}\ln\left(D_{T,i}\right)}-\tr(D_{T})\right)\\
 & \leq O\left(R_{\infty}^{2}\sqrt{d}\sqrt{\frac{\left(\sum_{t=0}^{T-1}\gamma_{t}^{2}\left\Vert \xi_{t}\right\Vert ^{2}\right)}{R_{\infty}^{2}}\ln\left(\frac{\left(\sum_{t=0}^{T-1}\gamma_{t}^{2}\left\Vert \xi_{t}\right\Vert ^{2}\right)}{R_{\infty}^{2}}\right)}\right)\ .
\end{align*}

Next, we take expectation and use the fact that $\sqrt{x\ln x}$ is
concave, $\gamma_{t}=O(t)\leq T$, and the assumption $\E\left[\left\Vert \xi_{t}\right\Vert ^{2}\right]\leq\sigma^{2}$.
We obtain 
\[
\mathbb{E}\left[(\diamond)\right]\leq O\left(R_{\infty}^{2}\sqrt{d}\sqrt{\frac{T^{3}\sigma^{2}}{R_{\infty}^{2}}\ln\left(\frac{T^{3}\sigma^{2}}{R_{\infty}^{2}}\right)}\right)=O\left(T^{3/2}\sigma R_{\infty}\sqrt{d}\sqrt{\ln\left(\frac{T\sigma}{R_{\infty}}\right)}\right)\ .
\]
Also, by assumption \eqref{eq:stoch-assumption-unbiased}, we have

\[
\mathbb{E}\left[\left\langle \xi_{t},z_{t}-x^{*}\right\rangle \vert z_{t}\right]=0\ .
\]
Thus taking expectation over the entire history we obtain that 
\[
\E\left[(\diamond\diamond)\right]=\mathbb{E}\left[\sum_{t=1}^{T}\gamma_{t}\left\langle \xi_{t},z_{t}-x^{*}\right\rangle \right]=0\ .
\]
Putting everything together, we obtain that, assuming that the condition
stated in Lemma \ref{lem:acsa-conditional-telescope-stochastic} holds:
\begin{align*}
 & \mathbb{E}\left[\left(\alpha_{T}-1\right)\gamma_{T}\left(f(y_{T})-f(x^{*})\right)-\left(\alpha_{0}-1\right)\gamma_{0}\left(f(y_{0})-f(x^{*})\right)\right]\\
 & \leq O\left(R_{\infty}^{2}\sum_{i=1}^{d}\beta_{i}\ln\left(2\beta_{i}\right)\right)+O\left(R_{\infty}^{2}d\right)+O\left(T^{3/2}\sigma R_{\infty}\sqrt{d}\sqrt{\ln\left(\frac{T\sigma}{R_{\infty}}\right)}\right)\\
 & =O\left(R_{\infty}^{2}\sum_{i=1}^{d}\beta_{i}\ln\left(2\beta_{i}\right)+T^{3/2}R_{\infty}\sqrt{d}\sigma\sqrt{\ln\left(\frac{T\sigma}{R_{\infty}}\right)}\right)\ .
\end{align*}
Once again, picking $\gamma_{t}=\alpha_{t}=\frac{t}{3}+1$ we easily
verify that the the required conditions hold, and thus: 
\[
\mathbb{E}\left[f(y_{T})-f(x^{*})\right]=O\left(\frac{R_{\infty}^{2}\sum_{i=1}^{d}\beta_{i}\ln\left(2\beta_{i}\right)}{T^{2}}+\frac{R_{\infty}\sqrt{d}\sigma\sqrt{\ln\left(\frac{T\sigma}{R_{\infty}}\right)}}{\sqrt{T}}\right)\ .
\]

\subsection{Analysis for Non-smooth Functions}

The analysis is an extension of the analysis in Section \ref{sec:analysis-acc-nonsmooth-1},
and it mainly consists of bounding the additional error term arising
from stochasticity as in the previous section. We use the following
version of Lemma \ref{lem:acsa-basic-bound-nonsmooth}. 
\begin{lem}
\label{lem:acsa-basic-bound-nonsmooth-stochastic}We have that 
\end{lem}

\begin{align*}
\alpha_{t}\gamma_{t}\left(f(y_{t+1})-f(x^{*})\right) & \leq\left(\alpha_{t}-1\right)\gamma_{t}\left(f(y_{t})-f(x^{*})\right)\\
 & +\frac{1}{2}\left\Vert z_{t}-x^{*}\right\Vert _{D_{t}}^{2}-\frac{1}{2}\left\Vert z_{t+1}-x^{*}\right\Vert _{D_{t}}^{2}-\frac{1}{2}\left\Vert z_{t}-z_{t+1}\right\Vert _{D_{t}}^{2}+2G\gamma_{t}\left\Vert z_{t+1}-z_{t}\right\Vert \\
 & +\gamma_{t}\left\langle \nabla f(x_{t})-\widetilde{\nabla}f(x_{t}),z_{t+1}-x^{*}\right\rangle .
\end{align*}

\begin{proof}
We follow the proof of Lemma \ref{lem:acsa-basic-bound-nonsmooth},
except that instead of smoothness we use convexity and Cauchy-Schwarz.
Specifically, we write: 
\begin{align*}
\alpha_{t}\gamma_{t}f(y_{t+1}) & \leq\alpha_{t}\gamma_{t}\left(f(x_{t})+\left\langle \nabla f(y_{t+1}),y_{t+1}-x_{t}\right\rangle \right)\\
 & =\alpha_{t}\gamma_{t}\left(f(x_{t})+\left\langle \nabla f(x_{t}),y_{t+1}-x_{t}\right\rangle +\left\langle \nabla f(y_{t+1})-\nabla f(x_{t}),y_{t+1}-x_{t}\right\rangle \right)\\
 & \leq\alpha_{t}\gamma_{t}\left(f(x_{t})+\left\langle \nabla f(x_{t}),y_{t+1}-x_{t}\right\rangle +\left\Vert \nabla f(y_{t+1})-\nabla f(x_{t})\right\Vert \left\Vert y_{t+1}-x_{t}\right\Vert \right)\\
 & \leq\alpha_{t}\gamma_{t}\left(f(x_{t})+\left\langle \nabla f(x_{t}),y_{t+1}-x_{t}\right\rangle +2G\left\Vert y_{t+1}-x_{t}\right\Vert \right)\ .
\end{align*}
Together with the fact that 
\[
\alpha_{t}\gamma_{t}\left(f(x_{t})+\left\langle \nabla f(x_{t}),y_{t+1}-x_{t}\right\rangle \right)\leq\left(\alpha_{t}-1\right)\gamma_{t}\cdot f(y_{t})+\gamma_{t}\left(f(x_{t})+\left\langle \nabla f(x_{t}),z_{t+1}-x_{t}\right\rangle \right)\,,
\]
which we can see in the proof of Lemma \ref{lem:acsa-basic-bound-nonsmooth},
we obtain that 
\begin{align*}
\alpha_{t}\gamma_{t}f(y_{t+1}) & \leq\left(\alpha_{t}-1\right)\gamma_{t}\cdot f(y_{t})+\gamma_{t}\left(f(x_{t})+\left\langle \nabla f(x_{t}),z_{t+1}-x_{t}\right\rangle \right)+\alpha_{t}\gamma_{t}\cdot2G\left\Vert y_{t+1}-x_{t}\right\Vert \\
 & =\left(\alpha_{t}-1\right)\gamma_{t}\cdot f(y_{t})+\gamma_{t}\left(f(x_{t})+\left\langle \nabla f(x_{t}),z_{t+1}-x_{t}\right\rangle \right)+2G\gamma_{t}\left\Vert z_{t+1}-z_{t}\right\Vert \\
 & =\left(\alpha_{t}-1\right)\gamma_{t}\cdot f(y_{t})+\underbrace{\gamma_{t}\left(f(x_{t})+\left\langle \widetilde{\nabla}f(x_{t}),z_{t+1}-x_{t}\right\rangle \right)}_{(\diamond)}+2G\gamma_{t}\left\Vert z_{t+1}-z_{t}\right\Vert +\\
 & +\gamma_{t}\left\langle \nabla f(x_{t})-\widetilde{\nabla}f(x_{t}),z_{t+1}-x_{t}\right\rangle \ ,
\end{align*}
where the first identity follows from \eqref{eq:acsa-yx-zz}.

Next we bound $(\diamond)$. Just like in the proof of Lemma \ref{lem:acsa-basic-bound},
we prove that 
\begin{align*}
(\diamond) & =\gamma_{t}\left(f(x_{t})+\left\langle \widetilde{\nabla}f(x_{t}),z_{t+1}-x_{t}\right\rangle \right)\\
 & \leq\gamma_{t}\left(f(x_{t})+\left\langle \widetilde{\nabla}f(x_{t}),x^{*}-x_{t}\right\rangle \right)+\frac{1}{2}\left\Vert x^{*}-z_{t}\right\Vert _{D_{t}}^{2}-\frac{1}{2}\left\Vert x^{*}-z_{t+1}\right\Vert _{D_{t}}^{2}-\frac{1}{2}\left\Vert z_{t+1}-z_{t}\right\Vert _{D_{t}}^{2}\ .
\end{align*}
This follows from the fact that the function 
\[
\phi_{t}(u)=\gamma_{t}\left(f(x_{t})+\left\langle \nabla f(x_{t}),u-x_{t}\right\rangle \right)+\frac{1}{2}\left\Vert u-z_{t}\right\Vert _{D_{t}}^{2}
\]
is strongly convex with respect to $\left\Vert \cdot\right\Vert _{D_{t}}$
and $z_{t+1}=\arg\min_{u\in\dom}\phi_{t}(u)$ by definition. Thus
$\phi_{t}(u)\geq\phi_{t}(z_{t+1})+\frac{1}{2}\left\Vert u-z_{t+1}\right\Vert _{D_{t}}^{2}$,
which gives us what we needed after substituting $u=x^{*}$. Thus,
by convexity: 
\begin{align*}
(\diamond) & \leq\gamma_{t}f(x^{*})+\frac{1}{2}\left\Vert x^{*}-z_{t}\right\Vert _{D_{t}}^{2}-\frac{1}{2}\left\Vert x^{*}-z_{t+1}\right\Vert _{D_{t}}^{2}-\frac{1}{2}\left\Vert z_{t+1}-z_{t}\right\Vert _{D_{t}}^{2}\\
 & +\gamma_{t}\left\langle \widetilde{\nabla}f(x_{t})-\nabla f(x_{t}),x^{*}-x_{t}\right\rangle 
\end{align*}
Combining with the bound on $\alpha_{t}\gamma_{t}f(y_{t+1})$, we
get that 
\begin{align*}
\alpha_{t}\gamma_{t}f(y_{t+1}) & \leq\left(\alpha_{t}-1\right)\gamma_{t}\cdot f(y_{t})\\
 & +\left(\gamma_{t}f(x^{*})+\frac{1}{2}\left\Vert x^{*}-z_{t}\right\Vert _{D_{t}}^{2}-\frac{1}{2}\left\Vert x^{*}-z_{t+1}\right\Vert _{D_{t}}^{2}-\frac{1}{2}\left\Vert z_{t+1}-z_{t}\right\Vert _{D_{t}}^{2}\right)+2G\gamma_{t}\left\Vert z_{t+1}-z_{t}\right\Vert \\
 & +\gamma_{t}\left\langle \nabla f(x_{t})-\widetilde{\nabla}f(x_{t}),z_{t+1}-x_{t}\right\rangle +\gamma_{t}\left\langle \widetilde{\nabla}f(x_{t})-\nabla f(x_{t}),x^{*}-x_{t}\right\rangle \ ,
\end{align*}
and thus 
\begin{align*}
\alpha_{t}\gamma_{t}\left(f(y_{t+1})-f(x^{*})\right) & \leq\left(\alpha_{t}-1\right)\gamma_{t}\cdot\left(f(y_{t})-f(x^{*})\right)\\
 & +\frac{1}{2}\left\Vert x^{*}-z_{t}\right\Vert _{D_{t}}^{2}-\frac{1}{2}\left\Vert x^{*}-z_{t+1}\right\Vert _{D_{t}}^{2}-\frac{1}{2}\left\Vert z_{t+1}-z_{t}\right\Vert _{D_{t}}^{2}+2G\gamma_{t}\left\Vert z_{t+1}-z_{t}\right\Vert \\
 & +\gamma_{t}\left\langle \nabla f(x_{t})-\widetilde{\nabla}f(x_{t}),z_{t+1}-x^{*}\right\rangle \ ,
\end{align*}
which is what we needed. 
\end{proof}
Now we telescope the terms from Lemma \ref{lem:acsa-basic-bound-nonsmooth-stochastic}.
The proof is identical to that of Lemma \ref{lem:acsa-conditional-telescope},
so we omit it. 
\begin{lem}
\label{lem:acsa-conditional-telescope-stochastic-nonsmooth}Suppose
that the parameters $\left\{ \alpha_{t}\right\} _{t}$, $\left\{ \gamma_{t}\right\} _{t}$
satisfy $\alpha_{t}\geq\gamma_{t}$ and 
\[
0<\left(\alpha_{t+1}-1\right)\gamma_{t+1}\leq\alpha_{t}\gamma_{t}\ ,
\]
for all $t\geq0$. Then 
\begin{align*}
 & \left(\alpha_{T}-1\right)\gamma_{T}\left(f(y_{T})-f(x^{*})\right)-\left(\alpha_{0}-1\right)\gamma_{0}\left(f(y_{0})-f(x^{*})\right)\\
 & \leq\frac{1}{2}R_{\infty}^{2}\tr\left(D_{T-1}\right)+\sum_{t=0}^{T-1}\left(2G\gamma_{t}\left\Vert z_{t}-z_{t+1}\right\Vert -\frac{1}{2}\left\Vert z_{t}-z_{t+1}\right\Vert _{D_{t}}^{2}\right)\\
 & +\sum_{t=0}^{T-1}\gamma_{t}\left\langle \nabla f(x_{t})-\widetilde{\nabla}f(x_{t}),z_{t+1}-x^{*}\right\rangle \ .
\end{align*}
\end{lem}

From here on we are concerned with upper bounding 
\begin{align*}
 & \frac{1}{2}R_{\infty}^{2}\tr\left(D_{T-1}\right)+\sum_{t=0}^{T-1}\left(2G\gamma_{t}\left\Vert z_{t}-z_{t+1}\right\Vert -\frac{1}{2}\left\Vert z_{t}-z_{t+1}\right\Vert _{D_{t}}^{2}\right)+\sum_{t=0}^{T-1}\gamma_{t}\left\langle \nabla f(x_{t})-\widetilde{\nabla}f(x_{t}),z_{t+1}-x^{*}\right\rangle \\
 & =\frac{1}{2}R_{\infty}^{2}\tr\left(D_{T-1}\right)+\frac{\sqrt{2}-1}{2}R_{\infty}^{2}\tr(D_{T})+\underbrace{2G\sum_{t=0}^{T-1}\gamma_{t}\left\Vert z_{t+1}-z_{t}\right\Vert }_{(\star)}-\underbrace{\frac{1}{2}\sum_{t=0}^{T-1}\left\Vert z_{t}-z_{t+1}\right\Vert _{D_{t}}^{2}}_{(\star\star)}\\
 & +\underbrace{\left(\sum_{t=0}^{T-1}\gamma_{t}\left\langle \nabla f(x_{t})-\widetilde{\nabla}f(x_{t}),z_{t+1}-z_{t}\right\rangle -\frac{\sqrt{2}-1}{2}R_{\infty}^{2}\tr(D_{T})\right)}_{(\diamond)}+\underbrace{\sum_{t=0}^{T-1}\gamma_{t}\left\langle \nabla f(x_{t})-\widetilde{\nabla}f(x_{t}),z_{t}-x^{*}\right\rangle }_{(\diamond\diamond)}\ .
\end{align*}
For $(\star)$ we write, similarly to the proof from Section \ref{sec:analysis-acc-nonsmooth-1},
\[
(\star)=2G\sum_{t=0}^{T-1}\gamma_{t}\left\Vert z_{t+1}-z_{t}\right\Vert \leq2G\gamma_{T}\sum_{t=0}^{T-1}\sqrt{\left\Vert z_{t+1}-z_{t}\right\Vert ^{2}}\leq2G\gamma_{T}T^{1/2}\sqrt{\sum_{t=0}^{T-1}\left\Vert z_{t+1}-z_{t}\right\Vert ^{2}}\ .
\]
We apply Lemma \ref{lem:inequalities} with $d_{t}^{2}=\left(z_{t+1,i}-z_{t,i}\right)^{2}\leq R_{\infty}^{2}$
and $R^{2}=2R_{\infty}^{2}$, and obtain 
\[
\sum_{t=0}^{T-1}\left(z_{t+1,i}-z_{t,i}\right)^{2}\leq8R_{\infty}^{2}\ln\left(D_{T,i}\right)\ ,
\]
which gives us that 
\[
(\star)\leq4\sqrt{2}G\gamma_{T}T^{1/2}R_{\infty}\sqrt{\sum_{i=1}^{d}\ln\left(D_{T,i}\right)}\ .
\]
Next, we lower bound $(\star\star)$. We apply Lemma \ref{lem:inequalities}
with $d_{t}^{2}=\left(z_{t,i}-z_{t+1,i}\right)^{2}$ and $R^{2}=2R_{\infty}^{2}$,
and obtain 
\[
\sum_{t=0}^{T-1}D_{t,i}\left(z_{t,i}-z_{t+1,i}\right)^{2}\geq4R_{\infty}^{2}\left(D_{T,i}-D_{0,i}\right)\ .
\]
Thus 
\[
(\star\star)=\frac{1}{2}\sum_{t=0}^{T-1}\left\Vert z_{t}-z_{t+1}\right\Vert _{D_{t}}^{2}\geq2R_{\infty}^{2}\left(\tr(D_{T})-\tr(D_{0})\right)
\]
For bounding $(\diamond)$ and $(\diamond\diamond)$, the analysis
is identical to that from Lemma \ref{lem:acsa-basic-bound-stoch}.
Hence we obtain that: 
\[
\mathbb{E}\left[(\diamond)\right]\le O\left(T^{3/2}\sigma R_{\infty}\cdot\sqrt{d}\sqrt{\ln\left(\frac{T\sigma}{R_{\infty}}\right)}\right)
\]
and 
\[
\mathbb{E}\left[(\diamond\diamond)\right]=0\ .
\]
Putting everything together, we obtain: 
\begin{align*}
 & \mathbb{E}\left[\left(\alpha_{T}-1\right)\gamma_{T}\left(f(y_{T})-f(x^{*})\right)-\left(\alpha_{0}-1\right)\gamma_{0}\left(f(y_{0})-f(x^{*})\right)\right]\\
 & \leq4\sqrt{2}G\gamma_{T}T^{1/2}R_{\infty}\sqrt{\sum_{i=1}^{d}\ln\left(D_{T,i}\right)}+\frac{1}{2}R_{\infty}^{2}\tr\left(D_{T-1}\right)+\frac{\sqrt{2}-1}{2}R_{\infty}^{2}\tr(D_{T})-2R_{\infty}^{2}\left(\tr(D_{T})-\tr(D_{0})\right)\\
 & +O\left(T^{3/2}\sigma R_{\infty}\cdot\sqrt{d}\sqrt{\ln\left(\frac{T\sigma}{R_{\infty}}\right)}\right)\\
 & \leq4\sqrt{2}G\gamma_{T}T^{1/2}R_{\infty}\sqrt{\sum_{i=1}^{d}\ln\left(D_{T,i}\right)}-\frac{4-\sqrt{2}}{2}R_{\infty}^{2}\tr(D_{T})+2R_{\infty}^{2}\tr(D_{0})+O\left(T^{3/2}\sigma R_{\infty}\cdot\sqrt{d}\sqrt{\ln\left(\frac{T\sigma}{R_{\infty}}\right)}\right)\\
 & \leq O\left(\sqrt{d}G\gamma_{T}T^{1/2}R_{\infty}\sqrt{\log\left(\frac{G\gamma_{T}T}{R_{\infty}}\right)}\right)+2R_{\infty}^{2}\tr(D_{0})+O\left(T^{3/2}\sigma R_{\infty}\cdot\sqrt{d}\sqrt{\ln\left(\frac{T\sigma}{R_{\infty}}\right)}\right)\ ,
\end{align*}
where the last inequality follows from Lemma \ref{lem:phiz}. Setting
$\gamma_{t}=\alpha_{t}=t/3+1$, which satisfies the conditions required
for our inequalities to hold, the previous bounds simplifies to 
\[
O\left(\sqrt{d}GT^{3/2}R_{\infty}\sqrt{\ln\left(\frac{GT}{R_{\infty}}\right)}+R_{\infty}^{2}d+T^{3/2}\sigma R_{\infty}\cdot\sqrt{d}\sqrt{\ln\left(\frac{T\sigma}{R_{\infty}}\right)}\right)\ .
\]
Hence we have 
\[
f(y_{T})-f(x^{*})=O\left(\frac{R_{\infty}\sqrt{d}G\sqrt{\ln\left(\frac{GT}{R_{\infty}}\right)}+R_{\infty}\sqrt{d}\sigma\sqrt{\ln\left(\frac{T\sigma}{R_{\infty}}\right)}}{\sqrt{T}}+\frac{R_{\infty}^{2}d}{T^{2}}\right)\ ,
\]
which completes our convergence analysis.

\section{Analysis of $\protect\adaagdplus$ for Smooth Functions}

\label{sec:analysis-acc-smooth}

We make the following observations that will be used in the analysis.
We note that we have $D_{t+1,i}^{2}\leq2D_{t,i}^{2}$, which will
play an important role in our analysis. The solution $y_{t}$ is the
primal solution, $z_{t}$ is the dual solution, and $x_{t}$ is the
solution at which we compute the gradient. Unrolling the recurrence
gives

\begin{align*}
x_{t} & =\frac{\sum_{i=1}^{t-1}a_{i}z_{i}+a_{t}z_{t-1}}{A_{t}}=\frac{\sum_{i=1}^{t-1}a_{i}z_{i}+a_{t}z_{t}+a_{t}(z_{t-1}-z_{t})}{A_{t}}=y_{t}+\frac{a_{t}}{A_{t}}\left(z_{t-1}-z_{t}\right)\ ,\\
y_{t} & =\frac{\sum_{i=1}^{t-1}a_{i}z_{i}+a_{t}z_{t}}{A_{t}}\ .
\end{align*}

Following \citet{CohenDO18}, we analyze the convergence of the algorithm
using suitable upper and lower bounds on the optimal function value
$f(x^{*})$. For the upper bound, we simply use the value $f(y_{t})$
of the primal solution:

\[
U_{t}:=f(y_{t})\geq f(x^{*})\ .
\]
To lower bound $f(x^{*})$, we take convex combinations of the lower
bounds provided by convexity. By convexity, for each iteration $i$,
we have 
\[
f(x^{*})\geq f(x_{i})+\left\langle \nabla f(x_{i}),x^{*}-x_{i}\right\rangle \ .
\]
By taking a convex combination of these inequalities with coefficients
$a_{i}=i$ and $A_{t}=\sum_{i=1}^{t}a_{i}=\frac{t(t+1)}{2}$, we obtain
the following lower bound on $f(x^{*})$: 
\begin{align*}
f(x^{*}) & \geq\frac{\sum_{i=1}^{t}a_{i}f(x_{i})+\sum_{i=1}^{t}a_{i}\left\langle \nabla f(x_{i}),x^{*}-x_{i}\right\rangle }{A_{t}}\\
 & =\frac{\sum_{i=1}^{t}a_{i}f(x_{i})-\frac{1}{2}\left\Vert x^{*}-z_{0}\right\Vert _{D_{t}}^{2}+\sum_{i=1}^{t}a_{i}\left\langle \nabla f(x_{i}),x^{*}-x_{i}\right\rangle +\frac{1}{2}\left\Vert x^{*}-z_{0}\right\Vert _{D_{t}}^{2}}{A_{t}}\\
 & \geq\frac{\sum_{i=1}^{t}a_{i}f(x_{i})-\frac{1}{2}\left\Vert x^{*}-z_{0}\right\Vert _{D_{t}}^{2}+\min_{u\in\dom}\left\{ \sum_{i=1}^{t}a_{i}\left\langle \nabla f(x_{i}),u-x_{i}\right\rangle +\frac{1}{2}\left\Vert u-z_{0}\right\Vert _{D_{t}}^{2}\right\} }{A_{t}}\\
 & :=L_{t}\ .
\end{align*}
Let 
\begin{align*}
\phi_{t}(u) & =\sum_{i=1}^{t}a_{i}\left\langle \nabla f(x_{i}),u-x_{i}\right\rangle +\frac{1}{2}\left\Vert u-z_{0}\right\Vert _{D_{t}}^{2}\\
\varphi_{t}(u) & =\sum_{i=1}^{t}a_{i}\left\langle \nabla f(x_{i}),u\right\rangle +\frac{1}{2}\left\Vert u-z_{0}\right\Vert _{D_{t}}^{2}\ .
\end{align*}
We have 
\[
\arg\min_{u\in\dom}\phi_{t}(u)=\arg\min_{u\in\dom}\varphi_{t}(u)=z_{t}\ .
\]
Therefore we can write the lower bound as

\[
L_{t}=\frac{\sum_{i=1}^{t}a_{i}f(x_{i})-\frac{1}{2}\left\Vert x^{*}-z_{0}\right\Vert _{D_{t}}^{2}+\phi_{t}(z_{t})}{A_{t}}\ .
\]
Thus we obtain an upper bound on the distance in function value at
iteration $t$ by considering the gap between the upper and lower
bound:

\[
G_{t}:=U_{t}-L_{t}\geq f(y_{t})-f(x^{*})\ .
\]
Our goal is to upper bound $G_{T}$. To this end, we analyze the difference
$A_{t}G_{t}-A_{t-1}G_{t-1}$. By telescoping the difference and using
an upper bound on $A_{1}G_{1}$, we obtain our convergence bound.

We first show the following lemma that only relies on convexity and
not use smoothness. 
\begin{lem}
\label{lem:gap1}We have 
\begin{align*}
A_{t}G_{t}-A_{t-1}G_{t-1} & \leq A_{t}\left(f(y_{t})-f(x_{t})\right)+A_{t-1}\left(f(x_{t})-f(y_{t-1})\right)-a_{t}\left\langle \nabla f(x_{t}),z_{t}-x_{t}\right\rangle \\
 & +\frac{1}{2}\left\Vert x^{*}-z_{0}\right\Vert _{D_{t}-D_{t-1}}^{2}-\frac{1}{2}\left\Vert z_{t}-z_{0}\right\Vert _{D_{t}-D_{t-1}}^{2}-\frac{1}{2}\left\Vert z_{t}-z_{t-1}\right\Vert _{D_{t-1}}^{2}\ .
\end{align*}
\end{lem}

\begin{proof}
We have 
\begin{align}
A_{t}U_{t}-A_{t-1}U_{t-1} & =A_{t}f(y_{t})-A_{t-1}f(y_{t-1})\nonumber \\
 & =a_{t}f(x_{t})+A_{t}\left(f(y_{t})-f(x_{t})\right)+A_{t-1}\left(f(x_{t})-f(y_{t-1})\right)\label{eq:ub}
\end{align}
We also have 
\begin{align}
 & A_{t}L_{t}-A_{t-1}L_{t-1}\nonumber \\
 & =\left(\sum_{i=1}^{t}a_{i}f(x_{i})-\frac{1}{2}\left\Vert x^{*}-z_{0}\right\Vert _{D_{t}}^{2}+\phi_{t}(z_{t})\right)-\left(\sum_{i=1}^{t-1}a_{i}f(x_{i})-\frac{1}{2}\left\Vert x^{*}-z_{0}\right\Vert _{D_{t-1}}^{2}+\phi_{t-1}(z_{t-1})\right)\nonumber \\
 & =a_{t}f(x_{t})-\frac{1}{2}\left\Vert x^{*}-z_{0}\right\Vert _{D_{t}-D_{t-1}}^{2}+\phi_{t}(z_{t})-\phi_{t-1}(z_{t-1})\ .\label{eq:lb1}
\end{align}
Additionally:

\begin{align}
 & \phi_{t}(z_{t})-\phi_{t-1}(z_{t-1})\nonumber \\
 & =\left(\sum_{i=1}^{t}a_{i}\left\langle \nabla f(x_{i}),z_{t}-x_{i}\right\rangle +\frac{1}{2}\left\Vert z_{t}-z_{0}\right\Vert _{D_{t}}^{2}\right)-\left(\sum_{i=1}^{t-1}a_{i}\left\langle \nabla f(x_{i}),z_{t-1}-x_{i}\right\rangle +\frac{1}{2}\left\Vert z_{t-1}-z_{0}\right\Vert _{D_{t-1}}^{2}\right)\nonumber \\
 & =a_{t}\left\langle \nabla f(x_{t}),z_{t}-x_{t}\right\rangle +\frac{1}{2}\left\Vert z_{t}-z_{0}\right\Vert _{D_{t}}^{2}-\left(\sum_{i=1}^{t-1}a_{i}\left\langle \nabla f(x_{i}),z_{t-1}-z_{t}\right\rangle +\frac{1}{2}\left\Vert z_{t-1}-z_{0}\right\Vert _{D_{t-1}}^{2}\right)\ .\label{eq:lb2}
\end{align}
Since $\phi_{t-1}$ is $1$-strongly convex with respect to $\left\Vert \cdot\right\Vert _{D_{t-1}}$and
$z_{t-1}=\arg\min_{u\in\dom}\phi_{t-1}(u)$, we have 
\begin{align*}
\phi_{t-1}(z_{t}) & \geq\phi_{t-1}(z_{t-1})+\underbrace{\left\langle \nabla\phi_{t-1}(z_{t-1}),z_{t}-z_{t-1}\right\rangle }_{\geq0}+\frac{1}{2}\left\Vert z_{t}-z_{t-1}\right\Vert _{D_{t-1}}^{2}\\
 & \geq\phi_{t-1}(z_{t-1})+\frac{1}{2}\left\Vert z_{t}-z_{t-1}\right\Vert _{D_{t-1}}^{2}\ .
\end{align*}
Plugging in the definition of $\phi_{t-1}$ and rearranging, we obtain
\begin{equation}
\sum_{i=1}^{t-1}a_{i}\left\langle \nabla f(x_{i}),z_{t-1}-z_{t}\right\rangle +\frac{1}{2}\left\Vert z_{t-1}-z_{0}\right\Vert _{D_{t-1}}^{2}\le\frac{1}{2}\left\Vert z_{t}-z_{0}\right\Vert _{D_{t-1}}^{2}-\frac{1}{2}\left\Vert z_{t}-z_{t-1}\right\Vert _{D_{t-1}}^{2}\ .\label{eq:lb3}
\end{equation}
By plugging in \eqref{eq:lb3} into \eqref{eq:lb2}, we obtain 
\begin{align}
\phi_{t}(z_{t})-\phi_{t-1}(z_{t-1}) & \geq a_{t}\left\langle \nabla f(x_{t}),z_{t}-x_{t}\right\rangle +\frac{1}{2}\left\Vert z_{t}-z_{0}\right\Vert _{D_{t}-D_{t-1}}^{2}+\frac{1}{2}\left\Vert z_{t}-z_{t-1}\right\Vert _{D_{t-1}}^{2}\ .\label{eq:lb4}
\end{align}
By plugging in \eqref{eq:lb4} into \eqref{eq:lb1}, we obtain 
\begin{align}
 & A_{t}L_{t}-A_{t-1}L_{t-1}\nonumber \\
 & \geq a_{t}f(x_{t})+a_{t}\left\langle \nabla f(x_{t}),z_{t}-x_{t}\right\rangle -\frac{1}{2}\left\Vert x^{*}-z_{0}\right\Vert _{D_{t}-D_{t-1}}^{2}+\frac{1}{2}\left\Vert z_{t}-z_{0}\right\Vert _{D_{t}-D_{t-1}}^{2}+\frac{1}{2}\left\Vert z_{t}-z_{t-1}\right\Vert _{D_{t-1}}^{2}\ .\label{eq:lb}
\end{align}
Using \eqref{eq:ub} and \eqref{eq:lb}, we obtain 
\begin{align*}
 & A_{t}G_{t}-A_{t-1}G_{t-1}\\
 & =\left(A_{t}U_{t}-A_{t-1}U_{t-1}\right)-\left(A_{t}L_{t}-A_{t-1}L_{t-1}\right)\\
 & \leq a_{t}f(x_{t})+A_{t}\left(f(y_{t})-f(x_{t})\right)+A_{t-1}\left(f(x_{t})-f(y_{t-1})\right)\\
 & -\left(a_{t}f(x_{t})+a_{t}\left\langle \nabla f(x_{t}),z_{t}-x_{t}\right\rangle -\frac{1}{2}\left\Vert x^{*}-z_{0}\right\Vert _{D_{t}-D_{t-1}}^{2}+\frac{1}{2}\left\Vert z_{t}-z_{0}\right\Vert _{D_{t}-D_{t-1}}^{2}+\frac{1}{2}\left\Vert z_{t}-z_{t-1}\right\Vert _{D_{t-1}}^{2}\right)\\
 & =A_{t}\left(f(y_{t})-f(x_{t})\right)+A_{t-1}\left(f(x_{t})-f(y_{t-1})\right)-a_{t}\left\langle \nabla f(x_{t}),z_{t}-x_{t}\right\rangle \\
 & +\frac{1}{2}\left\Vert x^{*}-z_{0}\right\Vert _{D_{t}-D_{t-1}}^{2}-\frac{1}{2}\left\Vert z_{t}-z_{0}\right\Vert _{D_{t}-D_{t-1}}^{2}-\frac{1}{2}\left\Vert z_{t}-z_{t-1}\right\Vert _{D_{t-1}}^{2}\ .
\end{align*}
\end{proof}
From this point onward, we use smoothness and obtain the following
bound. 
\begin{lem}
\label{lem:gap2}We have 
\[
A_{t}G_{t}-A_{t-1}G_{t-1}\leq\left\Vert z_{t}-z_{t-1}\right\Vert _{\sm}^{2}+\frac{1}{2}\left\Vert x^{*}-z_{0}\right\Vert _{D_{t}-D_{t-1}}^{2}-\frac{1}{2}\left\Vert z_{t}-z_{0}\right\Vert _{D_{t}-D_{t-1}}^{2}-\frac{1}{2}\left\Vert z_{t}-z_{t-1}\right\Vert _{D_{t-1}}^{2}\ .
\]
\end{lem}

\begin{proof}
By Lemma \ref{lem:gap1}, we have 
\begin{align*}
A_{t}G_{t}-A_{t-1}G_{t-1} & \leq A_{t}\left(f(y_{t})-f(x_{t})\right)+A_{t-1}\left(f(x_{t})-f(y_{t-1})\right)-a_{t}\left\langle \nabla f(x_{t}),z_{t}-x_{t}\right\rangle \\
 & +\frac{1}{2}\left\Vert x^{*}-z_{0}\right\Vert _{D_{t}-D_{t-1}}^{2}-\frac{1}{2}\left\Vert z_{t}-z_{0}\right\Vert _{D_{t}-D_{t-1}}^{2}-\frac{1}{2}\left\Vert z_{t}-z_{t-1}\right\Vert _{D_{t-1}}^{2}\ .
\end{align*}
Using smoothness and convexity, we upper bound 
\begin{align*}
 & A_{t}\underbrace{\left(f(y_{t})-f(x_{t})\right)}_{\text{smoothness}}+A_{t-1}\underbrace{\left(f(x_{t})-f(y_{t-1})\right)}_{\text{convexity}}-a_{t}\left\langle \nabla f(x_{t}),z_{t}-x_{t}\right\rangle \\
 & \leq A_{t}\left\langle \nabla f(x_{t}),y_{t}-x_{t}\right\rangle +A_{t}\frac{1}{2}\left\Vert y_{t}-x_{t}\right\Vert _{\sm}^{2}+A_{t-1}\left\langle \nabla f(x_{t}),x_{t}-y_{t-1}\right\rangle -a_{t}\left\langle \nabla f(x_{t}),z_{t}-x_{t}\right\rangle \\
 & =\left\langle \nabla f(x_{t}),\underbrace{A_{t}\left(y_{t}-x_{t}\right)+A_{t-1}\left(x_{t}-y_{t-1}\right)+a_{t}\left(x_{t}-z_{t}\right)}_{=0}\right\rangle +A_{t}\frac{1}{2}\left\Vert y_{t}-x_{t}\right\Vert _{\sm}^{2}\\
 & =\frac{1}{2}A_{t}\left\Vert y_{t}-x_{t}\right\Vert _{\sm}^{2}=\frac{1}{2}A_{t}\left\Vert \frac{a_{t}}{A_{t}}\left(z_{t}-z_{t-1}\right)\right\Vert _{\sm}^{2}=\frac{1}{2}\frac{a_{t}^{2}}{A_{t}}\left\Vert z_{t}-z_{t-1}\right\Vert _{\sm}^{2}\ .
\end{align*}
Since $a_{t}=t$ and $A_{t}=\frac{t(t+1)}{2}$, we have $\frac{a_{t}^{2}}{A_{t}}=t^{2}\cdot\frac{2}{t(t+1)}\leq2$.
Thus we obtain 
\begin{align*}
A_{t}G_{t}-A_{t-1}G_{t-1} & \leq\left\Vert z_{t}-z_{t-1}\right\Vert _{\sm}^{2}+\frac{1}{2}\left\Vert x^{*}-z_{0}\right\Vert _{D_{t}-D_{t-1}}^{2}-\frac{1}{2}\left\Vert z_{t}-z_{0}\right\Vert _{D_{t}-D_{t-1}}^{2}-\frac{1}{2}\left\Vert z_{t}-z_{t-1}\right\Vert _{D_{t-1}}^{2}\ .
\end{align*}
\end{proof}
By telescoping the difference, we obtain the following. 
\begin{lem}
\label{lem:total-gap}We have 
\begin{align*}
A_{T}G_{T}-A_{1}G_{1} & \leq\sum_{t=2}^{T}\left\Vert z_{t}-z_{t-1}\right\Vert _{\sm}^{2}+\frac{1}{2}R_{\infty}^{2}\left(\tr(D_{T})-\tr(D_{1})\right)\\
 & -\sum_{t=2}^{T}\frac{1}{2}\left\Vert z_{t}-z_{0}\right\Vert _{D_{t}-D_{t-1}}^{2}-\sum_{t=2}^{T}\frac{1}{2}\left\Vert z_{t}-z_{t-1}\right\Vert _{D_{t-1}}^{2}\ .
\end{align*}
\end{lem}

\begin{proof}
Summing the guarantee provided by Lemma \ref{lem:gap2}, we obtain
\begin{align*}
A_{T}G_{T}-A_{1}G_{1} & \leq\sum_{t=2}^{T}\left\Vert z_{t}-z_{t-1}\right\Vert _{\sm}^{2}+\sum_{t=2}^{T}\frac{1}{2}\left\Vert x^{*}-z_{0}\right\Vert _{D_{t}-D_{t-1}}^{2}\\
 & -\sum_{t=2}^{T}\frac{1}{2}\left\Vert z_{t}-z_{0}\right\Vert _{D_{t}-D_{t-1}}^{2}-\sum_{t=2}^{T}\frac{1}{2}\left\Vert z_{t}-z_{t-1}\right\Vert _{D_{t-1}}^{2}\ .
\end{align*}
We bound the second sum as follows:

\begin{align*}
\sum_{t=2}^{T}\frac{1}{2}\left\Vert x^{*}-z_{0}\right\Vert _{D_{t}-D_{t-1}}^{2} & \leq\frac{1}{2}R_{\infty}^{2}\sum_{t=2}^{T}\left(\tr(D_{t})-\tr(D_{t-1})\right)=\frac{1}{2}R_{\infty}^{2}\left(\tr(D_{T})-\tr(D_{1})\right)\ .
\end{align*}
\end{proof}
We analyze the upper bound provided by the above lemma using an analogous
argument to that we used in Section \ref{sec:analysis-adagrad+-smooth}.
As before, we split the upper bound into two terms and analyze each
of the terms analogously to Lemmas \ref{lem:error1} and \ref{lem:error2}.
We will only use the last negative sum, and drop the previous one.

\begin{align*}
A_{T}G_{T}-A_{1}G_{1} & \leq\sum_{t=2}^{T}\left\Vert z_{t}-z_{t-1}\right\Vert _{\sm}^{2}+\frac{1}{2}R_{\infty}^{2}\left(\tr(D_{T})-\tr(D_{1})\right)-\sum_{t=2}^{T}\frac{1}{2}\left\Vert z_{t}-z_{t-1}\right\Vert _{D_{t-1}}^{2}\\
 & =\underbrace{\sum_{t=2}^{T}\left\Vert z_{t}-z_{t-1}\right\Vert _{\sm}^{2}-\left(\frac{1}{2}-\frac{1}{2\sqrt{2}}\right)\sum_{t=2}^{T}\left\Vert z_{t}-z_{t-1}\right\Vert _{D_{t-1}}^{2}}_{(\star)}\\
 & +\underbrace{\frac{1}{2}R_{\infty}^{2}\left(\tr(D_{T})-\tr(D_{1})\right)-\frac{1}{2\sqrt{2}}\sum_{t=2}^{T}\left\Vert z_{t}-z_{t-1}\right\Vert _{D_{t-1}}^{2}}_{(\star\star)}\ .
\end{align*}

\begin{lem}
\label{lem:error1-acc}We have 
\[
(\star)\leq O\left(R_{\infty}^{2}\sum_{i=1}^{d}\beta_{i}\ln\left(2\beta_{i}\right)\right)\ .
\]
\end{lem}

\begin{proof}
Let $c=\frac{1}{2}-\frac{1}{2\sqrt{2}}$. Note that, for each coordinate
$i$, $D_{t,i}$ is increasing with $t$. For each coordinate $i\in[d]$
, we let $\tilde{T}_{i}$ be the last iteration $t$ for which $D_{t-1,i}\leq\frac{1}{c}\beta_{i}$;
if there is no such iteration, we let $\tilde{T}_{i}=-1$. We have
\begin{align*}
(\star) & =\sum_{t=2}^{T}\left\Vert z_{t}-z_{t-1}\right\Vert _{\sm}^{2}-c\sum_{t=2}^{T}\left\Vert z_{t}-z_{t-1}\right\Vert _{D_{t-1}}^{2}\\
 & =\sum_{i=1}^{d}\sum_{t=2}^{T}\left(\beta_{i}\left(z_{t,i}-z_{t-1,i}\right)^{2}-cD_{t-1,i}\left(z_{t,i}-z_{t-1,i}\right)^{2}\right)\\
 & \leq\sum_{i=1}^{d}\sum_{t=2}^{\tilde{T}_{i}}\beta_{i}\left(z_{t,i}-z_{t-1,i}\right)^{2}\ .
\end{align*}
We bound the above sum by considering each coordinate separately.
We apply Lemma \ref{lem:inequalities} with $d_{t}^{2}=\left(z_{t,i}-z_{t-1,i}\right)^{2}$
and $R^{2}=R_{\infty}^{2}\geq d_{t}^{2}$. Using the third inequality
in the lemma, we obtain 
\begin{align*}
\sum_{t=2}^{\tilde{T}_{i}}\left(z_{t,i}-z_{t-1,i}\right)^{2} & \leq2R_{\infty}^{2}+\sum_{t=2}^{\tilde{T}_{i}-2}\left(z_{t,i}-z_{t-1,i}\right)^{2}\leq2R_{\infty}^{2}+\sum_{t=1}^{\tilde{T}_{i}-2}\left(z_{t,i}-z_{t-1,i}\right)^{2}\\
 & \leq2R_{\infty}^{2}+4R_{\infty}^{2}\ln\left(\frac{D_{\tilde{T}_{i}-1,i}}{D_{1,i}}\right)\leq2R_{\infty}^{2}+4R_{\infty}^{2}\ln\left(\frac{1}{c}\beta_{i}\right)\ .
\end{align*}
Therefore 
\[
(\star)\leq O\left(R_{\infty}^{2}\sum_{i=1}^{d}\beta_{i}\ln\left(2\beta_{i}\right)\right)\ .
\]
\end{proof}
\begin{lem}
\label{lem:error2-acc}We have 
\[
(\star\star)\leq O(R_{\infty}^{2}\tr(D_{1}))
\]
\end{lem}

\begin{proof}
Using that $D_{t,i}^{2}\leq2D_{t-1,i}^{2}$ and thus $D_{t-1,i}\geq\frac{1}{\sqrt{2}}D_{t,i}$,
we obtain 
\[
\sum_{t=2}^{T-1}\left\Vert z_{t}-z_{t-1}\right\Vert _{D_{t-1}}^{2}=\sum_{i=1}^{d}\sum_{t=2}^{T-1}D_{t-1,i}\left(z_{t,i}-z_{t-1,i}\right)^{2}\geq\frac{1}{\sqrt{2}}\sum_{i=1}^{d}\sum_{t=2}^{T-1}D_{t,i}\left(z_{t,i}-z_{t-1,i}\right)^{2}\ .
\]
We apply Lemma \ref{lem:inequalities} with $d_{t}^{2}=\left(z_{t,i}-z_{t-1,i}\right)^{2}$
and $R^{2}=R_{\infty}^{2}$ and obtain 
\[
\sum_{t=2}^{T-1}D_{t,i}\left(z_{t,i}-z_{t-1,i}\right)^{2}\geq2R_{\infty}^{2}\left(D_{T,i}-D_{2,i}\right)\ .
\]
Therefore 
\[
\sum_{t=2}^{T-1}\left\Vert z_{t}-z_{t-1}\right\Vert _{D_{t-1}}^{2}\geq\sqrt{2}R_{\infty}^{2}\left(\tr(D_{T})-\tr(D_{2})\right)
\]
and 
\begin{align*}
(\star\star) & =\frac{1}{2}R_{\infty}^{2}\left(\tr(D_{T})-\tr(D_{1})\right)-\frac{1}{2\sqrt{2}}\sum_{t=2}^{T}\left\Vert z_{t}-z_{t-1}\right\Vert _{D_{t-1}}^{2}\\
 & \leq\frac{1}{2}R_{\infty}^{2}\left(\tr(D_{2})-\tr(D_{1})\right)\\
 & \leq O(R_{\infty}^{2}\tr(D_{1}))\ .
\end{align*}
In the last inequality, we have used that $D_{2}\leq\sqrt{2}D_{1}$. 
\end{proof}
Putting everything together, we obtain 
\[
A_{T}G_{T}-A_{1}G_{1}\leq O\left(R_{\infty}^{2}\sum_{i=1}^{d}\beta_{i}\ln\left(2\beta_{i}\right)\right)\ .
\]
Finally, we upper bound $A_{1}G_{1}$. 
\begin{lem}
We have 
\[
A_{1}G_{1}=O\left(R_{\infty}^{2}\left(\tr(\sm)+\tr(D_{1})\right)\right)=O\left(R_{\infty}^{2}\sum_{i=1}^{d}\beta_{i}\right)\ .
\]
\end{lem}

\begin{proof}
Since $y_{1}=z_{1}$ and $a_{1}=A_{1}=1$, we have

\begin{align*}
A_{1}G_{1} & =U_{1}-L_{1}\\
 & =f(y_{1})-\left(f(x_{1})-\frac{1}{2}\left\Vert x^{*}-z_{0}\right\Vert _{D_{1}}^{2}+\left\langle \nabla f(x_{1}),z_{1}-x_{1}\right\rangle +\frac{1}{2}\left\Vert z_{1}-z_{0}\right\Vert _{D_{1}}^{2}\right)\\
 & =\underbrace{f(y_{1})-f(x_{1})-\left\langle \nabla f(x_{1}),y_{1}-x_{1}\right\rangle }_{\text{smoothness}}+\frac{1}{2}\left\Vert x^{*}-z_{0}\right\Vert _{D_{1}}^{2}-\frac{1}{2}\left\Vert z_{1}-z_{0}\right\Vert _{D_{1}}^{2}\\
 & \leq\frac{1}{2}\left\Vert y_{1}-x_{1}\right\Vert _{\sm}^{2}+\frac{1}{2}\left\Vert x^{*}-z_{0}\right\Vert _{D_{1}}^{2}-\frac{1}{2}\left\Vert z_{1}-z_{0}\right\Vert _{D_{1}}^{2}\\
 & \leq\frac{1}{2}\left\Vert y_{1}-x_{1}\right\Vert _{\sm}^{2}+\frac{1}{2}\left\Vert x^{*}-z_{0}\right\Vert _{D_{1}}^{2}\\
 & \leq\frac{1}{2}R_{\infty}^{2}\tr(\sm)+\frac{1}{2}R_{\infty}^{2}\tr(D_{1})\\
 & =\frac{1}{2}R_{\infty}^{2}\left(\sum_{i=1}^{d}\beta_{i}+d\right)\ .
\end{align*}
\end{proof}
Since $A_{T}=\Theta(T^{2})$, we obtain our desired convergence: 
\[
f(y_{T})-f(x^{*})\leq G_{T}=O\left(\frac{R_{\infty}^{2}\sum_{i=1}^{d}\beta_{i}\ln\left(2\beta_{i}\right)}{T^{2}}\right)\ .
\]

\section{Analysis of $\protect\adaagdplus$ for Non-Smooth Functions}

\label{sec:analysis-acc-nonsmooth}

Throughout this section, the norm $\left\Vert \cdot\right\Vert $
without a subscript denotes the $\ell_{2}$ norm. We warn the reader
that the $G$ notation is overloaded: we use $G$ without a subscript
to denote the upper bound on norm of gradients, and we use $G_{t}$
to denote the function value gap at iteration $t$ (see Section \ref{sec:analysis-acc-smooth}
for the definition of $G_{t}$).

We follow the initial part of the analysis from Section \ref{sec:analysis-acc-smooth}
that uses only convexity, up to and including Lemma \ref{lem:gap1}.
By Lemma \ref{lem:gap1}, we have

\begin{align*}
A_{t}G_{t}-A_{t-1}G_{t-1} & =A_{t}\left(f(y_{t})-f(x_{t})\right)+A_{t-1}\left(f(x_{t})-f(y_{t-1})\right)-a_{t}\left\langle \nabla f(x_{t}),z_{t}-x_{t}\right\rangle \\
 & +\frac{1}{2}\left\Vert x^{*}-z_{0}\right\Vert _{D_{t}-D_{t-1}}^{2}-\frac{1}{2}\left\Vert z_{t}-z_{0}\right\Vert _{D_{t}-D_{t-1}}^{2}-\frac{1}{2}\left\Vert z_{t}-z_{t-1}\right\Vert _{D_{t-1}}^{2}\ .
\end{align*}
We proceed as follows:

\begin{align*}
 & A_{t}G_{t}-A_{t-1}G_{t-1}\\
 & =A_{t}\underbrace{\left(f(y_{t})-f(x_{t})\right)}_{\text{convexity}}+A_{t-1}\underbrace{\left(f(x_{t})-f(y_{t-1})\right)}_{\text{convexity}}-a_{t}\left\langle \nabla f(x_{t}),z_{t}-x_{t}\right\rangle \\
 & +\frac{1}{2}\left\Vert x^{*}-z_{0}\right\Vert _{D_{t}-D_{t-1}}^{2}-\frac{1}{2}\left\Vert z_{t}-z_{0}\right\Vert _{D_{t}-D_{t-1}}^{2}-\frac{1}{2}\left\Vert z_{t}-z_{t-1}\right\Vert _{D_{t-1}}^{2}\\
 & \leq A_{t}\left\langle \nabla f(y_{t}),y_{t}-x_{t}\right\rangle +A_{t-1}\left\langle \nabla f(x_{t}),x_{t}-y_{t-1}\right\rangle -a_{t}\left\langle \nabla f(x_{t}),z_{t}-x_{t}\right\rangle \\
 & +\frac{1}{2}\left\Vert x^{*}-z_{0}\right\Vert _{D_{t}-D_{t-1}}^{2}-\frac{1}{2}\left\Vert z_{t}-z_{0}\right\Vert _{D_{t}-D_{t-1}}^{2}-\frac{1}{2}\left\Vert z_{t}-z_{t-1}\right\Vert _{D_{t-1}}^{2}\\
 & =A_{t}\left\langle \nabla f(y_{t})-\nabla f(x_{t}),y_{t}-x_{t}\right\rangle +\left\langle \nabla f(x_{t}),\underbrace{A_{t}\left(y_{t}-x_{t}\right)+A_{t-1}\left(x_{t}-y_{t-1}\right)+a_{t}\left(x_{t}-z_{t}\right)}_{=0}\right\rangle \\
 & +\frac{1}{2}\left\Vert x^{*}-z_{0}\right\Vert _{D_{t}-D_{t-1}}^{2}-\frac{1}{2}\left\Vert z_{t}-z_{0}\right\Vert _{D_{t}-D_{t-1}}^{2}-\frac{1}{2}\left\Vert z_{t}-z_{t-1}\right\Vert _{D_{t-1}}^{2}\\
 & =A_{t}\left\langle \nabla f(y_{t})-\nabla f(x_{t}),y_{t}-x_{t}\right\rangle \\
 & +\frac{1}{2}\left\Vert x^{*}-z_{0}\right\Vert _{D_{t}-D_{t-1}}^{2}-\frac{1}{2}\left\Vert z_{t}-z_{0}\right\Vert _{D_{t}-D_{t-1}}^{2}-\frac{1}{2}\left\Vert z_{t}-z_{t-1}\right\Vert _{D_{t-1}}^{2}\ .
\end{align*}
Next, we use Cauchy-Schwarz and the fact that $y_{t}-x_{t}=\frac{a_{t}}{A_{t}}\left(z_{t-1}-z_{t}\right)$
and obtain 
\begin{align*}
A_{t}\left\langle \nabla f(y_{t})-\nabla f(x_{t}),y_{t}-x_{t}\right\rangle  & \leq A_{t}\left\Vert \nabla f(y_{t})-\nabla f(x_{t})\right\Vert \left\Vert y_{t}-x_{t}\right\Vert =a_{t}\left\Vert \nabla f(y_{t})-\nabla f(x_{t})\right\Vert \left\Vert z_{t-1}-z_{t}\right\Vert \ .
\end{align*}
Using the triangle inequality and the bound $G$ on gradient norms,
\begin{align*}
A_{t}\left\langle \nabla f(y_{t})-\nabla f(x_{t}),y_{t}-x_{t}\right\rangle  & \leq a_{t}\left(\left\Vert \nabla f(y_{t})\right\Vert +\left\Vert \nabla f(x_{t})\right\Vert \right)\left\Vert z_{t-1}-z_{t}\right\Vert \leq2Ga_{t}\left\Vert z_{t-1}-z_{t}\right\Vert \ .
\end{align*}
Plugging in, 
\begin{align*}
A_{t}G_{t}-A_{t-1}G_{t-1} & \leq2Ga_{t}\left\Vert z_{t-1}-z_{t}\right\Vert +\frac{1}{2}\left\Vert x^{*}-z_{0}\right\Vert _{D_{t}-D_{t-1}}^{2}-\frac{1}{2}\left\Vert z_{t}-z_{0}\right\Vert _{D_{t}-D_{t-1}}^{2}-\frac{1}{2}\left\Vert z_{t}-z_{t-1}\right\Vert _{D_{t-1}}^{2}\\
 & \leq2Ga_{t}\left\Vert z_{t-1}-z_{t}\right\Vert +\frac{1}{2}\left\Vert x^{*}-z_{0}\right\Vert _{D_{t}-D_{t-1}}^{2}-\frac{1}{2}\left\Vert z_{t}-z_{t-1}\right\Vert _{D_{t-1}}^{2}\ .
\end{align*}
Summing up, 
\[
A_{T}G_{T}-A_{1}G_{1}\leq\underbrace{2G\sum_{t=2}^{T}a_{t}\left\Vert z_{t-1}-z_{t}\right\Vert }_{(\star)}+\underbrace{\sum_{t=2}^{T}\frac{1}{2}\left\Vert x^{*}-z_{0}\right\Vert _{D_{t}-D_{t-1}}^{2}}_{(\star\star)}-\underbrace{\sum_{t=2}^{T}\frac{1}{2}\left\Vert z_{t}-z_{t-1}\right\Vert _{D_{t-1}}^{2}}_{(\star\star\star)}\ .
\]
To bound $(\star)$, we proceed analogously to the argument in Section
\ref{sec:analysis-adagrad+-nonsmooth}. We use that $a_{t}=t\leq T$
and the concavity of $\sqrt{z}$: 
\begin{align*}
(\star) & =2G\sum_{t=2}^{T}a_{t}\left\Vert z_{t-1}-z_{t}\right\Vert \leq2GT\sum_{t=2}^{T}\sqrt{\left\Vert z_{t-1}-z_{t}\right\Vert ^{2}}\leq2GT^{3/2}\sqrt{\sum_{t=2}^{T}\left\Vert z_{t-1}-z_{t}\right\Vert ^{2}}\ .
\end{align*}
For each coordinate separately, we apply Lemma \ref{lem:inequalities}
with $d_{t}^{2}=\left(z_{t,i}-z_{t-1,i}\right)^{2}$ and $R^{2}=R_{\infty}^{2}\geq d_{t}^{2}$,
and obtain 
\[
\sum_{t=2}^{T}\left\Vert z_{t-1}-z_{t}\right\Vert ^{2}\leq R_{\infty}^{2}d+\sum_{i=1}^{d}\sum_{t=1}^{T-1}\left(z_{t,i}-z_{t-1,i}\right)^{2}\leq R_{\infty}^{2}d+4R_{\infty}^{2}\sum_{i=1}^{d}\ln\left(\frac{D_{T,i}}{D_{1,i}}\right)=R_{\infty}^{2}d+4R_{\infty}^{2}\sum_{i=1}^{d}\ln\left(D_{T,i}\right)\ .
\]
Thus 
\begin{align*}
(\star) & \leq2GT^{3/2}\sqrt{\sum_{i=1}^{d}4R_{\infty}^{2}\ln\left(D_{T,i}\right)+R_{\infty}^{2}d}=2GT^{3/2}R_{\infty}\sqrt{\sum_{i=1}^{d}4\ln\left(D_{T,i}\right)+d}\ .
\end{align*}
We bound $(\star\star)$ as before: 
\begin{align*}
(\star\star) & =\sum_{t=2}^{T}\frac{1}{2}\left\Vert x^{*}-z_{0}\right\Vert _{D_{t}-D_{t-1}}^{2}\leq\frac{1}{2}R_{\infty}^{2}\left(\tr(D_{T})-\tr(D_{1})\right)\ .
\end{align*}
In the proof of Lemma \ref{lem:error2-acc}, we have shown that

\begin{align*}
(\star\star\star) & =\frac{1}{2}\sum_{t=2}^{T}\left\Vert z_{t}-z_{t-1}\right\Vert _{D_{t-1}}^{2}\geq\frac{\sqrt{2}}{2}R_{\infty}^{2}\left(\tr(D_{T})-\tr(D_{2})\right)\ .
\end{align*}
Putting everything together and using that $\tr(D_{2})\leq\sqrt{2}\tr(D_{1})=\sqrt{2}d$,
\begin{align*}
A_{T}G_{T}-A_{1}G_{1} & \leq2GT^{3/2}R_{\infty}\sqrt{\sum_{i=1}^{d}4\ln\left(D_{T,i}\right)+d}+\frac{1}{2}R_{\infty}^{2}\left(\tr(D_{T})-\tr(D_{1})\right)-\frac{\sqrt{2}}{2}R_{\infty}^{2}\left(\tr(D_{T})-\tr(D_{2})\right)\\
 & \leq4GT^{3/2}R_{\infty}\sqrt{\sum_{i=1}^{d}\ln\left(D_{T,i}\right)}-\frac{\sqrt{2}-1}{2}\sum_{i=1}^{d}D_{T,i}+2GT^{3/2}R_{\infty}\sqrt{d}+\frac{1}{2}R_{\infty}^{2}d\\
 & \leq O\left(\sqrt{d}GT^{3/2}R_{\infty}\sqrt{\ln\left(\frac{GT}{R_{\infty}}\right)}\right)+O\left(R_{\infty}^{2}d\right)\ .
\end{align*}
In the last inequality, we used Lemma \ref{lem:phiz}. Finally, we
bound $A_{1}G_{1}$. Since $y_{1}=z_{1}$ and $a_{1}=A_{1}=1$, we
have

\begin{align*}
A_{1}G_{1} & =U_{1}-L_{1}\\
 & =f(y_{1})-\left(f(x_{1})-\frac{1}{2}\left\Vert x^{*}-z_{0}\right\Vert _{D_{1}}^{2}+\left\langle \nabla f(x_{1}),z_{1}-x_{1}\right\rangle +\frac{1}{2}\left\Vert z_{1}-z_{0}\right\Vert _{D_{1}}^{2}\right)\\
 & =\underbrace{f(y_{1})-f(x_{1})}_{\text{convexity}}-\left\langle \nabla f(x_{1}),y_{1}-x_{1}\right\rangle +\frac{1}{2}\left\Vert x^{*}-z_{0}\right\Vert _{D_{1}}^{2}-\frac{1}{2}\left\Vert z_{1}-z_{0}\right\Vert _{D_{1}}^{2}\\
 & \leq\left\langle \nabla f(y_{1})-\nabla f(x_{1}),y_{1}-x_{1}\right\rangle +\frac{1}{2}\left\Vert x^{*}-z_{0}\right\Vert _{D_{1}}^{2}-\frac{1}{2}\left\Vert z_{1}-z_{0}\right\Vert _{D_{1}}^{2}\\
 & \leq\left\Vert \nabla f(y_{1})-\nabla f(x_{1})\right\Vert \left\Vert y_{1}-x_{1}\right\Vert +\frac{1}{2}\left\Vert x^{*}-z_{0}\right\Vert _{D_{1}}^{2}-\frac{1}{2}\left\Vert z_{1}-z_{0}\right\Vert _{D_{1}}^{2}\\
 & \leq\left(\left\Vert \nabla f(y_{1})\right\Vert +\left\Vert \nabla f(x_{1})\right\Vert \right)\left\Vert y_{1}-x_{1}\right\Vert +\frac{1}{2}\left\Vert x^{*}-z_{0}\right\Vert _{D_{1}}^{2}-\frac{1}{2}\left\Vert z_{1}-z_{0}\right\Vert _{D_{1}}^{2}\\
 & \leq2GR_{\infty}+\frac{1}{2}R_{\infty}^{2}\tr(D_{1})\\
 & =2GR_{\infty}+\frac{1}{2}R_{\infty}^{2}d\ .
\end{align*}
Since $A_{T}=\Theta(T^{2})$, we obtain

\begin{align*}
f(y_{T})-f(x^{*}) & =G_{T}\leq\frac{O\left(\sqrt{d}R_{\infty}G\sqrt{\ln\left(\frac{GT}{R_{\infty}}\right)}\right)T^{3/2}+O\left(R_{\infty}^{2}d\right)}{A_{T}}\\
 & =O\left(\frac{\sqrt{d}R_{\infty}G\sqrt{\ln\left(\frac{GT}{R_{\infty}}\right)}}{\sqrt{T}}+\frac{R_{\infty}^{2}d}{T^{2}}\right)\ .
\end{align*}

\section{Analysis of $\protect\adaagdplus$ in the Stochastic Setting}

\label{sec:analysis-agd+-stoch}

In this section, we extend the $\adaagdplus$ algorithm and its analysis
to the setting where, in each iteration, the algorithm receives a
stochastic gradient $\widetilde{\nabla}f(x_{t})$ that satisfies the
assumptions \eqref{eq:stoch-assumption-unbiased} and \eqref{eq:stoch-assumption-variance}:
$\E\left[\widetilde{\nabla}f(x)\vert x\right]=\nabla f(x)$ and $\E\left[\left\Vert \widetilde{\nabla}f(x)-\nabla f(x)\right\Vert ^{2}\right]\leq\sigma^{2}$.
The algorithm is shown in shown in Figure \ref{alg:acc-stoch}. Note
that we made a minor adjustment to the constant in the update in $D_{t}$.

\begin{figure}
\noindent %
\noindent\fbox{\begin{minipage}[t]{1\columnwidth - 2\fboxsep - 2\fboxrule}%
Let $D_{1}=I$, $z_{0}\in\dom$, $a_{t}=t$, $A_{t}=\sum_{i=1}^{t}a_{i}=\frac{t(t+1)}{2}$,
$R_{\infty}^{2}\geq\max_{x,y\in\dom}\left\Vert x-y\right\Vert _{\infty}^{2}$.

For $t=1,\dots,T$, update:

\begin{align*}
x_{t} & =\frac{A_{t-1}}{A_{t}}y_{t-1}+\frac{a_{t}}{A_{t}}z_{t-1}\ ,\\
z_{t} & =\arg\min_{u\in\dom}\left(\sum_{i=1}^{t}\left\langle a_{i}\widetilde{\nabla}f(x_{i}),u\right\rangle +\frac{1}{2}\left\Vert u-z_{0}\right\Vert _{D_{t}}^{2}\right)\ ,\\
y_{t} & =\frac{A_{t-1}}{A_{t}}y_{t-1}+\frac{a_{t}}{A_{t}}z_{t}\ ,\\
D_{t+1,i}^{2} & =D_{t,i}^{2}\left(1+\frac{\left(z_{t,i}-z_{t-1,i}\right)^{2}}{2R_{\infty}^{2}}\right)\ , & \text{for all }i\in[d].
\end{align*}
Return $y_{T}$.%
\end{minipage}}

\caption{$\protect\adaagdplus$ algorithm with stochastic gradients $\widetilde{\nabla}f(x_{t})$.}
\label{alg:acc-stoch} 
\end{figure}

\subsection{Analysis for Smooth Functions}

\label{subsec:smooth-adaagd+-stoch}

Since most of the analysis follows along the lines of that given in
Section \ref{sec:analysis-acc-smooth}, here we present the differences
introduced by the gradient stochasticity, and how they affect the
convergence. As in Section \ref{sec:analysis-acc-smooth}, we analyze
the convergence of the algorithm using suitable upper and lower bounds
on the optimal function value $f(x^{*})$. We use the same upper bound
as before:

\[
U_{t}:=f(y_{t})\geq f(x^{*})\ .
\]

We modify the lower bound to account for the stochastic gradients
used in the update:

\begin{align*}
f(x^{*}) & \geq\frac{\sum_{i=1}^{t}a_{i}f(x_{i})+\sum_{i=1}^{t}a_{i}\left\langle \nabla f(x_{i}),x^{*}-x_{i}\right\rangle }{A_{t}}\\
 & =\frac{\sum_{i=1}^{t}a_{i}f(x_{i})-\frac{1}{2}\left\Vert x^{*}-z_{0}\right\Vert _{D_{t}}^{2}+\sum_{i=1}^{t}a_{i}\left\langle \widetilde{\nabla}f(x_{i}),x^{*}-x_{i}\right\rangle +\frac{1}{2}\left\Vert x^{*}-z_{0}\right\Vert _{D_{t}}^{2}}{A_{t}}\\
 & +\frac{\sum_{i=1}^{t}a_{i}\left\langle \nabla f(x_{i})-\widetilde{\nabla}f(x_{i}),x^{*}-x_{i}\right\rangle }{A_{t}}\\
 & \geq\frac{\sum_{i=1}^{t}a_{i}f(x_{i})-\frac{1}{2}\left\Vert x^{*}-z_{0}\right\Vert _{D_{t}}^{2}+\min_{u\in\dom}\left\{ \sum_{i=1}^{t}a_{i}\left\langle \widetilde{\nabla}f(x_{i}),u-x_{i}\right\rangle +\frac{1}{2}\left\Vert u-z_{0}\right\Vert _{D_{t}}^{2}\right\} }{A_{t}}\\
 & +\frac{\sum_{i=1}^{t}a_{i}\left\langle \nabla f(x_{i})-\widetilde{\nabla}f(x_{i}),x^{*}-x_{i}\right\rangle }{A_{t}}\\
 & :=L_{t}\ .
\end{align*}
Using this, we obtain a slightly modified version of Lemma \ref{lem:gap1},
which bounds the change in gap between iterations:

\begin{align*}
A_{t}G_{t}-A_{t-1}G_{t-1} & \leq A_{t}\left(f(y_{t})-f(x_{t})\right)+A_{t-1}\left(f(x_{t})-f(y_{t-1})\right)-a_{t}\left\langle \widetilde{\nabla}f(x_{t}),z_{t}-x_{t}\right\rangle \\
 & +\frac{1}{2}\left\Vert x^{*}-z_{0}\right\Vert _{D_{t}-D_{t-1}}^{2}-\frac{1}{2}\left\Vert z_{t}-z_{0}\right\Vert _{D_{t}-D_{t-1}}^{2}-\frac{1}{2}\left\Vert z_{t}-z_{t-1}\right\Vert _{D_{t-1}}^{2}\\
 & +a_{t}\left\langle \nabla f(x_{t})-\widetilde{\nabla}f(x_{t}),x_{t}-x^{*}\right\rangle \ .
\end{align*}

The proof is almost identical to that of Lemma \ref{lem:gap1}. The
difference occurs when tracking the change in the lower bound $L_{t}$.
Here we are being charged differently, since $z_{t}$ is defined using
the history of noisy gradients seen so far. Now we can further upper
bound the change in gap, just like in Lemma \ref{lem:gap2}. We have
to be a bit careful about which terms involve the true gradient, and
which ones involve the noisy gradient.

We now use smoothness to upper bound $f(y_{t})-f(x_{t})$ and convexity
to upper bound $f(x_{t})-f(y_{t-1})$. Together with the definition
of the iteration and $y_{t}-x_{t}=\frac{a_{t}}{A_{t}}\left(z_{t}-z_{t-1}\right)$,
we obtain the following chain of inequalities: 
\begin{align*}
 & A_{t}G_{t}-A_{t-1}G_{t-1}\\
 & \leq A_{t}\underbrace{\left(f(y_{t})-f(x_{t})\right)}_{\text{smoothness}}+A_{t-1}\underbrace{\left(f(x_{t})-f(y_{t-1})\right)}_{\text{convexity}}-a_{t}\left\langle \widetilde{\nabla}f(x_{t}),z_{t}-x_{t}\right\rangle \\
 & +\frac{1}{2}\left\Vert x^{*}-z_{0}\right\Vert _{D_{t}-D_{t-1}}^{2}-\frac{1}{2}\left\Vert z_{t}-z_{0}\right\Vert _{D_{t}-D_{t-1}}^{2}-\frac{1}{2}\left\Vert z_{t}-z_{t-1}\right\Vert _{D_{t-1}}^{2}\\
 & +a_{t}\left\langle \nabla f(x_{t})-\widetilde{\nabla}f(x_{t}),x_{t}-x^{*}\right\rangle \\
 & \leq A_{t}\left(\left\langle \nabla f(x_{t}),y_{t}-x_{t}\right\rangle +\frac{1}{2}\left\Vert y_{t}-x_{t}\right\Vert _{\sm}^{2}\right)+A_{t-1}\left\langle \nabla f(x_{t}),x_{t}-y_{t-1}\right\rangle -a_{t}\left\langle \widetilde{\nabla}f(x_{t}),z_{t}-x_{t}\right\rangle \\
 & +\frac{1}{2}\left\Vert x^{*}-z_{0}\right\Vert _{D_{t}-D_{t-1}}^{2}-\frac{1}{2}\left\Vert z_{t}-z_{0}\right\Vert _{D_{t}-D_{t-1}}^{2}-\frac{1}{2}\left\Vert z_{t}-z_{t-1}\right\Vert _{D_{t-1}}^{2}\\
 & +a_{t}\left\langle \nabla f(x_{t})-\widetilde{\nabla}f(x_{t}),x_{t}-x^{*}\right\rangle \\
 & =\left\langle \widetilde{\nabla}f(x_{t}),\underbrace{A_{t}\left(y_{t}-x_{t}\right)+A_{t-1}\left(x_{t}-y_{t-1}\right)+a_{t}\left(x_{t}-z_{t}\right)}_{=0}\right\rangle \\
 & +\underbrace{\frac{1}{2}A_{t}\left\Vert y_{t}-x_{t}\right\Vert _{\sm}^{2}}_{=\frac{1}{2}\frac{a_{t}^{2}}{A_{t}}\left\Vert z_{t}-z_{t-1}\right\Vert _{\sm}^{2}\leq\left\Vert z_{t}-z_{t-1}\right\Vert _{\sm}^{2}}+\frac{1}{2}\left\Vert x^{*}-z_{0}\right\Vert _{D_{t}-D_{t-1}}^{2}-\frac{1}{2}\left\Vert z_{t}-z_{0}\right\Vert _{D_{t}-D_{t-1}}^{2}-\frac{1}{2}\left\Vert z_{t}-z_{t-1}\right\Vert _{D_{t-1}}^{2}\\
 & +\left\langle \nabla f(x_{t})-\widetilde{\nabla}f(x_{t}),a_{t}\left(x_{t}-x^{*}\right)+A_{t}\left(y_{t}-x_{t}\right)+A_{t-1}\left(x_{t}-y_{t-1}\right)\right\rangle \\
 & \leq\left\Vert z_{t}-z_{t-1}\right\Vert _{\sm}^{2}+\frac{1}{2}\left\Vert x^{*}-z_{0}\right\Vert _{D_{t}-D_{t-1}}^{2}-\frac{1}{2}\left\Vert z_{t}-z_{0}\right\Vert _{D_{t}-D_{t-1}}^{2}-\frac{1}{2}\left\Vert z_{t}-z_{t-1}\right\Vert _{D_{t-1}}^{2}\\
 & +\left\langle \nabla f(x_{t})-\widetilde{\nabla}f(x_{t}),a_{t}\left(x_{t}-x^{*}\right)+A_{t}\left(y_{t}-x_{t}\right)+A_{t-1}\left(x_{t}-y_{t-1}\right)\right\rangle 
\end{align*}

To shorten notation let $\xi_{t}=\nabla f(x_{t})-\widetilde{\nabla}f(x_{t})$.
Note that compared to the deterministic case, the change in gap contains
the following additional term: 
\begin{align*}
 & \left\langle \xi_{t},a_{t}\left(x_{t}-x^{*}\right)+A_{t}\left(y_{t}-x_{t}\right)+A_{t-1}\left(x_{t}-y_{t-1}\right)\right\rangle \\
 & =\left\langle \xi_{t},A_{t}\left(y_{t}-x_{t}\right)\right\rangle +\left\langle \xi_{t},a_{t}\left(x_{t}-x^{*}\right)+A_{t-1}\left(x_{t}-y_{t-1}\right)\right\rangle \\
 & =\left\langle \xi_{t},a_{t}\left(z_{t}-z_{t-1}\right)\right\rangle +\left\langle \xi_{t},a_{t}\left(x_{t}-x^{*}\right)+A_{t-1}\left(x_{t}-y_{t-1}\right)\right\rangle \,.
\end{align*}

On the last line, we have used that $y_{t}-x_{t}=\frac{a_{t}}{A_{t}}\left(z_{t}-z_{t-1}\right)$.

Plugging in into the previous inequality, we obtain

\begin{align*}
 & A_{t}G_{t}-A_{t-1}G_{t-1}\\
 & \leq\left\Vert z_{t}-z_{t-1}\right\Vert _{\sm}^{2}+\frac{1}{2}\left\Vert x^{*}-z_{0}\right\Vert _{D_{t}-D_{t-1}}^{2}-\frac{1}{2}\left\Vert z_{t}-z_{0}\right\Vert _{D_{t}-D_{t-1}}^{2}-\frac{1}{2}\left\Vert z_{t}-z_{t-1}\right\Vert _{D_{t-1}}^{2}\\
 & +\left\langle \xi_{t},a_{t}\left(z_{t}-z_{t-1}\right)\right\rangle +\left\langle \xi_{t},a_{t}\left(x_{t}-x^{*}\right)+A_{t-1}\left(x_{t}-y_{t-1}\right)\right\rangle 
\end{align*}
By telescoping the terms via the analysis from Section \ref{sec:analysis-acc-smooth},
and separating those involving $\xi_{t}$, we obtain: 
\begin{align*}
A_{T}G_{T}-A_{1}G_{1} & \leq\underbrace{\sum_{t=2}^{T}\left\Vert z_{t}-z_{t-1}\right\Vert _{\sm}^{2}-\left(\frac{1}{2}-\frac{2}{3\sqrt{2}}\right)\sum_{t=2}^{T}\left\Vert z_{t}-z_{t-1}\right\Vert _{D_{t-1}}^{2}}_{(\star)}\\
 & +\underbrace{\frac{1+1/3}{2}R_{\infty}^{2}\left(\tr(D_{T})-\tr(D_{1})\right)-\frac{2}{3\sqrt{2}}\sum_{t=2}^{T}\left\Vert z_{t}-z_{t-1}\right\Vert _{D_{t-1}}^{2}}_{(\star\star)}\\
 & +\underbrace{\left(\sum_{t=2}^{T}\left\langle \xi_{t},a_{t}\left(z_{t}-z_{t-1}\right)\right\rangle -\frac{1}{6}R_{\infty}^{2}\tr(D_{T})+\frac{1}{6}R_{\infty}^{2}\tr(D_{1})\right)}_{(\diamond)}\\
 & +\underbrace{\sum_{t=2}^{T}\left\langle \xi_{t},a_{t}\left(x_{t}-x^{*}\right)+A_{t-1}\left(x_{t}-y_{t-1}\right)\right\rangle }_{(\diamond\diamond)}\ .
\end{align*}
Following the proofs from Lemmas \ref{lem:error1-acc} and \ref{lem:error2-acc}
we bound $(\star)\leq O\left(R_{\infty}^{2}\sum_{i=1}^{d}\beta_{i}\ln\left(2\beta_{i}\right)\right)$
and $(\star\star)\leq O(R_{\infty}^{2}\tr(D_{1}))$.

To bound $(\diamond)$, similarly to \citep{BachL19}, we apply Cauchy-Schwarz
twice and obtain 
\begin{align*}
\sum_{t=2}^{T}\left\langle \xi_{t},a_{t}\left(z_{t}-z_{t-1}\right)\right\rangle  & \leq\sum_{t=1}^{T}a_{t}\left\Vert \xi_{t}\right\Vert \left\Vert z_{t}-z_{t-1}\right\Vert \leq\sqrt{\left(\sum_{t=1}^{T}a_{t}^{2}\left\Vert \xi_{t}\right\Vert ^{2}\right)\left(\sum_{t=1}^{T}\left\Vert z_{t}-z_{t-1}\right\Vert ^{2}\right)}\ .
\end{align*}
By applying Lemma \ref{lem:inequalities} separately for each coordinate,
with $d_{t}^{2}=\left(z_{t,i}-z_{t-1,i}\right)^{2}\leq R_{\infty}^{2}$
and $R^{2}=2R_{\infty}^{2}$, we obtain 
\begin{align}
\sum_{t=1}^{T}\left\Vert z_{t}-z_{t-1}\right\Vert ^{2} & \leq R_{\infty}^{2}\sum_{i=1}^{d}\left(1+8\ln\left(D_{T,i}\right)\right)\,.\label{eq:sumznorms}
\end{align}
Plugging in and using Lemma \ref{lem:phiz}, we obtain 
\begin{align*}
(\diamond) & \leq\sqrt{\left(\sum_{t=1}^{T}a_{t}^{2}\left\Vert \xi_{t}\right\Vert ^{2}\right)\cdot R_{\infty}^{2}\sum_{i=1}^{d}\left(1+8\ln\left(D_{T,i}\right)\right)}-\frac{1}{6}R_{\infty}^{2}\tr(D_{T})+\frac{1}{6}R_{\infty}^{2}\tr(D_{1})\\
 & \leq\sqrt{\left(\sum_{t=1}^{T}a_{t}^{2}\left\Vert \xi_{t}\right\Vert ^{2}\right)\cdot R_{\infty}^{2}d}+\frac{R_{\infty}^{2}}{6}\sqrt{\frac{288}{R_{\infty}^{2}}\left(\sum_{t=1}^{T}a_{t}^{2}\left\Vert \xi_{t}\right\Vert ^{2}\right)\cdot\left(\sum_{i=1}^{d}\ln\left(D_{T,i}\right)\right)}-\frac{R_{\infty}^{2}}{6}\tr(D_{T})+\frac{1}{6}R_{\infty}^{2}\tr(D_{1})\\
 & \leq\sqrt{\left(\sum_{t=1}^{T}a_{t}^{2}\left\Vert \xi_{t}\right\Vert ^{2}\right)\cdot R_{\infty}^{2}d}+\frac{R_{\infty}^{2}}{6}\sqrt{d}\sqrt{\frac{288}{R_{\infty}^{2}}\left(\sum_{t=1}^{T}a_{t}^{2}\left\Vert \xi_{t}\right\Vert ^{2}\right)\cdot\ln\left(\frac{288}{R_{\infty}^{2}}\left(\sum_{t=1}^{T}a_{t}^{2}\left\Vert \xi_{t}\right\Vert ^{2}\right)\right)}+\frac{1}{6}R_{\infty}^{2}\tr(D_{1})\ .
\end{align*}
Next, we take expectation and use the fact that $\sqrt{x}$ and $\sqrt{x\ln x}$
are concave, $a_{t}=t\leq T$, and the assumption $\E\left[\left\Vert \xi_{t}\right\Vert ^{2}\right]\leq\sigma^{2}$.
We obtain 
\begin{align*}
\mathbb{E}\left[(\diamond)\right] & \leq\sqrt{\left(\sum_{t=1}^{T}a_{t}^{2}\sigma^{2}\right)\cdot R_{\infty}^{2}d}+\frac{R_{\infty}^{2}}{6}\sqrt{d}\sqrt{\frac{288}{R_{\infty}^{2}}\left(\sum_{t=1}^{T}a_{t}^{2}\sigma^{2}\right)\cdot\ln\left(\frac{288}{R_{\infty}^{2}}\left(\sum_{t=1}^{T}a_{t}^{2}\sigma^{2}\right)\right)}+\frac{1}{6}R_{\infty}^{2}\tr(D_{1})\\
 & \leq O\left(R_{\infty}\sigma T^{3/2}\sqrt{d\ln\left(\frac{T\sigma}{R_{\infty}}\right)}\right)+\frac{1}{6}R_{\infty}^{2}\tr(D_{1})\ .
\end{align*}
Also, by assumption \eqref{eq:stoch-assumption-unbiased}, we have

\[
\mathbb{E}\left[\left\langle \xi_{t},a_{t}\left(x_{t}-x^{*}\right)+A_{t-1}\left(x_{t}-y_{t-1}\right)\right\rangle \vert x_{t}\right]=0\ .
\]
Taking expectation over the entire history we obtain that 
\[
\E\left[(\diamond\diamond)\right]=\mathbb{E}\left[\sum_{t=1}^{T}\left\langle \xi_{t},a_{t}\left(x_{t}-x^{*}\right)+A_{t-1}\left(x_{t}-y_{t-1}\right)\right\rangle \right]=0\ .
\]
Putting everything together, we obtain 
\[
\mathbb{E}\left[A_{T}G_{T}-A_{1}G_{1}\right]\leq O\left(R_{\infty}^{2}\sum_{i=1}^{d}\beta_{i}\ln\left(2\beta_{i}\right)\right)+O\left(R_{\infty}\sigma T^{3/2}\sqrt{d\ln\left(\frac{T\sigma}{R_{\infty}}\right)}\right)+\frac{1}{6}R_{\infty}^{2}\tr(D_{1})\ .
\]
Finally we bound $\mathbb{E}\left[A_{1}G_{1}\right]$, which per Lemma
\ref{lem:exp_a1g1} satisfies 
\[
\mathbb{E}\left[A_{1}G_{1}\right]\leq O\left(R_{\infty}^{2}\tr(\sm)+R_{\infty}^{2}d\right)+\sigma R_{\infty}\sqrt{d}\ .
\]
Therefore, since by definition $A_{T}=\Theta(T^{2})$, we have that
\[
\mathbb{E}\left[f\left(y_{T}\right)-f\left(x^{*}\right)\right]=O\left(\frac{R_{\infty}^{2}\sum_{i=1}^{d}\beta_{i}\ln\left(2\beta_{i}\right)}{T^{2}}+\frac{R_{\infty}\sigma\sqrt{d\ln\left(\frac{T\sigma}{R_{\infty}}\right)}}{\sqrt{T}}\right)\ .
\]

\begin{lem}
\label{lem:exp_a1g1}We have 
\begin{align*}
\mathbb{E}\left[A_{1}G_{1}\right] & =O\left(R_{\infty}^{2}\tr(\sm)+R_{\infty}^{2}d\right)+\sigma R_{\infty}\sqrt{d}\ .
\end{align*}
\end{lem}

\begin{proof}
We have $a_{1}=A_{1}=1$ and $y_{1}=z_{1}$. By definition, $z_{1}=\arg\min_{u\in\dom}\left\{ \left\langle \widetilde{\nabla}f(x_{1}),u-x_{1}\right\rangle +\frac{1}{2}\left\Vert u-z_{0}\right\Vert _{D_{1}}^{2}\right\} $.
Thus

\begin{align*}
A_{1}G_{1} & =G_{1}=U_{1}-L_{1}\\
 & =f(z_{1})-\left(f(x_{1})-\frac{1}{2}\left\Vert x^{*}-z_{0}\right\Vert _{D_{1}}^{2}+\left\langle \widetilde{\nabla}f(x_{1}),z_{1}-x_{1}\right\rangle +\frac{1}{2}\left\Vert z_{1}-z_{0}\right\Vert _{D_{1}}^{2}\right)\\
 & -\left\langle \nabla f(x_{1})-\widetilde{\nabla}f(x_{1}),x^{*}-x_{1}\right\rangle \\
 & =\underbrace{f(z_{1})-f(x_{1})-\left\langle \nabla f(x_{1}),z_{1}-x_{1}\right\rangle }_{\text{smoothness}}+\frac{1}{2}\left\Vert x^{*}-z_{0}\right\Vert _{D_{1}}^{2}-\frac{1}{2}\left\Vert z_{1}-z_{0}\right\Vert _{D_{1}}^{2}\\
 & +\underbrace{\left\langle \nabla f(x_{1})-\widetilde{\nabla}f(x_{1}),z_{1}-x^{*}\right\rangle }_{\text{Cauchy-Schwarz}}\\
 & \leq\frac{1}{2}\left\Vert z_{1}-x_{1}\right\Vert _{\sm}^{2}+\frac{1}{2}\left\Vert x^{*}-z_{0}\right\Vert _{D_{1}}^{2}+\left\Vert \nabla f(x_{1})-\widetilde{\nabla}f(x_{1})\right\Vert \left\Vert z_{1}-x^{*}\right\Vert \\
 & \leq\frac{1}{2}R_{\infty}^{2}\tr\left(\sm\right)+\frac{1}{2}R_{\infty}^{2}\tr\left(D_{1}\right)+\left\Vert \nabla f(x_{1})-\widetilde{\nabla}f(x_{1})\right\Vert \cdot R_{\infty}\sqrt{d}\\
 & =\frac{1}{2}R_{\infty}^{2}\tr\left(\sm\right)+\frac{1}{2}R_{\infty}^{2}\tr\left(D_{1}\right)+\left(\sqrt{\frac{R_{\infty}\sqrt{d}}{\sigma}}\left\Vert \nabla f(x_{1})-\widetilde{\nabla}f(x_{1})\right\Vert \right)\cdot\left(\sqrt{\sigma R_{\infty}\sqrt{d}}\right)\\
 & \leq\frac{1}{2}R_{\infty}^{2}\tr\left(\sm\right)+\frac{1}{2}R_{\infty}^{2}\tr\left(D_{1}\right)+\frac{1}{2}\frac{R_{\infty}\sqrt{d}}{\sigma}\left\Vert \nabla f(x_{1})-\widetilde{\nabla}f(x_{1})\right\Vert ^{2}+\frac{1}{2}\sigma R_{\infty}\sqrt{d}
\end{align*}

In the last inequality, we have used the inequality $ab\leq\frac{1}{2}a^{2}+\frac{1}{2}b^{2}$.

Taking expectation and using the assumption $\E\left[\left\Vert \nabla f(x_{1})-\widetilde{\nabla}f(x_{1})\right\Vert ^{2}\right]\leq\sigma^{2}$,
we obtain

\[
\E\left[A_{1}G_{1}\right]\leq\frac{1}{2}R_{\infty}^{2}\tr\left(\sm\right)+\frac{1}{2}R_{\infty}^{2}\tr\left(D_{1}\right)+\sigma R_{\infty}\sqrt{d}
\]
\end{proof}

\subsection{Analysis for Non-smooth Functions}

The analysis is an extension of the analysis in Section \ref{sec:analysis-acc-nonsmooth},
and it mainly consists of bounding the additional error term arising
from stochasticity as in the previous section.

To track the effect of using stochastic gradients, we follow the initial
part of the analysis from Section \ref{subsec:smooth-adaagd+-stoch}
that uses only convexity, up to the point where we bound:

\begin{align*}
 & A_{t}G_{t}-A_{t-1}G_{t-1}\\
 & \leq A_{t}\underbrace{\left(f(y_{t})-f(x_{t})\right)}_{\text{convexity}}+A_{t-1}\underbrace{\left(f(x_{t})-f(y_{t-1})\right)}_{\text{convexity}}-a_{t}\left\langle \widetilde{\nabla}f(x_{t}),z_{t}-x_{t}\right\rangle \\
 & +\frac{1}{2}\left\Vert x^{*}-z_{0}\right\Vert _{D_{t}-D_{t-1}}^{2}-\frac{1}{2}\left\Vert z_{t}-z_{0}\right\Vert _{D_{t}-D_{t-1}}^{2}-\frac{1}{2}\left\Vert z_{t}-z_{t-1}\right\Vert _{D_{t-1}}^{2}\\
 & +a_{t}\left\langle \nabla f(x_{t})-\widetilde{\nabla}f(x_{t}),x_{t}-x^{*}\right\rangle \\
 & \leq A_{t}\left\langle \nabla f(y_{t}),y_{t}-x_{t}\right\rangle +A_{t-1}\left\langle \nabla f(x_{t}),x_{t}-y_{t-1}\right\rangle -a_{t}\left\langle \widetilde{\nabla}f(x_{t}),z_{t}-x_{t}\right\rangle \\
 & +\frac{1}{2}\left\Vert x^{*}-z_{0}\right\Vert _{D_{t}-D_{t-1}}^{2}-\frac{1}{2}\left\Vert z_{t}-z_{0}\right\Vert _{D_{t}-D_{t-1}}^{2}-\frac{1}{2}\left\Vert z_{t}-z_{t-1}\right\Vert _{D_{t-1}}^{2}\\
 & +a_{t}\left\langle \nabla f(x_{t})-\widetilde{\nabla}f(x_{t}),x_{t}-x^{*}\right\rangle \\
 & =A_{t}\left\langle \nabla f(y_{t})-\nabla f(x_{t}),y_{t}-x_{t}\right\rangle +\left\langle \nabla f(x_{t}),\underbrace{A_{t}\left(y_{t}-x_{t}\right)+A_{t-1}\left(x_{t}-y_{t-1}\right)+a_{t}\left(x_{t}-z_{t}\right)}_{=0}\right\rangle \\
 & +\frac{1}{2}\left\Vert x^{*}-z_{0}\right\Vert _{D_{t}-D_{t-1}}^{2}-\frac{1}{2}\left\Vert z_{t}-z_{0}\right\Vert _{D_{t}-D_{t-1}}^{2}-\frac{1}{2}\left\Vert z_{t}-z_{t-1}\right\Vert _{D_{t-1}}^{2}\\
 & +a_{t}\left\langle \nabla f(x_{t})-\widetilde{\nabla}f(x_{t}),z_{t}-x^{*}\right\rangle \ .
\end{align*}
To shorten notation, we let $\xi_{t}=\nabla f(x_{t})-\widetilde{\nabla}f(x_{t})$.
From here on the proof is almost identical to the one from Section
\ref{sec:analysis-acc-nonsmooth}, thus showing that

\begin{align*}
A_{T}G_{T}-A_{1}G_{1} & \leq\underbrace{2G\sum_{t=2}^{T}a_{t}\left\Vert z_{t-1}-z_{t}\right\Vert +\sum_{t=2}^{T}\frac{1}{2}\left\Vert x^{*}-z_{0}\right\Vert _{D_{t}-D_{t-1}}^{2}-\sum_{t=2}^{T}\frac{1}{4}\left\Vert z_{t}-z_{t-1}\right\Vert _{D_{t-1}}^{2}}_{(\square)}\\
 & +\underbrace{\sum_{t=2}^{T}a_{t}\left\langle \xi_{t},z_{t}-z_{t-1}\right\rangle -\sum_{t=2}^{T}\frac{1}{4}\left\Vert z_{t}-z_{t-1}\right\Vert _{D_{t-1}}^{2}}_{(\diamond)}\\
 & +\underbrace{\sum_{t=2}^{T}a_{t}\left\langle \xi_{t},z_{t-1}-x^{*}\right\rangle }_{(\diamond\diamond)}\ .
\end{align*}
Note that unlike in the original analysis, here we broke $\sum_{t=2}^{T}\frac{1}{2}\left\Vert z_{t}-z_{t-1}\right\Vert _{D_{t-1}}^{2}$
in two components, one of which we use to control $(\diamond)$. Using
a similar analysis to the one in Section \ref{sec:analysis-acc-nonsmooth},
we bound 
\[
(\square)=O\left(\sqrt{d}R_{\infty}GT^{3/2}\sqrt{\ln\left(\frac{GT}{R_{\infty}}\right)}\right)+O\left(R_{\infty}^{2}d\right)\ .
\]
For each coordinate separately, we apply Lemma \ref{lem:inequalities}
with $d_{t}^{2}=\left(z_{t,i}-z_{t-1,i}\right)^{2}$ and $R^{2}=2R_{\infty}^{2}$,
and obtain

\begin{align*}
\sum_{t=2}^{T}\left\Vert z_{t-1}-z_{t}\right\Vert _{D_{t}}^{2} & \geq4R_{\infty}^{2}\left(\tr\left(D_{T+1}\right)-\tr\left(D_{2}\right)\right)\ ,\\
\sum_{t=2}^{T}\left\Vert z_{t-1}-z_{t}\right\Vert ^{2} & \leq8R_{\infty}^{2}\sum_{i=1}^{d}\ln\left(D_{T+1,i}\right)\ .
\end{align*}
Since $D_{t}\leq\sqrt{2}\cdot D_{t-1}$ by definition, this also gives
that 
\[
\sum_{t=2}^{T}\left\Vert z_{t-1}-z_{t}\right\Vert _{D_{t-1}}^{2}\geq\frac{4}{\sqrt{2}}R_{\infty}^{2}\left(\tr\left(D_{T+1}\right)-\tr\left(D_{2}\right)\right)\ .
\]
By twice applying Cauchy-Schwarz, and applying the bound on the variance,
we now bound: 
\begin{align*}
(\diamond) & =\sum_{t=2}^{T}a_{t}\left\langle \xi_{t},z_{t}-z_{t-1}\right\rangle -\sum_{t=2}^{T}\frac{1}{4}\left\Vert z_{t}-z_{t-1}\right\Vert _{D_{t-1}}^{2}\\
 & \leq\sum_{t=2}^{T}a_{t}\left\Vert \xi_{t}\right\Vert \left\Vert z_{t}-z_{t-1}\right\Vert -\sum_{t=2}^{T}\frac{1}{4}\left\Vert z_{t}-z_{t-1}\right\Vert _{D_{t-1}}^{2}\\
 & \leq\sqrt{\left(\sum_{t=1}^{T}a_{t}^{2}\left\Vert \xi_{t}\right\Vert ^{2}\right)\left(\sum_{t=2}^{T}\left\Vert z_{t}-z_{t-1}\right\Vert ^{2}\right)}-\sum_{t=2}^{T}\frac{1}{4}\left\Vert z_{t}-z_{t-1}\right\Vert _{D_{t-1}}^{2}\\
 & \leq\sqrt{\sum_{t=1}^{T}a_{t}^{2}\left\Vert \xi_{t}\right\Vert ^{2}\left(\sum_{t=2}^{T}\left\Vert z_{t}-z_{t-1}\right\Vert ^{2}\right)}-\sum_{t=2}^{T}\frac{1}{4}\left\Vert z_{t}-z_{t-1}\right\Vert _{D_{t-1}}^{2}\\
 & \leq\sqrt{\sum_{t=1}^{T}a_{t}^{2}\left\Vert \xi_{t}\right\Vert ^{2}\cdot8R_{\infty}^{2}\sum_{i=1}^{d}\ln\left(D_{T+1,i}\right)}-\frac{1}{\sqrt{2}}R_{\infty}^{2}\left(\tr\left(D_{T+1}\right)-\tr\left(D_{2}\right)\right)\\
 & =\frac{R_{\infty}^{2}}{\sqrt{2}}\left(\sqrt{\sum_{t=1}^{T}a_{t}^{2}\left\Vert \xi_{t}\right\Vert ^{2}\cdot\frac{16}{R_{\infty}^{2}}\cdot\sum_{i=1}^{d}\ln\left(D_{T+1,i}\right)}-\tr\left(D_{T+1}\right)\right)+\frac{1}{\sqrt{2}}R_{\infty}^{2}\tr\left(D_{2}\right)\\
 & \leq\frac{R_{\infty}^{2}}{\sqrt{2}}\sqrt{d}\cdot\sqrt{\frac{1}{2}\sum_{t=1}^{T}a_{t}^{2}\left\Vert \xi_{t}\right\Vert ^{2}\cdot\frac{16}{R_{\infty}^{2}}\cdot\ln\left(\sum_{t=1}^{T}a_{t}^{2}\left\Vert \xi_{t}\right\Vert ^{2}\cdot\frac{16}{R_{\infty}^{2}}\right)}+R_{\infty}^{2}d\\
 & =2R_{\infty}\sqrt{d}\cdot\sqrt{\sum_{t=1}^{T}a_{t}^{2}\left\Vert \xi_{t}\right\Vert ^{2}\cdot\ln\left(\sum_{t=1}^{T}a_{t}^{2}\left\Vert \xi_{t}\right\Vert ^{2}\cdot\frac{16}{R_{\infty}^{2}}\right)}+R_{\infty}^{2}d\ ,
\end{align*}
In the last inequality, we used Lemma \ref{lem:phiz} and $\tr(D_{2})\leq\sqrt{2}\tr(D_{1})=\sqrt{2}d$.
Taking the expectation and applying the concavity of $\sqrt{x\ln x}$
we obtain that 
\begin{align*}
\mathbb{E}\left[(\diamond)\right] & \leq\mathbb{E}\left[2R_{\infty}\sqrt{d}\cdot\sqrt{\sum_{t=1}^{T}a_{t}^{2}\left\Vert \xi_{t}\right\Vert ^{2}\cdot\ln\left(\sum_{t=1}^{T}a_{t}^{2}\left\Vert \xi_{t}\right\Vert ^{2}\cdot\frac{16}{R_{\infty}^{2}}\right)}+R_{\infty}^{2}d\right]\\
 & \leq2R_{\infty}\sqrt{d}\cdot\sqrt{\mathbb{E}\left[\sum_{t=1}^{T}a_{t}^{2}\left\Vert \xi_{t}\right\Vert ^{2}\right]\cdot\ln\left(\mathbb{E}\left[\sum_{t=1}^{T}a_{t}^{2}\left\Vert \xi_{t}\right\Vert ^{2}\right]\cdot\frac{16}{R_{\infty}^{2}}\right)}+R_{\infty}^{2}d\\
 & \leq2R_{\infty}\sqrt{d}\cdot\sqrt{T^{3}\sigma^{2}\cdot\ln\left(T^{3}\sigma^{2}\cdot\frac{16}{R_{\infty}^{2}}\right)}+R_{\infty}^{2}d\\
 & =O\left(R_{\infty}\sqrt{d}\cdot T^{3/2}\sigma\cdot\sqrt{\ln\left(\frac{T\sigma}{R_{\infty}}\right)}+R_{\infty}^{2}d\right)\ .
\end{align*}
By assumption \eqref{eq:stoch-assumption-unbiased} we have 
\[
\mathbb{E}\left[\left\langle \nabla f(x_{t})-\widetilde{\nabla}f(x_{t}),x_{t}-x^{*}\right\rangle \bigg\vert x_{t}\right]=0\ ,
\]
Taking expectation over the entire history we obtain that 
\[
\E\left[(\diamond\diamond)\right]=\mathbb{E}\left[\sum_{t=1}^{T}a_{t}\left\langle \nabla f(x_{t})-\widetilde{\nabla}f(x_{t}),x_{t}-x^{*}\right\rangle \right]=0\ .
\]
Finally we upper bound $\mathbb{E}\left[A_{1}G_{1}\right]$. We bound,
exactly as in the proof of Lemma \ref{lem:exp_a1g1}: 
\begin{align*}
G_{1} & \leq f(y_{1})-f(x_{1})+\frac{1}{2}\left\Vert x^{*}-z_{0}\right\Vert _{D_{1}}^{2}+\left\langle \nabla f(x_{1}),z_{1}-x_{1}\right\rangle -\left\langle \nabla f(x_{1})-\widetilde{\nabla}f(x_{1}),x^{*}-z_{1}\right\rangle \ .
\end{align*}
After which we apply standard inequalities to obtain: 
\begin{align*}
G_{1} & \leq\left\langle \nabla f(y_{1}),y_{1}-x_{1}\right\rangle +\frac{1}{2}\left\Vert x^{*}-z_{0}\right\Vert _{D_{1}}^{2}+\left\langle \nabla f(x_{1}),z_{1}-x_{1}\right\rangle -\left\langle \nabla f(x_{1})-\widetilde{\nabla}f(x_{1}),x^{*}-z_{1}\right\rangle \\
 & \leq\left\Vert \nabla f(y_{1})\right\Vert \left\Vert y_{1}-x_{1}\right\Vert +\frac{1}{2}\left\Vert x^{*}-z_{0}\right\Vert _{D_{1}}^{2}+\left\Vert \nabla f(x_{1})\right\Vert \left\Vert z_{1}-x_{1}\right\Vert +\left\Vert \nabla f(x_{1})-\widetilde{\nabla}f(x_{1})\right\Vert \left\Vert x^{*}-z_{1}\right\Vert \\
 & \leq GR_{\infty}d^{1/2}+\frac{1}{2}R_{\infty}^{2}d+GR_{\infty}d^{1/2}+\left\Vert \nabla f(x_{1})-\widetilde{\nabla}f(x_{1})\right\Vert \cdot R_{\infty}d^{1/2}\\
 & =GR_{\infty}d^{1/2}+\frac{1}{2}R_{\infty}^{2}d+GR_{\infty}d^{1/2}+\left(\sqrt{\frac{R_{\infty}d^{1/2}}{\sigma}}\left\Vert \nabla f(x_{1})-\widetilde{\nabla}f(x_{1})\right\Vert \right)\cdot\left(\sqrt{\sigma R_{\infty}d^{1/2}}\right)\\
 & \leq GR_{\infty}d^{1/2}+\frac{1}{2}R_{\infty}^{2}d+GR_{\infty}d^{1/2}+\frac{1}{2}\frac{R_{\infty}d^{1/2}}{\sigma}\left\Vert \nabla f(x_{1})-\widetilde{\nabla}f(x_{1})\right\Vert ^{2}+\frac{1}{2}\sigma R_{\infty}d^{1/2}\ .
\end{align*}

In the last inequality, we have used the inequality $ab\leq\frac{1}{2}a^{2}+\frac{1}{2}b^{2}$.

Taking the expectation and using the assumption $\E\left[\left\Vert \nabla f(x_{1})-\widetilde{\nabla}f(x_{1})\right\Vert ^{2}\right]\leq\sigma^{2}$,
we obtain 
\[
\mathbb{E}\left[A_{1}G_{1}\right]=O\left(GR_{\infty}d^{1/2}+R_{\infty}^{2}d+\sigma R_{\infty}d^{1/2}\right)\ .
\]
Combining with the rest we get that 
\begin{align*}
\mathbb{E}\left[f\left(y_{T}\right)-f\left(x^{*}\right)\right] & =O\left(\frac{R_{\infty}\sqrt{d}\cdot G\sqrt{\ln\left(\frac{GT}{R_{\infty}}\right)}+\sigma R_{\infty}\sqrt{d}\sqrt{\ln\left(\frac{T\sigma}{R_{\infty}}\right)}}{\sqrt{T}}+\frac{R_{\infty}^{2}d}{T^{2}}\right)\ .
\end{align*}

\section{Analysis of $\protect\adagrad$ for Unconstrained Convex Optimization}

\label{sec:adagrad-unconstrained}

Here we provide a sharper analysis that saves the $\log\left(2\beta_{i}\right)$
factor for smooth functions of the $\adagrad$ scheme \citep{duchi2011adaptive,McMahanS10}
in the unconstrained setting $\dom=\R^{d}$ (Figure \ref{alg:adagrad-unconstrained}).
The analysis we provide is a generalization to the vector setting
of the analysis in \citep{levy2018online}.

\begin{figure}
\noindent %
\noindent\fbox{\begin{minipage}[t]{1\columnwidth - 2\fboxsep - 2\fboxrule}%
Let $x_{0}\in\R^{d}$, $D_{0}=I$.

For $t=0,\dots,T-2$, update: 
\begin{align*}
x_{t+1} & =x_{t}-\eta D_{t}^{-1}\nabla f(x_{t})\ ,\\
D_{t+1,i}^{2} & =D_{t,i}^{2}+\left(\nabla_{i}f(x_{t+1})\right)^{2}\ , & \text{for all }i\in[d]\ .
\end{align*}
Return $\overline{x}_{T}=\frac{1}{T}\sum_{t=0}^{T-1}x_{t}$.%
\end{minipage}}

\caption{$\protect\adagrad$ algorithm \citep{duchi2011adaptive,McMahanS10}.}
\label{alg:adagrad-unconstrained} 
\end{figure}

We note that there is a small difference between the above algorithm
and the algorithm that we obtain by specializing $\adagradplus$ to
the unconstrained setting. The difference is in the definition of
the scaling:

\begin{align*}
\adagrad & :D_{t,i}^{2}=1+\sum_{s=0}^{t}\left(\nabla_{i}f(x_{s})\right)^{2}\\
\adagradplus & :D_{t,i}^{2}=1+\sum_{s=0}^{t-1}\left(\nabla_{i}f(x_{s})\right)^{2}
\end{align*}

In the constrained setting, we are forced to use a step that is ``off-by-one,''
i.e., it does not include the most recent iterate movement. However,
as we noted earlier, our update ensures that $D_{t+1,i}^{2}\leq2D_{t,i}^{2}$,
which allows us to deal with this complication very cleanly in our
analysis.

In the remainder of the section, we analyze the $\adagrad$ algorithm
in the smooth setting and show the following guarantee. We emphasize
that the algorithm does not know the smoothness parameters and it
automatically adapts to them. 
\begin{thm}
\label{thm:adagrad-smooth-unconstrained}Let $f$ be a convex function
that is $1$-smooth with respect to the norm $\left\Vert \cdot\right\Vert _{\sm}$,
where $\sm=\diag(\beta_{1},\dots,\beta_{d})$ is a diagonal matrix
with $\beta_{1},\dots,\beta_{d}\geq1$. Let $x^{*}\in\arg\min_{x\in\mathbb{R}^{d}}f(x)$.
Let $x_{t}$ be the iterates constructed by the algorithm in Figure
\ref{alg:adagrad-unconstrained} and let $\overline{x}_{T}=\frac{1}{T}\sum_{t=0}^{T-1}x_{t}$.
We have

\[
f(\overline{x}_{T})-f(x^{*})\leq\frac{1}{T}\sum_{t=0}^{T-1}\left(f(x_{t})-f(x^{*})\right)\leq\frac{R_{\infty}^{2}\sum_{i=1}^{d}\beta_{i}}{T}\ ,
\]
where $R_{\infty}\geq\max_{0\leq t\leq T}\left\Vert x_{t}-x^{*}\right\Vert _{\infty}$
and $\eta=\frac{1}{\sqrt{2}}R_{\infty}$. 
\end{thm}

We extend the analysis to the stochastic setting in Section \ref{sec:adagrad-unconstrained-stochastic}.

We will use the following standard lemma for smooth convex functions
that can be found, e.g., in the textbook \citep{nesterov1998introductory}. 
\begin{lem}
\label{lem:unconstrained-smoothness}Let $f$ be a convex function
that is $\beta$-smooth with respect to the norm $\left\Vert \cdot\right\Vert $
, with dual norm $\left\Vert \cdot\right\Vert _{*}$. We have 
\[
f(x)-f(y)\leq\left\langle \nabla f(x),x-y\right\rangle -\frac{1}{2\beta}\left\Vert \nabla f(x)-\nabla f(y)\right\Vert _{*}^{2}\quad\forall x,y\in\R^{d}
\]
\end{lem}

We will apply Lemma \ref{lem:unconstrained-smoothness} with $y=x^{*}=\arg\min_{x\in\R^{d}}f(x)$
is the (unconstrained) minimum of $f$. Thus we have $\nabla f(x^{*})=0$.
The $f$ is $1$-smooth with respect to the norm $\left\Vert \cdot\right\Vert _{\sm}$
and its dual norm is $\left\Vert \cdot\right\Vert _{\sm^{-1}}$. Thus
we obtain

\[
f(x_{t})-f(x^{*})\leq\left\langle \nabla f(x_{t}),x-x^{*}\right\rangle -\frac{1}{2}\left\Vert \nabla f(x_{t})\right\Vert _{\sm^{-1}}^{2}\ .
\]
Thus 
\begin{equation}
\sum_{t=0}^{T-1}\left(f(x_{t})-f(x^{*})\right)\leq\underbrace{\sum_{t=0}^{T-1}\left\langle \nabla f(x_{t}),x_{t}-x^{*}\right\rangle }_{\regret}-\frac{1}{2}\sum_{t=0}^{T-1}\left\Vert \nabla f(x_{t})\right\Vert _{\sm^{-1}}^{2}\ .\label{eq:adagrad1}
\end{equation}
We now analyze the regret using the standard $\adagrad$ analysis
\citep{duchi2011adaptive,McMahanS10}. We have

\begin{align*}
\frac{1}{2\eta}\left\Vert x_{t+1}-x^{*}\right\Vert _{D_{t}}^{2}-\frac{1}{2\eta}\left\Vert x_{t}-x^{*}\right\Vert _{D_{t}}^{2} & =\frac{1}{2\eta}\left\Vert x_{t+1}-x_{t}+x_{t}-x^{*}\right\Vert _{D_{t}}^{2}-\frac{1}{2\eta}\left\Vert x_{t}-x^{*}\right\Vert _{D_{t}}^{2}\\
 & =\frac{1}{2\eta}\left\Vert x_{t+1}-x_{t}\right\Vert _{D_{t}}^{2}+\frac{1}{\eta}\left\langle D_{t}\left(x_{t+1}-x_{t}\right),x_{t}-x^{*}\right\rangle \\
 & =\frac{\eta}{2}\left\Vert \nabla f(x_{t})\right\Vert _{D_{t}^{-1}}^{2}-\left\langle \nabla f(x_{t}),x_{t}-x^{*}\right\rangle \ .
\end{align*}
On the last line, we have used the update rule $x_{t+1}=x_{t}-\eta D_{t}^{-1}\nabla f(x_{t})$.
Rearranging and summing up over all iterations,

\begin{align*}
 & \sum_{t=0}^{T-1}\left\langle \nabla f(x_{t}),x_{t}-x^{*}\right\rangle \\
 & =\sum_{t=0}^{T-1}\left(\frac{1}{2\eta}\left\Vert x_{t}-x^{*}\right\Vert _{D_{t}}^{2}-\frac{1}{2\eta}\left\Vert x_{t+1}-x^{*}\right\Vert _{D_{t}}^{2}\right)+\sum_{t=0}^{T-1}\frac{\eta}{2}\left\Vert \nabla f(x_{t})\right\Vert _{D_{t}^{-1}}^{2}\\
 & =\sum_{t=0}^{T-1}\left(\frac{1}{2\eta}\left\Vert x_{t}-x^{*}\right\Vert _{D_{t}}^{2}-\frac{1}{2\eta}\left\Vert x_{t+1}-x^{*}\right\Vert _{D_{t+1}}^{2}+\frac{1}{2\eta}\left\Vert x_{t+1}-x^{*}\right\Vert _{D_{t+1}-D_{t}}^{2}\right)+\sum_{t=0}^{T-1}\frac{\eta}{2}\left\Vert \nabla f(x_{t})\right\Vert _{D_{t}^{-1}}^{2}\\
 & =\frac{1}{2\eta}\left\Vert x_{0}-x^{*}\right\Vert _{D_{0}}^{2}-\frac{1}{2\eta}\left\Vert x_{T}-x^{*}\right\Vert _{D_{T}}^{2}+\frac{1}{2\eta}\sum_{t=0}^{T-1}\left\Vert x_{t+1}-x^{*}\right\Vert _{D_{t+1}-D_{t}}^{2}+\sum_{t=0}^{T-1}\frac{\eta}{2}\left\Vert \nabla f(x_{t})\right\Vert _{D_{t}^{-1}}^{2}\ .
\end{align*}
Letting $R_{\infty}^{2}$ be an upper bound on $\left\Vert x_{t}-x^{*}\right\Vert _{\infty}^{2}$
for all $0\leq t\leq T$ and using the bound $\left\Vert z\right\Vert _{D}^{2}\leq\left\Vert z\right\Vert _{\infty}\tr(D)$,
we obtain

\begin{align}
 & \sum_{t=0}^{T-1}\left\langle \nabla f(x_{t}),x_{t}-x^{*}\right\rangle \nonumber \\
 & \leq\frac{1}{2\eta}\left\Vert x_{0}-x^{*}\right\Vert _{D_{0}}^{2}-\frac{1}{2\eta}\left\Vert x_{T}-x^{*}\right\Vert _{D_{T}}^{2}+\frac{1}{2\eta}R_{\infty}^{2}\sum_{t=0}^{T-1}\left(\tr(D_{t+1})-\tr(D_{t})\right)+\sum_{t=0}^{T-1}\frac{\eta}{2}\left\Vert \nabla f(x_{t})\right\Vert _{D_{t}^{-1}}^{2}\nonumber \\
 & =\frac{1}{2\eta}\left\Vert x_{0}-x^{*}\right\Vert _{D_{0}}^{2}-\frac{1}{2\eta}\left\Vert x_{T}-x^{*}\right\Vert _{D_{T}}^{2}+\frac{1}{2\eta}R_{\infty}^{2}\left(\tr(D_{T})-\tr(D_{0})\right)+\sum_{t=0}^{T-1}\frac{\eta}{2}\left\Vert \nabla f(x_{t})\right\Vert _{D_{t}^{-1}}^{2}\nonumber \\
 & \leq\frac{1}{2\eta}R_{\infty}^{2}\tr(D_{T})+\frac{\eta}{2}\sum_{t=0}^{T-1}\left\Vert \nabla f(x_{t})\right\Vert _{D_{t}^{-1}}^{2}\ .\label{eq:adagrad2}
\end{align}
Next, we use the following standard inequality \citep{duchi2011adaptive,McMahanS10}.
Let $\left\{ a_{t}\right\} $ be positive scalars and $A_{t}=\sum_{i=1}^{t}a_{t}$.
We have 
\begin{equation}
\sqrt{A_{T}}\leq\sum_{t=1}^{T}\frac{a_{t}}{\sqrt{A_{t}}}\leq2\sqrt{A_{T}}\ .\label{eq:ineq-sqrt}
\end{equation}
Recall that $D_{t,i}=\sqrt{1+\sum_{s=0}^{t}\left(\nabla_{i}f(x_{s})\right)^{2}}$.
We apply \eqref{eq:ineq-sqrt} for each coordinate $i$ separately,
with $a_{t}=\left(\nabla_{i}f(x_{t})\right)^{2}$ and obtain 
\[
\sum_{t=0}^{T-1}\frac{\left(\nabla_{i}f(x_{t})\right)^{2}}{D_{t,i}}\leq2D_{T-1,i}\ ,
\]
and thus 
\[
\sum_{t=0}^{T-1}\left\Vert \nabla f(x_{t})\right\Vert _{D_{t}^{-1}}^{2}\leq2\tr(D_{T-1})\ .
\]
Plugging in into \eqref{eq:adagrad2} and setting $\eta=\frac{1}{\sqrt{2}}R_{\infty}$,
we obtain 
\begin{equation}
\sum_{t=0}^{T-1}\left\langle \nabla f(x_{t}),x_{t}-x^{*}\right\rangle \leq\sqrt{2}R_{\infty}\tr(D_{T-1})\ .\label{eq:adagrad3}
\end{equation}
Plugging in \eqref{eq:adagrad3} into \eqref{eq:adagrad1}, we obtain

\begin{align*}
\sum_{t=0}^{T-1}\left(f(x_{t})-f(x^{*})\right) & \leq\sqrt{2}R_{\infty}\tr(D_{T-1})-\frac{1}{2}\sum_{t=0}^{T-1}\left\Vert \nabla f(x_{t})\right\Vert _{\sm^{-1}}^{2}\\
 & =\sqrt{2}R_{\infty}\sum_{i=1}^{d}\sqrt{\sum_{t=0}^{T-1}\left(\nabla_{i}f(x_{t})\right)^{2}}-\frac{1}{2}\sum_{i=1}^{d}\frac{1}{\beta_{i}}\sum_{t=0}^{T-1}\left(\nabla_{i}f(x_{t})\right)^{2}\ .
\end{align*}
Letting $z_{i}=\sum_{t=0}^{T-1}\left(\nabla_{i}f(x_{t})\right)^{2}$,
the bound becomes 
\begin{align*}
\sum_{t=0}^{T-1}\left(f(x_{t})-f(x^{*})\right) & \leq\sum_{i=1}^{d}\left(\sqrt{2}R_{\infty}\sqrt{z_{i}}-\frac{1}{2\beta_{i}}z_{i}\right)\leq\sum_{i=1}^{d}\max_{z\geq0}\left(\sqrt{2}R_{\infty}\sqrt{z}-\frac{1}{2\beta_{i}}z\right)\leq R_{\infty}^{2}\sum_{i=1}^{d}\beta_{i}\ .
\end{align*}
In the last inequality, we have used that the function $\phi(z)=a\sqrt{z}-\frac{1}{2b}z$
is concave over $z\geq0$ and it is maximized at $z^{*}=(ab)^{2}$.

Therefore

\[
f(\overline{x}_{T})-f(x^{*})\leq\frac{1}{T}\sum_{t=0}^{T-1}\left(f(x_{t})-f(x^{*})\right)\leq\frac{R_{\infty}^{2}\sum_{i=1}^{d}\beta_{i}}{T}\ .
\]

\subsection{Stochastic Setting}

\label{sec:adagrad-unconstrained-stochastic}

In this section, we extend the $\adagrad$ algorithm and its analysis
from Section \ref{sec:adagrad-unconstrained} to the stochastic setting.
We extend the $\adagrad$ algorithm in the natural way, and the resulting
algorithm is shown in Figure \ref{alg:adagrad-unconstrained-stochastic}.
The analysis we provide is a generalization to the vector setting
of the analysis in \citep{levy2018online}. In each iteration, we
receive a stochastic gradient $\widetilde{\nabla}f(x_{t})$ satisfying
the assumptions \eqref{eq:stoch-assumption-unbiased} and \eqref{eq:stoch-assumption-variance}:
$\E\left[\widetilde{\nabla}f(x_{t})\vert x_{t}\right]=\nabla f(x_{t})$
and $\E\left[\left\Vert \widetilde{\nabla}f(x_{t})-\nabla f(x_{t})\right\Vert ^{2}\right]\leq\sigma^{2}$.
Throughout this section, the norm $\left\Vert \cdot\right\Vert $
without a subscript denotes the $\ell_{2}$-norm.

\begin{figure}
\noindent %
\noindent\fbox{\begin{minipage}[t]{1\columnwidth - 2\fboxsep - 2\fboxrule}%
Let $x_{0}\in\R^{d}$, $D_{0}=I$.

For $t=0,\dots,T-2$, update: 
\begin{align*}
x_{t+1} & =x_{t}-\eta D_{t}^{-1}\widetilde{\nabla}f(x_{t})\ ,\\
D_{t+1,i}^{2} & =D_{t,i}^{2}+\left(\widetilde{\nabla}_{i}f(x_{t+1})\right)^{2}\ , & \forall i\in[d]\ .
\end{align*}

Return $\overline{x}_{T}=\frac{1}{T}\sum_{t=0}^{T-1}x_{t}$.%
\end{minipage}}

\caption{$\protect\adagrad$ algorithm with stochastic gradients $\widetilde{\nabla}f(x_{t})$.}
\label{alg:adagrad-unconstrained-stochastic} 
\end{figure}

In the remainder of this section, we prove the following convergence
guarantee. We emphasize that the algorithm does not know the variance
parameter $\sigma$ or the smoothness parameters, and it automatically
adapts to them. 
\begin{thm}
\label{thm:stoch-unconstr}Let $f$ be a convex function that is $1$-smooth
with respect to the norm $\left\Vert \cdot\right\Vert _{\sm}$, where
$\sm=diag(\beta_{1},\dots,\beta_{d})$ is a diagonal matrix with $\beta_{1},\dots,\beta_{d}\geq1$.
Let $x^{*}\in\arg\min_{x\in\mathbb{R}^{d}}f(x)$. Let $x_{t}$ be
the iterates constructed by the algorithm in Figure \ref{alg:adagrad-unconstrained-stochastic}
and let $\overline{x}_{T}=\frac{1}{T}\sum_{t=0}^{T-1}x_{t}$. If the
stochastic gradients satisfy the assumptions (\ref{eq:stoch-assumption-unbiased})
and (\ref{eq:stoch-assumption-variance}), we have 
\[
\E\left[f(\overline{x}_{T})-f(x^{*})\right]\leq O\left(\frac{R_{\infty}\sqrt{d}\sigma}{\sqrt{T}}+\frac{R_{\infty}^{2}\sum_{i=1}^{d}\beta_{i}}{T}\right)\ .
\]
where $R_{\infty}$ is a fixed scalar for which we have $R_{\infty}\geq\max_{0\leq t\leq T}\left\Vert x_{t}-x^{*}\right\Vert _{\infty}$
with probability one, and $\eta=\frac{1}{\sqrt{2}}R_{\infty}$. 
\end{thm}

We note that the regret analysis from Section \ref{sec:adagrad-unconstrained}
(which is the standard $\adagrad$ analysis from previous work \citep{duchi2011adaptive,McMahanS10}),
applies to this setting as well and it provides the following guarantee
with probability one: 
\begin{equation}
\sum_{t=0}^{T-1}\left\langle \widetilde{\nabla}f(x_{t}),x_{t}-x^{*}\right\rangle \leq\sqrt{2}R_{\infty}\tr(D_{T-1})\ .\label{eq:adagrad-stoch1}
\end{equation}
By assumption (\ref{eq:stoch-assumption-unbiased}), we have

\[
\E\left[\left\langle \widetilde{\nabla}f(x_{t}),x_{t}-x^{*}\right\rangle \vert x_{t}\right]=\left\langle \nabla f(x_{t}),x_{t}-x^{*}\right\rangle \ .
\]
Taking expectation over the entire history, we have

\[
\E\left[\left\langle \widetilde{\nabla}f(x_{t}),x_{t}-x^{*}\right\rangle \right]=\E\left[\left\langle \nabla f(x_{t}),x_{t}-x^{*}\right\rangle \right]\ .
\]
By linearity of expectation,

\[
\E\left[\sum_{t=0}^{T-1}\left\langle \widetilde{\nabla}f(x_{t}),x_{t}-x^{*}\right\rangle \right]=\E\left[\sum_{t=0}^{T-1}\left\langle \nabla f(x_{t}),x_{t}-x^{*}\right\rangle \right]\ .
\]
Combining with \eqref{eq:adagrad-stoch1}, we obtain

\begin{equation}
\E\left[\sum_{t=0}^{T-1}\left\langle \nabla f(x_{t}),x_{t}-x^{*}\right\rangle \right]\leq\sqrt{2}R_{\infty}\E\left[\tr(D_{T-1})\right]\ .\label{eq:adagrad-stoch2}
\end{equation}
As before, we apply Lemma \ref{lem:unconstrained-smoothness} with
$x^{*}=\arg\min_{x\in\R^{d}}f(x)$. Since $\nabla f(x^{*})=0$ and
$f$ is $1$-smooth with respect to $\left\Vert \cdot\right\Vert _{\sm}$,
we obtain 
\[
f(x_{t})-f(x^{*})\leq\left\langle \nabla f(x_{t}),x_{t}-x^{*}\right\rangle -\frac{1}{2}\left\Vert \nabla f(x_{t})\right\Vert _{\sm^{-1}}^{2}\ .
\]
Combining with \eqref{eq:adagrad-stoch2}, we obtain

\begin{align}
\E\left[\sum_{t=0}^{T-1}\left(f(x_{t})-f(x^{*})\right)\right] & \leq\sqrt{2}R_{\infty}\E\left[\tr(D_{T-1})\right]-\frac{1}{2}\E\left[\sum_{t=0}^{T-1}\left\Vert \nabla f(x_{t})\right\Vert _{\sm^{-1}}^{2}\right]\ .\label{eq:adagrad-stoch3}
\end{align}
For every coordinate $i\in[d]$, we define

\begin{align*}
\widetilde{G}_{i} & =\sqrt{1+\sum_{t=0}^{T-1}\left(\widetilde{\nabla}_{i}f(x_{t})\right)^{2}}=D_{T-1,i}\\
G_{i} & =\sqrt{1+\sum_{t=0}^{T-1}\left(\nabla_{i}f(x_{t})\right)^{2}}
\end{align*}
Using concavity of $\sqrt{x}$, we obtain 
\[
\E\left[\tr(D_{T-1})\right]=\sum_{i=1}^{d}\E\left[D_{T-1,i}\right]=\sum_{i=1}^{d}\E\left[\widetilde{G}_{i}\right]=\sum_{i=1}^{d}\E\left[\sqrt{\widetilde{G}_{i}^{2}}\right]\leq\sum_{i=1}^{d}\sqrt{\E\left[\widetilde{G}_{i}^{2}\right]}
\]
We have 
\[
\sum_{t=0}^{T-1}\left\Vert \nabla f(x_{t})\right\Vert _{\sm^{-1}}^{2}=\sum_{i=1}^{d}\frac{1}{\beta_{i}}\sum_{t=0}^{T-1}\left(\nabla_{i}f(x_{t})\right)^{2}=\sum_{i=1}^{d}\frac{1}{\beta_{i}}\left(G_{i}^{2}-1\right)\leq\sum_{i=1}^{d}\frac{1}{\beta_{i}}G_{i}^{2}
\]
Plugging in into \eqref{eq:adagrad-stoch3}, we obtain 
\begin{align*}
 & \E\left[\sum_{t=0}^{T-1}\left(f(x_{t})-f(x^{*})\right)\right]\leq\sqrt{2}R_{\infty}\sum_{i=1}^{d}\sqrt{\E\left[\widetilde{G}_{i}^{2}\right]}-\sum_{i=1}^{d}\frac{1}{2\beta_{i}}\E\left[G_{i}^{2}\right]\\
 & =\sqrt{2}R_{\infty}\underbrace{\left(\sum_{i=1}^{d}\sqrt{\E\left[\widetilde{G}_{i}^{2}\right]}-\sum_{i=1}^{d}\sqrt{\E\left[G_{i}^{2}\right]}\right)}_{(\star)}+\underbrace{\left(\sqrt{2}R_{\infty}\sum_{i=1}^{d}\sqrt{\E\left[G_{i}^{2}\right]}-\sum_{i=1}^{d}\frac{1}{2\beta_{i}}\E\left[G_{i}^{2}\right]\right)}_{(\star\star)}
\end{align*}
We upper bound each term $(\star)$ and $(\star\star)$ in term.

We upper bound $(\star)$ as follows. We start by showing that, for
every $i\in[d]$, we have 
\begin{equation}
\E\left[\widetilde{G}_{i}^{2}\right]\geq\E\left[G_{i}^{2}\right]\label{eq:adagrad-stoch4}
\end{equation}
The inequality is equivalent to 
\[
\sum_{t=0}^{T-1}\E\left[\left(\nabla_{i}f(x_{t})\right)^{2}\right]\leq\sum_{t=0}^{T-1}\E\left[\left(\widetilde{\nabla}_{i}f(x_{t})\right)^{2}\right]
\]
By assumption, for every $t$, we have 
\[
\E\left[\widetilde{\nabla}_{i}f(x_{t})\vert x_{t}\right]=\nabla_{i}f(x_{t})
\]
By squaring and using the fact that $x^{2}$ is convex, we obtain
\[
\left(\nabla_{i}f(x_{t})\right)^{2}=\left(\E\left[\widetilde{\nabla}_{i}f(x_{t})\vert x_{t}\right]\right)^{2}\leq\E\left[\left(\widetilde{\nabla}_{i}f(x_{t})\right)^{2}\vert x_{t}\right]
\]
Taking expectation over the entire history, we obtain 
\[
\E\left[\left(\nabla_{i}f(x_{t})\right)^{2}\right]\leq\E\left[\left(\widetilde{\nabla}_{i}f(x_{t})\right)^{2}\right]
\]
By summing up over all iterations $t$, we obtain \eqref{eq:adagrad-stoch4}.

Let us now note that, for any scalars $a,b$ satisfying $a\geq b\geq0$,
we have 
\[
\sqrt{a}-\sqrt{b}\leq\sqrt{a-b}
\]
We can verify the above inequality by squaring $\sqrt{a}\leq\sqrt{b}+\sqrt{a-b}$.
We apply the inequality with $a=\E\left[\widetilde{G}_{i}^{2}\right]$
and $b=\E\left[G_{i}^{2}\right]$, and obtain

\[
\sqrt{\E\left[\widetilde{G}_{i}^{2}\right]}-\sqrt{\E\left[G_{i}^{2}\right]}\leq\sqrt{\E\left[\widetilde{G}_{i}^{2}-G_{i}^{2}\right]}
\]
Summing up over all coordinates and using that $\sqrt{x}$ is concave
and the assumptions on the stochastic gradients, we obtain

\begin{align*}
(\star)= & \sum_{i=1}^{d}\left(\sqrt{\E\left[\widetilde{G}_{i}^{2}\right]}-\sqrt{\E\left[G_{i}^{2}\right]}\right)\leq\sum_{i=1}^{d}\sqrt{\E\left[\widetilde{G}_{i}^{2}-G_{i}^{2}\right]}\leq\sqrt{d}\sqrt{\E\left[\sum_{i=1}^{d}\widetilde{G}_{i}^{2}-\sum_{i=1}^{d}G_{i}^{2}\right]}\\
 & =\sqrt{d}\sqrt{\sum_{t=0}^{T-1}\E\left[\left\Vert \widetilde{\nabla}f(x_{t})\right\Vert ^{2}-\left\Vert \nabla f(x_{t})\right\Vert ^{2}\right]}\leq\sqrt{d}\sigma\sqrt{T}
\end{align*}
We now upper bound $(\star\star)$. We have 
\begin{align*}
(\star\star) & =\sum_{i=1}^{d}\left(\sqrt{2}R_{\infty}\sqrt{\E\left[G_{i}^{2}\right]}-\frac{1}{2\beta_{i}}\E\left[G_{i}^{2}\right]\right)\leq\sum_{i=1}^{d}\max_{z\geq0}\left(\sqrt{2}R_{\infty}z-\frac{1}{2\beta_{i}}z^{2}\right)\leq R_{\infty}^{2}\sum_{i=1}^{d}\beta_{i}
\end{align*}
In the last inequality, we have used that the function $\phi(z)=a\sqrt{z}-\frac{1}{2b}z$
is concave over $z\geq0$ and it is maximized at $z^{*}=(ab)^{2}$.

Plugging in into the previous inequality, we obtain 
\[
\E\left[f(\overline{x}_{T})-f(x^{*})\right]\le\frac{1}{T}\E\left[\sum_{t=0}^{T-1}\left(f(x_{t})-f(x^{*})\right)\right]\leq O\left(\frac{R_{\infty}\sqrt{d}\sigma}{\sqrt{T}}+\frac{R_{\infty}^{2}\sum_{i=1}^{d}\beta_{i}}{T}\right)\ .
\]

\section{$\textsc{\ensuremath{\protect\adamp}}$ Algorithm for Variational
Inequalities}

\label{sec:mirror-prox}

In this section, we extend the universal Mirror-Prox of \citet{BachL19}
to the vector setting, and resolve the open question asked by \citet{BachL19}.
The algorithm applies to the more general setting of solving variational
inequalities, which captures both convex minimization and convex-concave
zero-sum games (we refer the reader to \citep{BachL19} for the details).

\subsection{Variational Inequalities}

Here we review some definitions and facts from \citep{BachL19}. We
follow their setup and notation, and we include it here for completeness.

\paragraph{Definitions.}

Let $\dom\subseteq\R^{d}$ be a convex set and let $F\colon\dom\to\mathbb{R}^{d}$.
We say that $F$ is a \emph{monotone operator }if it satisfies

\[
\left\langle F(x)-F(y),x-y\right\rangle \geq0,\quad\forall(x,y)\in\dom\times\dom\ .
\]
We say that $F$ is $\beta$-smooth with respect to a norm $\left\Vert \cdot\right\Vert $with
dual norm $\left\Vert \cdot\right\Vert _{*}$ if

\[
\left\Vert F(x)-F(y)\right\Vert _{*}\leq\beta\left\Vert x-y\right\Vert \ .
\]
A \emph{gap function} is a function $\Delta\colon\dom\times\dom\to\mathbb{R}$
that is convex with respect to its first argument and it satisfies
\[
\left\langle F(x),x-y\right\rangle \geq\Delta(x,y)\quad\forall(x,y)\in\dom\times\dom\ .
\]
The \emph{duality gap} is defined as 
\[
\dualgap(x):=\max_{y\in\dom}\Delta(x,y)\ .
\]

\paragraph{Remark on the Gap Function.}

In light of the definition, the function $\left\langle F(x),x-y\right\rangle $
is a natural candidate for a gap function. However, it is not necessarily
convex with respect to its first argument. As we note below, the convexity
of the gap function allows us to analyze an iterative scheme via the
regret, and we require it for this reason. As a result, we do not
use the monotonicity of $F$ directly, and we only rely on the existence
of the gap function. Moreover, the existence of the gap function is
needed only for the analysis, and the algorithm does not use it.

\paragraph{Problem Definition.}

We assume that we are given black-box access to an evaluation oracle
for $F$ that, on input $x$, it returns $F(x)$. We also assume that
we are given black-box access to a projection oracle for $\dom$ that,
on input $x$, it returns $\Pi_{\dom}(x)=\arg\min_{y\in\dom}\left\Vert x-y\right\Vert $.
The goal is to find a solution $x\in\dom$ that minimizes the duality
gap $\dualitygap(x)$.

It was shown in \citep{BachL19} that this problem generalizes both
convex minimization $\min_{x\in K}f(x)$ and convex-concave zero-sum
games $\min_{x\in X}\max_{y\in Y}g(x,y)$, where $g$ is convex in
$x$ and concave in $y$. For the former problem, $F(x)=\nabla f(x)$
and $\dualitygap(x)=\max_{y\in\dom}\Delta(x,y)=f(x)-\min_{y\in\dom}f(y)$.
For the latter problem, we have $\dom=X\times Y$ and $F(x,y)=(\nabla_{x}g(x,y),-\nabla_{y}g(x,y))$
and $\dualgap(u,v)=\max_{y\in Y}g(u,y)-\min_{x\in X}g(x,v)$. For
both problems, there is a suitable gap function.

\paragraph{Analyzing Convergence via Regret.}

As noted in \citep{BachL19}, the convexity of the gap function allows
us to analyze convergence via the regret:

\[
\regret:=\sum_{t=1}^{T}\left\langle F(x_{t}),x_{t}\right\rangle -\min_{x\in\dom}\sum_{t=1}^{T}\left\langle F(x_{t}),x\right\rangle \ .
\]

We can translate a regret guarantee into a convergence guarantee using
Jensen's inequality, since the gap function is convex in its first
argument. Let $\overline{x}_{T}=\frac{1}{T}\sum_{t=1}^{T}x_{t}$.
For all $x\in\dom$, we have

\begin{align*}
\Delta(\overline{x}_{T},x) & \leq\frac{1}{T}\sum_{t=1}^{T}\Delta(x_{t},x)\leq\frac{1}{T}\sum_{t=1}^{T}\left\langle F(x_{t}),x_{t}-x\right\rangle \leq\frac{1}{T}\regret\ .
\end{align*}
Therefore 
\[
\dualgap(\overline{x}_{T})=\max_{x\in K}\Delta(\overline{x}_{T},x)\leq\frac{1}{T}\regret\ .
\]

\subsection{Analysis for Smooth Operators}

We borrow the initial part of the analysis from \citep{BachL19}.
As noted above, it suffices to analyze the regret:

\[
\regret:=\sum_{t=1}^{T}\left\langle F(x_{t}),x_{t}\right\rangle -\min_{x\in\dom}\sum_{t=1}^{T}\left\langle F(x_{t}),x\right\rangle \ .
\]
Letting 
\[
x^{*}=\arg\min_{x\in\dom}\sum_{t=1}^{T}\left\langle F(x_{t}),x\right\rangle \ ,
\]
the regret becomes 
\[
\regret=\sum_{t=1}^{T}\left\langle F(x_{t}),x_{t}-x^{*}\right\rangle \ .
\]
Following \citep{BachL19}, we write

\begin{align*}
\left\langle F(x_{t}),x_{t}-x^{*}\right\rangle  & =\left\langle F(x_{t}),x_{t}-y_{t}\right\rangle +\left\langle F(x_{t}),y_{t}-x^{*}\right\rangle \\
 & =\underbrace{\left\langle F(x_{t})-F(y_{t-1}),x_{t}-y_{t}\right\rangle }_{(A)}+\underbrace{\left\langle F(y_{t-1}),x_{t}-y_{t}\right\rangle }_{(B)}+\underbrace{\left\langle F(x_{t}),y_{t}-x^{*}\right\rangle }_{(C)}\ .
\end{align*}
We bound $(A)$, $(B)$, and $(C)$ as in \citep{BachL19}.

For $(A)$, we use Holder's inequality, smoothness, and the inequality
$ab\leq\frac{1}{2}a^{2}+\frac{1}{2}b^{2}$:

\begin{align*}
(A) & \leq\left\Vert F(x_{t})-F(y_{t-1})\right\Vert _{\sm^{-1}}\left\Vert x_{t}-y_{t}\right\Vert _{\sm}\leq\left\Vert x_{t}-y_{t-1}\right\Vert _{\sm}\left\Vert x_{t}-y_{t}\right\Vert _{\sm}\leq\frac{1}{2}\left(\left\Vert x_{t}-y_{t-1}\right\Vert _{\sm}^{2}+\left\Vert x_{t}-y_{t}\right\Vert _{\sm}^{2}\right)\ .
\end{align*}

For $(B)$ and $(C)$, we use the definition of $x_{t}$ and $y_{t}$.
Let

\begin{align*}
\phi_{t}(x) & =\left\langle F(y_{t-1}),x\right\rangle +\frac{1}{2}\left\Vert x-y_{t-1}\right\Vert _{D_{t}}^{2}\ .
\end{align*}

Since $\phi_{t}$ is $1$-strongly convex with respect to $\left\Vert \cdot\right\Vert _{D_{t}}$
and $x_{t}=\arg\min_{x\in K}\phi_{t}(x)$, for all $v\in\dom$, we
have

\[
\phi_{t}(x_{t})\leq\phi_{t}(v)-\frac{1}{2}\left\Vert x_{t}-v\right\Vert _{D_{t}}^{2}\ .
\]
Thus 
\[
\left\langle F(y_{t-1}),x_{t}-v\right\rangle \leq\frac{1}{2}\left\Vert v-y_{t-1}\right\Vert _{D_{t}}^{2}-\frac{1}{2}\left\Vert x_{t}-v\right\Vert _{D_{t}}^{2}-\frac{1}{2}\left\Vert x_{t}-y_{t-1}\right\Vert _{D_{t}}^{2}\ .
\]
Applying the above with $v=y_{t}$ gives 
\[
(B)=\left\langle F(y_{t-1}),x_{t}-y_{t}\right\rangle \leq\frac{1}{2}\left\Vert y_{t}-y_{t-1}\right\Vert _{D_{t}}^{2}-\frac{1}{2}\left\Vert x_{t}-y_{t}\right\Vert _{D_{t}}^{2}-\frac{1}{2}\left\Vert x_{t}-y_{t-1}\right\Vert _{D_{t}}^{2}\ .
\]
Similarly, let

\[
\varphi_{t}(x)=\left\langle F(x_{t}),x\right\rangle +\frac{1}{2}\left\Vert x-y_{t-1}\right\Vert _{D_{t}}^{2}\ .
\]
Since $\varphi_{t}$ is $1$-strongly convex with respect to $\left\Vert \cdot\right\Vert _{D_{t}}$
and $y_{t}=\arg\min_{x\in\dom}\varphi_{t}(x)$, for all $v\in\dom$,
we have 
\[
\varphi_{t}(y_{t})\leq\varphi_{t}(v)-\frac{1}{2}\left\Vert y_{t}-v\right\Vert _{D_{t}}^{2}\ .
\]
Thus 
\[
\left\langle F(x_{t}),y_{t}-v\right\rangle \leq\frac{1}{2}\left\Vert v-y_{t-1}\right\Vert _{D_{t}}^{2}-\frac{1}{2}\left\Vert y_{t}-v\right\Vert _{D_{t}}^{2}-\frac{1}{2}\left\Vert y_{t}-y_{t-1}\right\Vert _{D_{t}}^{2}\ .
\]
Applying the above with $v=x^{*}$ gives 
\[
(C)=\left\langle F(x_{t}),y_{t}-x^{*}\right\rangle \leq\frac{1}{2}\left\Vert x^{*}-y_{t-1}\right\Vert _{D_{t}}^{2}-\frac{1}{2}\left\Vert y_{t}-x^{*}\right\Vert _{D_{t}}^{2}-\frac{1}{2}\left\Vert y_{t}-y_{t-1}\right\Vert _{D_{t}}^{2}\ .
\]
Putting everything together and summing up,

\begin{align*}
\sum_{t=1}^{T}\left\langle F(x_{t}),x_{t}-x^{*}\right\rangle  & \leq\sum_{t=1}^{T}\left(\frac{1}{2}\left\Vert x_{t}-y_{t-1}\right\Vert _{\sm}^{2}+\frac{1}{2}\left\Vert x_{t}-y_{t}\right\Vert _{\sm}^{2}-\frac{1}{2}\left\Vert x_{t}-y_{t-1}\right\Vert _{D_{t}}^{2}-\frac{1}{2}\left\Vert x_{t}-y_{t}\right\Vert _{D_{t}}^{2}\right)\\
 & +\sum_{t=1}^{T}\left(\frac{1}{2}\left\Vert x^{*}-y_{t-1}\right\Vert _{D_{t}}^{2}-\frac{1}{2}\left\Vert x^{*}-y_{t}\right\Vert _{D_{t}}^{2}\right)\ .
\end{align*}
As in the standard $\adagrad$ analysis, we bound

\begin{align*}
 & \sum_{t=1}^{T}\left(\frac{1}{2}\left\Vert x^{*}-y_{t-1}\right\Vert _{D_{t}}^{2}-\frac{1}{2}\left\Vert x^{*}-y_{t}\right\Vert _{D_{t}}^{2}\right)\\
 & =\frac{1}{2}\left\Vert x^{*}-y_{0}\right\Vert _{D_{1}}^{2}-\frac{1}{2}\left\Vert x^{*}-y_{1}\right\Vert _{D_{1}}^{2}\\
 & +\sum_{t=2}^{T}\left(\frac{1}{2}\left\Vert x^{*}-y_{t-1}\right\Vert _{D_{t-1}}^{2}-\frac{1}{2}\left\Vert x^{*}-y_{t}\right\Vert _{D_{t}}^{2}+\frac{1}{2}\left\Vert x^{*}-y_{t-1}\right\Vert _{D_{t}-D_{t-1}}^{2}\right)\\
 & =\frac{1}{2}\left\Vert x^{*}-y_{0}\right\Vert _{D_{1}}^{2}-\frac{1}{2}\left\Vert x^{*}-y_{1}\right\Vert _{D_{1}}^{2}+\frac{1}{2}\left\Vert x^{*}-y_{1}\right\Vert _{D_{1}}^{2}-\frac{1}{2}\left\Vert x^{*}-y_{T}\right\Vert _{D_{T}}^{2}\\
 & +\sum_{t=2}^{T}\frac{1}{2}\left\Vert x^{*}-y_{t-1}\right\Vert _{D_{t}-D_{t-1}}^{2}\\
 & \leq\frac{1}{2}\left\Vert x^{*}-y_{0}\right\Vert _{D_{1}}^{2}+\frac{1}{2}R_{\infty}^{2}\left(\tr(D_{T})-\tr(D_{1})\right)\\
 & \leq\frac{1}{2}R_{\infty}^{2}\tr(D_{T})\ .
\end{align*}
Thus 
\begin{align*}
 & 2\sum_{t=1}^{T}\left\langle F(x_{t}),x_{t}-x^{*}\right\rangle \\
 & \leq\underbrace{\left(R_{\infty}^{2}\tr(D_{T})-\frac{1}{2}\sum_{t=1}^{T}\left(\left\Vert x_{t}-y_{t-1}\right\Vert _{D_{t}}^{2}+\left\Vert x_{t}-y_{t}\right\Vert _{D_{t}}^{2}\right)\right)}_{(\star)}\\
 & +\underbrace{\sum_{t=1}^{T}\left(\left(\left\Vert x_{t}-y_{t-1}\right\Vert _{\sm}^{2}+\left\Vert x_{t}-y_{t}\right\Vert _{\sm}^{2}\right)-\frac{1}{2}\left(\left\Vert x_{t}-y_{t-1}\right\Vert _{D_{t}}^{2}+\left\Vert x_{t}-y_{t}\right\Vert _{D_{t}}^{2}\right)\right)}_{(\star\star)}\ .
\end{align*}
We now use the arguments from Lemmas \ref{lem:error1} and \ref{lem:error2}
to bound $(\star)$ and $(\star\star)$.

For each coordinate separately, we apply Lemma \ref{lem:inequalities}
with $d_{t}^{2}=\left(x_{t,i}-y_{t-1,i}\right)^{2}+\left(x_{t,i}-y_{t,i}\right)^{2}$
and $R^{2}=2R_{\infty}^{2}\geq d_{t}^{2}$, and obtain 
\begin{align*}
\frac{1}{2}\sum_{t=1}^{T}\left(\left\Vert x_{t}-y_{t-1}\right\Vert _{D_{t}}^{2}+\left\Vert x_{t}-y_{t}\right\Vert _{D_{t}}^{2}\right) & \geq\frac{1}{2}\sum_{i=1}^{d}\sum_{t=1}^{T-1}D_{t,i}\left(\left(x_{t,i}-y_{t-1,i}\right)^{2}+\left(x_{t,i}-y_{t,i}\right)^{2}\right)\\
 & \geq2R_{\infty}^{2}\left(\tr(D_{T})-\tr(D_{1})\right)\ .
\end{align*}
Therefore 
\begin{align*}
(\star) & =R_{\infty}^{2}\tr(D_{T})-\frac{1}{2}\sum_{t=1}^{T}\left(\left\Vert x_{t}-y_{t-1}\right\Vert _{D_{t}}^{2}+\left\Vert x_{t}-y_{t}\right\Vert _{D_{t}}^{2}\right)\leq2R_{\infty}^{2}\tr(D_{1})=2R_{\infty}^{2}d\ .
\end{align*}
Note that the scaling $D_{t,i}$ is increasing with $t$. Let $\tilde{T}_{i}$
be the last iteration $t$ for which $D_{t,i}\leq2\beta_{i}$; we
let $\tilde{T}_{i}=-1$ if there is no such iteration. We have 
\begin{align*}
(\star\star) & =\sum_{t=1}^{T}\left(\left(\left\Vert x_{t}-y_{t-1}\right\Vert _{\sm}^{2}+\left\Vert x_{t}-y_{t}\right\Vert _{\sm}^{2}\right)-\frac{1}{2}\left(\left\Vert x_{t}-y_{t-1}\right\Vert _{D_{t}}^{2}+\left\Vert x_{t}-y_{t}\right\Vert _{D_{t}}^{2}\right)\right)\\
 & =\sum_{i=1}^{d}\sum_{t=1}^{T}\left(\beta_{i}\left(\left(x_{t,i}-y_{t-1,i}\right)^{2}+\left(x_{t,i}-y_{t,i}\right)^{2}\right)-\frac{1}{2}D_{t,i}\left(\left(x_{t,i}-y_{t-1,i}\right)^{2}+\left(x_{t,i}-y_{t,i}\right)^{2}\right)\right)\\
 & \leq\sum_{i=1}^{d}\sum_{t=1}^{\tilde{T}_{i}}\beta_{i}\left(\left(x_{t,i}-y_{t-1,i}\right)^{2}+\left(x_{t,i}-y_{t,i}\right)^{2}\right)\ .
\end{align*}
For each coordinate separately, we apply Lemma \ref{lem:inequalities}
with $d_{t}^{2}=\left(x_{t,i}-y_{t-1,i}\right)^{2}+\left(x_{t,i}-y_{t,i}\right)^{2}$
and $R^{2}=2R_{\infty}^{2}\geq d_{t}^{2}$, and obtain 
\begin{align*}
\sum_{t=1}^{\tilde{T}_{i}}\left(\left(x_{t,i}-y_{t-1,i}\right)^{2}+\left(x_{t,i}-y_{t,i}\right)^{2}\right) & \leq2R_{\infty}^{2}+\sum_{t=1}^{\tilde{T}_{i}-1}\left(\left(x_{t,i}-y_{t-1,i}\right)^{2}+\left(x_{t,i}-y_{t,i}\right)^{2}\right)\\
 & \leq2R_{\infty}^{2}+8R_{\infty}^{2}\ln\left(\frac{D_{\tilde{T}_{i},i}}{D_{1,i}}\right)=2R_{\infty}^{2}+8R_{\infty}^{2}\ln\left(2\beta_{i}\right)\ .
\end{align*}
Therefore 
\[
(\star\star)\leq O\left(R_{\infty}^{2}\sum_{i=1}^{d}\beta_{i}\ln\left(2\beta_{i}\right)\right)\ .
\]
Thus we obtain 
\[
\dualitygap(\overline{x}_{T})\leq\frac{\regret}{T}=O\left(\frac{R_{\infty}^{2}\sum_{i=1}^{d}\beta_{i}\ln\left(2\beta_{i}\right)}{T}\right)\ .
\]

\subsection{Analysis for Non-Smooth Operators}

Throughout this section, $\left\Vert \cdot\right\Vert $ without a
subscript denotes the $\ell_{2}$ norm. We borrow the initial part
of the analysis from \citep{BachL19}. As noted above, it suffices
to analyze the regret:

\begin{align*}
\regret & :=\sum_{t=1}^{T}\left\langle F(x_{t}),x_{t}\right\rangle -\min_{x\in\dom}\sum_{t=1}^{T}\left\langle F(x_{t}),x\right\rangle \\
 & =\sum_{t=1}^{T}\left\langle F(x_{t}),x_{t}-x^{*}\right\rangle \ .
\end{align*}

We follow \citet{BachL19} and we use the same argument as in the
previous section up to the final error analysis. We let $G\geq\max_{x\in\dom}\left\Vert F(x)\right\Vert $.

\begin{align*}
 & \left\langle F(x_{t}),x_{t}-x^{*}\right\rangle \\
 & =\left\langle F(x_{t}),x_{t}-y_{t}\right\rangle +\left\langle F(x_{t}),y_{t}-x^{*}\right\rangle \\
 & =\left\langle F(x_{t})-F(y_{t-1}),x_{t}-y_{t}\right\rangle +\left\langle F(y_{t-1}),x_{t}-y_{t}\right\rangle +\left\langle F(x_{t}),y_{t}-x^{*}\right\rangle \\
 & \le\left\Vert F(x_{t})-F(y_{t-1})\right\Vert \left\Vert x_{t}-y_{t}\right\Vert -\frac{1}{2}\left\Vert x_{t}-y_{t}\right\Vert _{D_{t}}^{2}-\frac{1}{2}\left\Vert x_{t}-y_{t-1}\right\Vert _{D_{t}}^{2}+\frac{1}{2}\left\Vert x^{*}-y_{t-1}\right\Vert _{D_{t}}^{2}-\frac{1}{2}\left\Vert x^{*}-y_{t}\right\Vert _{D_{t}}^{2}\\
 & \leq\left(\left\Vert F(x_{t})\right\Vert +\left\Vert F(y_{t-1})\right\Vert \right)\left\Vert x_{t}-y_{t}\right\Vert -\frac{1}{2}\left\Vert x_{t}-y_{t}\right\Vert _{D_{t}}^{2}-\frac{1}{2}\left\Vert x_{t}-y_{t-1}\right\Vert _{D_{t}}^{2}+\frac{1}{2}\left\Vert x^{*}-y_{t-1}\right\Vert _{D_{t}}^{2}-\frac{1}{2}\left\Vert x^{*}-y_{t}\right\Vert _{D_{t}}^{2}\\
 & \leq2G\left\Vert x_{t}-y_{t}\right\Vert -\frac{1}{2}\left\Vert x_{t}-y_{t}\right\Vert _{D_{t}}^{2}-\frac{1}{2}\left\Vert x_{t}-y_{t-1}\right\Vert _{D_{t}}^{2}+\frac{1}{2}\left\Vert x^{*}-y_{t-1}\right\Vert _{D_{t}}^{2}-\frac{1}{2}\left\Vert x^{*}-y_{t}\right\Vert _{D_{t}}^{2}\ .
\end{align*}
Summing up and using the standard $\adagrad$ analysis as before,
we obtain 
\begin{align*}
\sum_{t=1}^{T}\left\langle F(x_{t}),x_{t}-x^{*}\right\rangle  & \leq\underbrace{\sum_{t=1}^{T}2G\left\Vert x_{t}-y_{t}\right\Vert }_{(\star)}-\underbrace{\frac{1}{2}\sum_{t=1}^{T}\left(\left\Vert x_{t}-y_{t}\right\Vert _{D_{t}}^{2}+\left\Vert x_{t}-y_{t-1}\right\Vert _{D_{t}}^{2}\right)}_{(\star\star)}+\frac{1}{2}R_{\infty}^{2}\tr(D_{T})\ .
\end{align*}
We bound $(\star)$ and $(\star\star)$ as in Section \ref{sec:analysis-adagrad+-nonsmooth}.
Since $\sqrt{z}$ is concave, we have 
\begin{align*}
(\star) & =2G\sum_{t=1}^{T}\sqrt{\left\Vert x_{t}-y_{t}\right\Vert ^{2}}\leq2G\sqrt{T}\sqrt{\sum_{t=1}^{T}\left\Vert x_{t}-y_{t}\right\Vert ^{2}}\leq2G\sqrt{T}\sqrt{\sum_{t=1}^{T}\left(\left\Vert x_{t}-y_{t}\right\Vert ^{2}+\left\Vert x_{t}-y_{t-1}\right\Vert ^{2}\right)}\ .
\end{align*}
For each coordinate separately, we apply Lemma \ref{lem:inequalities}
with $d_{t}^{2}=\left(x_{t,i}-y_{t-1,i}\right)^{2}+\left(x_{t,i}-y_{t,i}\right)^{2}$
and $R^{2}=2R_{\infty}^{2}\geq d_{t}^{2}$, and obtain 
\begin{align*}
\sum_{t=1}^{T}\left(\left\Vert x_{t}-y_{t}\right\Vert ^{2}+\left\Vert x_{t}-y_{t-1}\right\Vert ^{2}\right) & \leq R_{\infty}^{2}d+\sum_{i=1}^{d}\sum_{t=1}^{T-1}\left(\left(x_{t,i}-y_{t-1,i}\right)^{2}+\left(x_{t,i}-y_{t,i}\right)^{2}\right)\\
 & \leq2R_{\infty}^{2}d+8R_{\infty}^{2}\sum_{i=1}^{d}\ln\left(\frac{D_{T,i}}{D_{1,i}}\right)=2R_{\infty}^{2}d+8R_{\infty}^{2}\sum_{i=1}^{d}\ln\left(D_{T,i}\right)\ .
\end{align*}
Therefore 
\begin{align*}
(\star) & \leq2G\sqrt{T}\sqrt{2R_{\infty}^{2}d+8R_{\infty}^{2}\sum_{i=1}^{d}\ln\left(D_{T,i}\right)}\leq2G\sqrt{T}\left(\sqrt{2}R_{\infty}\sqrt{d}+2\sqrt{2}R_{\infty}\sqrt{\sum_{i=1}^{d}\ln\left(D_{T,i}\right)}\right)\ .
\end{align*}
As shown in the previous section, we have 
\[
(\star\star)=\frac{1}{2}\sum_{t=1}^{T}\left(\left\Vert x_{t}-y_{t}\right\Vert _{D_{t}}^{2}+\left\Vert x_{t}-y_{t-1}\right\Vert _{D_{t}}^{2}\right)\geq2R_{\infty}^{2}\left(\tr(D_{T})-\tr(D_{1})\right)\ .
\]
Putting everything together,

\begin{align*}
 & \sum_{t=1}^{T}\left\langle F(x_{t}),x_{t}-x^{*}\right\rangle \\
 & \le2G\sqrt{T}\left(\sqrt{2}R_{\infty}\sqrt{d}+2\sqrt{2}R_{\infty}\sqrt{\sum_{i=1}^{d}\ln\left(D_{T,i}\right)}\right)-R_{\infty}^{2}\left(\tr(D_{T})-\tr(D_{1})\right)+\frac{1}{2}R_{\infty}^{2}\tr(D_{T})\\
 & =4\sqrt{2}R_{\infty}G\sqrt{T}\sqrt{\sum_{i=1}^{d}\ln\left(D_{T,i}\right)}-\frac{1}{2}R_{\infty}^{2}\tr(D_{T})+2\sqrt{2}R_{\infty}G\sqrt{d}\sqrt{T}+R_{\infty}^{2}d\\
 & =4\sqrt{2}R_{\infty}G\sqrt{T}\sqrt{\sum_{i=1}^{d}\ln\left(D_{T,i}\right)}-\frac{1}{2}R_{\infty}^{2}\sum_{i=1}^{d}D_{T,i}+2\sqrt{2}R_{\infty}G\sqrt{d}\sqrt{T}+R_{\infty}^{2}d\\
 & \leq O\left(\sqrt{d}R_{\infty}G\sqrt{\ln\left(\frac{GT}{R_{\infty}}\right)}\right)\sqrt{T}+O\left(R_{\infty}^{2}d\right)\ .
\end{align*}
In the last inequality, we used Lemma \ref{lem:phiz}.

Therefore 
\[
\dualitygap(\overline{x}_{T})=O\left(\frac{\sqrt{d}R_{\infty}G\sqrt{\ln\left(\frac{GT}{R_{\infty}}\right)}}{\sqrt{T}}+\frac{R_{\infty}^{2}d}{T}\right)\ .
\]

\section{Experimental Evaluation}

\label{sec:Experiments}

\begin{figure}
\centering{}
\includegraphics[scale=0.49]{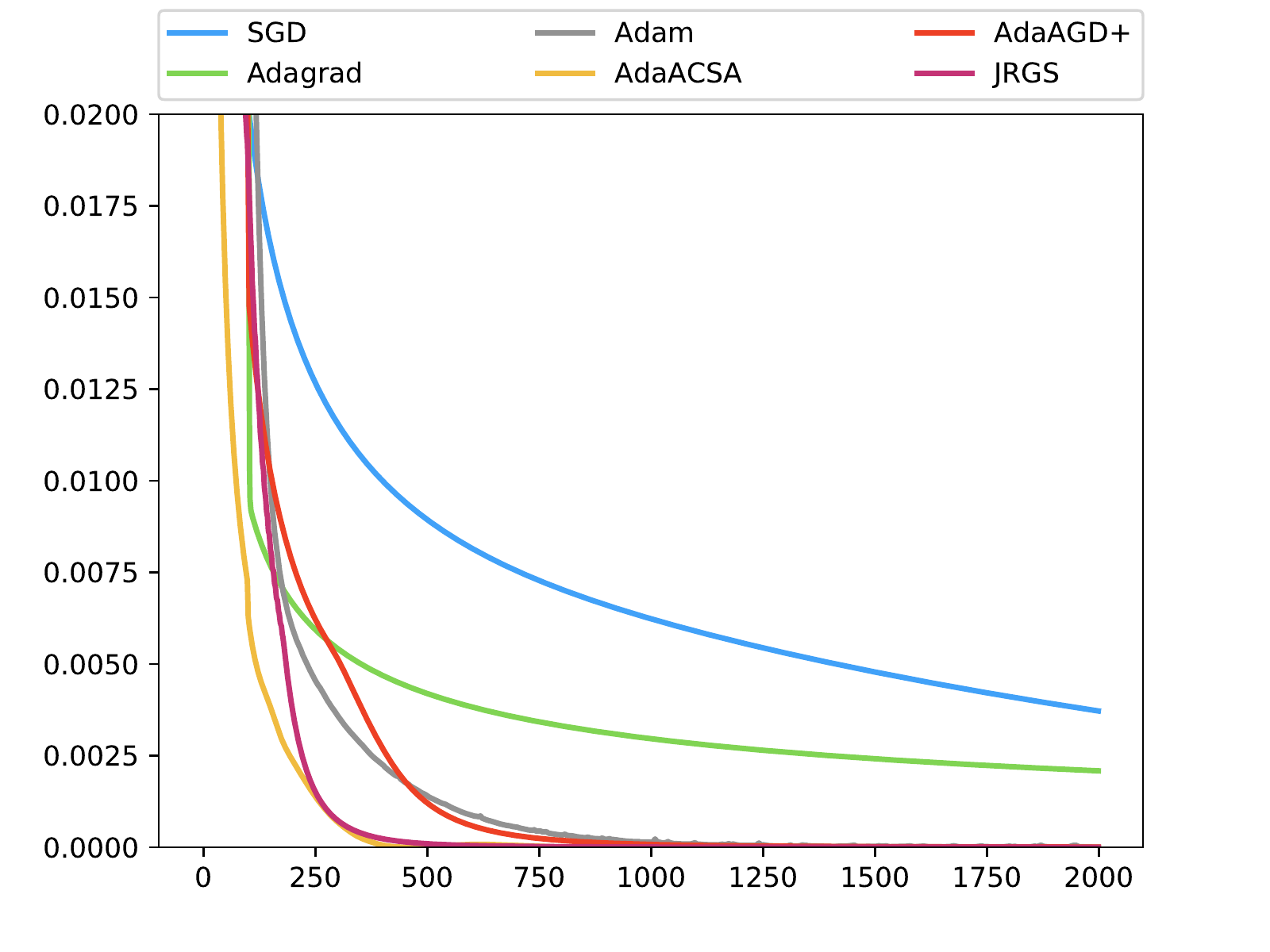}
\includegraphics[scale=0.49]{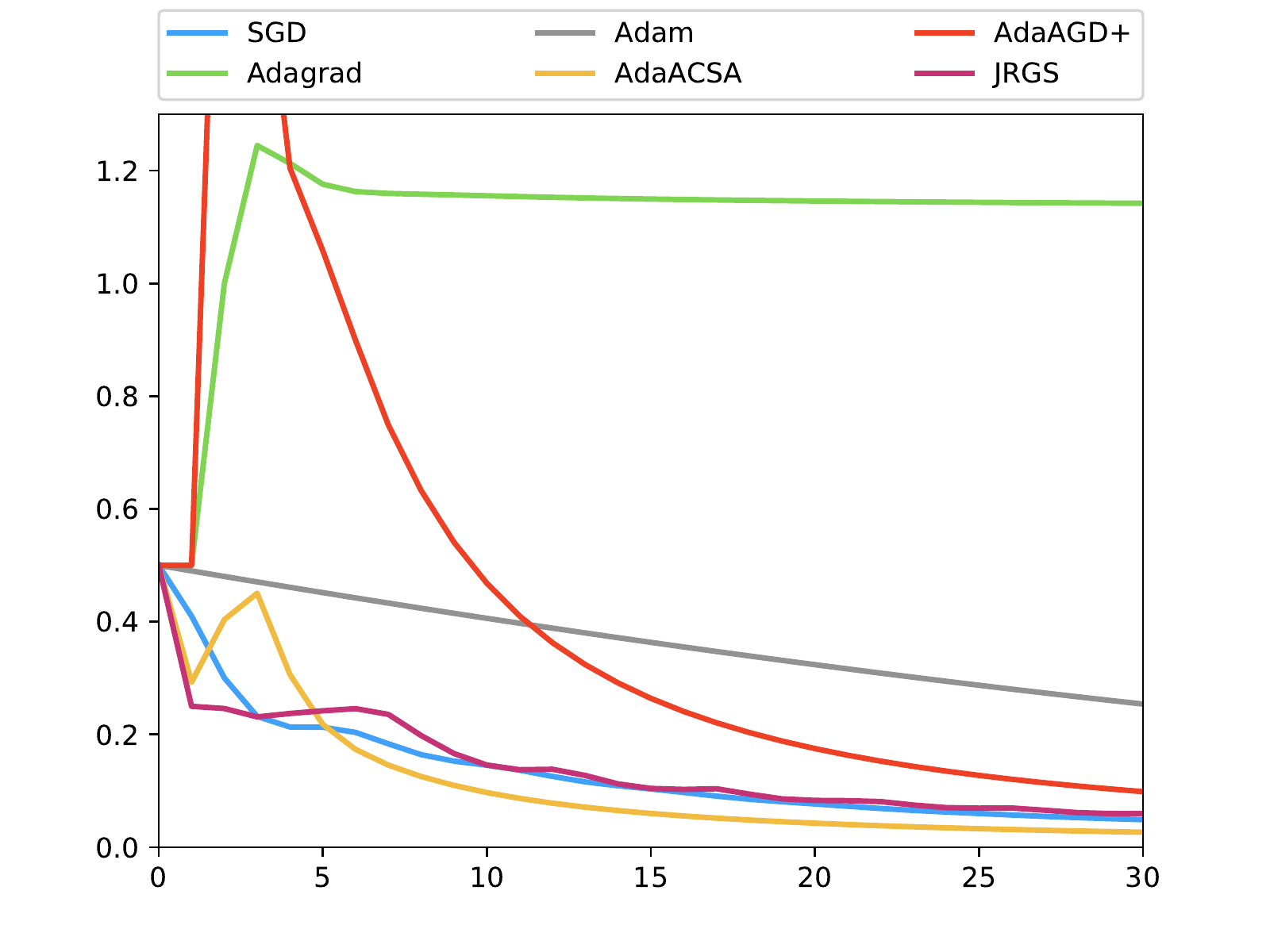}
\caption{Function value achieved using 2000 iterations and 30 iterations, respectively,
for Nesterov's \textquotedblleft worst function\textquotedblright .}
\label{fig:worst-fn}
\end{figure}

\begin{table}
\begin{centering}
\begin{tabular}{|r|c|c|c|c|c|}
\hline 
 & {\small{}1e-1} & {\small{}1e-2} & {\small{}1e-3} & {\small{}1e-4} & {\small{}1e-5}\tabularnewline
\hline 
{\small{}$\sgd$} & {\small{}16} & {\small{}400} & {\small{}$>$2000} & {\small{}$>$2000} & {\small{}$>$2000}\tabularnewline
\hline 
{\small{}$\adagrad$} & {\small{}101} & {\small{}104} & {\small{}$>$2000} & {\small{}$>$2000} & {\small{}$>$2000}\tabularnewline
\hline 
{\small{}$\adam$} & {\small{}64} & {\small{}149} & {\small{}570} & {\small{}1055} & {\small{}1697}\tabularnewline
\hline 
{\small{}$\adaacsa$} & \textbf{\small{}10} & \textbf{\small{}73} & \textbf{\small{}275} & \textbf{\small{}387} & \textbf{\small{}431}\tabularnewline
\hline 
{\small{}$\adaagdplus$} & {\small{}30} & {\small{}154} & {\small{}525} & {\small{}934} & {\small{}1633}\tabularnewline
\hline 
{\small{}$\jrgs$} & {\small{}18} & {\small{}136} & {\small{}278} & {\small{}495} & {\small{}880}\tabularnewline
\hline 
\end{tabular}
\par\end{centering}
\caption{Evaluation on Nesterov's \textquotedblleft worst function in the world\textquotedblright .
For each method, we display the number of iterations before the first
iterate with target error is encountered.}
\label{synth-it}
\end{table}

\begin{figure}
\centering{}
\includegraphics[scale=0.33]{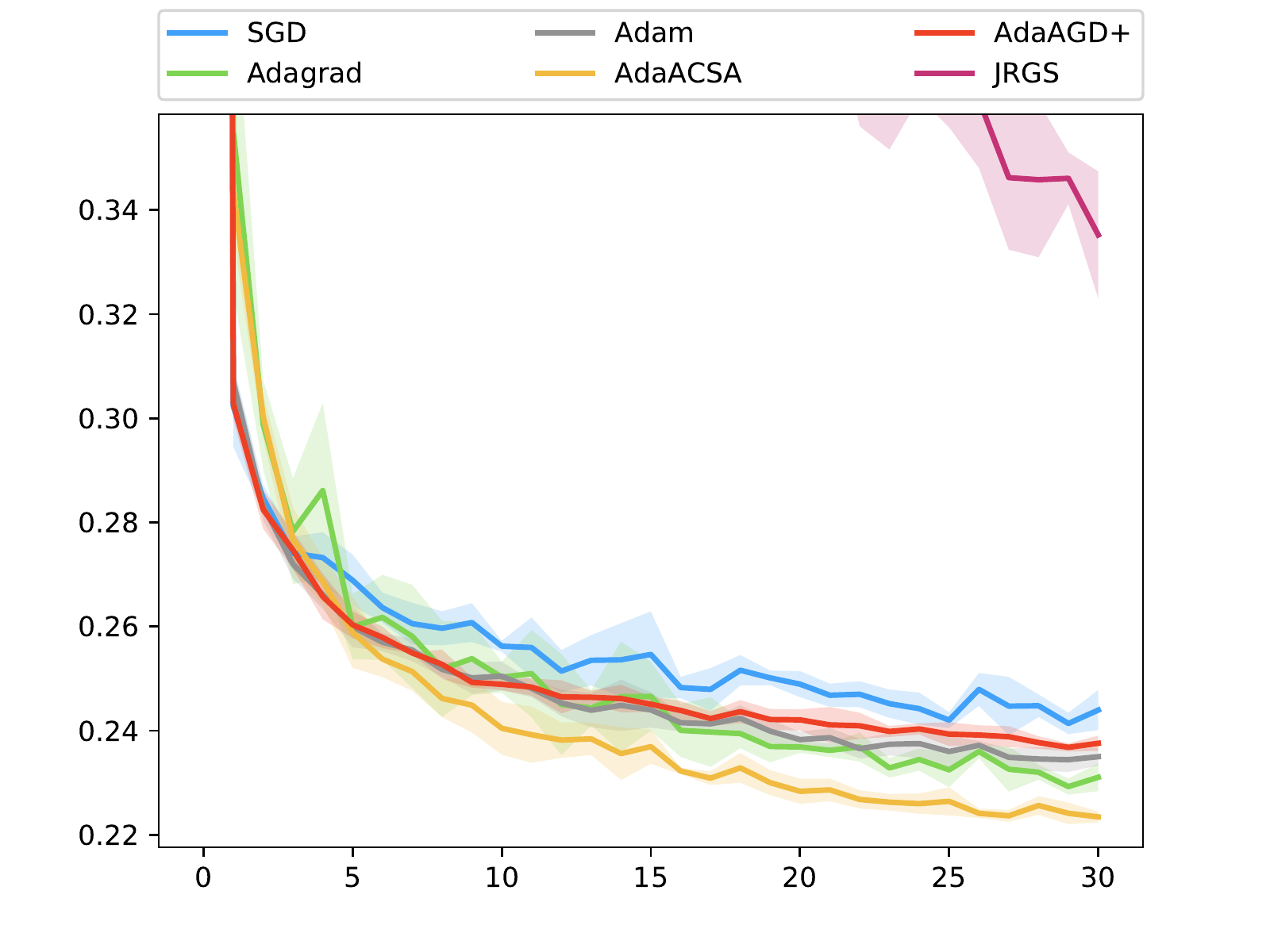}
\includegraphics[scale=0.33]{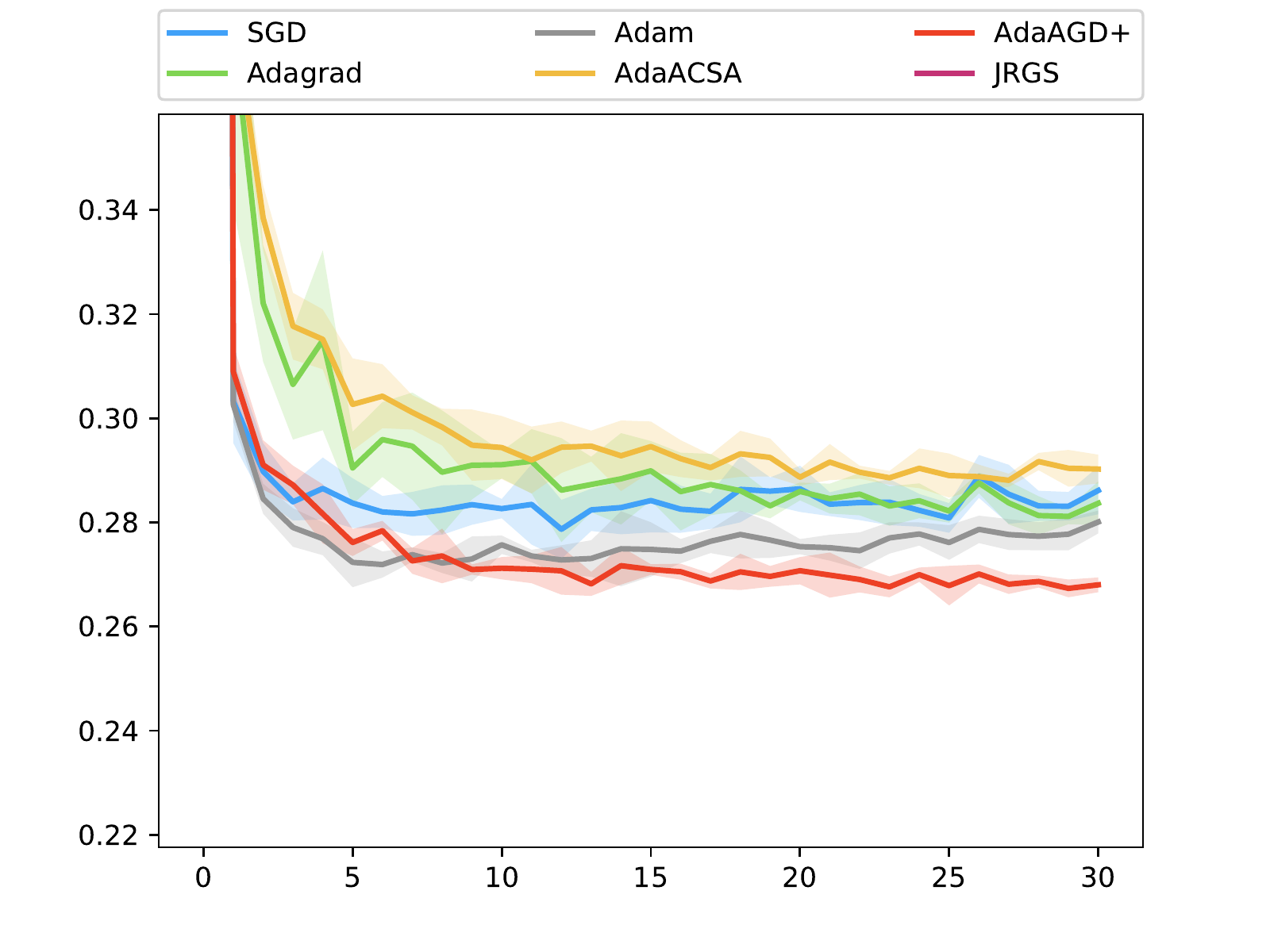}
\includegraphics[scale=0.33]{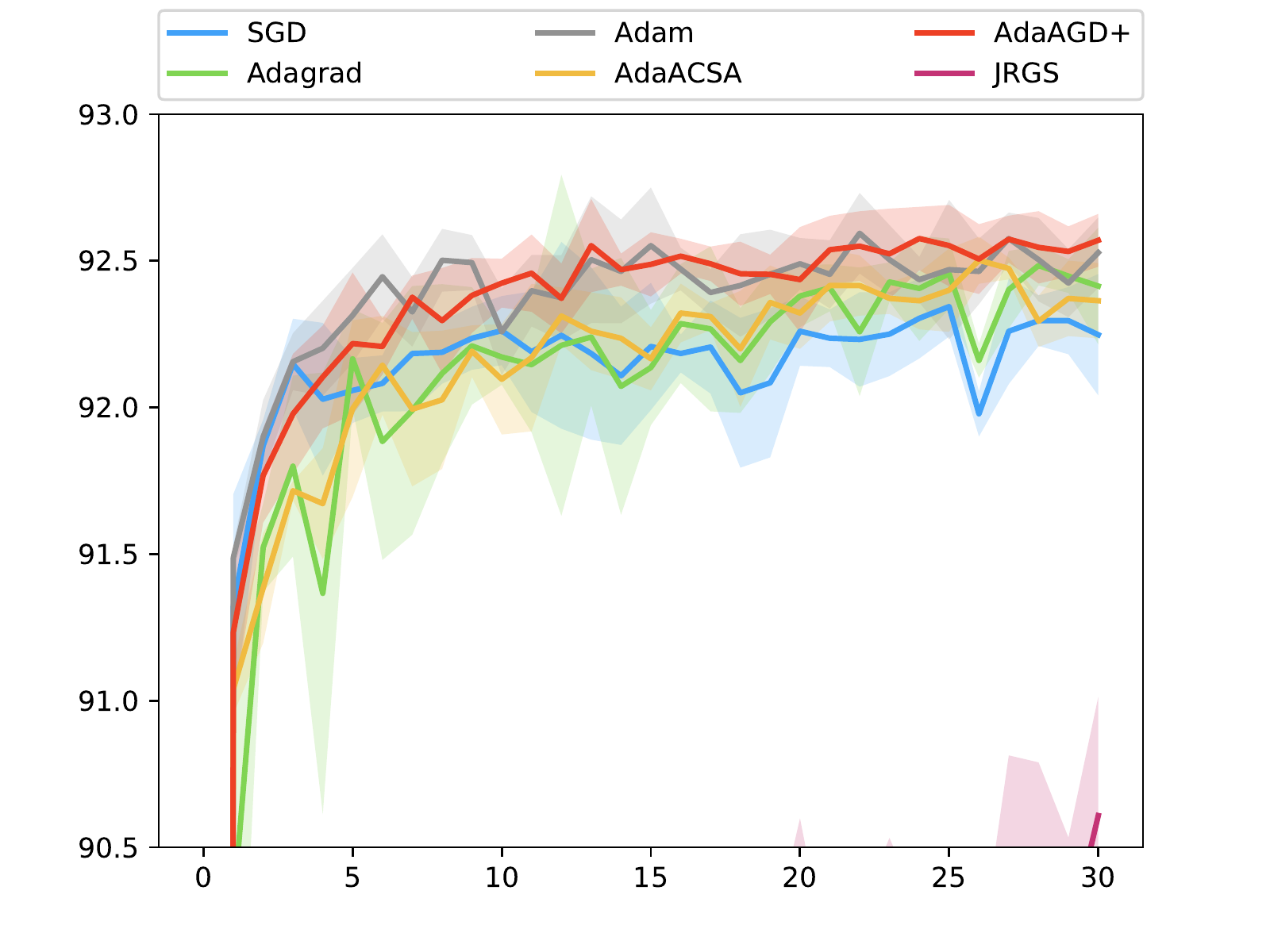}
\includegraphics[scale=0.33]{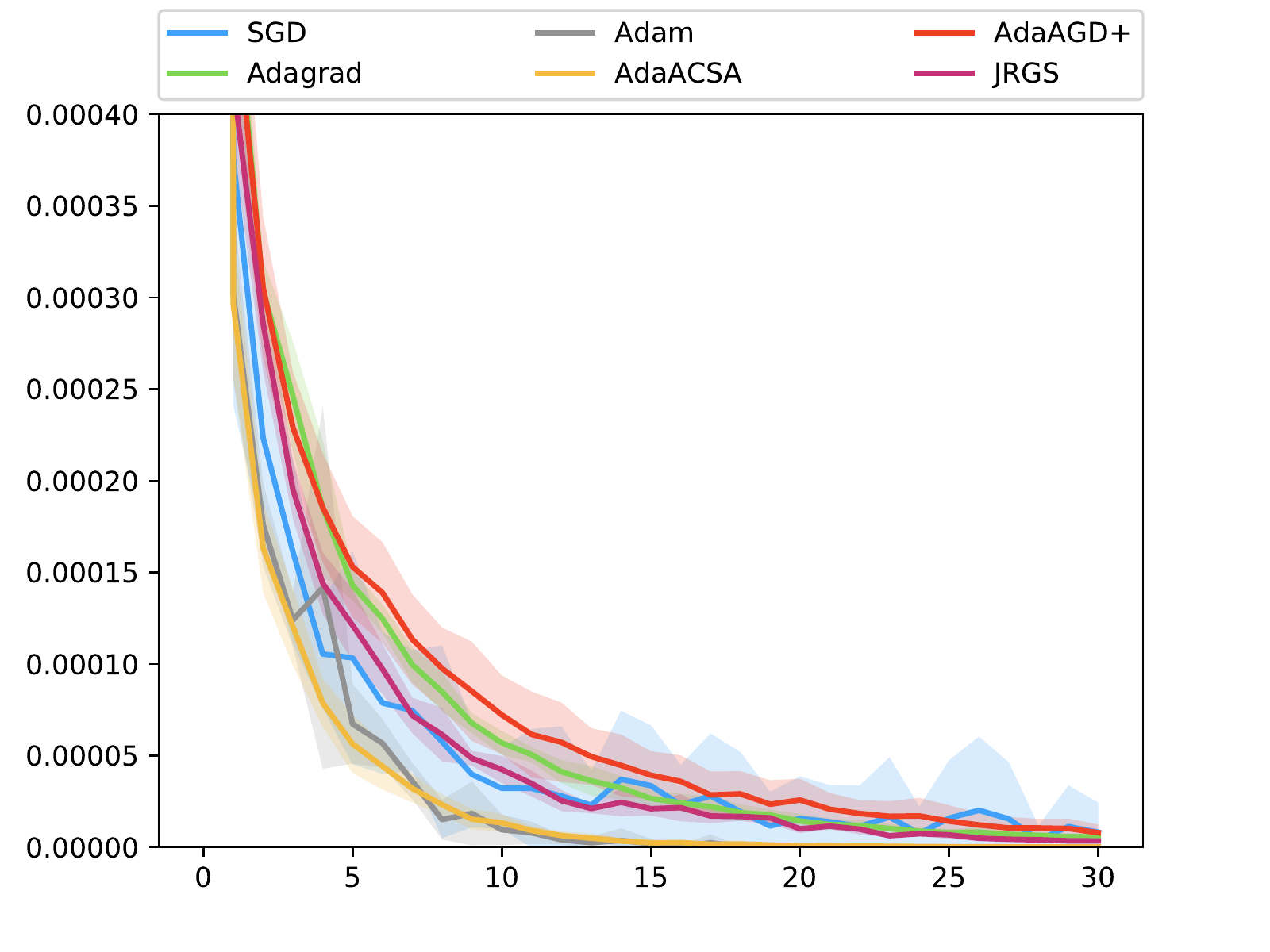}
\includegraphics[scale=0.33]{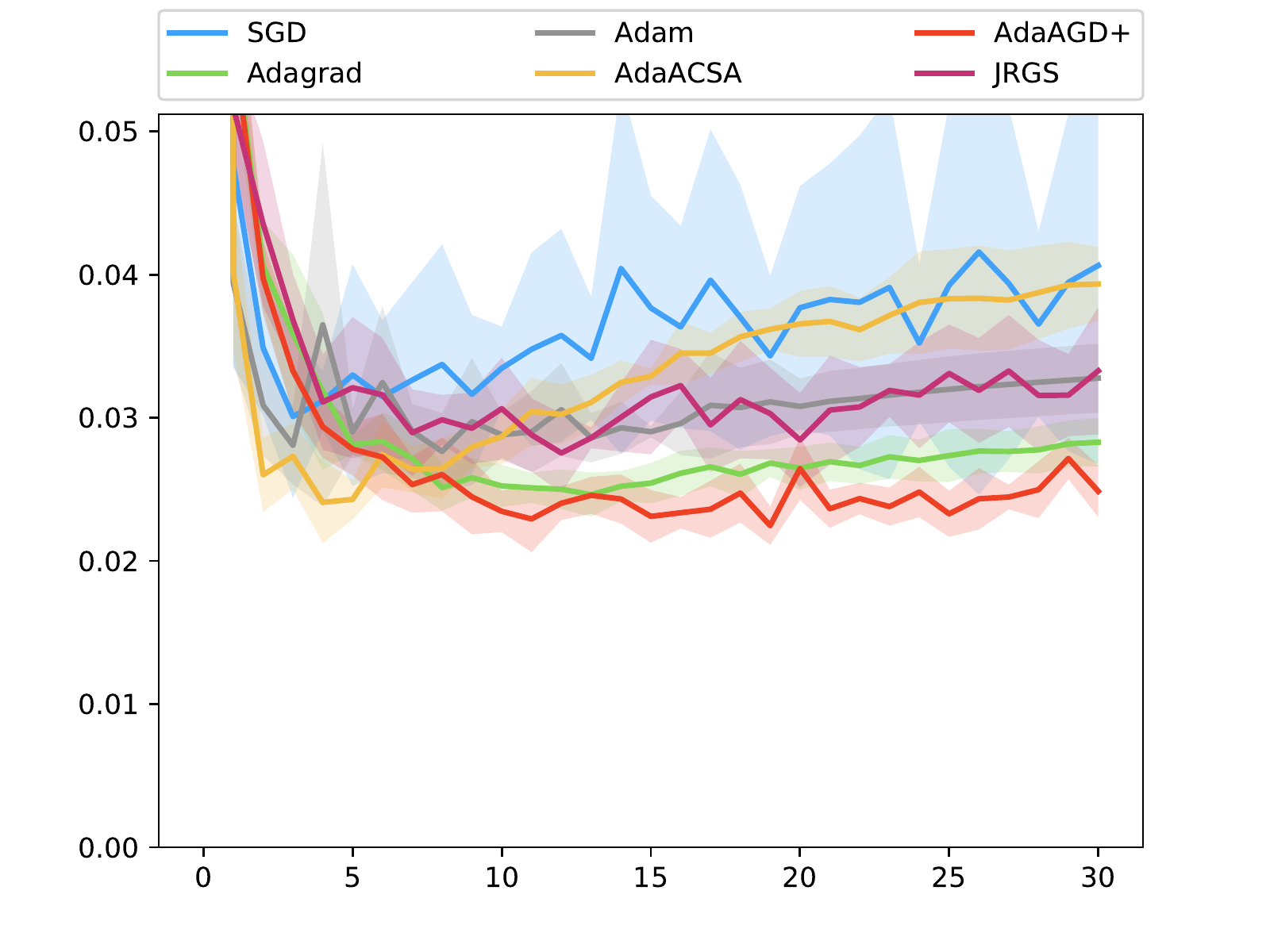}
\includegraphics[scale=0.33]{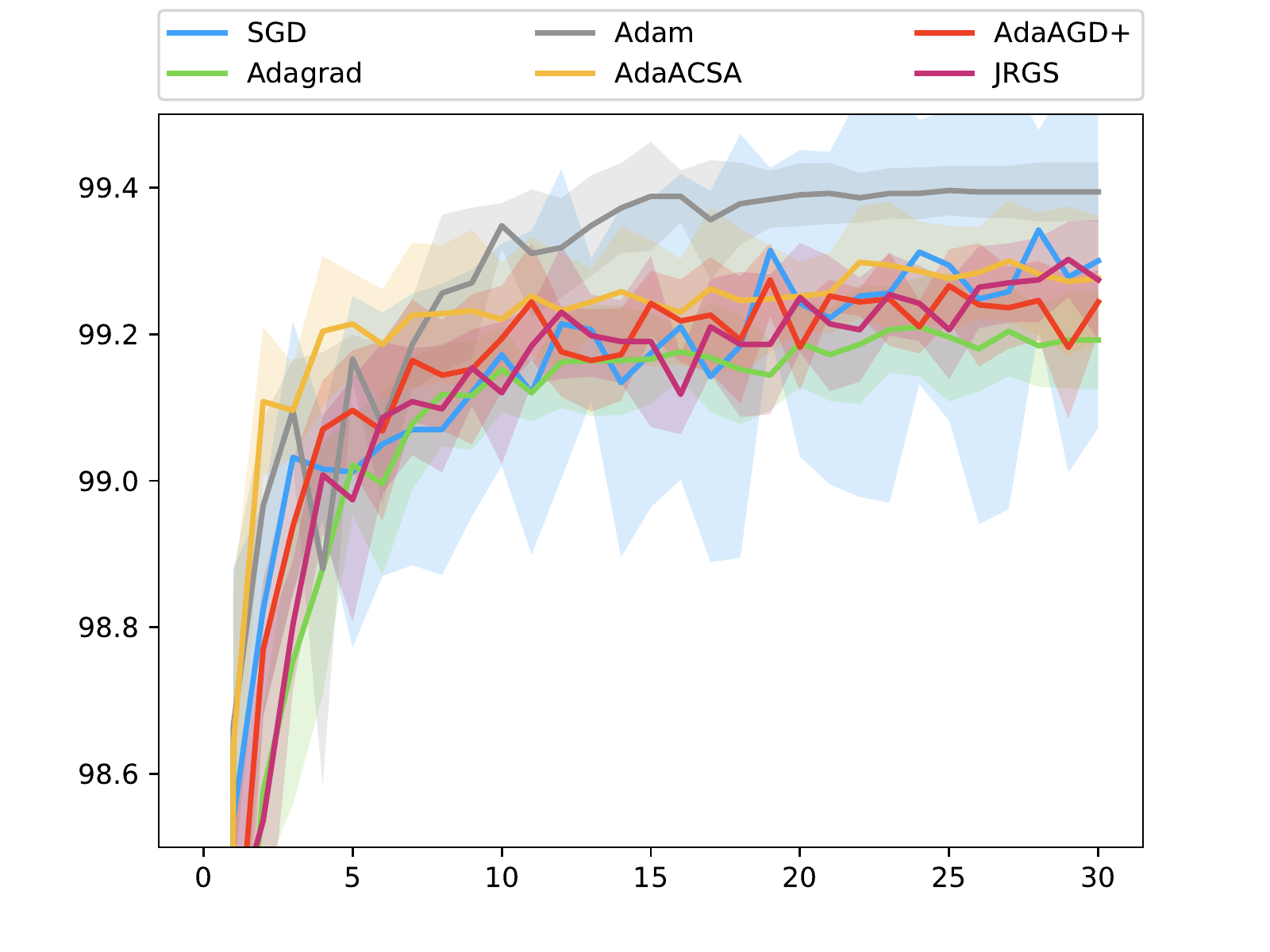}
\includegraphics[scale=0.33]{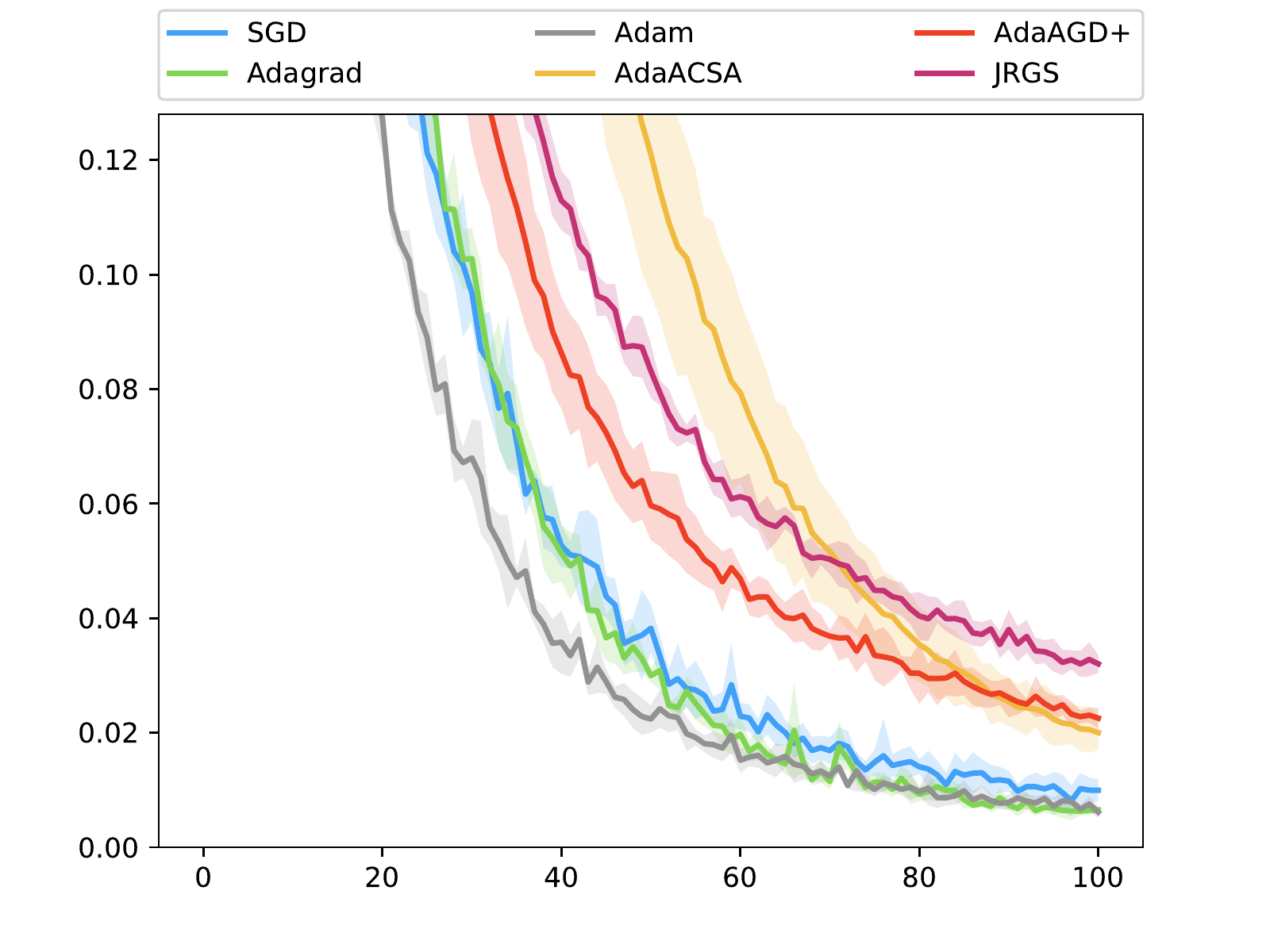}
\includegraphics[scale=0.33]{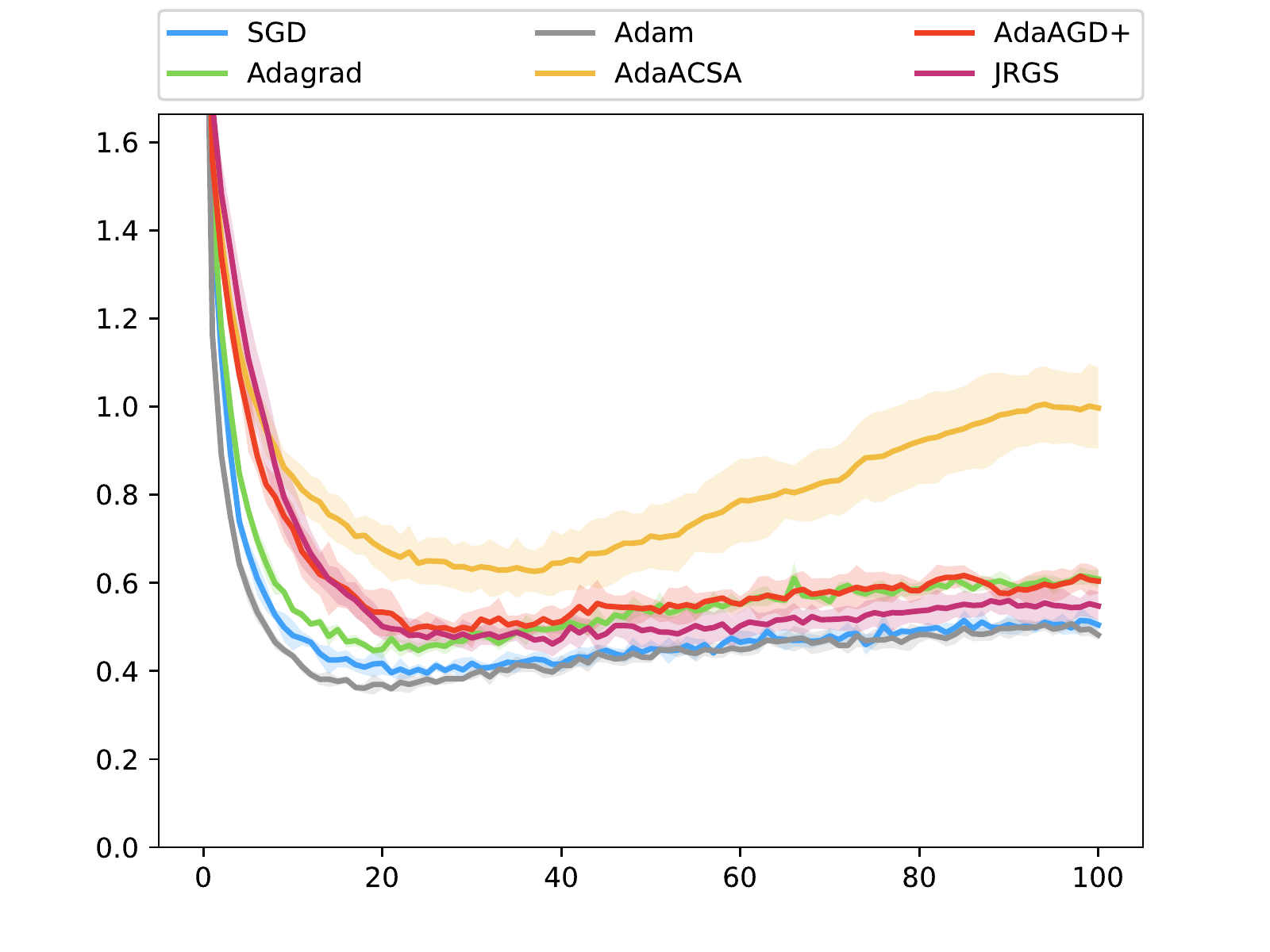}
\includegraphics[scale=0.33]{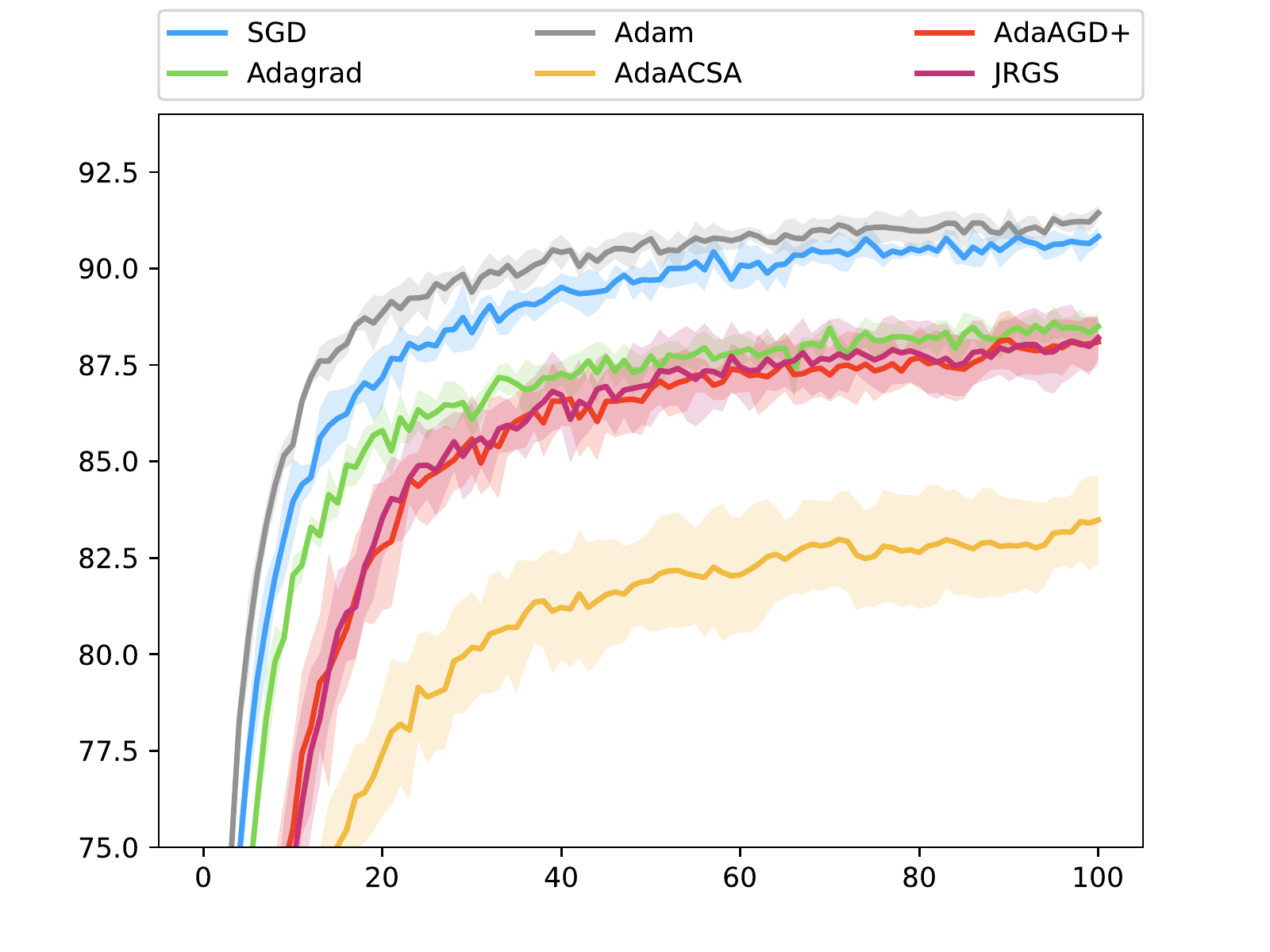}

\caption{Train losses, test losses, and test accuracies. Top row: logistic
regression on MNIST. Middle row: convolutional neural network on MNIST.
Bottom row: residual network on CIFAR-10. The plotted lines correspond
to values averaged over $5$ runs (except for $\protect\adagrad$
on CIFAR-10, where one run failed to converge). Shaded areas represent
the standard deviation.}
\label{fig:all-losses}
\end{figure}

\begin{table}
\begin{centering}
\begin{tabular}{|r|c|c|c|}
\hline 
\textbf{\small{}logistic} & {\small{}train loss} & {\small{}test loss} & {\small{}test accuracy}\tabularnewline
\hline 
{\small{}$\sgd$} & {\small{}2.44e-1\textpm 0.38e-2} & {\small{}2.86e-1\textpm 0.46e-2} & {\small{}92.25\textpm 0.21}\tabularnewline
\hline 
{\small{}$\adagrad$} & {\small{}2.31e-1\textpm 0.27e-2} & {\small{}2.84e-1\textpm 0.42e-2} & {\small{}92.41\textpm 0.20}\tabularnewline
\hline 
{\small{}$\adam$} & {\small{}2.35e-1\textpm 0.17e-2} & {\small{}2.80e-1\textpm 0.22e-2} & {\small{}92.53\textpm 0.12}\tabularnewline
\hline 
{\small{}$\adaacsa$} & \textbf{\small{}2.23e-1\textpm 0.10e-2} & {\small{}2.90e-1\textpm 0.28e-2} & {\small{}92.36\textpm 0.13}\tabularnewline
\hline 
{\small{}$\adaagdplus$} & {\small{}2.38e-1\textpm 0.15e-2} & \textbf{\small{}2.68e-1\textpm 0.14e-2} & \textbf{\small{}92.57\textpm 0.09}\tabularnewline
\hline 
{\small{}$\jrgs$} & {\small{}3.35e-1\textpm 1.22e-2} & {\small{}4.42e-1\textpm 1.15e-2} & {\small{}90.61\textpm 0.40}\tabularnewline
\hline 
\hline 
\textbf{\small{}CNN} & {\small{}train loss} & {\small{}test loss} & {\small{}test accuracy}\tabularnewline
\hline 
{\small{}$\sgd$} & {\small{}10.46e-4\textpm 207.62e-5} & {\small{}4.06e-2\textpm 1.39e-2} & {\small{}99.30\textpm 0.23}\tabularnewline
\hline 
{\small{}$\adagrad$} & {\small{}7.32e-4\textpm 25.06e-5} & {\small{}2.83e-2\textpm 0.17e-2} & {\small{}99.19\textpm 0.07}\tabularnewline
\hline 
{\small{}$\adam$} & \textbf{\small{}0.05e-4\textpm 0.04e-5} & {\small{}3.28e-2\textpm 0.24e-2} & \textbf{\small{}99.39\textpm 0.04}\tabularnewline
\hline 
{\small{}$\adaacsa$} & {\small{}0.18e-4\textpm 0.71e-5} & {\small{}3.93e-2\textpm 0.26e-2} & {\small{}99.28\textpm 0.08}\tabularnewline
\hline 
{\small{}$\adaagdplus$} & {\small{}10.18e-4\textpm 57.92e-5} & \textbf{\small{}2.49e-2\textpm 0.18e-2} & {\small{}99.24\textpm 0.04}\tabularnewline
\hline 
{\small{}$\jrgs$} & {\small{}4.43e-4\textpm 7.32e-5} & {\small{}3.33e-2\textpm 0.44e-2} & {\small{}99.27\textpm 0.08}\tabularnewline
\hline 
\hline 
\textbf{\small{}ResNet18} & {\small{}train loss} & {\small{}test loss} & {\small{}test accuracy}\tabularnewline
\hline 
{\small{}$\sgd$} & {\small{}0.10e-1\textpm 0.19e-2} & {\small{}0.50\textpm 1.11e-2} & {\small{}90.83\textpm 0.24}\tabularnewline
\hline 
{\small{}$\adagrad${*}} & {\small{}0.07e-1\textpm 0.05e-2} & {\small{}0.61\textpm 2.01e-2} & {\small{}88.50\textpm 0.35}\tabularnewline
\hline 
{\small{}$\adam$} & \textbf{\small{}0.06e-1\textpm 0.09e-2} & \textbf{\small{}0.48\textpm 1.15e-2} & \textbf{\small{}91.44\textpm 0.18}\tabularnewline
\hline 
{\small{}$\adaacsa$} & {\small{}0.20e-1\textpm 0.31e-2} & {\small{}1.00\textpm 9.13e-2} & {\small{}83.48\textpm 1.15}\tabularnewline
\hline 
{\small{}$\adaagdplus$} & {\small{}0.23e-1\textpm 0.18e-2} & {\small{}0.60\textpm 2.79e-2} & {\small{}88.10\textpm 0.60}\tabularnewline
\hline 
{\small{}$\jrgs$} & {\small{}0.32e-1\textpm 0.16e-2} & {\small{}0.55\textpm 3.24e-2} & {\small{}88.20\textpm 0.55}\tabularnewline
\hline 
\end{tabular}
\par\end{centering}
\caption{Comparison between optimization methods for logistic regression on
MNIST/convolutional neural network on MNIST/residual network on CIFAR-10.}
\label{full-comparison}
\end{table}

We empirically validated our adaptive accelerated algorithms, $\adaacsa$
and $\adaagdplus$, by testing them on a series of standard models
encountered in machine learning. While the analyses we provided are
specifically crafted for convex objectives, we observed that these
methods exhibits good behavior in the non-convex settings corresponding
to training deep learning models. This may be motivated by the fact
that a significant part of the optimization performed when training
such a model occurs within convex regions \citep{leclerc2020two}.

\paragraph*{Algorithms. }

We evaluated our $\adaacsa$ and $\adaagdplus$ algorithms against
three popular methods, $\sgd$ with momentum, $\adagrad$, and $\adam$.
We also evaluated the algorithms against the recent method of \citet{joulanisimpler}
which we refer to as $\jrgs$. We performed extensive hyper-parameter
tuning, such that each method we compare against has the opportunity
to exhibit its best possible performance. We give the complete experimental
details in Sections \ref{subsec:exp_setup}-\ref{subsec:radiusandlr}.

\paragraph*{Synthetic experiment.}

First, we tested all the methods on a synthetic example, known as
Nesterov's ``worst function in the world'', which is a canonical
example used for testing accelerated gradient methods~\citep{nesterov2013introductory}:
\[
f(x)=\frac{1}{2}\left(x_{1}^{2}+x_{n}^{2}+\sum_{i=1}^{n-1}\left(x_{i}-x_{i+1}\right)^{2}\right)-x_{1}
\]
We see that $\adaacsa$ easily beats all the other methods. In Table
\ref{synth-it} we show for each method the number of iterations before
finding a solution with a fixed target error in function value. We
plot the values of $f$ in Figure \ref{fig:worst-fn}.

\paragraph*{Classification experiments.}

Additionally, we tested these optimization methods on three different
classification models typically encountered in machine learning. The
first one is logistic regression on the MNIST dataset. This is a simple
convex objective for which $\adaacsa$ achieves the best training
loss, while $\adaagdplus$ achieves the best test loss. The second
is a convolutional neural network on the MNIST dataset. Despite non-convexity,
both our methods behave well, and $\adaagdplus$ achieves the best
test loss. The third is a residual network for the CIFAR-10 classification
task. We report the training losses, test losses and test accuracies
achieved in Table \ref{full-comparison}. In Figure \ref{fig:all-losses}
we plot these values, averaged over $5$ runs.

\paragraph*{Discussion. }

We verified experimentally that $\adaacsa$ and $\adaagdplus$ behave
very well on convex objectives, as anticipated by theory. For practical
non-convex objectives, they show a remarkable degree of robustness,
managing to reach close to zero training loss. By contrast, $\jrgs$
requires a significant amount of tuning in order to converge -- our
experiments show that in non-convex settings, without properly constraining
the domain to a small $\ell_{\infty}$ ball, it is very hard for it
to achieve any nontrivial progress.

\subsection{Experimental Setup}

\label{subsec:exp_setup}

We tested each method with its optimal hyper-parameter initialization.
More specifically, for each of these methods we first search the hyper-parameter
configuration that returns the best training loss after a fixed number
of epochs. Then we report experimental results under this choice of
hyper-parameters.

For each tested method, we do a grid search over learning rates in
$\left\{ 1,0.5\right\} \times\left\{ 10^{0},10^{-1},\dots,10^{-4}\right\} $.
For $\sgd$ we test multiple values for the momentum $\mu=\left\{ 0.0,0.5,0.9\right\} $.
For $\adam$ we we test both the standard algorithm and the AMSGrad
version \citep{reddi2018convergence}.

Notably, $\adaagdplus$ and $\jrgs$ are methods designed specifically
for the constrained setting. Indeed, there are simple cases where
removing the constraint in the optimization steps may cause them to
diverge. We observed that they tend to behave better when adding constraints,
even for simple examples. Therefore we run them with a generic $\ell_{\infty}$
radius of $1$, unless otherwise specified. We discuss more about
this aspect in Section \ref{subsec:radiusandlr}.

\subsubsection{Models and Learning Rates}

Here we describe the specific network architectures we used, and the
learning rates that achieved the best results after a hyper-parameter
search, which we used for our final experiments.

\paragraph*{Synthetic Experiment.}

We verify experimentally that $\adaacsa$ indeed achieves an accelerated
convergence rate for convex objectives. To this end, we test all the
methods on Nesterov's ``worst function in the world'' \citet{nesterov2013introductory}:
\begin{equation}
f(x)=\frac{1}{2}\left(x_{1}^{2}+x_{n}^{2}+\sum_{i=1}^{n-1}\left(x_{i}-x_{i+1}\right)^{2}\right)-x_{1}\ ,\label{eq:worst-fn}
\end{equation}
which is a canonical example that proves tightness of accelerated
gradient methods. In our setup we used $n=100$. We ran each method
for $2000$ iterations after grid searching for the best hyper-parameter
configuration.

We optimized the function from \eqref{eq:worst-fn}. We found the
best results using a learning rate of $0.1$ for SGD with $0.9$ momentum,
learning rate of $1.0$ for $\adagrad$, learning rate of $0.01$
for Adam (with identical behavior whether AMSGrad was used or not),
learning rate of $1.0$ for $\adaacsa$ and learning rate of $0.5$
for JRGS. Since $\adagrad$ and $\adaacsa$ achieved the best convergence
with rate $1.0$ we additionally tested them with learning rate $10.0$,
but failed to achieve comparable results. With this higher rate, $\adaacsa$
eventually reached the same precision, but required more iterations.

\paragraph*{Logistic Regression.}

We test these optimization methods on the multi-class logistic regression
using the MNIST dataset, which exhibits a simple convex objective.
More specifically, our model consists of a single linear layer with
cross entropy loss. We train with a minibatch size of 128, as in standard
setups.

The model consists of a single linear layer with cross entropy loss.
The hyper-parameter search revealed a learning rate of $0.05$ for
SGD with $0.5$ momentum, learning rate $0.1$ for $\adagrad$, learning
rate $0.001$ for Adam with the AMSGrad option activated, learning
rate $0.1$ for AdaACSA, and $0.005$ for JRGS.

We ran each optimization method for $30$ epochs with the optimal
hyper-parameters setting found via grid search; in each case we ran
the method starting from $5$ different random seeds, and reported
the mean and standard deviation of train/test losses and test accuracies.
The averaged values are reported in Figure \ref{fig:all-losses}.
The graph for $\jrgs$ does not entirely appear in the figure, since
the losses it achieves are significantly larger.

\paragraph*{Convolutional Neural Network.}

We also test the methods on convolutional neural networks (CNN's).
Again, we consider the MNIST classification task. Similarly to the
logistic regression experiment, we train with a batch size of $128$
for $30$ epochs.

The model is standard -- it contains two stages of 2-d convolutions
with a kernel of size $5$, followed by max pooling and ReLU. These
are followed by a linear layer, a ReLU layer, and another linear layer,
to which we apply cross entropy loss.

We use a learning rate of $0.05$ for SGD with momentum $0.9$, learning
rate $0.01$ for $\adagrad$, learning rate $0.001$ for Adam with
the AMSgrad option activated, learning rate $0.01$ for $\adaacsa$,
and learning rate $0.005$ for JRGS.

This experiment confirms that $\jrgs$ is meant to be used only for
optimizing convex functions. We notice a drastic difference when switching
the architecture from linear to CNN. In the former case the method
converges fast, in the latter it fails completely when running it
in the same regime as $\adaagdplus$, with a radius of $1.0$. We
reduced the radius to $0.1$ and noticed that all of a sudden the
accuracy improved drastically, from as good as random to close to
$100
$. This suggests that within a small radius the function is locally
convex, and hence JRGS exhibits appropriate behavior. However, as
opposed to the other methods, which seem to exhibit significant tolerance
to non-convexity, this one is extremely fragile. We therefore include
results for $\jrgs$ with a radius of $0.1$. In Section \ref{subsec:empirical-failure-jrgs}
we discuss more about this aspect.

Train losses, test losses and test accuracies over $30$ epochs are
reported in Figure \ref{fig:all-losses}. The values achieved at the
end are reported in Table \ref{full-comparison}.

\paragraph*{Residual Network.}

We also run tests on the CIFAR-10 dataset, for which we train a standard
ResNet18 model~\citep{he2016deep}. We used the ResNet18 model as
described in \citep{he2016deep}, for which we used a standard implementation.

In order to pick hyperparameters, we repeat the experiments previously
described and run each setup for $40$ epochs. For $\jrgs$ and $\adaagdplus$,
which are better suited for constrained optimization we pick a radius
of $2.0$. We made this choice since models we trained with vanilla
$\adam$ achieved the best test accuracy with weights that are at
most $1.3$ in $\ell_{\infty}$.

For $\adagrad$ one of the runs failed to converge, so we discarded
it, and returned the average of the $4$ remaining runs.

We plot train losses, test losses and test accuracies in Figure \ref{fig:all-losses}.
While $\adam$ and $\sgd$ seem to obtain the best test accuracies,
we see that in this regard $\adaagdplus$, $\jrgs$ and $\adagrad$
are competitive. We note however that $\adagrad$ exhibits some significant
lack of robustness, since for one run it failed to converge, and that
$\jrgs$ could be run only after some tuning by constraining it to
run only within a specific small region around the origin.

We found the best results using a learning rate of $0.1$ for SGD
with $0.5$ momentum, learning rate of $0.01$ for Adagrad, learning
rate of $0.001$ for Adam with AMSgrad, learning rate of $0.1$ for
AdaACSA, learning rate of $0.5$ for AdaAGD+ and and learning rate
of $0.01$ for JRGS.

\subsection{Failure of $\protect\jrgs$ on Nonconvex Domains}

\label{subsec:empirical-failure-jrgs}

\begin{figure}
\begin{centering}
\includegraphics[scale=0.35]{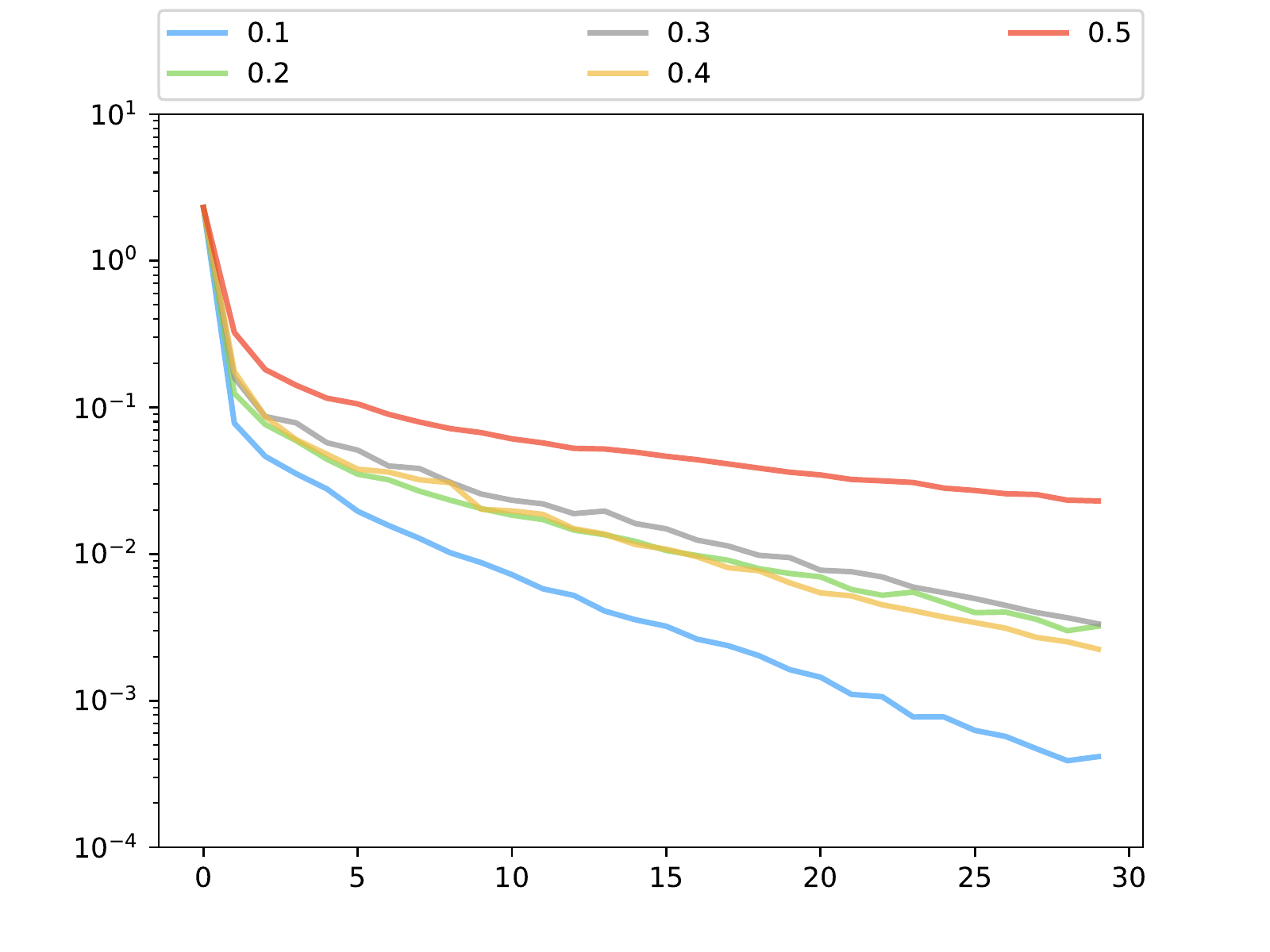}
\includegraphics[scale=0.35]{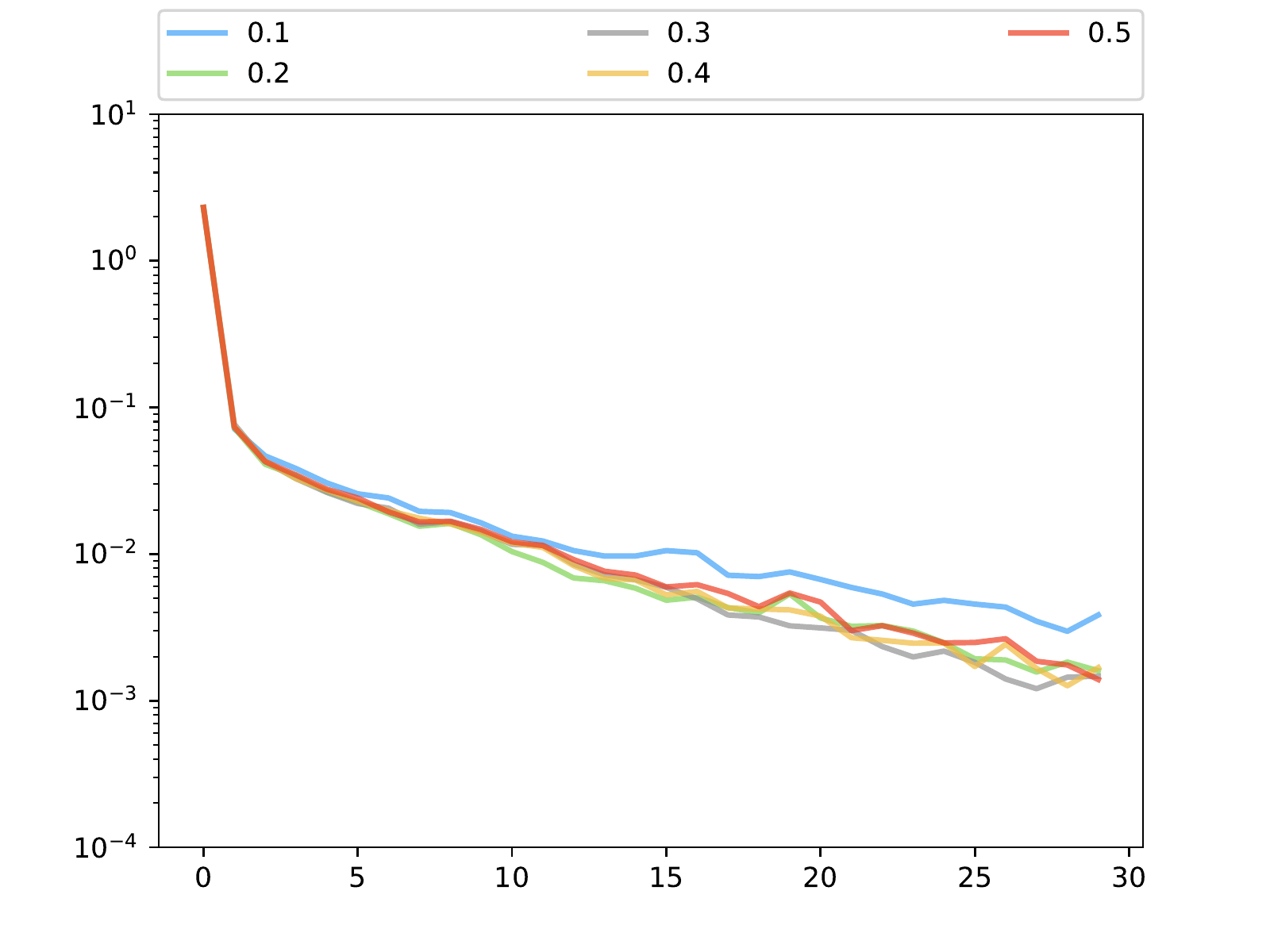}
\includegraphics[scale=0.35]{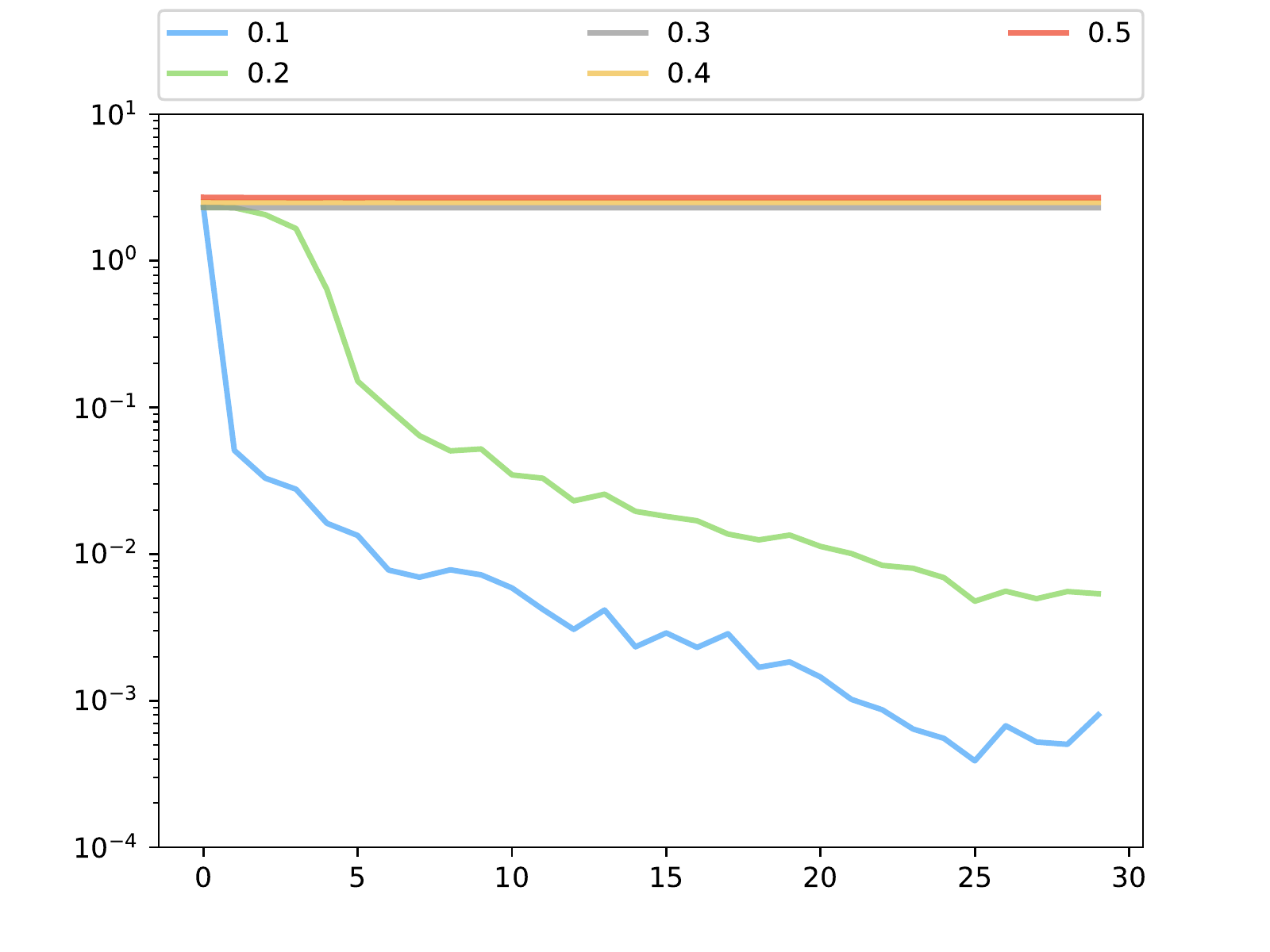}
\par\end{centering}
\caption{Train losses on a logarithmic scale for AdaACSA, AdaAGD+ and JRGS
under different radius constraints. We notice that for non-convex
instances (CNN architecture on MNIST) $\protect\jrgs$ fails to make
any progress unless the optimization domain is sufficiently constrained.
By contrast, $\protect\adaacsa$ and $\protect\adaagdplus$ are fairly
robust to non-convexity, as they both empirically exhibit linear convergence
even without overconstraining the domain.}
\label{fig:varying-radius-losses}
\end{figure}

Following the discussion from Section \ref{sec:Experiments}, we show
that picking a larger radius makes the $\jrgs$ method fail to converge.
This stands opposite to the other methods we presented which, although
they are designed and analyzed specifically for convex functions,
exhibit some degree of robustness to non-convexity. In this experiment
we used the learing rates from the CNN experiment in Section \ref{sec:Experiments}.
We run the constrained versions of $\adaacsa$, $\adaagdplus$ and
$\jrgs$ on the MNIST classification task, this time varying the radii
of the constrained domain in $\{0.1,0.2,0.3,0.4,0.5\}$. In Figure
\ref{fig:varying-radius-losses} we see that in the case of $\jrgs$,
the performance degrades drastically as radius is increased past $0.2$.
On the other hand, the performance of $\adaacsa$ degrades gracefully,
as larger radius naturally corresponds to a slower convergence rate.
$\adaagdplus$ is pretty much unaffected by these constraints. 

\subsection{Discussion on Implementing the Methods}

\label{subsec:radiusandlr}

\paragraph*{Radius and Learning Rates.}

Here we discuss relevant implementation matters. We implemented $\adaacsa$,
$\adaagdplus$ and $\jrgs$ as PyTorch libraries. Compared to the
description of the theoretical algorithms, we implemented a series
of standard tweaks, to make these more amenable to deploying in the
wild.

First, observe that the algorithms we describe do not explicitly include
a learning rate. Per the discussion from Section \ref{sec:adaptive-schemes},
the learning rate corresponds to the radius $R_{\infty}$. However,
we allow some further slack by making the learning rate $\eta$ and
the radius $R_{\infty}$ two independent parameters. This way we use
$\eta$ to adjust the size of the steps we take, and $R_{\infty}$
to define the feasible domain over wich we optimize. This is also
a theoretically sound choice, since there is a natural tradeoff between
$\eta$ and $R_{\infty}$ which our algorithms optimize by setting
these parameters to be equal. In practice, setting them to different
values shows some advantage, especially in the case where optimization
happens in a much smaller region than the anticipated $\ell_{\infty}$
ball of radius $R_{\infty}$.

\paragraph*{Constrained vs Unconstrained Optimization.}

\begin{figure}
\noindent %
\noindent\fbox{\begin{minipage}[t]{1\columnwidth - 2\fboxsep - 2\fboxrule}%
Let $D_{0}=I$, $z_{0}=x_{0}$, $\gamma_{0}=1$, $\eta>0$.

For $t=0,\dots,T-1$, update:
\begin{align*}
g_{t} & =\widetilde{\nabla}f(x_{t})\ , & \text{(get stochastic gradient)}\\
D_{t+1,i}^{2} & =D_{t,i}^{2}+\frac{\gamma_{t}^{2}}{\eta^{2}}\left(g_{t}\right)_{i}^{2}\ ,\text{for all }i\in[d].\\
z_{t+1} & =z_{t}-\gamma_{t}D_{t+1}^{-1}g_{t}\ , & \text{(gradient step)}\\
y_{t+1} & =x_{t}-D_{t}^{-1}g_{t}\ , & \text{(mirror step)}\\
\gamma_{t+1} & =\frac{1}{2}\left(1+\sqrt{1+4\gamma_{t}^{2}}\right)\\
x_{t+1} & =\left(1-\gamma_{t+1}^{-1}\right)y_{t+1}+\gamma_{t+1}^{-1}z_{t+1}\ , & \text{(linear coupling)}
\end{align*}

Return $y_{T}$.%
\end{minipage}}

\caption{Unconstrained $\protect\adaacsa$ algorithm.}

\label{alg:lin-coup}
\end{figure}

While the purpose of this paper is to provide algorithms for constrained
optimization, in practical instances it may be desirable to not impose
explicit constraints on the domain. The reason is that it is possible
that during the entire run of the algorithm, the iterates stay within
a small region, and hence adding radius constraints only introduce
the need for additional tuning. Also, we note that the standard PyTorch
implementations of $\adagrad$ and $\adam$ do not constrain the domain
in any way, yet they both achieve good practical performance.

We therefore included in our implementation an unconstrained version
of $\adaacsa$. The algorithm is very similar to the one described
in Figures \ref{alg:acsa} and \ref{alg:acc-stoch-1}, with a few
key differences. First, just as in the $\adagrad$ and $\adam$ implementations,
we do not constrain the domain. Because of this, the movement $z_{t+1}-z_{t}$
which determines the update to the preconditioner $D_{t}$ can be
explicitly written as $z_{t+1}-z_{t}=-\gamma_{t}D_{t}^{-1}\nabla f(x_{t})$.
We can therefore use this to slightly alter the algorithm and make
it more similar to the vanilla unconstrained $\adagrad$, in the sense
that rather than using the off-by-one scaling (see more about this
in Section \ref{sec:adagrad-unconstrained}) we first update the preconditioner
and then perform the step. This does not fundamentally change the
analysis --- as a matter of fact the only difference is that the
speed of convergence will now depend on $R_{\infty}=\max_{t}\left\Vert x_{t}-x^{*}\right\Vert _{\infty}$,
which is an unknown parameter.

In Figure \ref{alg:lin-coup} we describe the unconstrained stochastic
algorithm, which turns out to be an adaptive version of the linear
coupling method of \citet{allen2017linear}.

We can verify that this algorithm matches the $\adaacsa$ algorithm
by moving the preconditioner update before the update in $z$ and
$y$. Expanding the iterations from Figure \ref{alg:acc-stoch-1}
and setting $\alpha_{t}=\gamma_{t}$ we obtain: 
\begin{align*}
x_{t} & =\left(1-\gamma_{t}^{-1}\right)y_{t}+\gamma_{t}^{-1}z_{t}\ ,\\
D_{t+1,i}^{2} & =D_{t,i}^{2}\left(1+\frac{\left(z_{t+1,i}-z_{t,i}\right)^{2}}{2R_{\infty}^{2}}\right)=D_{t,i}^{2}+\frac{\gamma_{t}^{2}\left(g_{t}\right)_{i}^{2}}{2R_{\infty}^{2}}\ ,\text{for all }i\in[d].\\
z_{t+1} & =\arg\min_{u\in\dom}\left\{ \gamma_{t}\left\langle \widetilde{\nabla}f(x_{t}),u\right\rangle +\frac{1}{2}\left\Vert u-z_{t}\right\Vert _{D_{t}}^{2}\right\} =z_{t}-\gamma_{t}D_{t}^{-1}g_{t}\ ,\\
y_{t+1} & =\left(1-\gamma_{t}^{-1}\right)y_{t}+\gamma_{t}^{-1}z_{t+1}=x_{t}+\gamma_{t}^{-1}\left(z_{t+1}-z_{t}\right)=x_{t}-D_{t}^{-1}\nabla f(x_{t})\ .
\end{align*}
Since $\gamma_{0}=1$, we also have $x_{0}=z_{0}$, so we can move
the update in $x$ after the one for $y$, and hence we derive exactly
the algorithm in Figure \ref{alg:lin-coup}, after setting $\eta=R_{\infty}\sqrt{2}$.
Note that we used the specific steps described in Remark \ref{rem:step-sizes},
i.e. we make the inequality $\left(\alpha_{t+1}-1\right)\gamma_{t+1}\leq\alpha_{t}\gamma_{t}$
tight in order to optimize the growth of $\gamma_{t}$.

\end{document}